\documentclass[preprint,cleveref,12pt]{colt2025} % Include author names

\usepackage{times} % 

\title[Logarithmic regret of exploration in MDPs]{
    Logarithmic regret of exploration in average reward Markov decision processes
}

\coltauthor{%
    \Name{Victor Boone} \Email{victor.boone@irit.fr}\\
    \Name{Bruno Gaujal} \Email{bruno.gaujal@inria.fr}\\
    \addr Univ.~Grenoble Alpes, Inria, CNRS, Grenoble INP, LIG, 38000 Grenoble, France 
}

\usepackage[utf8]{inputenc}
\usepackage[T1]{fontenc}
\usepackage[fontsize=11.5]{scrextend} % Change font size
\usepackage{subfiles}
\usepackage{lipsum} % ghost text
\usepackage{microtype} % haha you fonts

\usepackage{hyperref}       % hyperlinks
\usepackage{url}            % simple URL typesetting

\usepackage{algorithm}
\usepackage{algorithmic}

\usepackage[most]{tcolorbox}

\newenvironment{blackblock}{
    \begin{tcolorbox}[
        enhanced,
        breakable,
        colback=black!10,
        leftrule=1mm,
        toprule=0pt,
        bottomrule=0pt,
        rightrule=0pt,
        arc=0pt,
        before skip=1em plus 2pt,
        after skip=1em plus 2pt,
    ]
}{
    \end{tcolorbox}
}

\usepackage{dsfont}
\usepackage{mathrsfs, mathtools}
\usepackage{mleftright} % \middle command

\newtheorem{assumption}{Assumption}
\newaliascnt{informalproperty}{theorem}
\newtheorem{informalproperty}[informalproperty]{Informal Property}
\aliascntresetthe{informalproperty}
\crefname{informalproperty}{informal Property}{informal Properties} % we are cheating with Capitalize ...

\usepackage[showdeletions]{color-edits}	% for editing macros / use [suppress] for final
\newcommand{\draft}[1]{#1}	% for removing macro coloring
\newcommand{\newmacro}[2]{\newcommand{#1}{\draft{#2}}}	% for semantic definitions
\newcommand{\renewmacro}[2]{\renewcommand{#1}{\draft{#2}}}	% for semantic definitions
\makeatletter
\newcommand{\subsetcong}{\mathrel{\mathpalette\subset@cong\relax}}
\newcommand{\subsetsim}{\mathrel{\mathpalette\subset@sim\relax}}

\newcommand{\subset@cong}[2]{%
  \vbox{\offinterlineskip\m@th
    \ialign{\hfil$#1##$\hfil\cr
      \sim\cr\subset\cr
    }%
  }%
}
\newcommand{\subset@sim}[2]{%
  \vtop{\offinterlineskip\m@th
    \ialign{\hfil$#1##$\hfil\cr
      \subset\cr\noalign{\kern1pt}\sim\cr
    }%
  }%
}
\makeatother

\DeclarePairedDelimiter{\braces}{\{}{\}}	% for braces
\DeclarePairedDelimiter{\brackets}{[}{]}	% for brackets
\DeclarePairedDelimiter{\parens}{(}{)}	% for parentheses
\DeclarePairedDelimiter{\abs}{\lvert}{\rvert}	% for absolute value
\DeclarePairedDelimiter{\ceil}{\lceil}{\rceil}	% for ceiling
\DeclarePairedDelimiter{\floor}{\lfloor}{\rfloor}	% for floor
\DeclarePairedDelimiter{\pospart}{[}{]_{+}}	% for positive par
\DeclarePairedDelimiter{\norm}{\lVert}{\rVert}	% for norm
\newcommand{\tsqrt}[1]{{\textstyle\sqrt{#1}}} % text style square root
\newmacro{\Real}{\mathbf{R}}
\newmacro{\Nat}{\mathbf{N}}
\newmacro{\Rel}{\mathbf{Z}}
\newmacro{\R}{\Real}
\newmacro{\N}{\Nat}
\newmacro{\Q}{\mathbf{Q}}
\newmacro{\Z}{\Rel}

\def\product{\prod}
\def\integral{\int}

\newmacro{\tlim}{\lim\limits} % inline limit
\newmacro{\tlimsup}{\limsup\limits} % inline limsup
\newmacro{\tliminf}{\liminf\limits} % inline liminf
\newmacro{\tsum}{\sum\limits} % inline sum
\newmacro{\tprod}{\prod\limits} % inline product

\newmacro{\EE}{\mathbf{E}} % Expectation
\newmacro{\Var}{\mathbf{V}} % Variance
\renewmacro{\Pr}{\mathbf{P}} % Probability
\newmacro{\dd}{\mathrm{d}} % for integrals
\newmacro{\probabilities}{\mathcal{P}} % Space of probability distributions
\newmacro{\ind}{\mathbf{1}} % Indicator
\newcommand{\indic}[1]{\ind\parens{#1}}
\def\indicator{\indic}
\newcommand{\eqindicator}[1]{\ind\parens*{#1}}

\newcommand{\vecspan}[1]{\mathrm{sp}\parens{#1}}

\DeclareMathOperator*{\oh}{\mathrm{o}}
\DeclareMathOperator*{\OH}{\mathrm{O}}
\DeclareMathOperator*{\Clim}{\mathrm{Cesaro-lim}}

\newmacro{\event}{\mathcal{E}}

\newmacro{\KL}{\mathrm{KL}}
\newmacro{\kl}{\mathrm{kl}}
\newmacro{\entropy}{\mathrm{Ent}}

\newmacro{\mdp}{M}
\def\model{\mdp}
\newmacro{\mdps}{\mathcal{M}}
\def\models{\mdps}
\newmacro{\modelss}{\mathfrak{M}}
\newmacro{\modelsseq}{\mathfrak{M}}

\newmacro{\state}{s} % Space of actions
\newmacro{\State}{S} % random state
\newmacro{\states}{\mathcal{S}} % Space of states
\newmacro{\action}{a} % action
\newmacro{\Action}{A} % random action
\newmacro{\actions}{\mathcal{A}} % Space of actions
\newmacro{\pair}{z} % Pair state-action
\newmacro{\Pair}{Z} % random pair
\newmacro{\pairs}{\mathcal{Z}} % State-action space

\newmacro{\history}{o} % instance of history
\newmacro{\histories}{\mathcal{O}} % history space
\newmacro{\History}{O} % random history

\newmacro{\structure}{\mathfrak{Z}} % Co-exploration structure

\newmacro{\kerneldistribution}{p} % kernel as a distribution
 
\newmacro{\kernelvector}{p} % kernel as a tensor/vector (bold)
\def\kernel{\kernelvector}
\newmacro{\kernels}{\mathcal{P}}
\newmacro{\Kernel}{P} % kernel matrix
\newmacro{\kernelrewarddistribution}{q} % kernel or reward distribution?

\newmacro{\kernelreward}{q} % kernel or reward?
\def\kerrew{\kernelreward}
\newmacro{\kerrews}{\mathcal{Q}}

\newmacro{\rewarddistribution}{r} % reward as a distribution

\newmacro{\rewardvector}{r} % (mean) reward as a vector (bold)
\newmacro{\Reward}{R}
\def\reward{\rewardvector}
\newmacro{\rewards}{\mathcal{R}}

\newmacro{\bellman}{L} % Bellman operator

\newmacro{\opt}{*} 
\newmacro{\optopt}{\opt\opt} 

\newmacro{\policy}{\pi} % policy
\newmacro{\policies}{\Pi} % stationary deterministic policies
\newmacro{\deterministicpolicies}{\policies^\mathrm{SD}} % deterministic policies
\newmacro{\randomizedpolicies}{\policies^\mathrm{SR}} % randomized policies
\newmacro{\planners}{\policies^\mathrm{HR}} % history dependent policies
\newmacro{\planner}{(\policy_t)}

\newmacro{\optpolicy}{\policy^\opt} % optimal policy
\newmacro{\optimalpolicies}{\policies^\opt} % (gain) optimal policies
\newmacro{\bellmanpolicies}{\policies^{\opt}_\mathrm{Bell}} % Bellman optimal policies
\newmacro{\biaspolicies}{\policies^{\opt}_\mathrm{bias}} % bias optimal policies
\def\optpolicies{\optimalpolicies}
\newmacro{\approximalpolicies}{\policies{}^\opt_\epsilon} % epsilon-gain optimal policies

\newmacro{\optimalpairs}{\pairs^{\opt\opt}} % X_** optimal pairs
\newmacro{\weakoptimalpairs}{\pairs^{\opt}} % X_* weakly optimal pairs
\newmacro{\suboptimalpairs}{\pairs^-} % X_- suboptimal pairs
\def\optpairs{\optimalpairs}

\def\wkoptpairs{\weakoptimalpairs}

\newmacro{\gain}{g}
\newmacro{\bias}{h}
\newmacro{\biases}{\mathcal{H}}
\newmacro{\optimalgain}{\gain^\opt}
\def\optgain{\optimalgain} % shortcut for gain
\newmacro{\optimalbias}{\bias^\opt}
\def\optbias{\optimalbias} % shortcut for bias
\newmacro{\ogaps}{\Delta^{\opt}}
\newmacro{\gaps}{\Delta}
\newmacro{\gaingap}{\gaps_{\gain}} % gain gap
\newmacro{\biasgap}{\gaps_{\bias}} % bias gap

\newmacro{\bellmanoperator}{L} % Bellman operator
\newmacro{\diameter}{D}
\newmacro{\uniformdiameter}{\diameter_u}
\newmacro{\worstdiameter}{\diameter_{\opt}}

\newmacro{\invariantmeasures}{\mathrm{Inv}} % Space of invariant measures
\newmacro{\invariantmeasure}{\mu} % Notation for invariant measure
\newmacro{\candidatemeasures}{\mathcal{I}} % Condidate exploration measures 

\def\imeasure{\invariantmeasure}

\newmacro{\Reg}{\mathrm{\normalfont Reg}} % regret 
\newmacro{\FOReg}{\mathrm{\normalfont FOReg}} % regret : sum of gaps (sometimes pseudo-regret)
\newmacro{\stochasticregret}{\Delta} % stochastic regret
\newmacro{\regretlowerbound}{K} % regret lower bound
\newmacro{\minimaxcomplexity}{\regretlowerbound}

\newmacro{\alternative}{\textrm{\normalfont Alt}} % Alternative models
\newmacro{\confusing}{\textrm{\normalfont Cnf}} % Confusing models

\newmacro{\loss}{\ell} % loss function

\DeclareMathOperator{\RegExp}{\mathrm{RegExp}}

\newmacro{\alg}{\Lambda} % Algorithm
\newmacro{\learner}{\alg} % Algorithm
\newmacro{\visits}{N} % Visit counts
\newmacro{\stimes}{\mathcal{T}} % special times 
\newmacro{\episodes}{\mathcal{K}} % set of epsiodes
\newmacro{\explorationepisodes}{\mathcal{K}^-} % exploration episodes
\newmacro{\exploitationepisodes}{\mathcal{K}^+} % exploitation episodes

\usepackage{tikz}
\usetikzlibrary{backgrounds}

\tikzset{
    state/.style={
        draw,
        circle,
        fill=none,
    },
    action/.style={
        draw,
        circle,
        color=white,
        fill=black,
        inner sep=0pt,
        minimum size=4mm,
        anchor=center,
    },
    reward/.style={
        font={\scriptsize}
    },
    transition/.style={
        ->,
        >=stealth,
    },
}

\def\Reach{\mathrm{Reach}}

\def\STEP#1{(\strong{STEP~#1})}
\def\QED{$\blacksquare$}

\newcommand{\strong}[1]{\textbf{#1}}
\newmacro{\Bernoulli}{\mathrm{Ber}}

\begin{document}

    \allowdisplaybreaks % to allow aligns to break other multiple lines
    \maketitle
    
    \begin{abstract}%
        In average reward Markov decision processes, state-of-the-art algorithms for regret minimization follow a well-established framework:
        They are model-based, optimistic and episodic. 
        First, they maintain a confidence region from which optimistic policies are computed using a well-known subroutine called Extended Value Iteration (\texttt{EVI}).
        Second, these policies are used over time windows called episodes, each ended by the Doubling Trick \eqref{equation_doubling_trick} rule or a variant thereof.
        In this work, without modifying \texttt{EVI}, we show that there is a significant advantage in replacing \eqref{equation_doubling_trick} by another simple rule, that we call the Vanishing Multiplicative \eqref{equation_vanishing_multiplicative} rule.
        When managing episodes with \eqref{equation_vanishing_multiplicative}, the algorithm's regret is, both in theory and in practice, as good if not better than with \eqref{equation_doubling_trick}, while the one-shot behavior is greatly improved. 
        More specifically, the management of bad episodes (when sub-optimal policies are being used) is much better under \eqref{equation_vanishing_multiplicative} than \eqref{equation_doubling_trick} by making the regret of exploration logarithmic rather than linear.
        These results are made possible by a new in-depth understanding of the contrasting behaviors of confidence regions during good and bad episodes.
    \end{abstract}


    \section{Introduction}

    Regret minimization in average reward Markov decision processes is a classical problem with a rich literature and landscape of methods. 
    Regarding theoretical guarantees (especially in the minimax setting), the most successful line of algorithms adapts the famous \texttt{UCB} algorithm of \cite{auer_using_2002} to Markov decision processes. 
    This includes \cite{auer_logarithmic_2006,tewari_optimistic_2007,auer_near_optimal_2009,bartlett_regal_2009,filippi_optimism_2010,fruit_efcient_2018,tossou_near_optimal_2019,fruit_improved_2020,bourel_tightening_2020,zhang_regret_2019,boone_achieving_2024} in particular, that are the focus of this work. 
    All these algorithms are episodic and follow the \emph{optimism-in-the-face-of-uncertainty} principle:
    During learning, they maintain a confidence region of plausible environments from which decisions are taken. 
    Specifically, they deploy policies achieving the highest average gain among all MDPs in the confidence region.
    This policy is used for a whole time interval called an episode, and is only updated when deemed necessary. 

    \begin{figure}[ht]
        \vspace{-.66em}

        \begin{tikzpicture}
            \node at (0, 0) {\includegraphics[width=.49\linewidth]{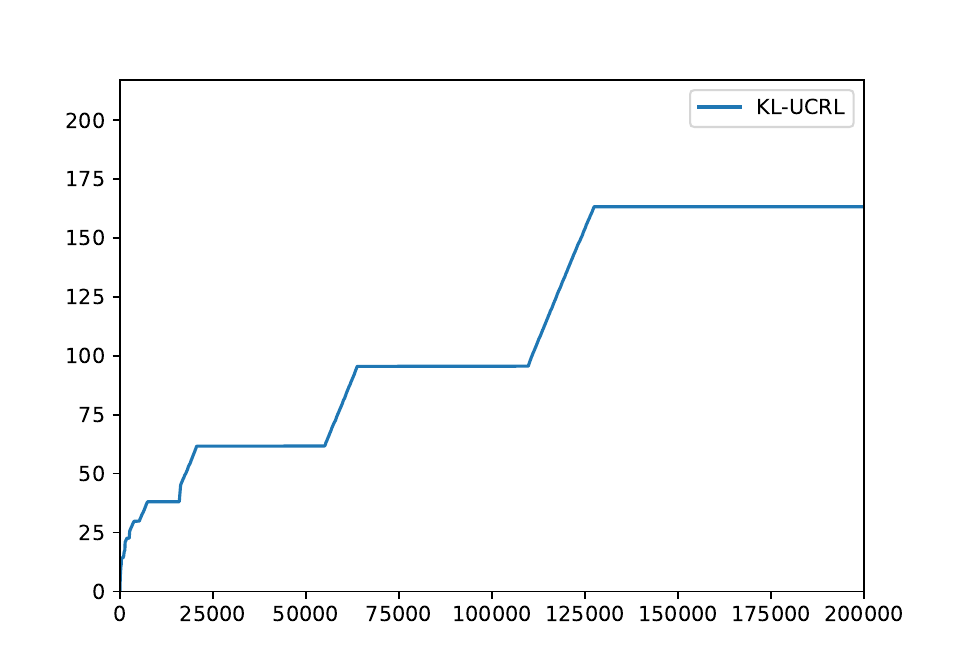}};
            \fill[color=red, opacity=0.3] (-2.37, -2.02) to (-2.37, 1.97) to (-2.18, 1.97) to (-2.18, -2.02);
            \fill[color=red, opacity=0.3] (-1.23, -2.02) to (-1.23, 1.97) to (-.93, 1.97) to (-.93, -2.02);
            \fill[color=red, opacity=0.3] (.35, -2.02) to (.35, 1.97) to (.92, 1.97) to (.92, -2.02);
        \end{tikzpicture}
        \hfill
        \begin{tikzpicture}
            \node at (0, 0) {\includegraphics[width=.49\linewidth]{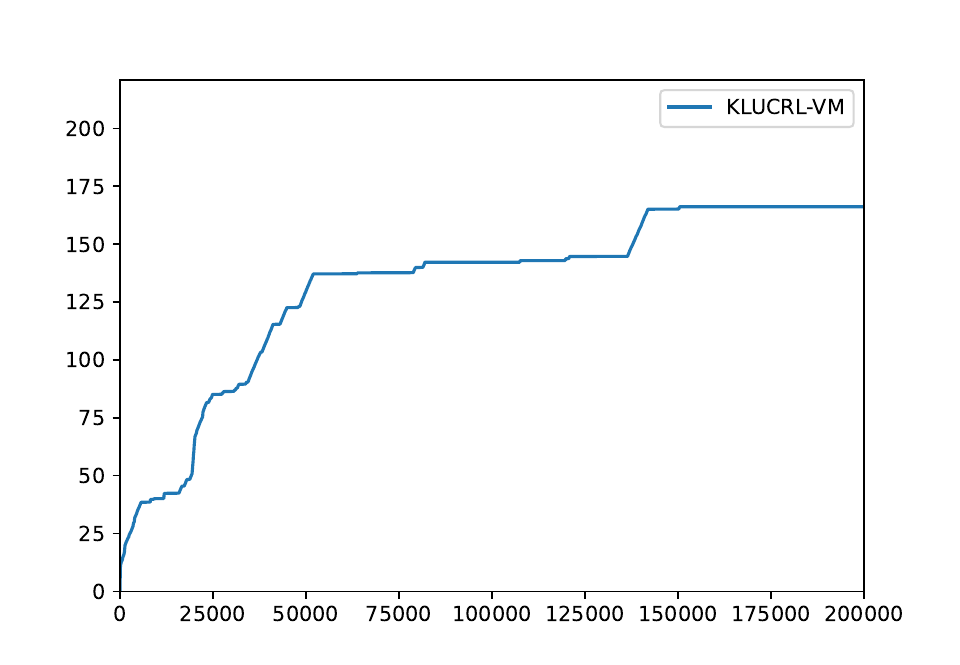}};
            \fill[color=red, opacity=0.3] (-1.8, -2.02) to (-1.8, 1.97) to (-1.71, 1.97) to (-1.71, -2.02);
            \fill[color=red, opacity=0.3] (-1.68, -2.02) to (-1.68, 1.97) to (-1.61, 1.97) to (-1.61, -2.02);
            \fill[color=red, opacity=0.3] (-1.57, -2.02) to (-1.57, 1.97) to (-1.5, 1.97) to (-1.5, -2.02);
            \fill[color=red, opacity=0.3] (-1.42, -2.02) to (-1.42, 1.97) to (-1.3, 1.97) to (-1.3, -2.02);
            \fill[color=red, opacity=0.3] (-0.97, -2.02) to (-0.97, 1.97) to (-.95, 1.97) to (-.95, -2.02);
            \fill[color=red, opacity=0.3] (-0.52, -2.02) to (-0.52, 1.97) to (-.49, 1.97) to (-.49, -2.02);
            \fill[color=red, opacity=0.3] (-0.45, -2.02) to (-0.45, 1.97) to (-.42, 1.97) to (-.42, -2.02);
            \fill[color=red, opacity=0.3] (.31, -2.02) to (.31, 1.97) to (.33, 1.97) to (.33, -2.02);
            \fill[color=red, opacity=0.3] (.65, -2.02) to (.65, 1.97) to (.67, 1.97) to (.67, -2.02);
            \fill[color=red, opacity=0.3] (.69, -2.02) to (.69, 1.97) to (.71, 1.97) to (.71, -2.02);
            \fill[color=red, opacity=0.3] (1.15, -2.02) to (1.15, 1.97) to (1.31, 1.97) to (1.31, -2.02);
            \fill[color=red, opacity=0.3] (1.54, -2.02) to (1.54, 1.97) to (1.56, 1.97) to (1.56, -2.02);
        \end{tikzpicture}

        \vspace{-.66em}
        \caption{
            \label{figure_introduction}
            The left plot displays the regret of \texttt{KLUCRL} \cite{filippi_optimism_2010} over a single run with highlighted periods of sub-optimal play that are increasing in duration. 
            In comparison, the right plot displays the regret of our proposed algorithm, where periods of sub-optimal play are much shorter resulting in a smoother regret curve. 
            \vspace{-1.5em}
        }
    \end{figure}
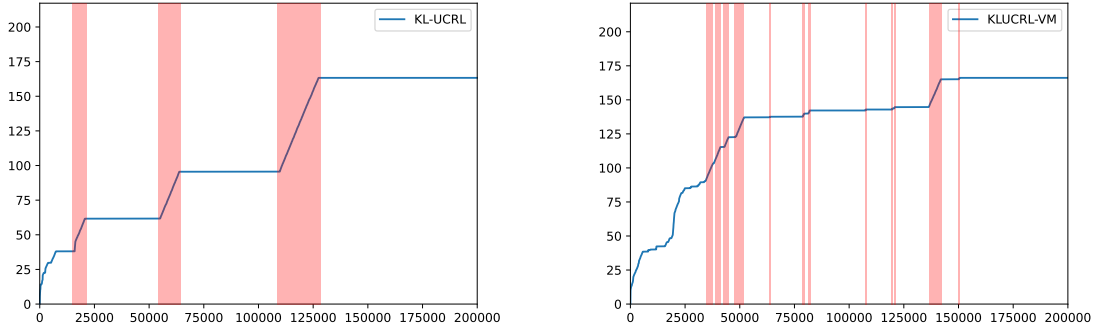

    This paper is not about improving the regret guarantees of these algorithms.
    Instead, we are interested in improving their long term behavior over a single run.
    In particular, we argue that state-of-the-art algorithms renew their policy too lazily, leading to long sequences of sub-optimal play, even when the learning process is well advanced. 
    This phenomenon appears strikingly during experiments: When running the classical \texttt{KLUCRL} of \cite{filippi_optimism_2010}, the algorithm displays periods of sub-optimal play that last for increasingly long durations, even after the initial burn-in phase is ended, see \Cref{figure_introduction}. 
    Such episodes of sub-optimal play are generally inevitable and correspond to the explorative part of the learning task; The learner has to make sure that seemingly bad actions are bad indeed. 
    The issue rather lies in the fact that the current design of all these algorithms makes such episodes increase exponentially in size. 
    This phenomenon was recently pointed out by \cite{boone_regret_2023} and measured by a new performance metric called the \strong{regret of exploration} (see \Cref{definition_regret_of_exploration}).
    The authors further suggest a way to obtain regret of exploration guarantees by refining the management of episodes. 
    However, their solution is computationally heavy and is only shown to work in the very restricted setting of Markov decision processes with deterministic transition kernels. 

    \paragraph{Contribution}
    In this paper, we go beyond \cite{boone_regret_2023} and provide a solution with better guarantees, both theoretically and experimentally. 
    We introduce a new simpler rule to end episodes, and show that the performance under the new episode rule guarantees logarithmic regret of exploration for two classes of MDPs: ergodic, and communicating MDPs with prior information on the support of the transition kernel.
    Our analysis is generic and focuses on when and how the confidence region used by an optimistic algorithm is well-behaved so that episodes of sub-optimal play are short and isolated. 
    We further show that the regret guarantees remain mostly intact, both in the model independent (minimax) and model dependent settings. 

    \section{Preliminaries}

    \paragraph{General notations}
    Given a finite set $\mathcal{X}$, we denote $\probabilities(\mathcal{X})$ the set of probability measures over $\mathcal{X}$.
    For $\kerrew \in \probabilities(\mathcal{X})$ and $f : \mathcal{X} \to \R$ a measurable map, we write $\kerrew f := \integral f(x) d\kerrew(x)$ the average of $f$ against $\kerrew$. 
    The Kullback-Leibler divergence from distribution $\kerrew'$ to $\kerrew'$ is denoted $\KL(\kerrew||\kerrew')$ and we further write $\kl(p,p') := \KL(\Bernoulli(p)||\Bernoulli(p')) = p \log(\frac p{p'}) + (1-p) \log(\frac{1-p}{1-p'})$ the divergence from Bernoulli distribution of parameters $p'$ to $p$. 
    Given a finite set $\states$, we denote $e = (1, \ldots, 1) \in \R^\states$ the constant unitary vector and $(e_\state)_{\state \in \states}$ the canonical basis of $\R^\states$. 
    The \strong{span semi-norm} of a vector $u \in \R^\states$ is $\vecspan{u} := \max(u) - \min(u)$. 

    \subsection{Markov decision processes in average reward}

    This work uses standard notations for Markov decision processes in average reward in the style of \cite[§8-9]{puterman_markov_1994}.
    A \strong{Markov decision process} (or \strong{model}) consists in a tuple $\model \equiv (\states, \actions, \kernel, \reward)$ made of a state space $\states$ and an action space $\actions \equiv \bigcup_{\state \in \states} \actions(\state)$ together forming a state-action pair space $\pairs := \bigcup_{\state \in \states} \braces{\state} \times \actions(\state)$, a transition kernel $\kernel : \pairs \to \probabilities(\states)$ and reward distributions $\reward: \pairs \to \probabilities(\R)$. 
    The {reward-kernel pair} is $\kerrew := (\reward, \kernel) : \pairs \to \probabilities(\R) \times \probabilities(\states)$.

    \begin{assumption}
    \label{assumption_finite_bernoulli}
        The pair space $\pairs$ is known and finite and rewards are Bernoulli.
    \end{assumption}

    Since rewards have Bernoulli distributions (\Cref{assumption_finite_bernoulli}), we will  use a harmless abuse of notations and write $\reward(\pair) \in [0,1]$ for the mean reward function at $\pair \in \pairs$.

    \subsubsection{Interacting with a Markov decision process using policies}

    The set of stationary deterministic policies is $\policies \equiv \states \to \actions$.
    We denote $\State_t, \Action_t, \Reward_t$ the random state, action and reward observed at time $t$, and $\Pair_t := (\State_t, \Action_t)$ is the associated pair. 
    By construction, $\State_{t+1} \sim \kernel(\Pair_t)$ and $\Reward_t \sim \Bernoulli(\reward(\Pair_t))$. 
    The (observed) {history of play} is $\History_t := (\State_1, \Action_1, \Reward_1, \ldots, \State_t)$ and $\histories$ is the space of all possible histories. 
    Fixing the environment $\model$, the policy $\policy \in \policies$ and the initial state $\state \in \states$ properly defines a probability space on the set of histories, or, said more loosely, determines the distribution of $(\State_t, \Action_t, \Reward_t)_{t \ge 1}$.
    We write $\EE_{\state}^{\model, \policy}[-]$ and $\Pr_{\state}^{\model, \policy}\parens{-}$ the associated expectation and probability operators. 

    The \strong{visit count} of a pair $\pair \in \pairs$ is written $\visits_T(\pair) := \sum_{t=1}^{T-1} \indicator{\Pair_t = \pair}$. 

    The \strong{gain} and \strong{bias} functions of a policy $\policy \in \policies$ from the initial state $\state \in \states$ are denoted $\gain^\policy(\state; \model)$ and $\bias^\policy(\state; \model)$ and given by the formulas
    \begin{align*}
        \gain^\policy(\state; \model)
        & := 
        \lim_{T \to \infty}
        \EE_{\state}^{\model, \policy} \brackets*{
            \frac 1T
            \sum_{t=1}^{T}
            \Reward_t
        }, 
        \\
        \bias^\policy(\state; \model)
        & :=
        \Clim_{T \to \infty}
        \EE_{\state}^{\model, \policy} \brackets*{
            \sum_{t=1}^{T}
            \parens*{
                \Reward_t - \gain^{\policy}(\State_t; \model)
            }
        }.
    \end{align*}
    The optimal gain and bias functions are respectively $\optgain(\model)$ and $\optbias(\model)$ and the set of gain optimal policies, i.e., policies $\policy \in \policies$ such that $\gain^\policy(\model) = \optgain(\model)$, are denoted $\optpolicies(\model)$. 
    In particular, $\optgain(\model) := \max_{\policy \in \policies} \gain^{\policy}(\model)$ and $\optbias(\model) := \max_{\policy \in \optpolicies(\model)} \bias^\policy(\model)$. 
    We define the \strong{Bellman gaps} $\ogaps(-;\model) : \pairs \to \R$ as the gaps in Bellman's optimality equations:
    \begin{equation}
    \label{equation_bellman_gaps}
        \ogaps(\state, \action; \model) 
        := 
        \optgain(\state; \model) + \optbias(\state; \model) - \reward(\state, \action) - \kernel(\state, \action) \optbias(\model)
        .
    \end{equation}
    A pair $\pair \in \pairs$ is said \strong{weakly-optimal}, written $\pair \in \wkoptpairs(\model)$, if it has Bellman gap $\ogaps(\pair; \model) = 0$; and \strong{sub-optimal} otherwise, written $\pair \in \suboptimalpairs(\model)$. 
    A pair $\pair \in \pairs$ is said \strong{optimal}, written $\pair \in \optpairs(\model)$, if $\pair \in \wkoptpairs(\model)$ and it is visited infinitely often (almost surely) under some gain optimal policy.
    Note that $\optpairs(\model) \subseteq \wkoptpairs(\model)$ by definition.

    \subsubsection{Communicating Markov decision processes}

    All throughout the paper, the models that we consider are always \strong{communicating} (\Cref{assumption_communicating}).

    \begin{assumption}
    \label{assumption_communicating}
        In this work, all Markov decision processes are {communicating}, i.e., that every state is reachable from any other under the right policy, meaning that the {diameter} is finite:
        \begin{equation}
        \label{equation_diameter}
            \diameter(\model)
            :=
            \max_{\state \ne \state'}
            \min_{\policy \in \policies}
            \EE_{\state}^{\model, \policy}
            \brackets*{
                \inf \braces*{
                    t \ge 1
                    :
                    \State_t = \state'
                }
            }
            <
            \infty
            .
        \end{equation}
    \end{assumption}

    \Cref{assumption_communicating} is pretty common nowadays.
    It is the core assumption made in the seminal paper of \cite{auer_near_optimal_2009} and most subsequent works; It is much more general than the ergodic assumption of \cite{agrawal_adaptive_1990,burnetas_optimal_1997,pesquerel_imed_rl_2022}; It is not completely general either.
    This assumption is absolutely necessary for the well-behavior of the \texttt{EVI}-subroutine of \cite{auer_near_optimal_2009} (discussed downstream), which is the common foundation of the algorithms of interest in this paper. 

    Under \Cref{assumption_communicating}, the optimal gain is a constant vector with $\optgain(\state; \model) \in \R e$ so that we write $\optgain(\model) \in \R$ in place of $\optgain(\state; \model)$; And Bellman gaps are non-negative, i.e., $\ogaps(\pair; \model) \ge 0$ for all $\pair \in \pairs$. 
    In the sequel, the dependency in $\model$ is dropped when unambiguous.

    \subsection{Reinforcement learning and regret minimization}

    A \strong{learning algorithm} is formally a measurable map $\learner : \histories \to \probabilities(\actions)$, mapping histories of observations to probabilistic choices of actions. 
    Similarly to policies, fixing the environment $\model$, a learner $\learner$ and the initial state $\state \in \states$ properly defines the distribution of $(\State_t, \Action_t, \Reward_t)_{t \ge 1}$ and we write $\EE_{\state}^{\model, \learner}[-]$ and $\Pr_{\state}^{\model, \learner}\parens{-}$ the associated expectation and probability operators. 
    The objective of the learner is to maximize $\Reward_1 + \ldots + \Reward_T$, and their ability to do so is measured by the \strong{regret}, that compares the amount of reward that a gain optimal policy $\optpolicy \in \optpolicies(\model)$ (dependent on $\model$) and the learner are able to collect within the same time budget. 
    Following standard MDP theory, $\abs{\EE_{\state}^{\model, \optpolicy}[\Reward_1 + \ldots + \Reward_T] - T \optgain} \le \vecspan{\optbias}$ so that in this  setting, the regret is usually defined as $T \optgain - \sum_{t=1}^T \Reward_t$, see \cite{auer_near_optimal_2009}. 
    In this work, we consider a pseudo-regret instead (\Cref{definition_regret}), to remove random noise over which the learner as no control.
    The expected regret defined below is equal to the classical one, up to an inconsequential additive constant, with $\abs{\Reg(T; \model, \learner, \state) - \EE_{\state}^{\model, \learner}[T \optgain - \sum_{t=1}^T \Reward_t]} \le \vecspan{\optbias}$.

    \begin{definition}
    \label{definition_regret}
        The \strong{pseudo-regret} of an algorithm $\learner$ over $\model$ is the random variable given by:
        \begin{equation}
            \stochasticregret(1, T) 
            \equiv
            \stochasticregret(T) 
            := 
            \sum_{t=1}^T \ogaps(\Pair_t; \model)
        \end{equation}
        and the \strong{expected regret} is $\Reg(T; \model, \learner, \state) := \EE_{\state}^{\model, \learner}[\stochasticregret(T)]$. 
    \end{definition}

    The lower the regret, the better the learner performs.
    A learner $\learner$ is said \strong{no-regret} relatively to a set of models $\models^0$, or {Hannan consistent} \cite{hannan_approximation_1957}, if for all {communicating} $\model \in \models^0$ and regardless of the initial state $\state \in \states$, $\Reg(T; \model, \learner, \state) = \oh(T)$. 
    $\models^0$ will be called the \strong{ambient set} and is a form of prior information.
    We further assume that $\models^0$ is in product form.

    \begin{assumption}
    \label{assumption_product_form}
        The ambient set $\models^0$ is in \strong{product form}, i.e., it is of the form $\models^0 \equiv \product_{\pair \in \pairs} (\rewards_\pair^0 \times \kernels^0_\pair)$ where $\rewards_\pair^0 \subseteq [0, 1]$ and $\kernels_\pair^0 \subseteq \probabilities(\states)$. 
    \end{assumption}

    \subsection{Optimistic model-based and \texttt{EVI}-based algorithms}
    \label{section_optimistic_model_based}

    There is a large literature on algorithms with regret guarantees. 
    In this paper, we focus on \strong{optimistic model-based algorithms}, which is a line of algorithms adapted from the well-known \texttt{UCB} \cite{auer_using_2002}. 
    They follow the \emph{optimism-in-the-face-of-uncertainty} (OFU) principle: when unsure about the value of an action or a policy, estimate that value as the highest that is statistically plausible. 
    Ever since \texttt{UCRL} \cite{auer_logarithmic_2006}, the main incarnation of this principle is the following.
    Over time, maintain a confidence region $\models(t)$ that contains $\model$ with high probability and work in an episodic fashion.
    An \strong{episode} is a time segment $\braces{t_k, \ldots, t_{k+1}-1}$ during which the algorithm plays a fixed policy $\policy_{t_k} \in \policies$. 
    This policy is computed as the policy achieving the highest gain on the best plausible model at time $t_k$.
    More formally, we define the \strong{optimistic gain} of $\policy$ in $\models(t)$ from $\state \in \states$ as 
    \begin{equation}
    \label{equation_optimistic_gain}
        \gain^\policy(\state; \models(t)) 
        := 
        \sup_{\model' \in \models(t)} \gain^\policy(\state; \model')
        .
    \end{equation}
    An \strong{optimistic policy} $\pi$ is such that $\ \gain^\policy(\models(t)) := \sup_{\policy \in \policies} \gain^\policy(\models(t))$. 
    Perhaps surprisingly, optimistic policies are easy to compute from $\models(t)$ via a process called {Extended Value Iteration} (\texttt{EVI}), see \cite{auer_near_optimal_2009}.
    \Cref{algorithm_optimistic} provides the general architecture of these algorithms. 
    
    \begin{algorithm}[h]
        \begin{algorithmic}[1]
            \FOR {$t = 1, 2, \ldots$}
                \IF {the current policy $\policy_{t_k}$ is obsolete}
                    \STATE Increase $k$, set $t_k \gets t$ and compute $\policy_{t_k} \gets \texttt{EVI}(\models(t_k))$;
                \ENDIF
                \STATE Set $\policy_t \gets \policy_{t_k}$ and play $\Action_t := \policy_{t}(\State_t)$;
            \ENDFOR
        \end{algorithmic}
        \caption{
            \label{algorithm_optimistic}
            \texttt{EVI}-based algorithms.
        }
    \end{algorithm}

    \paragraph{Selected focus: \texttt{KLUCRL}}
    The scheme of \texttt{EVI}-based algorithms (\Cref{algorithm_optimistic}) can be improved along two axis. 
    The first is the choice of confidence region.
    According to Sanov's theorem, the tightest way to construct $\models(t)$ so that it contains $\model$ with high probability is to rely on KL divergences (see \Cref{appendix_choice_klucrl}), leading to a region $\models(t) \equiv \prod_{\pair \in \pairs} \rewards_\pair(t) \times \kernels_\pair(t)$ with:
    \begin{equation}
    \label{equation_confidence_region}
    \begin{gathered}
        \rewards_\pair(t)
        := 
        \braces*{
            \tilde{\reward}_\pair \in [0, 1]
            :
            \visits_\pair (t)
            \KL(\hat{\reward}_\pair(t)||\tilde{\reward}_\pair)
            \le
           \log(2 e t)
        }
        \cap 
        \rewards_\pair^0
        \\
        \kernels_\pair(t)
        := 
        \braces*{
            \tilde{\kernel}_\pair \in \probabilities(\states)
            :
            \visits_\pair (t)
            \KL(\hat{\kernel}_\pair(t)||\tilde{\kernel}_\pair)
            \le
            \abs{\states} \log(2 e t)
        }
        \cap
        \kernels_\pair^0
    \end{gathered}
    \end{equation}
    where $\hat{\reward}_\pair (t)$ and $\hat{\kernel}_\pair (t)$ are the {empirically} observed reward and kernels after $t$ learning steps. 
    The expression of \eqref{equation_confidence_region} is tuned so that $\Pr(\exists t \ge T: \model \notin \models(t)) \le {2\abs{\pairs}}T^{-1}$, see \Cref{appendix_confidence_region}.
    Our work could be adapted to other types of confidence region, e.g., $\ell_1$ or $\ell_2$ norms, or Bernstein's style inequalities but the above will be the selected focus for its superiority over the others, both theoretically and empirically.
    Note that in \eqref{equation_confidence_region}, $\models(t)$ is constrained to the ambiant set of MDPs $\models^0$.
    This is a form of prior information. 

    The second potential  improvement axis is the way to determine whether the current policy is obsolete or not.
    Most of the literature relies on the \strong{doubling trick} \eqref{equation_doubling_trick} or variants thereof, that essentially wait for a pair to increase its visit count multiplicatively---by $2$ for the doubling trick.
    \begin{equation}
    \tag{\texttt{DT}}
    \label{equation_doubling_trick}
        \visits_t(\State_t, \policy_{t_k}(\State_t))
        \ge
        \max \braces[\big]{
            2 \visits_{t_k}(\State_t, \policy_{t_k}(\State_t))
            ,
            1
        }.
    \end{equation}
    Choosing $\models(t)$ as in \eqref{equation_confidence_region} and managing episodes with \eqref{equation_doubling_trick} leads to our variant of the algorithm \texttt{KLUCRL} of \cite{filippi_optimism_2010} that can take the prior information $\models^0$ into account.
    From now on and to streamline the discussion, \texttt{KLUCRL} is the main focus. 

    \section{The regret of exploration of episodic algorithms}

    In this work, we move beyond regret minimization, by investigating additional regret guarantees localized in time. 
    To that end, the regret notations of \Cref{definition_regret} are overloaded as such: we denote $\stochasticregret(\tau, \tau') \equiv \stochasticregret(\tau, \tau'; \model) := \sum_{t=\tau}^{\tau'} \ogaps(\Pair_t)$ the pseudo-regret endured from $\tau$ to $\tau'$ where $\tau \le \tau'$ are two stopping times of the stochastic process. 
    We further write $\Reg(\tau, \tau'; \model) = \EE_{\state}^{\model, \learner}[\stochasticregret(\tau, \tau'; \model)]$ the associated expected regret. 

    \subsection{The definition of the regret of exploration beyond ergodic Markov decision processes}

    We start by generalizing the definition of the regret of exploration of \cite{boone_regret_2023} beyond ergodic models. 
    To measure the instantaneous performance of an algorithm that has exploration phases, one may be tempted to monitor the regret starting at times when the algorithm drops an optimal policy for a sub-optimal one, i.e., at times
    \begin{equation}
    \label{equation_intention}
        \braces*{
            t_k
            : 
            \policy_{t_{k}-1} \in \optpolicies(\model)
            \text{~and~}
            \policy_{t_k} \notin \optpolicies(\model)
        }
        .
    \end{equation}
    \Cref{equation_intention} captures the idea.
    However, beyond ergodic environments, \eqref{equation_intention} is not precise enough and is ill-behaved in general.
    The main reason is that deployed policies can be partially optimal and multi-chain. 
    For instance, \eqref{equation_intention} fails to capture time-instants where $\policy_{t_k-1}$ is gain optimal from the current state but not globally, while such times should be considered as exploration times too. 
    Indeed, the behavior of the algorithm does not depend on the actions chosen by the policy from states that can never be reached. 
    To account for such cases, the final definition of \strong{exploration times} (\Cref{definition_exploration_time}) is slightly more complex.

    \begin{definition}[Exploration]
    \label{definition_exploration_episode}
    \label{definition_exploration_time}
        An episode $k$ is an \strong{exploration episode} and $t_k$ is an \strong{exploration time} if the two following conditions are satisfied:
        {\upshape (1)} $\optgain(\model) = \gain^{\policy_{t_{k}-1}}(\State_{t_k}; \model)$; 
        and {\upshape (2)} we have $\Pr_{\state}^{\policy_{t_k}} \parens{\exists t \ge 1: \ogaps(\Pair_t; \model) \ne 0} > 0$.
        The set of exploration episodes is denoted $\explorationepisodes$.
    \end{definition}

    Written differently, $t_k$ is an exploration time if the learning agent drops a policy that is gain optimal from the current state for a policy that may use a sub-optimal pair if iterated over and over from the current state.
    When the underlying model is ergodic, the exploration times given by \Cref{definition_exploration_time} are equivalent to those defined using \eqref{equation_intention} and by \cite{boone_regret_2023}.

    We enumerate $\explorationepisodes$ as $(t_{k(i)})$ where $k(i)$ denotes the $i$-th exploration episode and $t_{k(i)}$ is the associated $i$-th initial exploration time. 
    Formally, $t_{k(1)} := \inf \explorationepisodes$ and $t_{k(i+1)} := \inf \braces{t_k > t_{k(i)} : k \in \explorationepisodes}$. 
    The \strong{regret of exploration} is defined as the worst expected regret at exploration times asymptotically.

    \begin{definition}[Regret of exploration]
    \label{definition_regret_of_exploration}
        Let $(t_{k(i)})$ be the enumeration of exploration times.
        The \strong{regret of exploration} is given by:
        \begin{equation}
            \RegExp(T)
            \equiv \RegExp(T; \model)
            :=
            \limsup_{i \to \infty} \Reg(t_{k(i)}, t_{k(i)}+T; M)
            .
        \end{equation}
    \end{definition}
    
    \subsection{Explorative Markov decision processes}
    \label{section_well_definition}

    The regret of exploration is only worth studying if there are infinitely many exploration times $t_{k(i)}$.
    This is not always the case.
    In fact, there exist Markov decision processes for which infinite exploration is somehow unnecessary, making them conceptually easier to learn than bandits. 
    In \Cref{appendix_exploration}, we provide a complete characterization of the set of Markov decision processes for which the number of exploration episodes is infinite and where the regret of exploration is well-defined:
    For such models, there exist \strong{consistent} (see \cite{salomon_lower_2013}) learners $\learner$ achieving $\Reg(T; \model, \learner) = \oh(\log(T))$. 
    This result is surprisingly difficult to establish and is peripheral to our work, hence completely deferred to \Cref{appendix_exploration}.
    Let us insist on the fact that given a class $\models$ of environments, the regret of exploration may be ill-defined for large sub-spaces of $\models$, even for reasonable learning algorithms. 
    This motivates the following definitions.

    \begin{definition}[Non-degeneracy]
    \label{definition_non_degeneracy}
        A model $\model \equiv (\pairs, \reward, \kernel)$ is said \strong{non-degenerate} if 
        {\upshape (1)} it has a unique Bellman optimal policy, i.e., there is a unique $\policy \in \policies$ that satisfies the first two orders optimality equations:
        \begin{align*}
            \forall \state \in \states,
            \quad
            \gain^\policy(\state) 
            & = 
            \max_{\action \in \actions(\state)} \braces*{
                \kernel(\state, \action) \gain^\policy
            }
            \\
            \forall \state \in \states,
            \quad
            \gain^\policy(\state) + \bias^\policy(\state)
            & = 
            \max_{\action \in \actions(\state)} \braces*{
                \reward(\state, \action)
                + \kernel(\state, \action) \bias^\policy
            }
        \end{align*}
        and {\upshape (2)} that unique Bellman optimal policy is unichain.  
    \end{definition}

    In \Cref{figure_non_degenerate}, we provide an example of degenerate and non-degenerate Markov decision processes.  
    On the left model, there are multiple Bellman optimal policies, as one can achieve optimal gain by either looping on the left or the right loop. 
    As a matter of fact, every policy excepted $1 \leftrightarrow 2$ is Bellman optimal. 
    On the right model, we add a small noise to the reward function.
    That noise breaks ties and the right loop becomes better than the other, so that the Bellman optimal policy is indeed unique ($1 \to 2 \to 2$) and unichain. 

    \begin{figure}[t]
        \centering
        \begin{tikzpicture}
            \node[state] (1) at (0, 0) {$1$};
            \node[state] (2) at (3, 0) {$2$};
            \draw[transition, loop] (1) to node[midway, above] {$0.5$} (1);
            \draw[transition, loop] (2) to node[midway, above] {$0.5$} (2);
            \draw[transition] (1) to[bend left] node[midway, above] {$0.1$} (2);
            \draw[transition] (2) to[bend left] node[midway, below] {$0.1$} (1);
        \end{tikzpicture}
        \begin{tikzpicture}
            \node[state] (1) at (0, 0) {$1$};
            \node[state] (2) at (3, 0) {$2$};
            \draw[transition, loop] (1) to node[midway, above] {$0.49$} (1);
            \draw[transition, loop] (2) to node[midway, above] {$0.51$} (2);
            \draw[transition] (1) to[bend left] node[midway, above] {$0.09$} (2);
            \draw[transition] (2) to[bend left] node[midway, below] {$0.08$} (1);
        \end{tikzpicture}

        \caption{
        \label{figure_non_degenerate}
            Examples of non-degeneracy (\Cref{definition_non_degeneracy}).
            A degenerate Markov decision process (to the left) and a non-degenerate Markov decision process (to the right).
            Both models have deterministic transitions represented with arrows. 
            Labels are reward means. 
        }
    \end{figure}
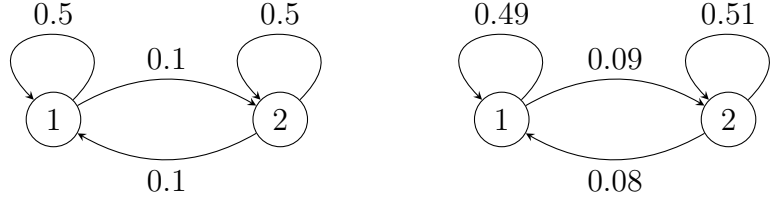

    The main observation from \Cref{figure_non_degenerate} can be generalized: All ``noisy'' versions of the non-degenerate model (to the left) are non-degenerate. 
    This means that ``almost-all'' communicating Markov decision processes are non-degenerate and makes the non-degeneracy assumption mild.
    More details are found in \Cref{appendix_technical}.

    From now on, we will focus on non-degenerate Markov decision processes. 

    \begin{definition}[Explorative models]
    \label{definition_explorative}
        Given a space of Markov decision processes $\models$, its \strong{explorative sub-space} $\models^+$ is the set of non-degenerate models $M \in \models$ such that every algorithm {\upshape (1)} with sub-linearly many episodes and {\upshape (2)} which is no-regret on $\models$, has infinitely many exploration episodes almost surely.
        Non-explorative models are said \strong{exploration-free}.
    \end{definition}

    Note that the explorative character of a Markov decision process depends on the ambient space. 
    This makes explorative models slightly more difficult to describe than non-degenerate ones.
    That discussion is deferred to \Cref{appendix_exploration}.

    \subsection{The doubling trick leads to linear regret of exploration}
    \label{section_regexp}

    In \Cref{figure_introduction}, we observe that the regret at exploration times of \texttt{KLUCRL} that uses \eqref{equation_doubling_trick} increases overall.
    This follows from a general principle that is quite intuitive:
    If a change of episode requires an increase of visits relatively to the initial visit count vector, and if deployed policies do not play actions with vanishing probabilities (see \eqref{equation_linear_condition}), then the regret of exploration grows linearly on recurrent models at least. 
    The following theorem is an alternative version of \cite[Theorem~1]{boone_regret_2023}, adapted to our definition of exploration times.

    \begin{theorem}
    \label{theorem_linear_regexp}
        Fix a pair space $\pairs$ and let $\models$ be the space of all recurrent models with pairs $\pairs$. 
        Let $f : \N \to (0, \infty)$ be such that $\lim f(n) = + \infty$. 
        Any no-regret episodic learner $\planner$, i.e., using fixed policies over episodes $\braces{t_k, \ldots, t_{k+1}-1}$, satisfying
        \begin{equation}
        \label{equation_linear_condition}
        \begin{gathered}
            \forall k \ge 1,
            \exists \pair \in \pairs,
            \quad
            \visits_\pair({t_{k+1}}) \ge \visits_\pair({t_k}) + f\parens*{\visits_\pair({t_k})}
            \\
            \exists c > 0,
            \forall t \ge 0, \forall (\state, \action) \in \pairs,
            \quad
            \policy_t(\action|\pair) \ge c \text{~or~} \policy_t(\action|\pair) = 0
        \end{gathered}
        \end{equation}
        has linear regret of exploration on the explorative sub-space of $\models$, i.e., for all $\model \in \models^+$, we have $\RegExp(T) = \Omega(T)$ a.s.~when $T \to \infty$.
    \end{theorem}

    This result applies to \texttt{KLUCRL} and more generally to all algorithms relying on the doubling trick \eqref{equation_doubling_trick} to manage episodes, corresponding to  $f(n) = n \vee 1$.
    This includes \texttt{UCRL2} \cite{auer_near_optimal_2009}, \texttt{REGAL} \cite{bartlett_regal_2009}, \texttt{KLUCRL} \cite{filippi_optimism_2010}, \texttt{UCRL2B} \cite{fruit_improved_2020}, \texttt{SCAL} \cite{fruit_efcient_2018}, \texttt{UCRL3} \cite{bourel_tightening_2020}, \texttt{EBF} \cite{zhang_regret_2019} and also \texttt{PMEVI} \cite{boone_achieving_2024} (up to mild modifications of \eqref{equation_linear_condition} for a few of them).
    Therefore, \Cref{theorem_linear_regexp} pinpoints an issue: 
    The local regret of current optimistic algorithms is the worst possible, because the regret of exploration of these methods grows linearly. 
    In this paper, we will fix this problem without too many side-effects. 
    We alter these algorithms (focusing on \texttt{KLUCRL}) and achieve sub-linear regret of exploration without hurting the minimax regret guarantees.
    This is achieved by a small and cost-less modification of the episode stopping rule.

    \section{Logarithmic regret of exploration with the vanishing multiplicative condition}
    \label{section_vm_regret_guarantees}

    Our solution improves on \cite{boone_regret_2023}, where the regret of exploration guarantees are only proved for deterministic transition MDPs. 
    Their solution consists in stopping an episode if the current policy is no longer optimistically optimal \emph{enough}. This is done by introducing a function $\psi(t)$ and ending episode $k$ if $\gain^{\policy_{t_k}}(\models(t)) + \psi(t) < \optgain(\models(t))$. 
    This approach has the clear issue that one has to constantly  monitor the values of $\gain^{\policy_{t_k}}(\models(t))$ and $\optgain(\models(t))$. This has a high computational  cost.
    Despite these limitations, the main observation of \cite{boone_regret_2023} is key:
    if $\policy_{t_k}$ is sub-optimal, its optimistic gain $\gain^{\policy_{t_k}}(\models(t))$ decreases quickly over $\braces{t_k, \ldots, t}$ and otherwise, if $\policy_{t_k}$ is optimal, then $\gain^{\policy_{t_k}}(\models(t))$ rather behaves like a random walk.
    Essentializing their argument, it actually boils down to show that the behaviors of confidence regions $\rewards_\pair(t)$ and $\kernels_\pair(t)$ are very different at high and low visit counts of $\pair$. 
    At high visit counts, the evolution of confidence regions is slow and they mostly behave like random walks.
    At low visit counts, they \emph{shrink} quickly as the number of visits increases. We call these two distinct behaviors the \strong{shrinking-shaking effect}.
    It motivates a different and simpler approach than the one of \cite{boone_regret_2023}: Renew the episode when there is an increase of information.
    This is actually the idea of the \eqref{equation_doubling_trick}, but instead of asking for a visit count to double, we suggest to wait for a multiplicative increase with respect to a vanishing time-dependent factor, i.e.,
    \begin{equation}
    \tag{\texttt{VM}}
    \label{equation_vanishing_multiplicative}
        \visits_t(\State_t, \policy_{t_k}(\State_t))
        > 
        (1 + f(t_k)) 
        \max\braces*{
            1,
            \visits_{t_k}(\State_t, \policy_{t_k}(\State_t))
        }
    \end{equation}
    \noindent
    where $f : \N \to [0,1]$ is a non-increasing vanishing function of $t$. 
    The above condition will be referred to as the $f$-\strong{Vanishing Multiplicative condition}, or $f$-\eqref{equation_vanishing_multiplicative}, or even more simply \eqref{equation_vanishing_multiplicative}.
    Remark that \eqref{equation_doubling_trick} is also of the form \eqref{equation_vanishing_multiplicative} with $f \equiv 1$, except that this function is not vanishing.

    By changing \eqref{equation_doubling_trick} for \eqref{equation_vanishing_multiplicative}, we get the following range of regret guarantees for \texttt{KLUCRL}.

    \begin{blackblock}
    \begin{theorem}[Main result]
    \label{theorem_main}
        Let $f : \N \to [0, 1]$ and consider running \texttt{KLUCRL} with episodes managed by $f$-\eqref{equation_vanishing_multiplicative} with prior information $\models^0$ and let $\models_\diameter$ be the set of Markov decision processes with diameter less than $\diameter$. 
        We have:
        \begin{itemize}
            \item[1.] \emph{Minimax}: For $f = \Omega(t^{-1/2})$, $\displaystyle \sup_{\model' \in \models_\diameter} \Reg(T; \model') = \textstyle \OH(\diameter \State \sqrt{\Action T \log(T)})$; 
        \end{itemize}
        Moreover, if $\model$ satisfies \Cref{assumption_interior} and is explorative, we have:
        \begin{itemize}
            \item[2.] \emph{Model dependent}: If $f > 0$, then $\Reg(T; \model) = \OH(\log(T)\log\log(T))$;
            \vspace{-1em}
            \item[3.] \emph{Regret of exploration}: If $f(t) = \oh\parens*{\frac 1{\log(t)}}$, then $\RegExp(T; \model) = \OH(\log(T))$.

        \end{itemize}
    \end{theorem}
    \end{blackblock}

    Note that the model dependent regret guarantees and the regret of exploration guarantees only hold if $\model$ satisfies a structural assumption with respect to the ambient set of models, $\models^0$. 

    \begin{assumption}[Interior assumption]
    \label{assumption_interior}
        For $\pair \in \pairs$, $\reward(\pair)$ and $\kernel(\pair)$ are in the interior of $\rewards_\pair^0$ and $\kernels_\pair^0$ respectively, i.e., $\mathrm{supp}(\reward(\pair)) = \braces{0, 1}$ and $\mathrm{supp}(\kernel(\pair)) \supseteq \mathrm{supp}(\kernel'(\pair))$ for all $\kernel'(\pair) \in \kernels_\pair^0$. 
    \end{assumption}

    \paragraph{Comment 1}
    The minimax regret guarantees (\textit{1.}) given in \Cref{theorem_main} are the same as for the original \eqref{equation_doubling_trick} version of \texttt{KLUCRL}.
    The model dependent guarantees (\textit{2.}) only suffer from an additional $\log \log(T)$ factor, while the regret of exploration guarantees are improved from $\Omega(T)$ to $\OH(\log(T))$.
    Moreover, although \Cref{theorem_main} is stated specifically for \texttt{KLUCRL}, it can be generalized to other types of confidence regions such as $\ell_1$, $\ell_2$ or Bernstein-type., inducing similar results for other algorithms such as \texttt{UCRL}, \texttt{UCRL2},  \texttt{UCRL3}, \texttt{EBF} or \texttt{PMEVI}.

    \paragraph{Comment 2}
    When $\models^0 = \product_{\pair \in \pairs} ([0, 1] \times \probabilities(\states))$ (no prior information), $\model$ satisfies \Cref{assumption_interior} if, and only if $\model$ is an ergodic model with fully-supported transition kernels. 
    However, non-ergodic models can also be covered when prior information on the support of $\kernel$ is available.

    \medskip

    \paragraph{Outline of the proof}
    The minimax guarantees directly follow from a straight-forward upper bound of the number of episodes under $f$-\eqref{equation_vanishing_multiplicative} and are detailed in \Cref{appendix_minimax}.
    The model dependent guarantees are proved in \Cref{appendix_model_dependent}.
    In the remaining of the paper, we focus on the analysis of the regret of exploration.
    The proof is challenging and requires an in-depth understanding of the behavior of optimistic algorithms in the long run. 

    First, we establish a noteworthy difference between the visit rates of optimal and non-optimal pairs in \Cref{section_macroscopic_coherence}: the first are visited linearly and the others at most logarithmically.
    Following this observation, we explain in \Cref{section_shrinking_shaking} how these distinct rates imply drastically different behaviors of the associated confidence regions, referred to as the \strong{shrinking-shaking} dichotomy. 
    In turn, it leads to a general conceptual property that we call \strong{coherence} in \Cref{section_coherence}.  
    Lastly, we show that regret of exploration is logarithmic because of coherence, see \Cref{section_establishing_coherence}.

    \subsection{Visit rates of optimal and non-optimal pairs}
    \label{section_macroscopic_coherence}

    The result below describes the almost-sure asymptotic regime of versions of \texttt{KLUCRL} managing episodes with \eqref{equation_vanishing_multiplicative}. 
    Up to the non-degeneracy of the underlying model, the visit counts can be split into two regimes: $N_z(t)$ grows linearly with $t$ for $z \in \pairs^{**}(M)$ while $N_z(t)$ grows sub-logarithmically for $z \notin \pairs^{**}(M)$ (including $\optpairs(\model) \setminus \pairs^*(M)$ in particular).

    \begin{lemma}[Almost-sure asymptotic regime]
    \label{lemma_asymptotic_regime}
        Let $\model \in \models$ be a non-degenerate model satisfying \Cref{assumption_interior}.
        Assume that \texttt{KLUCRL} is run while managing episodes with $f$-\eqref{equation_vanishing_multiplicative} with arbitrary $f > 0$.
        There exists $\lambda > 0$ such that:
        \begin{align*}
            \forall z \notin \optpairs(M), \quad
            & \Pr^M \parens*{
                \exists T, \forall t \ge T :
                N_z(t) < \lambda \log(t)
            } = 1, 
            \text{~and}
            \\
            \forall z \in \optpairs(M), \quad
            & \Pr^M \parens*{
                \exists T, \forall t \ge T :
                N_z(t) > \tfrac 1\lambda t
            } = 1.
        \end{align*}
    \end{lemma}

    The result holds in particular for \eqref{equation_doubling_trick}.
    This is not much of a surprise, since \texttt{KLUCRL} is known to have logarithmic model dependent regret.
    In contrast, this result is remarkable for \eqref{equation_vanishing_multiplicative} because the number of episodes can arbitrarily greater than logarithmic. 

    Note that \eqref{equation_doubling_trick} and \eqref{equation_vanishing_multiplicative} differ in the amount of time a sub-sampled pair can be visited during an episode. 
    Indeed, for $z \in \pairs$ with $\Pr^M(\exists T, \forall t \ge T : N_z(t) < \lambda \log(t)) = 1$, we have
    \begin{equation}
    \label{equation_vanishing_ends}
        N_z(t_{k+1}) 
        \le \floor{\parens*{1 + f(t_k)} N_z(t_k)} + 1
        = N_z(t_k) + 1 + \floor{\lambda f(t_k) \log(t_k)}.
    \end{equation}
    For $f(t) = \oh\parens*{\frac 1{\log(t)}}$, we have $\floor{\lambda f(t_k) \log(t_k)} = 0$ provided that $t_k$ is large enough.
    So, following \eqref{equation_vanishing_ends}, sub-sampled pairs are visited at most once per episode in the long run.
    So, under \eqref{equation_vanishing_multiplicative}, \texttt{KLUCRL} almost instantly refreshes its policy when sub-optimal pairs are visited. 

    \subsection{The shrinking-shaking dichotomy in the behavior of confidence regions}
    \label{section_shrinking_shaking}

    The shrinking-shaking effect concerns the way confidence regions evolve over time under low (shrinking) and high (shaking) amounts of information.
    This is illustrated in \Cref{figure_shrinking_shaking}.

    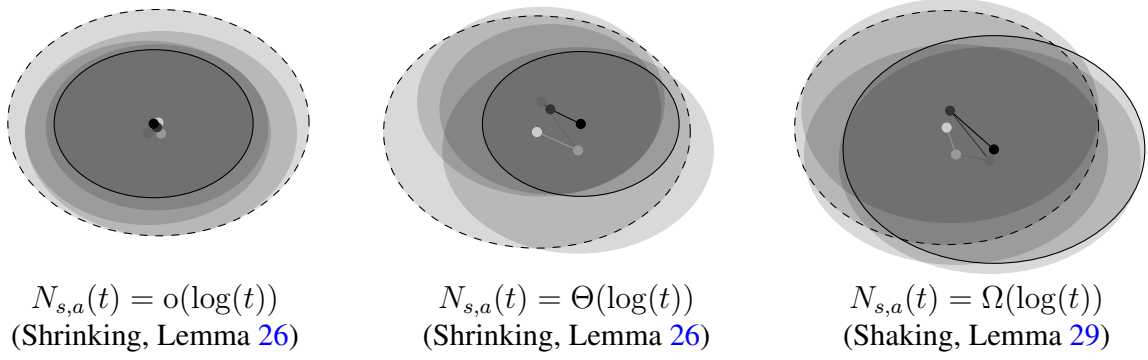
\begin{figure}[h]
        \centering
        \begin{tikzpicture}
            \begin{scope}[scale=1]
                \draw[fill=black, opacity=0.15] (0.045,0.1) ellipse(1.992 and 1.497);
                \draw[dashed] (0.045,0.1) ellipse(1.992 and 1.497);
                \draw[fill=black, opacity=0.15] (0.08,-0.044) ellipse(1.792 and 1.349);
                \draw[fill=black, opacity=0.15] (-0.093,-0.035) ellipse(1.618 and 1.212);
                \draw[fill=black, opacity=0.15] (0.027,0.039) ellipse(1.464 and 1.092);
                \draw[fill=black, opacity=0.15] (-0.018,0.086) ellipse(1.316 and 0.979);
                \draw (-0.018,0.086) ellipse(1.316 and 0.979);
                \draw[color=black!40] (0.045,0.1) to (0.08,-0.044);
                \draw[color=black!60] (0.08,-0.044) to (-0.093,-0.035);
                \draw[color=black!80] (-0.093,-0.035) to (0.027,0.039);
                \draw[color=black!100] (0.027,0.039) to (-0.018,0.086);
                \draw[fill=black!20, color=black!20] (0.045, 0.1) circle(0.06);
                \draw[fill=black!40, color=black!40] (0.08, -0.044) circle(0.06);
                \draw[fill=black!60, color=black!60] (-0.093, -0.035) circle(0.06);
                \draw[fill=black!80, color=black!80] (0.027, 0.039) circle(0.06);
                \draw[fill=black!100, color=black!100] (-0.018, 0.086) circle(0.06);
            \end{scope}
            \node at (0, -2.25) {$N_{s,a}(t)=\oh(\log(t))$};
            \node at (0, -2.75) {(Shrinking, \Cref{lemma_shrinking})};
        \end{tikzpicture}
        \hfill
        \begin{tikzpicture}
            \begin{scope}[scale=1]
                \draw[fill=black, opacity=0.15] (-0.392,-0.024) ellipse(2.028 and 1.536);
                \draw[dashed] (-0.392,-0.024) ellipse(2.028 and 1.536);
                \draw[fill=black, opacity=0.15] (0.152,-0.267) ellipse(1.79 and 1.359);
                \draw[fill=black, opacity=0.15] (-0.339,0.378) ellipse(1.631 and 1.248);
                \draw[fill=black, opacity=0.15] (-0.214,0.277) ellipse(1.46 and 1.12);
                \draw[fill=black, opacity=0.15] (0.188,0.083) ellipse(1.301 and 0.958);
                \draw (0.188,0.083) ellipse(1.301 and 0.958);
                \draw[color=black!40] (-0.392,-0.024) to (0.152,-0.267);
                \draw[color=black!60] (0.152,-0.267) to (-0.339,0.378);
                \draw[color=black!80] (-0.339,0.378) to (-0.214,0.277);
                \draw[color=black!100] (-0.214,0.277) to (0.188,0.083);
                \draw[fill=black!20, color=black!20] (-0.392, -0.024) circle(0.06);
                \draw[fill=black!40, color=black!40] (0.152, -0.267) circle(0.06);
                \draw[fill=black!60, color=black!60] (-0.339, 0.378) circle(0.06);
                \draw[fill=black!80, color=black!80] (-0.214, 0.277) circle(0.06);
                \draw[fill=black!100, color=black!100] (0.188, 0.083) circle(0.06);
            \end{scope}
            \node at (0, -2.25) {$N_{s,a}(t)=\Theta(\log(t))$};
            \node at (0, -2.75) {(Shrinking, \Cref{lemma_shrinking})};
        \end{tikzpicture}
        \hfill
        \begin{tikzpicture}
            \begin{scope}[scale=1]
                \draw[fill=black, opacity=0.15] (-0.398,0.035) ellipse(2.01 and 1.548);
                \draw[dashed] (-0.398,0.035) ellipse(2.01 and 1.548);
                \draw[fill=black, opacity=0.15] (-0.274,-0.321) ellipse(2.004 and 1.453);
                \draw[fill=black, opacity=0.15] (0.171,-0.4) ellipse(1.99 and 1.494);
                \draw[fill=black, opacity=0.15] (-0.352,0.258) ellipse(1.96 and 1.48);
                \draw[fill=black, opacity=0.15] (0.228,-0.255) ellipse(1.996 and 1.506);
                \draw (0.228,-0.255) ellipse(1.996 and 1.506);
                \draw[color=black!40] (-0.398,0.035) to (-0.274,-0.321);
                \draw[color=black!60] (-0.274,-0.321) to (0.171,-0.4);
                \draw[color=black!80] (0.171,-0.4) to (-0.352,0.258);
                \draw[color=black!100] (-0.352,0.258) to (0.228,-0.255);
                \draw[fill=black!20, color=black!20] (-0.398, 0.035) circle(0.06);
                \draw[fill=black!40, color=black!40] (-0.274, -0.321) circle(0.06);
                \draw[fill=black!60, color=black!60] (0.171, -0.4) circle(0.06);
                \draw[fill=black!80, color=black!80] (-0.352, 0.258) circle(0.06);
                \draw[fill=black!100, color=black!100] (0.228, -0.255) circle(0.06);
            \end{scope}

            \node at (0, -2.25) {$N_{s,a}(t) = \Omega(\log(t))$};
            \node at (0, -2.75) {(Shaking, \Cref{lemma_shaking})};
        \end{tikzpicture}
        \caption{
            \label{figure_shrinking_shaking}
            An artist view of the shrinking/shaking behavior of the $\mathcal{Q}_{s,a}(t)$ as the number of new samples $N_{s,a}(t') - N_{s,a}(t) \ll N_{s,a}(t)$ increases (from dashed to solid line). 
        }
        \vspace{-2em}
    \end{figure}

    \begin{informalproperty}[{Shrinking-Shaking} effect]
    \label{lemma_informal_shrinking_shaking}
        Let $(t_{k(i)})$ be the enumeration of exploration episodes.
        Fix $T \ge 1$ and denote $\kerrews_\pair(t) := \rewards_\pair(t) \times \kernels_\pair(t)$.
        With high probability and uniformly over $t \in \braces{t_{k(i)}, \ldots, t_{k(i)} + T}$, we have
        \begin{equation*}
            \visits_\pair (t) 
            > 
            \visits_\pair (t_{k(i)})
            + 
            {
                \color{red!80!black}
                \indicator{\pair \notin \optpairs(\model)} 
            }
            C \log(T)
            \implies
            \kerrews_\pair (t) \subseteq \kerrews_\pair (t_{k(i)-1})
            .
        \end{equation*}
    \end{informalproperty}
    
    In \Cref{lemma_informal_shrinking_shaking}, which is informal and slightly wrong, \textcolor{red!80!black}{$\indicator{\pair \notin \optpairs(\model)}$} incarnates the shrinking-shaking dichotomy. 
    The rigorous treatment of this dichotomy is tedious and calculatory, while the idea behind the phenomenon is quite intuitive.
    We postpone the formal, precise and extensive description of the shrinking-shaking effect to \Cref{appendix_regret_of_exploration}.

    Instead, we provide here a heuristic derivation of what the shrinking-shaking phenomenon is and how it is proved. 
    As a matter of fact, the phenomenon already appears in the simple setting of bandits with Gaussian rewards. 
    Let us introduce specialized notations for that purpose. 
    Let $(X_i)$ be a sequence of i.i.d.~random variables of distribution $\mathrm{N}(\mu, \sigma^2)$ and let $\widehat{\mu}(n) := \frac 1n \sum_{k=1}^n X_k$ the empirical average after $n$ samples. 
    The typical way to construct a confidence region for $\mu$ follows from Azuma-Hoeffding's inequality, with
    $
        \mathcal{I}(n) 
        :=
        \braces{
            \widetilde{\mu} \in \R
            :
            \abs*{\widetilde{\mu} - \widehat{\mu}(n)}
            \le
            \sigma \sqrt{{\log(1/\delta)}/{n}}
        }
        .
    $
    The supremum of $\mathcal{I}(n)$ is the largest plausible value for $\mu$, and is given by:
    \begin{equation}
    \label{equation_ucb_index}
        \widetilde{\mu}(n) 
        := 
        \widehat{\mu}(n) + \sigma \sqrt{\frac{\log(1/\delta)}n}
        \equiv
        \widehat{\mu}(n) + \sigma \sqrt{\frac{\log(t)}n}
        .
    \end{equation}
    Regarding our setting, $\widetilde{\mu}(n)$ is the analogue of $\sup \rewards_\pair(t)$, i.e., the highest plausible reward for a given pair at a given time. 
    The quantity $\log(1/\delta)$ can be seen as $\log(t)$ as our confidence regions are tuned for the confidence threshold at $\delta = \frac 1t$.
    The \strong{shrinking-shaking effect} is a general observation about the evolution of $\widetilde{\mu}(n + \dd n)$ after a few samples $\dd n$. 
    It starts with the first order Taylor expansion of $\widetilde{\mu}(n + \dd n)$, giving
    \begin{equation}
    \label{equation_ucb_update}
        \widetilde{\mu}(n + \dd n)
        =
        \widetilde{\mu}(n) 
        + 
        \underbrace{
            \widehat{\mu}(n + \dd n) - \widehat{\mu}(n)
        }_{\textsc{Empirical update}}
        - 
        \underbrace{
            \frac{\sqrt{\log(t)} \dd n}{2 n \sqrt{n}}
        }_{\textsc{Optimism drop}}
        .
    \end{equation}
    The update in the optimistic estimate is the sum of two quantities: the empirical update and the optimism drop. 
    The empirical update is the change of $\widehat{\mu}(n)$ and is roughly noise. 
    Indeed, from the law of large numbers, we have $\widehat{\mu}(n) \approx \mu$ and $\widehat{\mu}(n + \dd n) - \widehat{\mu}(n) \approx \frac 1n \sum_{k=n}^{n + \dd n}(X_k - \mu)$, hence is the sum of $\dd n$ i.i.d.~centered random variables. 
    By Azuma-Hoeffding's inequality, we find that $\widehat{\mu}(n + \dd n) - \widehat{\mu}(n) \approx \frac 1n \sigma \sqrt{\log(1/\delta')} \sqrt{\dd n}$ with probability $1 - \delta'$.

    Taking $n = \OH(\log(t))$ leads to the \strong{shrinking effect}: $\widetilde{\mu}(n + \dd n)$ tends to decrease after a few additional samples $\dd n$.
    Indeed, to have a decrease in the optimistic estimate in \eqref{equation_ucb_update} from $n$ to $n + \dd n$, we need
    \begin{equation}
    \label{equation_ucb_condition}
        \sigma \sqrt{\log\parens*{\frac 1{\delta'}}}
        \sqrt{\dd n}
        \le 
        \frac 12 \sqrt{\frac{\log(t)}n} \dd n
        .
    \end{equation}
    When $n = \OH(\log(t))$, \eqref{equation_ucb_condition} states that as soon as $\dd n = \Omega(\sigma^2 \log(\frac 1{\delta'}))$, the noise due to the empirical updates becomes negligible with respect to the optimism drop.
    Then, the optimistic estimate $\widetilde{\mu}(n + \dd n)$ starts to decrease with quantifiable speed. 
    
    Taking $n = \Omega(t)$, we get the opposite; This is the \strong{shaking effect}.
    More specifically, the optimism drop kills the noise of the empirical update.
    Indeed, by setting $n = t$ in \eqref{equation_ucb_update}, the optimism drop is of order $\frac 1n (\frac{\log(t)}t)^{1/2} \dd n$.
    This quantity is eventually negligible in front of the noise, that is of order $\frac 1n \sqrt{\dd n}$ when $t \to \infty$.
    Hence the shaking effect.

    \paragraph{What it means for \texttt{KLUCRL}}
    When running \texttt{KLUCRL}, we have argued in \Cref{section_macroscopic_coherence} that optimal pairs satisfy $\visits_\pair (t) = \Omega(t)$ while non-optimal pairs satisfy $\visits_\pair (t) = \OH(\log(t))$. 
    When \texttt{KLUCRL} deploys a sub-optimal policy $\policy$, this policy uses non-optimal pairs, for which the confidence regions $\rewards_\pair(t)$ and $\kernels_\pair(t)$ tend to shrink while all the others are negligibly shaking. 
    When iterating $\policy$, these non-optimal pairs are eventually visited enough, $\rewards_\pair(t)$ and $\kernels_\pair(t)$ eventually shrink, so the optimistic gain of $\policy$ decreases.
    Hence, $\policy$ won't be used as an optimistic policy anymore if the episode is updated---and it is updated quickly thanks to \eqref{equation_vanishing_multiplicative}, see \eqref{equation_vanishing_ends}.

    \subsection{A central conceptual property: coherence}
    \label{ssec:coherence}
    \label{section_coherence}

    The point of the shrinking-shaking effect is to establish a \strong{coherence property} defined below. 
    
    \begin{definition}[Coherence]
        \label{definition_coherence}
        We say that an algorithm is \strong{$\boldsymbol{(F, \tau, T, \varphi)}$-coherent} if $F \equiv (F_t : t \ge 1)$ is an adapted sequence of events, $\tau$ a stopping time, $T \ge 1$ is a scalar and $\varphi : \N \to [0, \infty)$ is a function such that, for all $t \in \braces{\tau, \ldots, \tau + T - 1}$, 
        \begin{equation*}
            F_t 
            \subseteq 
            \braces*{
                g^{\pi_t}(S_t) < g^*(\State_t)
                \Rightarrow
                \exists \pair \equiv (\state,\action) \in \Reach(\policy_t, \State_t) 
                % \setminus \pairs^*(\model)
                :
                \brackets*{
                    \hspace{-.3em}
                    \begin{array}{c}
                        N_z(t) - N_z(\tau) \le \varphi(\tau)
                        \\
                        \text{and~}
                        \gain^{\policy_t}(\state) < \optgain(\state)
                    \end{array}
                    \hspace{-.3em}
                }
            }
        \end{equation*}
        where $\pair\equiv(\state,\action) \in \Reach(\policy_t, \State_t)$ stands for $\policy(\action|\state) > 0$ and $\Pr^\policy_{\State_t}(\tau_{\state}<\infty) > 0$.  
    \end{definition}

    Roughly speaking, coherence states that the iteration of a sub-optimal policy is linked to a lack of information (quantified by a budget $\varphi(\tau)$) that has positive probability to be recovered by iterating that policy only. 
    The purpose of the coherence property is its link with local regret guarantees, as shown by \Cref{lemma_coherence} below.
    However, the coherence property may only be conveniently used if the episodes of the algorithm are \strong{weakly regenerative}, meaning that episodes may only end if the current state has already been visited during the episode. 
    This property makes sure that the sub-sampled state-action pair, of which coherence ensures the existence, is reached and visited during the episode with positive probability.

    \begin{definition}[Weakly regenerative episodes]
    \label{definition_regenerative_episodes}
        We say that the episodes of an algorithm are \strong{weakly regenerative} if, for all $k \ge 1$, there exists $t \in \braces{t_k, \ldots, t_{k+1}-1}$ such that $\State_t = \State_{t_{k+1}}$. 
    \end{definition}

    \begin{lemma}[Coherence and local regret]
        \label{lemma_coherence}
        Assume that $\model$ is non-degenerate (\Cref{definition_non_degeneracy}).
        If the algorithm is $(F, \tau, T, \varphi)$-coherent and has weakly regenerative episodes, then there exist model dependent constants $C_1, C_2, C_3, C_4 > 0$ such that: 
        \begin{equation*}
            \forall x \ge 0,
            \quad
            \Pr \parens*{
                \stochasticregret(\tau, \tau + T) 
                \ge 
                x + C_4 \varphi(\tau)
                \text{~and~}
                \bigcap_{t=\tau}^{\tau+T-1} F_t
            }
            \le
            C_1 T^{C_3}
            \exp\parens*{
                - \frac x{C_2}
            }
            .
        \end{equation*}
        More specifically, $C_1, C_2, C_3, C_4$ only depend on $M$ and are independent of $F, \tau, T$ and $\varphi$.
    \end{lemma}

    Using the shorthand $F_{\tau:\tau+T} := \bigcap_{t=\tau}^{\tau+T-1} F_t$, this means that on a good event $F_{\tau: \tau+T}$, the local regret $\stochasticregret(\tau, \tau+T)$ has sub-exponential tails.
    The above result can also be written in the form $\Pr(\stochasticregret(\tau, \tau+T) \ge C_1 + C_4 \varphi(\tau) + (\eta C_2 + C_3) \log(T), F_{\tau:\tau+T}) \le T^{-\eta}$ for all $\eta > 0$.

    The proof of \Cref{lemma_coherence} is difficult and deferred to \Cref{section_lemma_coherence_proof}.

    \subsection{Establishing regret of exploration guarantees via coherence}
    \label{section_establishing_coherence}

    Based on the shrinking and shaking effects discussed upstream, we show that \eqref{equation_vanishing_multiplicative} guarantees \strong{local} coherence properties that, once combined with \Cref{lemma_asymptotic_regime}, become regret of exploration guarantees. 
    The exact coherence property is detailed in \Cref{lemma_local_coherence} below.
    Once \Cref{lemma_local_coherence} is established (see \Cref{appendix_local_coherence}), \Cref{theorem_main} follows instantly. 

    \begin{lemma}
    \label{lemma_local_coherence}
        Let $M \in \models^+$ be a non-degenerate explorative model.
        Consider running \texttt{KLUCRL} with model satisfying \Cref{assumption_interior} and assume that episodes are managed with the $f$-\eqref{equation_vanishing_multiplicative} with $f(t) = \oh\parens{\frac 1{\log(t)}}$.
        Let $(t_{k(i)})$ be the enumeration of exploration episodes. 
        Then, there exists a constant $C(\model) > 0$ such that, for all $T \ge 1$ and $\delta > 0$, there is an adapted sequence of events $(E_t)$ and a function $\varphi : \N \to \R$ such that:
        \begin{enumerate}
            \item 
                For all $i \ge 1$, the algorithm is $(E_t, t_{k(i)}, T, \varphi)$-coherent;
                \vspace{-.66em}
            \item 
                $\Pr\parens[\big]{\bigcup_{t=t_{k(i)}}^{t_{k(i)}+T-1} E_t^c} \le \delta + \oh(1)$ when $i \to \infty$;
                \vspace{-.66em}
            \item 
                $\varphi(t) \le 1 + C \log\parens{\frac T\delta} + \oh(1)$ when $t \to \infty$.
        \end{enumerate}
    \end{lemma}

    \begin{proof}\hspace{-0.2em}\strong{ of \Cref{theorem_main}, assertion 3}~
        Use the coherence property of \Cref{lemma_local_coherence} with $\delta = \frac 1T$ and apply \Cref{lemma_coherence}.
        We obtain:
        \begin{align*}
            & \limsup_{i \to \infty}
            \Pr \parens*{
                \stochasticregret(t_{k(i)}, t_{k(i)}+T)
                \ge 
                x + C_4 \varphi(t_{k(i)})
            }
            \\
            & \le
            \limsup_{i \to \infty}
            \braces*{
                \Pr \parens*{
                    \stochasticregret(t_{k(i)}, t_{k(i)}+T)
                    \ge 
                    x + C_4 \varphi(t_{k(i)})
                    ,
                    \bigcap_{t=t_{k(i)}}^{t_{k(i)}+T-1}
                    E_t
                }
                +
                \Pr \parens*{
                    \bigcup_{t=t_{k(i)}}^{t_{k(i)}+T-1}
                    E_t^c
                }
            }
            \\
            & \le
            \exp\parens*{
                - \frac x{C_2} + C_3 \log(T) + \log(C_1)
            }
            + \frac 1T
        \end{align*}
        which is bounded by $\frac 2T$ for $x \ge C_2 (1+C_3) \log(T) + C_2 \log(C_1)$, where $C_1, C_2, C_3, C_4$ are the constants provided by \Cref{lemma_coherence}. 
        Using that $\limsup_{i \to \infty} \varphi(t_{k(i)}) \le 1 + 2C \log(T)$ and setting $\psi(T) := (C_2 (1 + C_3) + 2 C_4 C)\log(T) + C_2 \log (C_1) + C_4$, we obtain:
        \begin{equation}
            \RegExp(T)
            \le
            \limsup_{i \to \infty} 
            \braces*{
                \psi(T)
                + 
                T 
                \cdot
                \Pr \parens*{
                    \Reg(t_{k(i)}, t_{k(i)}+T)
                    \ge 
                    \psi(T)
                }
            }
            \le
            \psi(T) + 2.
        \end{equation}
        This concludes the proof of \Cref{theorem_main}.
    \end{proof}

    \section{Beyond asymptotic guarantees}

    Our theoretical results (\Cref{theorem_main}) are only asymptotic. However, the behavior of \texttt{KLUCRL} with episodes managed by \eqref{equation_vanishing_multiplicative} is remarkably better than its \eqref{equation_doubling_trick} version over a single run and for reasonably small time horizons, as displayed in \Cref{figure_introduction}.
    In \Cref{appendix_experiments}, we provide additional numerical insights with thorough evidence of the smoother behavior of \eqref{equation_vanishing_multiplicative} over \eqref{equation_doubling_trick} for \texttt{KLUCRL} as well as for other learning algorithms (\texttt{UCRL2},  \texttt{UCRL2B}, ...).

    \section*{Acknowledgements}

    This research was supported in part by the French National Research Agency (ANR) in the framework of the PEPR IA FOUNDRY project (ANR-23-PEIA-0003), and benefitted from the support of the FMJH Program PGMO. 

    \bibliography{bibliography}

    \appendix
    \allowdisplaybreaks

    \crefalias{section}{appendix} % uncomment if you are using cleveref

    \clearpage
    \section{Experiments: The Vanishing Multiplicative condition in practice}
    \label{appendix_experiments}

    In this appendix, we provide a few numerical insights in \texttt{KLUCRL} and the differences between \eqref{equation_doubling_trick} and \eqref{equation_vanishing_multiplicative}. 
    In \Cref{appendix_choice_klucrl}, we justify the choice of \texttt{KLUCRL} both from a theoretical and experimental viewpoint. 
    In \Cref{appendix_experiments_regret}, we show that in practice, the regret of \texttt{KLUCRL} is slightly better with \eqref{equation_vanishing_multiplicative} than with \eqref{equation_doubling_trick} in expectation, in distribution and sometimes in variance.
    In \Cref{appendix_experiments_regret_of_exploration}, we provide an experimental proxy for the regret of exploration, that displays a clear advantage in using \eqref{equation_vanishing_multiplicative} over \eqref{equation_doubling_trick}.

    \subsection{On the choice of \texttt{KLUCRL} as a reference algorithm}
    \label{appendix_choice_klucrl}

    The confidence region $\models(t)$ is built in product form, i.e., $\models(t) \equiv \product_{\pair \in \pairs} (\rewards_\pair (t) \times \kernels_\pair (t))$ where $\rewards_\pair (t) \subseteq [0, 1]$ and $\kernels_\pair (t) \subseteq \probabilities(\states)$. 
    Both are confidence regions for categorical distributions (of dimension $d = 2$ for rewards and $d = \abs{\states}$ for kernels), in which case the constructions of $\rewards_\pair (t)$ and $\kernels_\pair (t)$ are traditionally relying on concentration inequalities.
    These relate how far is the empirical estimate ($\hat{\reward}_t(\pair)$ and $\hat{\kernel}_t(\pair)$) from the true expected value ($\reward(\pair)$ and $\kernel(\pair)$).
    In the literature of model-based optimistic algorithms for Markov decision processes, three main concentration inequalities are being used, taking the form below:
    \begin{align*}
        \text{Weissman's inequality:}
        & \quad
        \visits_\pair (t) \norm{\hat{\kernel}_t(\pair) - \kernel(\pair)}_1^2
        \le
        f(t)
        ;
        \\
        \text{Bernstein's inequality:}
        & \quad 
        \abs{\hat{\kernel}_t(\state|\pair) - \kernel(\state|\pair)}
        \le 
        \sqrt{\frac{2 \hat{\kernel}_t(\state|\pair) (1 - \hat{\kernel}_t(\state|\pair)) f(t)}{\visits_\pair (t)}}
        + \frac{7 f(t)}{3 \visits_\pair(t)}
        ;
        \\
        \text{Empirical likelihoods:}
        & \quad
        \visits_\pair (t) \KL(\hat{\kernel}_t(\pair)||\kernel(\pair)) 
        \le 
        f(t)
        .
    \end{align*}
    Choosing one among these three respectively provides (in order) \texttt{UCRL2} \cite{auer_near_optimal_2009}, \texttt{UCRL2B} \cite{fruit_improved_2020} and \texttt{UCRL3} \cite{bourel_tightening_2020}, and \texttt{KLUCRL} \cite{filippi_optimism_2010}. 

    \paragraph{Best confidence region in theory}
    On the theoretical side, asymptotically tight concentration inequalities are based on empirical likelihoods.
    The reason for that is Sanov's theorem.
    Let $\kernel \in \probabilities(\states)$ be a categorical distribution and $(X_k)$ i.i.d.~random variables of distribution $\kernel$.
    Let $\hat{\kernel}_n := \frac 1n \sum_{k=1}^n e_{X_k}$ denote the empirical estimation of $\kernel$ after $n$ independent samples of it.
    Sanov's theorem states that for all $\mathcal{U} \subseteq \probabilities(\states)$, we have
    $
        \Pr \parens*{
            \hat{\kernel}_n \in \mathcal{U}
        }
        =
        \exp \braces*{
            - n \inf_{\kernel' \in \mathcal{U}} \KL(\kernel'||\kernel)
            + \oh(n)
        }
        .
    $
    This justifies that $\KL(-||-)$ is a natural object to measure where $\hat{\kernel}_n$ concentrates.
    We also find:
    \begin{align*}
        \Pr \parens*{
            \KL(\hat{\kernel}_n||\kernel) > x
        }
        & = \Pr \parens*{
            \hat{\kernel}_n \in \braces{\kernel' : \KL(\kernel'||\kernel) > x}
        }
        \\
        & = \exp \braces*{
            - n \inf_{\kernel' \in \braces{\kernel'' : \KL(\kernel''||\kernel) > x}} \KL(\kernel'||\kernel) 
            + \oh(n)
        }
        \\
        & = \exp \braces*{
            - n x + \oh(n)
        }
        .
    \end{align*}
    Therefore, $\KL$-semi-balls naturally provide tight confidence regions. 
    Such confidence regions are also variance aware (see \cite{talebi_variance_aware_2018}), which is crucial to provide minimax regret bounds that are better than the usual $\OH(\diameter \State\sqrt{\Action T \log(T)})$.

    \paragraph{Best confidence region in practice}
    In practice, \texttt{KLUCRL} is known to perform well. 
    This is displayed in \Cref{figure_comparison_regret_all}, where a selection of algorithms is run. 
    The regret is averaged over $100$ runs and at each run, the environment is picked uniformly at random among ergodic Markov decision processes---the environment is re-rolled every time to mitigate the possible over-specialization of some algorithms for some environments. 
    We compare the performance of \texttt{UCRL2} \cite{auer_near_optimal_2009}, \texttt{UCRL2B} \cite{fruit_improved_2020} and \texttt{KLUCRL} \cite{filippi_optimism_2010} that are all optimistic algorithms relying on \texttt{EVI} (\Cref{appendix_evi}) to compute optimistic policies from their confidence regions.
    This is the framework described in details in \Cref{appendix_minimax}, for which our proof techniques for regret of exploration guarantees (\eqref{equation_vanishing_multiplicative}, \Cref{appendix_coherence,appendix_regret_of_exploration}) are applicable. 
    For fairness, these algorithms are reworked with state-of-the-art confidence regions in $\ell_1$-norm, Bernstein's style and in empirical likelihoods. 

    In \Cref{figure_comparison_regret_all}, we observe that \texttt{KLUCRL} has better expected regret than the other algorithms. 

    \begin{figure}[h]
        \centering
        \includegraphics[width=0.6\linewidth]{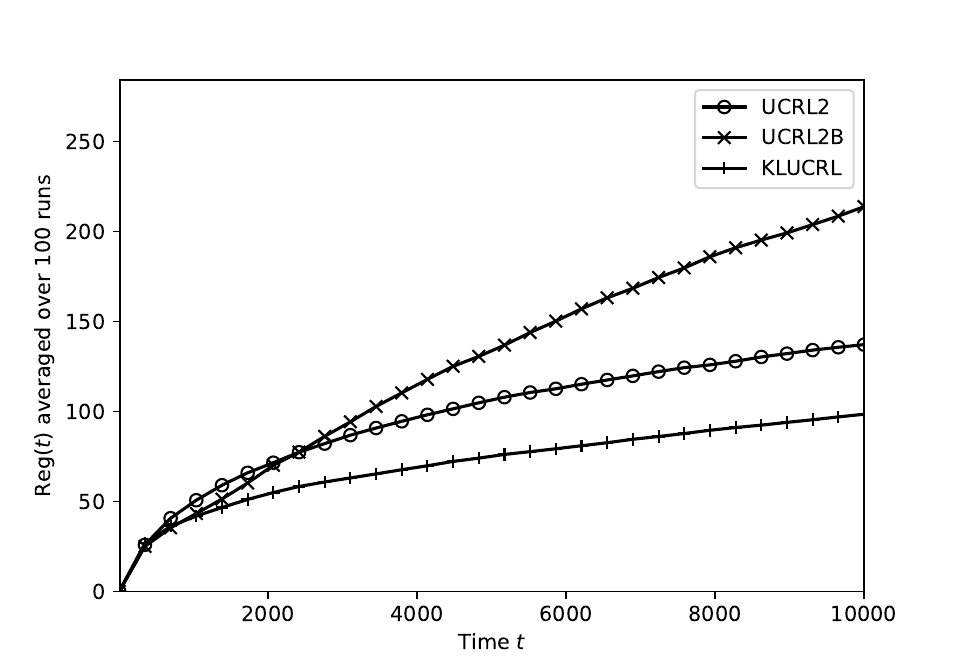}
        \caption{
            \label{figure_comparison_regret_all}
            Bayesian regret of \texttt{UCRL2}, \texttt{UCRL2B} and \texttt{KLUCRL}. 
            Each algorithm is run 100 times on an ergodic environment with $5$ states and $2$ actions, picked at random and renewed for every run. 
        }
    \end{figure}

    \subsection{regret guarantees of \eqref{equation_vanishing_multiplicative} on experiments}
    \label{appendix_experiments_regret}

    The model independent regret of an algorithm is quite difficult to measure experimentally, because it is found as the expected regret on the worst environment $\model \in \models$, that depends on the algorithm and is hard to find. 
    Instead, we focus on the model dependent regret. 

    \begin{figure}[ht]
        \includegraphics[width=0.5\linewidth]{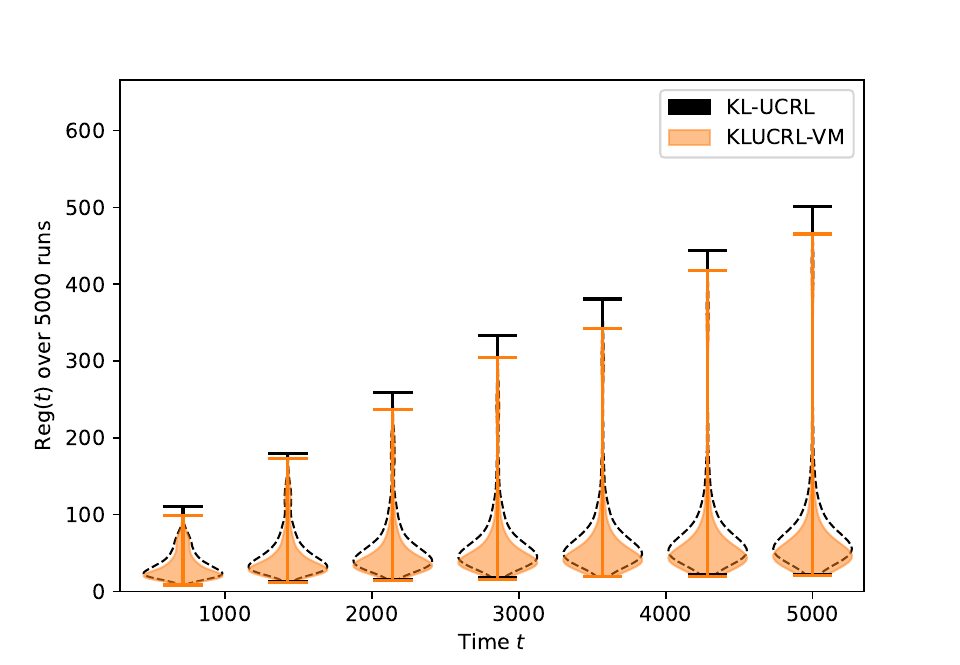}
        \hfill
        \includegraphics[width=0.5\linewidth]{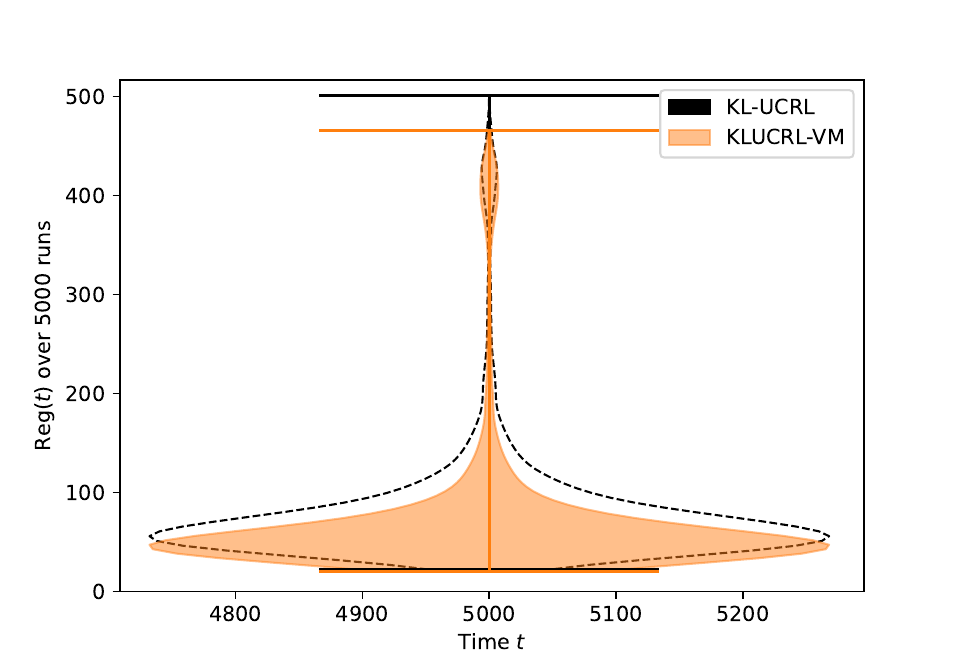}
        \caption{
        \label{figure_violin_klucrl}
            Violin plots of the regret of \texttt{KLUCRL} with episodes managed by \eqref{equation_doubling_trick} (in \strong{black} with dashed lines) and by \eqref{equation_vanishing_multiplicative} (in \textcolor{orange}{\strong{orange}} with solid lines) on a small ergodic environment. 
            By changing \eqref{equation_doubling_trick} to \eqref{equation_vanishing_multiplicative}, we observe a slight improvement of the expected regret with an overall shift of its distribution to smaller values.
            These observations are uniform over the time horizon.
        }
    \end{figure}

    \begin{figure}[ht]
        \centering
        \includegraphics[width=.49\linewidth]{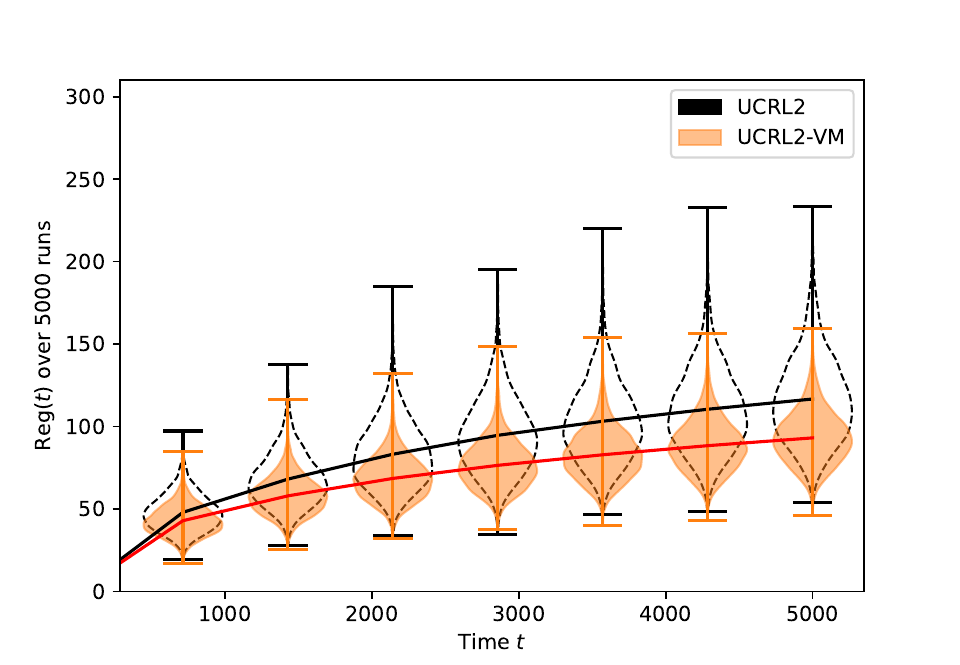}
        \hfill
        \includegraphics[width=.49\linewidth]{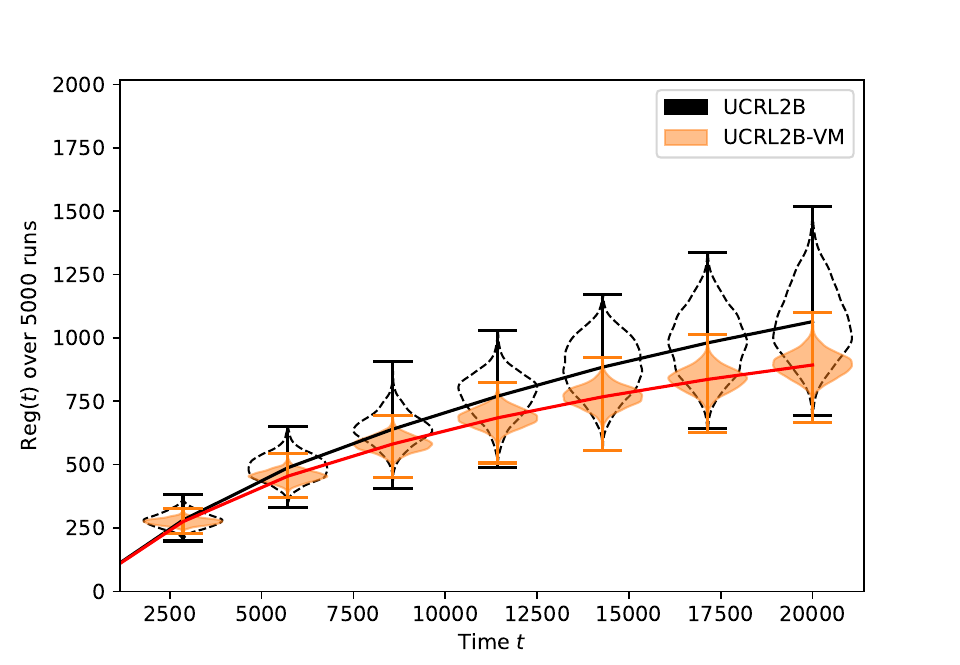}
        \caption{
        \label{figure_violin_others}
            Violin plots of the regret of \texttt{UCRL2} and \texttt{UCRL2B} with episodes managed by \eqref{equation_doubling_trick} (in \strong{black} with dashed lines) and by \eqref{equation_vanishing_multiplicative} (in \textcolor{orange}{\strong{orange}} with solid lines) on a small ergodic environment. 
            For these algorithms, we further observe a reduction of the variance.
        }
    \end{figure}

    In \Cref{figure_violin_klucrl}, we compare the behavior of \texttt{KLUCRL} when managing episodes with \eqref{equation_doubling_trick} and with $f$-\eqref{equation_vanishing_multiplicative} for $f(t) = \sqrt{\log(1+t)/t}$.
    The chosen environment is a small ergodic Markov decision process and a huge number of runs is done to accurately determine the distribution of the regret for both algorithms. 
    We observe that both are bimodal with a concentration around the expected value, with \texttt{KLUCRL-\eqref{equation_vanishing_multiplicative}} being better than \texttt{KLUCRL-\eqref{equation_doubling_trick}} overall, both in expectation and in distribution (stochastic dominance). 

    In \Cref{figure_violin_others}, we run the same experiments as in \Cref{figure_violin_klucrl} but with \texttt{UCRL2} \cite{auer_near_optimal_2009} and \texttt{UCRL2B} \cite{fruit_improved_2020}.
    The same observation than with \texttt{KLUCRL} can be made: the version relying on \eqref{equation_vanishing_multiplicative} stochastically dominates the version relying on \eqref{equation_doubling_trick}.
    A phenomenon that is hard to see for \texttt{KLUCRL} but that is striking for \texttt{UCRL2} and \texttt{UCRL2B} is the reduction of the variance.
    Indeed, we can see that the distributions is much more concentrated around its mean with \eqref{equation_vanishing_multiplicative} than with \eqref{equation_doubling_trick}.

    \begin{figure}[ht]
        \centering
        \includegraphics[width=.49\linewidth]{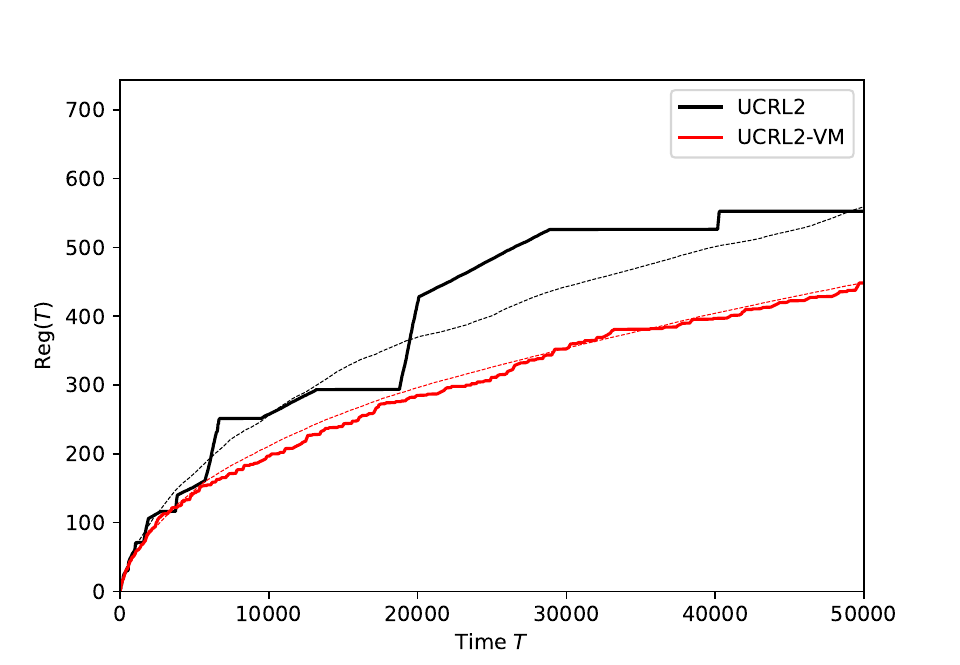}
        \hfill
        \includegraphics[width=.49\linewidth]{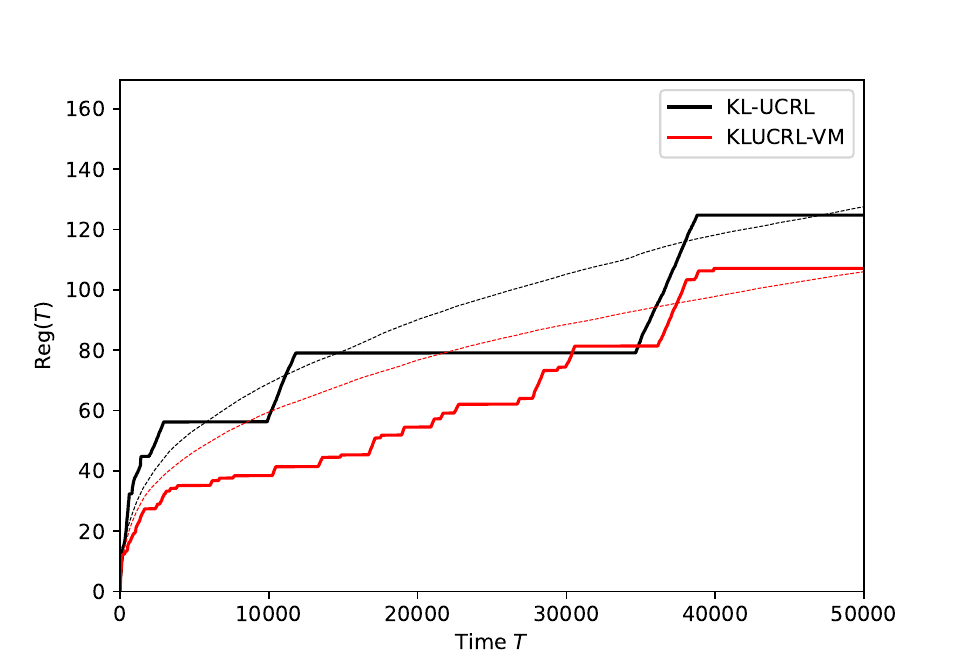}

        \includegraphics[width=.49\linewidth]{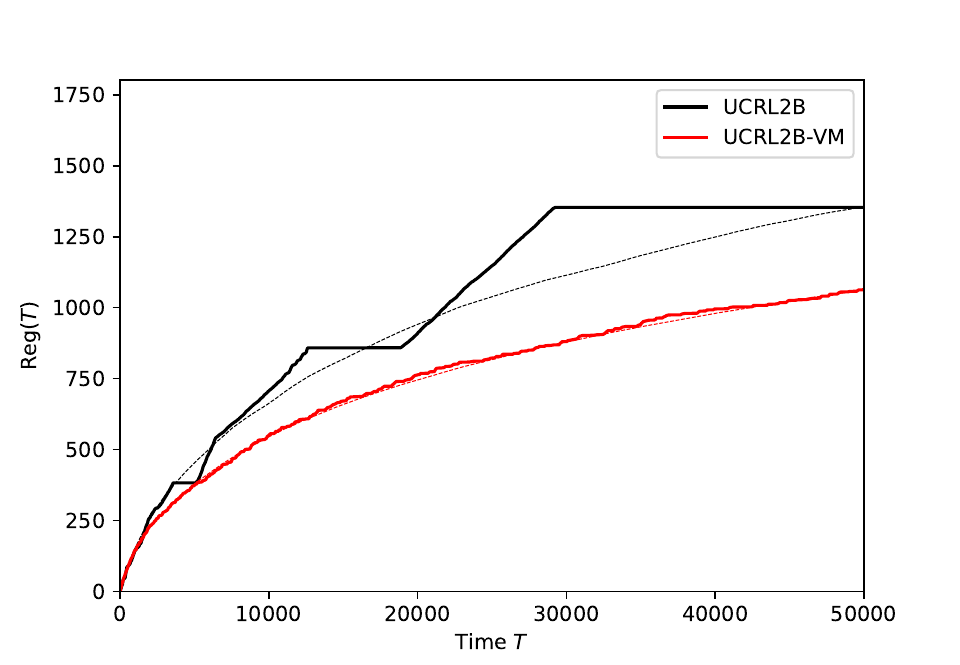}
        \caption{
        \label{figure_single_trajectories}
            Pseudo-regret of a selection of algorithms on a fixed ergodic environment with $5$ states and $2$ actions picked at random.
            The dashed line is the average over $256$ runs, while the solid line displays the pseudo-regret over a single trajectory. 
        }
    \end{figure}

    In \Cref{figure_single_trajectories}, we run \texttt{UCRL2} \cite{auer_near_optimal_2009}, \texttt{UCRL2B} \cite{fruit_improved_2020} and \texttt{KLUCRL} \cite{filippi_optimism_2010} in their vanilla (with episodes managed by \eqref{equation_doubling_trick}) and their reworked versions (with episodes managed with $f$-\eqref{equation_vanishing_multiplicative} for $f(t) = \sqrt{\log(t)/t}$) tagged with \eqref{equation_vanishing_multiplicative} in the legend.  
    The environment is a fixed ergodic Markov decision process with $5$ states and $2$ actions per state, picked at random. 
    We display the regret averaged over $256$ runs with a dashed line, and the solid line is the pseudo-regret over a single trajectory, picked among those that minimize $\stochasticregret(50000; \model) - \Reg(T; \model)$ for readability. 
    The plotted average pseudo-regrets show that using \eqref{equation_vanishing_multiplicative} rather than \eqref{equation_doubling_trick} has a real advantage regarding regret minimization already.
    Looking at the single trajectory curves, we observe that the duration of periods of sub-optimal play is much shorter under $f$-\eqref{equation_vanishing_multiplicative} than under \eqref{equation_doubling_trick}, for all three algorithms.
    Note that not all bad episodes are guaranteed to be small (see for e.g.~the plot of \texttt{KLUCRL}).
    This is consistent with theory: A bound on the regret of exploration guarantees that periods of sub-optimal play are short in average, but does not rule out the existence of long periods of sub-optimal play. 

    \subsection{The regret of exploration under \eqref{equation_vanishing_multiplicative}}
    \label{appendix_experiments_regret_of_exploration}

    As the regret of exploration is a $\limsup$, it is impossible to measure it experimentally. 
    We approximate it in finite time by looking at the quantity:
    \begin{equation}
    \label{equation_proxy_regret_of_exploration}
        T \mapsto 
        \max \braces*{
            \Reg(t_{k(i)}, t_{k(i)} + T)
            :
            t_{k(i)} \in \braces{\psi(T_\text{max}), \ldots, T_\text{max}}
        }
    \end{equation}
    where $T_\text{max} \ge 1$ is the number of learning steps in the experiment and $\psi : \N \to \N$ is a threshold function. 
    The threshold function satisfies $\psi(t) < t$. 
    First, we want $\psi(t) \to \infty$ to remove the burn-in phase of the learning algorithm.
    Second, we want $\psi(t) = \oh(t)$ to make sure that $\braces{\psi(t), \ldots, t}$ contains many episodes of exploration so that the regret of exploration is estimated correctly. 

    \begin{figure}[h]
        \centering
        \includegraphics[width=.6\linewidth]{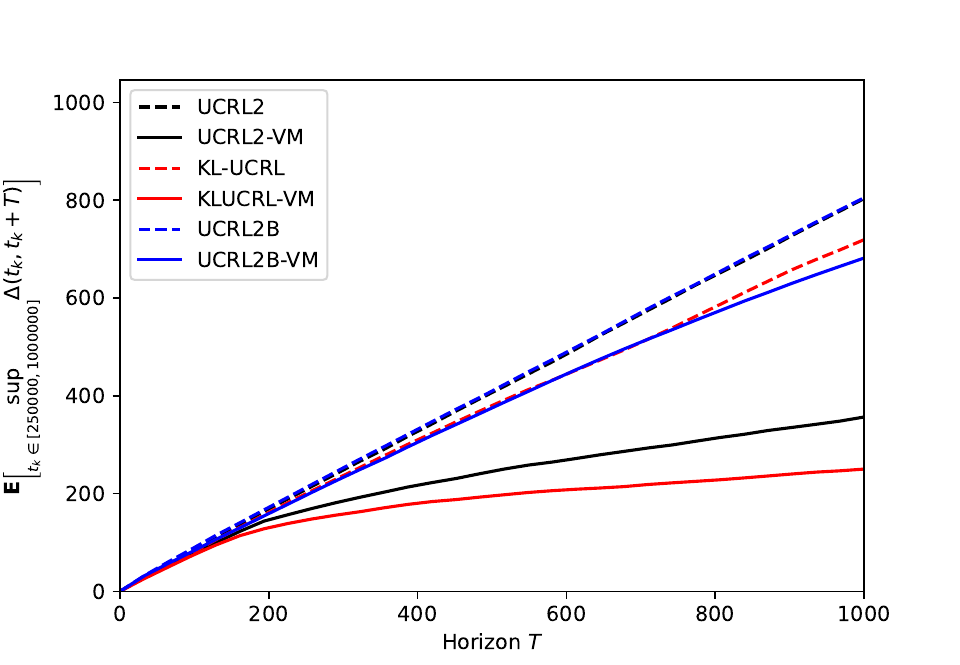}
        \caption{
        \label{figure_proxy_regret_of_exploration}
            Estimation of the regret of exploration of several algorithms, following the proxy \eqref{equation_proxy_regret_of_exploration}.
        }
    \end{figure}
    
    In \Cref{figure_proxy_regret_of_exploration}, we plot a proxy for the regret of exploration in the form of \eqref{equation_proxy_regret_of_exploration}, with $T_\text{max} = 10^6$ and $\psi(T_\text{max}) = 10^5$.
    The environment is a  small River-Swim with $3$ states, known to be a hard-to-learn environment (see \cite[Figure~4]{bourel_tightening_2020}).
    We run \texttt{UCRL2} \cite{auer_near_optimal_2009}, \texttt{UCRL2B} \cite{fruit_improved_2020} and \texttt{KLUCRL} \cite{filippi_optimism_2010} in their vanilla (with episodes managed by \eqref{equation_doubling_trick}) and their reworked versions (with episodes managed with $f$-\eqref{equation_vanishing_multiplicative} for $f(t) = \sqrt{\log(t)/t}$) tagged with \eqref{equation_vanishing_multiplicative} in the legend.  
    
    We observe that the regret of exploration indeeds is a sub-linear function of $T$ for all three algorithms under \eqref{equation_vanishing_multiplicative}, while their \eqref{equation_doubling_trick} versions display a linear regret of exploration. 

    \clearpage
    \section{Minimax regret guarantees under \texorpdfstring{\eqref{equation_vanishing_multiplicative}}{(VM)}}
    \label{appendix_minimax}

    In this appendix, we establish the minimax regret guarantees as given by \Cref{theorem_main}, assertion 1.
    We further provide a large range of general\footnote{With the exception of the pioneer work of \cite{auer_near_optimal_2009}, previous works tend to overlook the well-behavior of \texttt{EVI} from a theoretical perspective. \Cref{appendix_evi} provides a more rigorous treatment.} results on \texttt{EVI} (\Cref{appendix_evi}) and confidence regions (\Cref{appendix_confidence_region}) that will be used in other sections.  
    In \Cref{appendix_regret_episodes}, we provide a regret bound for instances of \texttt{KLUCRL} running with the $f$-\eqref{equation_vanishing_multiplicative} rule for general non-increasing $f : \N \to [0, 1]$.
    This bound is to be combined with the bound on the number of episodes provided in \Cref{appendix_number_episodes} to obtain \Cref{theorem_minimax_regret}.

    \begin{blackblock}
        \begin{theorem}
        \label{theorem_minimax_regret}
            Let $f : \N \to [0, 1]$ and consider running \texttt{KLUCRL} with episodes managed by $f$-\eqref{equation_vanishing_multiplicative} and let $\models_\diameter$ be the set of Markov decision processes with diameter less than $\diameter$. 
            If $f(t) = \Omega(t^{-1/2})$, then:
            \begin{equation*}
                \sup_{\model' \in \models_\diameter} \Reg(T; \model')
                =
                \OH \parens*{
                    \diameter \State \tsqrt{\Action T \log(T)}
                }
            \end{equation*}
        \end{theorem}
    \end{blackblock}

    \paragraph{Notations.}
    The empirical transition kernel and mean reward vector at learning step $t$ are denoted $\hat{\kernel}_t$ and $\hat{\reward}_t$.
    The policy played at time $t$ is $\policy_t$. 
    By design of \texttt{EVI}, the policy $\policy$ that it returns at time $t$ satisfies a Poisson equation (\Cref{corollary_evi_convergence}) of the form $\tilde{\gain}_t + \tilde{\bias}_t = \tilde{\reward}_t + \tilde{\kernel}_t \tilde{\bias_t}$ where $\vecspan{\tilde{\gain}_t} = 0$, $\tilde{\reward}_t(\state, \policy(\state)) \in \rewards_{\state, \policy(\state)}(t)$ and $\tilde{\kernel}_t(\state, \policy(\state)) \in \kernels_{\state, \policy(\state)}(t)$ for all $\state \in \states$; and $\tilde{\bias}_t$ is the bias function of the Markov reward process $(\tilde{\reward}_t, \tilde{\kernel}_t)$. 
    As shown formally in \Cref{appendix_evi} thereafter, $\tilde{\gain}_t = \gain^*(\models(t))$ is the optimal gain of $\models(t)$ and
    \begin{equation}
    \label{equation_bias_diameter}
        \vecspan{\tilde{\bias}_t}
        \le
        \diameter(\models(t))
    \end{equation}
    which is bounded by $\diameter(\model)$ as soon as $\model \in \models(t)$. 
    We further introduce $\episodes(T) := \braces{k \in \N: t_k \le T}$ the set of episodes starting prior to $T \ge 1$. 

    \subsection{Properties of Extended Value Iteration (\texttt{EVI}) and extended MDPs}
    \label{appendix_evi}

    When the confidence region $\models(t)$ is in product form $\models(t) \equiv \product_{\pair \in \pairs} \rewards_\pair(t) \times \kernels_\pair (t)$, such as in our case (see \Cref{section_optimistic_model_based} and \Cref{equation_confidence_region}), it can be seen as a single Markov decision process with compact action space by \emph{extending} actions.
    This extended formulation of $\models(t)$ goes back to \cite{auer_near_optimal_2009} and is what allows to interpret the optimistic gain \eqref{equation_optimistic_gain} as the optimal gain function of $\models(t)$ seen as a Markov decision process.
    Specifically, the extended action space of $\models(t)$ from $\state \in \states$ is:
    \begin{equation*}
        \tilde{\actions}(\state; t) 
        :=
        \product_{\action \in \actions(\state)}
        \parens*{
            \rewards_{\state, \action}(t) \times \kernels_{\state, \action}(t)
        }
    \end{equation*}
    that we may more simply write $\tilde{\actions}(\state)$.
    Accordingly, a choice of action in $\models(t)$ consists in a choice of a vanilla action from $\state$ (i.e., $\action \in \actions(\state)$) as well as a plausible reward and transition kernel for that choice of action. 
    Policies of $\models(t)$ are \strong{extended policies}, and take the form of $(\policy, \reward', \kernel')$ where $\policy$ is a policy of $\model$ and $\reward', \kernel'$ are plausible choices of reward function and transition kernel for that policy. 
    Extended Value Iteration (\texttt{EVI}) consists in iterating the \strong{Bellman operator} \cite[§8.5]{puterman_markov_1994} of $\models(t)$ seen as an extended MDP.
    In the case of $\models(t)$, its Bellman operator is given by $\mathcal{\bellman}(t) \equiv \bellman(\models(t)): \R^\states \to \R^\states$, 
    \begin{equation}
    \label{equation_extended_bellman_operator}
        \parens*{\mathcal{\bellman}(t) u}(\state)
        =
        \max_{\action \in \actions(\state)}
        \max_{\reward \in \rewards_{\state, \action}(t)}
        \max_{\kernel \in \kernels_{\state, \action}(t)}
        \braces[\big]{
            \reward(\state, \action) 
            + 
            \kernel(\state, \action) u
        }
        .
    \end{equation}
    \texttt{EVI} consists in iterating $\mathcal{\bellman}(t)$ until convergence to a near span-fixpoint, i.e., in computing $u_{n+1} = \mathcal{\bellman}(t) u_n$ until $\vecspan{u_{n+1} - u_n} < \epsilon$ where $\epsilon$ is the desired numerical precision. 
    Once the condition ``$\vecspan{u_{n+1} - u_n} < \epsilon$'' is reached, the algorithm returns the policy $\policy : \states \to \actions$ such that $\policy(\state)$ is a choice of action achieving the maximum in \eqref{equation_extended_bellman_operator} for $u = u_n$. 
    This is the algorithm Value Iteration \cite[§8.5]{puterman_markov_1994} applied to $\models(t)$. 

    While \texttt{EVI} performs very well in practice and rarely struggles to converge, there is actually no existing theoretical guarantees regarding its convergence. 
   
    MDPs with compact action spaces are to be treated with care, especially because the existence of solutions to the Bellman equations is not always guaranteed.
    As a consequence to this, the convergence of the iterates of Bellman operators is not guaranteed in general. This issue has been largely overlooked in the reinforcement learning literature and curious readers can take a look at \cite{schweitzer_undiscounted_1985} for that matter.
    In \cite{auer_near_optimal_2009} for \texttt{UCRL2}, the authors do address this issue and argue that the maximum in \eqref{equation_extended_bellman_operator} must be achieved at some vertex of the polytope given by the $\ell_1$-ball spawn by the confidence region.
    Therefore, the maximum is always a maximum over finitely many elements, so $\models(t)$ can be reduced to an extended MDP with \emph{finite} action space.
    This argument can be replicated for confidence regions based on empirical Bernstein inequalities such as for \texttt{UCRL2B} \cite{fruit_improved_2020}, although not explicitly mentioned. 
    It fails completely when $\kernels_\pair (t)$ has smooth boundary, such as for \texttt{KLUCRL} \cite{filippi_optimism_2010} and confidence regions used here.
    Thankfully and in general, MDPs with compact action spaces are much better behaved when they are \strong{communicating} (see \Cref{assumption_communicating}).
    Thankfully again, this is the case of $\models(t)$. 

    \begin{proposition}[\cite{schweitzer_brouwer_1987}]
    \label{proposition_compact_mdps}
        Let $\models$ be a communicating Markov decision process with finite state space $\states$ and compact action space $\actions$.
        Assume that $\reward(\pair) \in [0, 1]$ and that $\action \in \actions(\state) \mapsto \reward(\state, \action)$ and $\action \in \actions(\state) \mapsto \kernel(\state, \state)$ are continuous functions. 
        Then:
        \begin{enumerate}
            \item 
                Its Bellman operator $\bellman : \R^\state \to \R^\states$ given by $(\bellman u)(\state) = \max_{\action \in \actions(\state)} \braces{\reward(\state, \action) + \kernel(\state, \action) u}$ admits a span-fixpoint, i.e., $\exists u \in \R^\states$ such that $\vecspan{\bellman u - u} = 0$;
            \item 
                If $\kernel(\state|\state,\action) > 0$ for all $(\state, \action) \in \pairs$, then the iterates of the Bellman operator converge to a span-fixpoint with linear convergence speed, i.e., there is $\gamma < 1$ such that for all $u \in \R^\states$, $\vecspan{\bellman^{n+1} u - \bellman^n u} = \OH(\gamma^n)$ when $n \to \infty$. 
        \end{enumerate}
    \end{proposition}

    The span-fixpoint to which $\bellman$ converges is denoted $\optbias$, is the optimal bias function of $\models$, and satisfies a Bellman equation $\optgain + \optbias = \max_{\action \in \actions(\state)} \braces{\reward(\state, \action) + \kernel(\state, \action) \optbias}$. 
    The technical condition ``$\kernel(\state|\state, \action) > 0$ for all $(\state, \action) \in \pairs$'' can always be guaranteed under an aperiodicity transform of $\models$, see \cite[§8.5.4]{puterman_markov_1994} and \cite[§4]{bartlett_regal_2009}, that consists in iterating $\frac 12 (\bellman + \mathrm{Id})$ instead of $\bellman$. 
    This aperiodicity transform indeed improves the convergence speed of \texttt{EVI} in practice, without modification of the quality of its output policy. 
    The communicativity assumption is always satisfied, because by design of the confidence region, $\kernels_\pair(t)$ contains fully-supported transition kernel for all $\pair \in \pairs$ and $t \ge 1$.\footnote{If \texttt{KLUCRL} is ran with prior information on the support of $\kernel(\pair)$, then $\kernels_\pair(t)$ always contains elements with the same support than $\kernel(\pair)$---and the communicativity assumption only depends on the support of transition kernels, independently of how small the transition probabilities can be. So $\models(t)$ is communicating when $\model$ is communicating, which is the case in this work.}
    In the end, we can provide generic guarantees for the convergence of \texttt{EVI}.

    \begin{corollary}
    \label{corollary_evi_convergence}
        Let $\model = (\pairs, \reward, \kernel)$ be a communicating Markov decision process. 
        Let $\models(t) \equiv \product_{\pair \in \pairs} \rewards_\pair(t) \times \kernels_\pair(t)$ be a compact confidence region for $\model$.
        Assume that $\rewards_\pair(t) \subseteq [0, 1]$ and that, for all $\pair \in \pairs$, $\kernels_\pair (t)$ contains some $\kernel'(\pair)$ with $\mathrm{supp}(\kernel'(\pair)) \supseteq \mathrm{supp}(\kernel(\pair))$.
        Then:
        \begin{enumerate}
            \item 
                The extended Bellman operator $\mathcal{\bellman}(t)$, see \eqref{equation_extended_bellman_operator}, admits a span-fixpoint and the optimistic gain $\optgain(\models(t))$ of \eqref{equation_optimistic_gain} is the optimal gain of the extended MDP $\models(t)$;

            \item 
                The iterates of $\frac 12(\mathcal{\bellman}(t) + \mathrm{Id})$ converge linearly fast to a span-fixpoint of $\mathcal{\bellman}(t)$, $\optbias(\models(t))$, that satisfies the Bellman equation 
                \begin{equation*}
                    \optgain(\state; \models(t)) + \optbias(\state; \models(t)) 
                    = 
                    \max_{\reward'(\state, \action) \in \rewards_{\state, \action}(t)}
                    \max_{\kernel'(\state, \action) \in \kernels_{\state, \action}(t)}
                    \braces[\big]{
                        \reward'(\state, \action)
                        +
                        \kernel'(\state, \action) \optbias(\models(t))
                    }
                    .
                \end{equation*}
                Therefore, the extended policy $(\policy, \reward', \kernel')$ achieving $\mathcal{\bellman}(t) \optbias(\models(t))$ satisfies the Poisson equation $\optgain(\state; \models(t)) + \optbias(\state; \models(t)) = \reward'(\state) + \kernel'(\state) \optbias(\models(t))$. 
        \end{enumerate}
    \end{corollary}

    A last property that is crucial in the regret analysis of \texttt{EVI}-based algorithms is that their optimal bias $\optbias$ given by \Cref{corollary_evi_convergence} have small span. 
    This result is well-known, see for example \cite[Proposition~3.6]{fruit_exploration_exploitation_2019}.
    We provide a short proof for self-containedness.

    \begin{lemma}
    \label{lemma_bias_bounded_diameter}
        Let $\models$ be a communicating Markov decision process with finite state space $\states$ and compact action space $\actions$. 
        Assume that $\reward(\pair) \in [0, 1]$ and that $\action \in \actions(\state) \mapsto \reward(\state, \action)$ and $\action \in \actions(\state) \mapsto \kernel(\state, \state)$ are continuous functions. 
        Let $\optbias$ be a span-fixpoint of its Bellman operator. 
        Then:
        \begin{equation*}
            \vecspan{\optbias} \le \diameter(\models)
        \end{equation*}
        where $\diameter(\models)$ is the diameter of $\models$, as given by \eqref{equation_diameter}.
    \end{lemma}
    \begin{proof}
        Fix two states $\state, \state' \in \states$ and let $\policy$ such that $\EE_{\state}^{\policy}[\tau_{\state'}] < \infty$ where $\tau_{\state'} := \inf\braces{t > 1: \State_t = \state'}$ is the reaching time to $\state'$. 
        Since $\optgain(\state) + \optbias(\state) \ge \reward(\state, \policy(\state)) + \kernel(\state, \policy(\state)) \optbias$, we have:
        \begin{align*}
            0 
            & \le
            \EE_{\state}^{\policy} \brackets*{
                \sum_{t=1}^{\tau_{\state'} - 1}
                \parens*{
                    \optgain(\State_t) - \reward(\Pair_t)
                    + 
                    \parens*{e_{\State_t} - \kernel(\Pair_t)} \optbias
                }
            }
            \\
            & \overset{(\dagger)}\le
            \vecspan{\reward}
            \EE_{\state}^\policy [\tau_{\state'}]
            + \optbias(\state) - \optbias(\state')
        \end{align*}
        where $(\dagger)$ follows from Doob's optional stopping theorem and that $\vecspan{\optgain - \reward} \le \vecspan{\reward}$. 
        By taking the policy minimizing $\EE_{\state}^\policy[\tau_{\state'}]$, we conclude that $\optbias(\state') - \optbias(\state) \le \diameter(\models)$.
        Because this holds for arbitrary $\state, \state' \in \states$, we conclude that $\vecspan{\optbias} = \max(\optbias) - \min(\optbias) \le \diameter(\models)$. 
    \end{proof}

    \subsection{The confidence region of \texttt{KLUCRL}}
    \label{appendix_confidence_region}

    The confidence region of \texttt{KLUCRL} is designed to hold with high probability (\Cref{lemma_confidence_region}). 

    \begin{lemma}
    \label{lemma_confidence_region}
        The confidence region holds with high probability
        \begin{equation*}
            \Pr \parens*{
                \exists t \ge T:
                \model \notin \models(t)
            }
            \le
            2 \abs{\pairs} T^{-1}
            .
        \end{equation*}
    \end{lemma}
    \begin{proof}
        This result follows by a time-uniform concentration inequality for empirical likelihoods of \cite{jonsson2020planning}, see their Proposition~1.
        \cite[Proposition~1]{jonsson2020planning} states the following:
        Given $d \ge 2$, and $\kernel \in \probabilities[d]$ a probability distribution over $\braces{1, \ldots, d}$, if $\hat{\kernel}_n \in \probabilities[d]$ denotes the empirical average of $n$ i.i.d.~samples of $\kernel$, then for all $\delta \ge 0$, 
        \begin{equation*}
            \Pr \parens*{
                \exists n \ge 1
                :
                n \KL(\hat{\kernel}_n||\kernel)
                > 
                \log\parens*{\frac 1\delta}
                + 
                (d-1) \log \parens*{
                    e \parens*{
                        1 + \frac n{d-1}
                    }
                }
            }
            \le
            \delta
            .
        \end{equation*}
        In our case, we readily obtain that for all $\pair \in \pairs$ and $t \ge 1$,
        \begin{equation}
        \label{equation_proof_confidence_region_1}
            \Pr \parens*{
                \visits_\pair (t)
                \KL(\hat{\kernel}_t(\pair)||\kernel(\pair))
                >
                \log(t)
                + 
                (\abs{\states}-1)
                \log \parens*{
                    e \parens*{
                        1 + \frac{\visits_\pair(t)}{\abs{\states}-1}
                    }
                }
            }
            \le
            \frac 1t
            .
        \end{equation}
        Since $\visits_\pair (t) \le t-1$, we have in particular that for all $\pair \in \pairs$ and $t \ge 1$,
        \begin{equation}
        \label{equation_proof_confidence_region_2}
            \Pr \parens[\Big]{
                \visits_\pair (t)
                \KL(\hat{\kernel}_t(\pair)||\kernel(\pair))
                >
                \abs{\states} \log(2 e t)
            }
            \le
            \frac 1t
        \end{equation}
        where we recognize the definition of $\kernels_\pair (t)$ in \eqref{equation_confidence_region}. 
        Rewards are done similarly, applying \cite[Proposition~1]{jonsson2020planning} for $d = 2$. 
        Conclude by union bound over $\kerrew \in \braces{\reward, \kernel}$ and $\pair \in \pairs$. 
    \end{proof}
    
    Note that there is a significant loss of information when going from \eqref{equation_proof_confidence_region_1} to \eqref{equation_proof_confidence_region_2}, in the sense that \eqref{equation_proof_confidence_region_1} is much more precise than \eqref{equation_proof_confidence_region_2}. 
    It means that we could take a much more precise confidence region than the one used in \eqref{equation_confidence_region}. 
    The confidence region has been simplified to ease the calculations in the proof of the shrinking effect (\Cref{lemma_shrinking}), see \Cref{appendix_shrinking}.

    \subsection{Bounds of classical error terms}
    \label{appendix_bounds_classical_errors}

    The maximal version of Hoeffding's inequality below (\Cref{lemma_hoeffding_maximal}) is a standard result from Hoeffding \cite{hoeffding_probability_1963}.
    It is used in the proofs of \Cref{lemma_expected_reward_error,lemma_expected_kernel_error} in integrated form to bound the error due to optimism. 

    \begin{lemma}[\cite{hoeffding_probability_1963}]
    \label{lemma_hoeffding_maximal}
        Let $(X_k)_{k \ge 1}$ be a sequence of i.i.d.~random variables in $[0, 1]$ and let $\hat{\mu}_n$ their empirical mean after $n$ samples. 
        Let $\mu := \EE[X_1]$ be the true mean.
        Then, for all $x \ge 0$ and $m \ge 1$, we have:
        \begin{equation*}
            \Pr \parens*{
                \max_{n \ge m} \braces*{\hat{\mu}_n - \mu}
                \ge 
                x
            }
            \le 
            \exp \braces*{-2 m x^2}
            .
        \end{equation*}
    \end{lemma}

    \begin{lemma}
    \label{lemma_expected_reward_error}
        The expected cumulative optimistic reward error is bounded as follows:
        \begin{equation*}
            \EE \brackets*{
                \sum_{k \in \episodes(T)}
                \sum_{t = t_k}^{t_{k+1}-1}
                \pospart*{
                    \tilde{\reward}_{t_k}(\Pair_t) - \reward(\Pair_t)
                }
            }
            =
            \OH \parens*{
                \tsqrt{\abs{\pairs} T \log(T)}
            }
            .
        \end{equation*}
    \end{lemma}
    \begin{proof}
        We write:
        \begin{align*}
            & \EE \brackets*{
                \sum_{k \in \episodes(T)}
                \sum_{t = t_k}^{t_{k+1}-1}
                \pospart*{
                    \tilde{\reward}_{t_k}(\Pair_t) - \reward(\Pair_t)
                }
            }
            \\
            & \overset{(\dagger)}\le
            \EE \brackets*{
                \sum_{k \in \episodes(T)}
                \sum_{t = t_k}^{t_{k+1}-1}
                \pospart*{
                    \tilde{\reward}_{t_k}(\Pair_t) - \hat{\reward}_{t_k}(\Pair_t)
                }
            }
            +
            \EE \brackets*{
                \sum_{k \in \episodes(T)}
                \sum_{t = t_k}^{t_{k+1}-1}
                \pospart*{
                    \hat{\reward}_{t_k}(\Pair_t) - \reward(\Pair_t)
                }
            }
            \\
            & \overset{(\ddagger)}\le 
            \EE \brackets*{
                \sum_{k \in \episodes(T)}
                \sum_{t = t_k}^{t_{k+1}-1}
                \sqrt{
                    2 \kl\parens*{
                        \hat{\reward}_{t_k}(\Pair_t)
                        ||
                        \tilde{\reward}_{t_k}(\Pair_t) 
                    }
                }
            }
            +
            \EE \brackets*{
                \sum_{k \in \episodes(T)}
                \sum_{t = t_k}^{t_{k+1}-1}
                \pospart*{
                    \hat{\reward}_{t_k}(\Pair_t) - \reward(\Pair_t)
                }
            }
            ,
        \end{align*}
        where $(\dagger)$ follows by sub-additivity of $\pospart{-}$ and $(\ddagger)$ by Pinsker's inequality. 

        The first term is bounded as follows.
        By construction of the confidence region \eqref{equation_confidence_region}, we have $\visits_{t_k}(\pair_t) \kl(\tilde{\reward}_{t_k}(\Pair_t)||\tilde{\reward}_{t_k}(\Pair_t)) \le \log(T e (1 + \visits_\pair (t_k))) \le 2 \log(T) + 1$.
        We obtain:
        \begin{align*}
            \EE \brackets*{
                \sum_{k \in \episodes(T)}
                \sum_{t = t_k}^{t_{k+1}-1}
                \sqrt{
                    2 \kl\parens*{
                        \hat{\reward}_{t_k}(\Pair_t)
                        ||
                        \tilde{\reward}_{t_k}(\Pair_t) 
                    }
                }
            }
            & \le
            \EE \brackets*{
                \sum_{k \in \episodes(T)}
                \sum_{t = t_k}^{t_{k+1}-1}
                \sqrt{
                    \frac{2 (2\log(T) + 1)}{\visits_{\Pair_t}(t_k)}
                }
            }
            \\
            & \overset{(\dagger)}\le 
            \EE \brackets*{
                \sum_{k \in \episodes(T)}
                \sum_{t = t_k}^{t_{k+1}-1}
                \sqrt{
                    \frac{4 (2\log(T) + 1)}{\visits_{\Pair_t}(t)}
                }
            }
            \\
            & \le 
            2\sqrt{2 \log(T) + 1}
            \cdot 
            \EE \brackets*{
                \sum_{\pair \in \pairs}
                \sum_{n=1}^{\visits_\pair(T)}
                \frac 1{\sqrt{n}}
            }
            \\
            & \le 
            4 \sqrt{2 \log(T) + 1}
            \cdot 
            \EE \brackets*{
                \sum_{\pair \in \pairs}
                \sqrt{\visits_\pair (T)}
            }
            \\
            & \overset{(\ddagger)}\le 
            4 \sqrt{2 \log(T) + 1}
            \cdot 
            \tsqrt{\abs{\pairs} T}
        \end{align*}
        where $(\dagger)$ follows from the observation that, under $f$-\eqref{equation_vanishing_multiplicative}, we have $\visits_\pair (t_k) \le 2 \visits_{\pair}(t)$ and $(\ddagger)$ by Cauchy-Schwartz' inequality. 

        We continue by bounding the second term.
        In the computation below, we denote $\hat{\reward}_{(n)}(\pair)$ the empirical reward at $\pair \in \pairs$ after exactly $n$ samples of it. 
        In particular, note that $\reward_{(\visits_{\pair}(t))}(\pair) = \reward_t(\pair)$. 
        We have:
        \begin{align*}
            & 
            \EE \brackets*{
                \sum_{k \in \episodes(T)}
                \sum_{t = t_k}^{t_{k+1}-1}
                \pospart*{
                    \hat{\reward}_{t_k}(\Pair_t) - \reward(\Pair_t)
                }
            }
            \\
            & =
            \EE \brackets*{
                \sum_{\pair \in \pairs}
                \sum_{k \in \episodes(T)}
                \sum_{t=t_k}^{t_{k+1}-1}
                \indicator{\Pair_t = \pair}
                \pospart*{
                    \hat{\reward}_{(\visits_{\pair}(t_k))}(\pair) - \reward(\pair)
                }
            }
            \\
            & \overset{(\dagger)}\le 
            \EE \brackets*{
                \sum_{\pair \in \pairs}
                \sum_{k \in \episodes(T)}
                \sum_{t=t_k}^{t_{k+1}-1}
                \indicator{\Pair_t = \pair}
                \max_{n \ge \floor{\frac 12 \visits_\pair (t)}}
                \pospart*{
                    \hat{\reward}_{(n)}(\pair) - \reward(\pair)
                }
            }
            \\
            & =
            \EE \brackets*{
                \sum_{\pair \in \pairs}
                \sum_{m=1}^{\visits_\pair(T)}
                \max_{n \ge \floor{\frac 12 m}}
                \pospart*{
                    \hat{\reward}_{(n)}(\pair) - \reward(\pair)
                }
            }
            \\
            & \le
            2 \cdot \EE \brackets*{
                \sum_{\pair \in \pairs}
                \sum_{m=1}^{\visits_\pair(T)}
                \max_{n \ge m}
                \pospart*{
                    \hat{\reward}_{(n)}(\pair) - \reward(\pair)
                }
            }
            \\
            & \overset{(\ddagger)}=
            2 \cdot \EE \brackets*{
                \sum_{\pair \in \pairs}
                \sum_{m=1}^{\visits_\pair(T)}
                \integral_0^\infty
                \Pr \parens*{
                    \max_{n \ge m}
                    \pospart*{
                        \hat{\reward}_{(n)}(\pair) - \reward(\pair)
                    }
                    \ge
                    x
                } \mathrm{d} x
            }
            \\
            & \overset{(\S)}\le
            2 \cdot \EE \brackets*{
                \sum_{\pair \in \pairs}
                \sum_{m=1}^{\visits_\pair(T)}
                \integral_0^\infty
                \exp \braces*{
                    -2m x^2
                }
                \mathrm{d} x
            }
            \\
            & =
            \EE \brackets*{
                \sum_{\pair \in \pairs}
                \sum_{m=1}^{\visits_\pair(T)}
                \sqrt{\frac{\pi}{2m}}
            }
            \le 
            \sqrt{2 \pi}
            \EE \brackets*{
                \sum_{\pair \in \pairs}
                \tsqrt{\visits_\pair (T)}
            }
            \overset{(\$)}\le
            \tsqrt{2 \pi \abs{\pairs} T}
        \end{align*}
        where
        $(\dagger)$ follows from the observation that $\visits_{t}(\pair) \le 2 \visits_{t_k}(\pair)$ for $t \in \braces{t_k, \ldots, t_{k+1}-1}$;
        $(\ddagger)$ follows from Doob's optional stopping theorem;
        $(\S)$ follows from \Cref{lemma_hoeffding_maximal} and 
        $(\$)$ is obtained with Cauchy-Schwartz' inequality. 
    \end{proof}

    We obtain a similar result for transition kernels, by changing $\pospart{-}$ to $\norm{-}_1$, invoking the time-uniform concentration result of \cite[Proposition~1]{jonsson2020planning} of empirical likelihoods in dimension $d = \abs{\states}$ instead of $d = 2$. \Cref{lemma_hoeffding_maximal} has to be modified to take into account these modifications, see \Cref{lemma_weissman_maximal} below, which is a maximal version of Weissman's inequality \cite{weissman_inequalities_2003}. 

    \begin{lemma}
    \label{lemma_weissman_maximal}
        Let $\kernel \in \probabilities[d]$ for $d \ge 2$ and let $(X_k)$ be a sequence of i.i.d.~samples of $\kernel$.
        Denote $\hat{\kernel}_n := \frac 1n (e_{X_1} + \ldots + e_{X_k})$ the empirical distribution after $n$ samples. 
        Then, for all $x \ge 0$ and $m \ge 1$, we have:
        \begin{equation*}
            \Pr \parens*{
                \max_{n \ge m}
                \norm{\hat{\kernel}_n - \kernel}_1
                \ge
                x
            }
            \le
            \exp \braces*{-2 m x^2 + \abs{\states} \log(2)}
            .
        \end{equation*}
    \end{lemma}
    \begin{proof}
        Note that $\norm{\hat{\kernel}_n - \kernel}_1 = \max_{u \in \braces{-1,1}^\states} (\hat{\kernel}_n - \kernel) \cdot u$.
        So, introducing the notations $X_k^u := e_{X_k} \cdot u$ for $u \in \braces{-1, 1}^\states$ together with $\hat{\mu}_n^u := \frac 1n(X_1^u + \ldots + X_n^u)$ and $\mu^u = \EE[X_1^u]$, we have $\norm{\hat{\kernel}_n - \kernel}_1 = \max_{u \in \braces{-1,1}^\states} (\hat{\mu}_n^u - \mu^u)$.
        So,
        \begin{align*}
            \Pr \parens*{
                \max_{n \ge m}
                \norm{\hat{\kernel}_n - \kernel}_1
                \ge
                x
            }
            & =
            \Pr \parens*{
                \exists u \in \braces{-1,1}^\states:
                \max_{n \ge m}
                \braces*{\hat{\mu}_n^u - \mu^u}
                \ge 
                x
            }
            \\
            & \le 
            \sum_{u \in \braces{-1,1}^\states}
            \Pr \parens*{
                \max_{n \ge m}
                \braces*{\hat{\mu}_n^u - \mu^u}
                \ge 
                x
            }
            \overset{(\dagger)}\le
            2^{\abs{\states}} \exp\braces*{-2m x^2}
        \end{align*}
        where $(\dagger)$ follows from \Cref{lemma_hoeffding_maximal}.
    \end{proof}

    \begin{lemma}
    \label{lemma_expected_kernel_error}
        The expected cumulative optimistic error on kernels is bounded as follows:
        \begin{equation*}
            \EE \brackets*{
                \sum_{k \in \episodes(T)}
                \sum_{t = t_k}^{t_{k+1}-1}
                \norm*{
                    \tilde{\kernel}_{t_k}(\Pair_t) - \kernel(\Pair_t)
                }_1
            }
            =
            \OH \parens*{
                \tsqrt{\abs{\states} \abs{\pairs} T \log(T)}
            }
            .
        \end{equation*}
    \end{lemma}
    \begin{proof}
        Same proof as \Cref{lemma_expected_reward_error}, changing $\pospart{-}$ for $\norm{-}_1$, taking care of the extra $\abs{\states}$ in the confidence region for kernels, and invoking \Cref{lemma_weissman_maximal} instead of \Cref{lemma_hoeffding_maximal}.
    \end{proof}

    \subsection{Bounding the regret relatively to the number of episodes}
    \label{appendix_regret_episodes}

    In this section, we prove in \Cref{lemma_minimax_regret_episodes} that the regret under \eqref{equation_vanishing_multiplicative} can be decoupled as the classical term $S \diameter \sqrt{A T \log(T)}$ and another which is proportional to the number of episodes. 
    The regret analysis is classical and inspired from \cite{auer_near_optimal_2009} for \texttt{UCRL2}, excepted that the analysis is written in expectation rather than in probability. 
    The analysis could be adapted to obtain a result in probability as well. 

    \begin{lemma}
    \label{lemma_minimax_regret_episodes}
        Consider running \texttt{KLUCRL} while managing episodes with the $f$-\eqref{equation_vanishing_multiplicative} rule for some non-increasing $f : \N \to [0, 1]$.
        For all $\model$ with diameter less than $\diameter > 0$, we have:
        \begin{equation*}
            \Reg(T; \model) 
            =
            \OH \parens*{
                \diameter \tsqrt{\abs{\states}\abs{\pairs} T \log(T)}
                + \diameter \EE^{\model}\abs[\big]{\episodes(T)}
            }
        \end{equation*}
    \end{lemma}
    \begin{proof}
        We have $\Reg(T; \model) = \EE[\sum_{t=1}^T \ogaps(\Pair_t; \model)]$, so
        \begin{align*}
            & \Reg(T; \model)
            \\
            & \le
            \EE \brackets*{
                \sum_{t=1}^T 
                (\optgain(\model) - \Reward_t)
            } 
            + \vecspan{\optbias(\model)}
            \\
            & \overset{(\dagger)}=
            \EE \brackets*{
                \sum_{k \in \episodes(T)}
                \sum_{t = t_k}^{t_{k+1} - 1}
                (\optgain(\model) - \reward(\Pair_t))
            }
            + \vecspan{\optbias(\model)}
            \\
            & \le 
            \tsqrt{\abs{\pairs} T}
            +
            \EE \brackets*{
                \sum_{k \in K(T)}
                \eqindicator{
                    t_k \ge \tsqrt{\abs{\pairs} T}
                    ,
                    \model \in \models(t_k)
                }
                \sum_{t = t_k}^{t_{k+1} - 1}
                (\optgain(\model) - \reward(\Pair_t))
            }
            \\
            & \phantom{{} \le \tsqrt{\abs{\pairs}T}}
            +
            T \cdot \Pr \parens*{
                \exists t \ge \tsqrt{\abs{\pairs}T} 
                :
                \model \notin \models(t)
            }
            + \vecspan{\optbias(\model)}
            \\
            & \overset{(\ddagger)}=
            \underbrace{
                \EE \brackets*{
                    \sum_{k \in K(T)}
                    \eqindicator{
                        t_k \ge \tsqrt{\abs{\pairs}T}
                        ,
                        \model \in \models(t_k)
                    }
                    \sum_{t = t_k}^{t_{k+1} - 1}
                    (\optgain(\model) - \reward(\Pair_t))
                }
            }_{\mathrm{A}}
            + \OH\parens*{\tsqrt{\abs{\pairs}T}}
        \end{align*}
        where 
        $(\dagger)$ follows from the tower property and
        $(\ddagger)$ follows from \Cref{lemma_confidence_region}, stating that $\Pr(\exists t \ge T: \model \notin \models(T)) \le 2 \abs{\pairs} T^{-1}$. 
        We focus on the first expectation.
        Further introduce the good event $\event_t := \braces{t \ge \sqrt{\abs{\pairs} T}, \model \in \models(t)}$.
        It is $\sigma(\History_t)$-measurable.
        At time $t_k$ and under $\event_{t_k}$, we have $\tilde{\gain}_{t_k} = \optgain(\tilde{\models}(t_k)) \ge \optgain(\model)$ and $\vecspan{\tilde{\bias}_{t_k}} \le \diameter(\model)$ (see \Cref{appendix_evi}). 
        So:
        \begin{align*}
            \mathrm{A} 
            & :=
            \EE \brackets*{
                \sum_{k \in K(T)}
                \indicator{\event_{t_k}}
                \sum_{t = t_k}^{t_{k+1} - 1}
                (\optgain(\model) - \reward(\Pair_t))
            }
            \\
            & \le 
            \EE \brackets*{
                \sum_{k \in K(T)}
                \indicator{\event_{t_k}}
                \sum_{t = t_k}^{t_{k+1} - 1}
                \parens*{
                    \tilde{\gain}_{t_k} - \reward(\Pair_t)
                }
            }
            \\
            & = 
            \EE \brackets*{
                \sum_{k \in K(T)}
                \indicator{\event_{t_k}}
                \parens*{
                    \sum_{t = t_k}^{t_{k+1} - 1}
                    \parens*{
                        \tilde{\gain}_{t_k} - \tilde{\reward}_{t_k}(\Pair_t)
                    }
                    +
                    \sum_{t = t_k}^{t_{k+1} - 1}
                    \parens*{
                        \tilde{\reward}_{t_k}(\Pair_t)
                        -
                        \reward(\Pair_t)
                    }
                }
            }
            \\
            & \le
            \EE \brackets*{
                \sum_{k \in K(T)}
                \indicator{\event_{t_k}}
                \sum_{t = t_k}^{t_{k+1} - 1}
                \parens*{
                    \tilde{\gain}_{t_k} - \tilde{\reward}_{t_k}(\Pair_t)
                }
            }
            +
            \EE \brackets*{
                \sum_{k \in K(T)}
                \sum_{t = t_k}^{t_{k+1} - 1}
                \pospart*{
                    \tilde{\reward}_{t_k}(\Pair_t)
                    -
                    \reward(\Pair_t)
                }
            }
            \\
            & \overset{(\dagger)}=
            \EE \brackets*{
                \sum_{k \in K(T)}
                \indicator{\event_{t_k}}
                \sum_{t = t_k}^{t_{k+1} - 1}
                \parens*{
                    \tilde{\gain}_{t_k} - \tilde{\reward}_{t_k}(\Pair_t)
                }
            }
            +
            \OH \parens*{
                \tsqrt{\abs{\pairs} T \log(T)}
            }
        \end{align*}
        where $(\dagger)$ follows from \Cref{lemma_expected_reward_error}.
        We proceed as follows:
        \begin{align*}
            &
            \EE \brackets*{
                \sum_{k \in K(T)}
                \indicator{\event_{t_k}}
                \sum_{t = t_k}^{t_{k+1} - 1}
                \parens*{
                    \tilde{\gain}_{t_k} - \tilde{\reward}_{t_k}(\Pair_t)
                }
            }
            \\
            & \overset{(\dagger)}=
            \EE \brackets*{
                \sum_{k \in K(T)}
                \indicator{\event_{t_k}}
                \sum_{t = t_k}^{t_{k+1} - 1}
                \parens*{
                    e_{\State_t} - \tilde{\kernel}_{t_k}(\Pair_t)
                } \tilde{\bias}_{t_k}
            }
            \\
            & =
            \EE \brackets*{
                \sum_{k \in K(T)}
                \indicator{\event_{t_k}}
                \parens*{
                    \sum_{t = t_k}^{t_{k+1} - 1}
                    \parens*{
                        e_{\State_t} - \kernel(\Pair_t)
                    } \tilde{\bias}_{t_k}
                    +
                    \sum_{t = t_k}^{t_{k+1} - 1}
                    \parens*{
                        \kernel(\Pair_t) - \tilde{\kernel}_{t_k}(\Pair_t)
                    } \tilde{\bias}_{t_k}
                }
            }
            \\
            & \overset{(\ddagger)}\le
            \EE \brackets*{
                \sum_{k \in K(T)}
                \indicator{\event_{t_k}}
                \parens*{
                    \tilde{\bias}_{t_k}(\State_{t_k})
                    -
                    \tilde{\bias}_{t_k}(\State_{t_{k+1}})
                }
            }
            +
            \frac{\diameter}2
            \EE \brackets*{
                \sum_{t = t_k}^{t_{k+1} - 1}
                \sum_{k \in K(T)}
                \norm*{
                    \tilde{\kernel}_{t_k}(\Pair_t)
                    -
                    \kernel(\Pair_t)
                }_1
            }
            \\
            & \overset{(\S)}\le 
            \EE \brackets*{
                \sum_{k \in K(T)}
                \diameter
            }
            +
            \OH \parens*{
                \diameter \tsqrt{
                    \abs{\states} \abs{\pairs} T \log(T)
                }
            }
            =
            \OH \parens*{
                \diameter \EE\abs{\episodes(T)}
                +
                \diameter \tsqrt{
                    \abs{\states} \abs{\pairs} T \log(T)
                }
            }
        \end{align*}
        where
        $(\dagger)$ follows from the Poisson equation $\tilde{\gain}_{t_k} - \tilde{\reward}_{t_k}(\state, \action) = \tilde{\bias}_{t_k}(\state) - \tilde{\kernel}_{t_k}(\state, \action) \tilde{\bias}_{t_k}$;
        $(\ddagger)$ is obtained using the telescopic nature of the first term and by using that $(\kernel(\pair) - \tilde{\kernel}_{t_k}(\pair)) \tilde{\bias}_{t_k} \le \frac 12 \vecspan{\tilde{\bias}_{t_k}} \norm{\tilde{\kernel}_{t_k}(\pair) - \kernel(\pair)}_1$ to bound the second, and further using that $\vecspan{\tilde{\bias}_{t_k}} \le \diameter(\models(t_k)) \le \diameter(\model)$ on $\event_{t_k}$ by \Cref{lemma_bias_bounded_diameter}; and
        $(\S)$ follows by bounding the first term using that $\indicator{\event_{t_k}}\vecspan{\tilde{\bias}_{t_k}} \le \diameter(\model)$ (\Cref{lemma_bias_bounded_diameter}) and by bounding the second using \Cref{lemma_expected_kernel_error}.
        We conclude accordingly. 
    \end{proof}

    \subsection{Bounding the number episodes under \texorpdfstring{$f$-\eqref{equation_vanishing_multiplicative}}{f-(VM)}}
    \label{appendix_number_episodes}

    The episodes under $f$-\eqref{equation_vanishing_multiplicative} are bounded in a similar than for \eqref{equation_doubling_trick}. 
    The technique that we provide below provides a result that ends up being asymptotically better than \cite[Proposition~18]{auer_near_optimal_2009} for the doubling trick, $\abs{\episodes(T)} \le \abs{\pairs} \log_2(\frac{8T}{\abs{\pairs}})$.

    \begin{lemma}
    \label{lemma_number_episodes}
        Assume that episodes are managed with $f$-\eqref{equation_vanishing_multiplicative} where $f : \N \to (0, 1]$ is non-increasing.
        Whatever $\model \in \models$, we have
        \begin{equation*}
            \abs*{\episodes(T)} 
            \le
            \frac{
                \abs{\pairs} \log\parens*{
                    \frac{2T + \OH(1)}{\abs{\pairs}}
                }
            }{
                \log(1 + f(T))
            }
            \quad
            \mathrm{a.s.}
        \end{equation*}
    \end{lemma}
    \begin{proof}
        Given $\pair \in \pairs$, let $\episodes_\pair (T) := \braces[\big]{k : t_k \le T \mathrm{~and~} \visits_\pair (t_{k+1}) > (1 + f(t_k)) \max\braces{1, \visits_\pair (t_k)}}$ the set of episodes that are ended by visiting $\pair \in \pairs$. 
        Remark that:
        \begin{equation*}
            \visits_\pair(t_{k+1})
            \ge
            \prod_{\ell \in \episodes_\pair(t_k)}
            \parens*{1 + f(t_\ell)}
            \ge
            (1 + f(T))^{\abs{\episodes_\pair(t_k)}}
            .
        \end{equation*}
        We have $t_{k+1} \le 2 t_k + \OH(1)$ when $k \to \infty$, hence summing the above over $\pair \in \pairs$, we obtain:
        \begin{equation*}
            2T
            \ge
            \sum_{\pair \in \pairs} (1 + f(T))^{\abs{\episodes_\pair(T)}}
            \ge 
            \inf_{\omega \in \probabilities(\pairs)}
            \braces*{
                \sum_{\pair \in \pairs} (1 + f(T))^{\omega_\pair \abs{\episodes(T)}}
            }
        \end{equation*}
        where the second inequality follows from the observation that the union $\episodes(T) = \bigcup_{\pair \in \pairs} \episodes_\pair (T)$ is disjoint. 
        The RHS of the above is the infemum of a convex function $\psi(\omega)$. 
        The KKT conditions show that this infemum is reached when $\omega_\pair = \abs{\pairs}^{-1}$ for all $\pair \in \pairs$. 
        Plugging these values in the above and solving in $\abs{\episodes(T)}$, we obtain the desired result.
    \end{proof}

    \clearpage
    \section{The coherence lemma: Proof of \Cref{lemma_coherence}}
    \label{appendix_coherence}
    \label{section_lemma_coherence_proof}

    In this appendix, we provide a proof of the coherence lemma (\Cref{lemma_coherence}).
    Stated in its general form, this lemma can be instantiated in various forms to obtain a large variety of results. 
    It is used \emph{twice} to provide the regret of exploration guarantees of \texttt{KLUCRL}, first in a macroscopic (or global) way to provide the asymptotic regime (see \Cref{section_macroscopic_coherence} and \Cref{lemma_asymptotic_regime}) and lastly in a microscopic (or local) way to finally provide regret of exploration guarantees (see \Cref{section_establishing_coherence} and \Cref{lemma_local_coherence}).
    It is also used to obtain instance dependent regret guarantees (\Cref{appendix_model_dependent} and \Cref{theorem_model_dependent_regret}), showing that every instance \texttt{KLUCRL} managing episodes with $f$-\eqref{equation_vanishing_multiplicative} is consistent on the sub-space of non-degenerate Markov decision processes. 

    We recall the statement of \Cref{lemma_coherence} below.

    \par
    \bigskip
    \noindent
    \textbf{\Cref{lemma_coherence} (Coherence and local regret)}
    \textit{
        Assume that $\model$ is non-degenerate (\Cref{definition_non_degeneracy}).
        If the algorithm is $(F, \tau, T, \varphi)$-coherent and has weakly regenerative episodes, then there exist model dependent constants $C_1, C_2, C_3, C_4 > 0$ such that: 
        \begin{equation*}
            \forall x \ge 0,
            \quad
            \Pr \parens*{
                \stochasticregret(\tau, \tau + T) 
                \ge 
                x + C_4 \varphi(\tau)
                \text{~and~}
                \bigcap_{t=\tau}^{\tau+T-1} F_t
            }
            \le
            C_1 T^{C_3}
            \exp\parens*{
                - \frac x{C_2}
            }
            .
        \end{equation*}
        More specifically, $C_1, C_2, C_3, C_4$ only depend on $M$ and are independent of $F, \tau, T$ and $\varphi$.
    }

    \paragraph{Outline of the proof}
    The whole appendix is dedicated to a proof of \Cref{lemma_coherence}.
    In \Cref{appendix_coherence_partioning}, the time-range $\braces{\tau, \ldots, \tau + T - 1}$ is partioned into segments $\braces{\tau_i, \ldots, \tau_{i+1}-1}$ alternating between periods of sub-optimal and optimal play. 
    We start by bounding the regret due to sub-optimal segments in \Cref{appendix_coherence_suboptimal_segments}.
    In \STEP{1}, we relate the total duration and the number of sub-optimal segments to the potential $\varphi(\tau)$, to show in \STEP{2} that the number of sub-optimal time-segments has sub-exponential tails under the good event $\bigcap_{t=\tau}^{\tau+T-1} F_t$.
    This leads to sub-exponential tails for the total duration of sub-optimal segments in \STEP{3} under the same good event.
    It provides an immediate regret bound for the regret induced by sub-optimal periods of play in \STEP{4}.
    In \Cref{appendix_coherence_optimal_segments}, we move to the bound of the regret on optimal segments, where the algorithm plays gain optimal policies. 
    However, even if the algorithm plays a gain optimal policy on $\braces{\tau_i, \ldots, \tau_{i+1}-1}$, it may play a few sub-optimal actions before the recurrent class of that policy is reached: This is the well-known ``cost'' induced by switching policies. 
    Therefore, we motivate in \Cref{appendix_coherence_optimal_segments} that we need to bound the time that the algorithm takes to reach the optimal class on \emph{all} optimal segments.
    This is related to the number of sub-optimal segments in \STEP{1}, and as optimal and sub-optimal segments alternate by construction, the work done in \Cref{appendix_coherence_suboptimal_segments} provides a bound on that number.
    This leads to a sub-exponential tails for the regret induced by optimal segments in \STEP{3}.
    Everything is combined in \Cref{appendix_coherence_combining} to conclude the proof of \Cref{lemma_coherence}.

    \paragraph{Notations}
    Given a policy $\policy \in \policies$ and a state $\state \in \states$, we write $\Reach(\state, \policy)$ the set of reachable pairs under $\policy$ from $\state$, i.e., the set of $\pair \in \pairs$ such that $\Pr_{\state}^{\policy}\parens*{\exists t \ge 1: \Pair_t = \pair} > 0$.

    \subsection{Partioning of \texorpdfstring{$\braces{\tau, \ldots, \tau+T-1}$}{\{tau, ... , tau + T-1\}} into optimal and sub-optimal segments}
    \label{appendix_coherence_partioning}

    The time segment of interest $[\tau, \tau+T)$ is partioned into sub-segments $\biguplus_{i=1}^I [\tau_i, \tau_{i+1})$ as follows:
        \vspace{-1em}
    \begin{equation}
    \notag
    \begin{split}
        \tau_1 & := \tau, 
        \\
        \tau_{i+1} & := 
        (\tau + T) 
        \wedge 
        \begin{cases}
            \inf \braces{t_k: t_k > \tau_i} 
            \\
            \inf \braces{t > \tau_i: \indicator{\gain^{\policy_t}(\State_t, \model) = \optgain(\model)} \ne \indicator{\gain^{\policy_{\tau_i}}(\State_{\tau_i}, \model) = \optgain(\model)}}
        \end{cases}
    \end{split}
    \end{equation}
    \noindent
    and we write $i \in \mathcal{I}_{\rm opt}$ if $\gain^{\policy_{\tau_i}}(\State_{\tau_i}; \model) = \optgain(\model)$ and $i \in \mathcal{I}_{\rm sub}$ if $\gain^{\policy_{\tau_i}}(\State_{\tau_i}; \model) < \optgain(\model)$, that we refer to as \strong{optimal} and \strong{sub-optimal} segments. 
    By design, every segment $[\tau_i, \tau_{i+1})$ is a subset of an episode and the sequence $(\tau_i)$ is a increasing sequence of stopping times. 
    The regret is decomposed according to this partition:
    \begin{equation}
    \label{equation_coherence_partioning}
        \Reg(\tau, \tau+T)
        = 
        \sum_{i \in \mathcal{I}_\mathrm{sub}} \sum_{t=\tau_i}^{\tau_{i+1}-1} \ogaps(\Pair_t)
        +
        \sum_{i \in \mathcal{I}_\mathrm{opt}} \sum_{t=\tau_i}^{\tau_{i+1}-1} \ogaps(\Pair_t)
        .
    \end{equation}
    Both terms are bounded separately.
    The first corresponds to the regret on segments where the current policy is sub-optimal, while the second corresponds to the regret on segments where the current policy is asymptotically optimal. 

    \subsection{Upper bounding the regret on sub-optimal segments}
    \label{appendix_coherence_suboptimal_segments}

    We have:
    \begin{equation}
    \label{equation_coherence_suboptimal_segments}
        \sum_{i \in \mathcal{I}_\mathrm{sub}} 
        \sum_{t=\tau_i}^{\tau_{i+1}-1} 
        \ogaps(\Pair_t)
        \le 
        \parens*{\max_{\pair \in \pairs} \ogaps(\pair)}
        \sum_{i \in \mathcal{I}_\mathrm{sub}} 
        (\tau_{i+1} - \tau_i)
        .
    \end{equation}
    We bound $\sum_{i \in \mathcal{I}_\mathrm{sub}} (\tau_{i+1} - \tau_i)$ directly.

    \bigskip
    \noindent
    \STEP{1}
    \textit{
        There exists a constant $\epsilon > 0$ such that, on $\bigcap_{t = \tau}^{\tau + T - 1} F_t$, we have:
        \begin{equation}
            \label{equation_coherence_suboptimal_1}
            \abs{\pairs}\parens{\varphi(\tau) + 1}
            \ge 
            \epsilon 
            \sum_{i \in \mathcal{I}_\mathrm{sub}} 
            (\tau_{i+1} - \tau_i)
            +
            \sum_{i \in \mathcal{I}_\mathrm{sub}} 
            \sum_{t=\tau_i}^{\tau_{i+1}-1}
            \parens*{
                e_{\State_{t+1}} - \kernel(\Pair_t)
            } \sum_{\pair \in \pairs} \bias^{\policy_{\tau_i}}(e_\pair, \kernel)
            - \frac {\abs{\mathcal{I}_\mathrm{sub}}}\epsilon
        \end{equation}
        with $\vecspan{\sum_{\pair \in \pairs} \bias^{\policy_{\tau_i}}(e_\pair, \kernel)} \le \frac 1\epsilon$, where $\bias^\policy(e_\pair, \kernel)$ is the bias function of the policy $\policy$ under the reward function $e_\pair$ and kernel $\kernel$.
        Moreover, $\epsilon$ can be chosen independently of $F, \tau, T$ and $\varphi$. 
    }
    \medskip
    \begin{proof}
        Let $i \in \mathcal{I}_\mathrm{sub}$ and fix $\pair \in \pairs$. 
        Because the segment $[\tau_i, \tau_{i+1})$ is a piece of episode, $\policy_{\tau_i}$ is used all throughout the segment. 
        The gain and bias functions of $\policy_{\tau_i}$ on the model with reward function $e_\pair$ (equal to one at $\pair$ and null elsewhere) and kernel $\kernel$ are respectively denoted $\gain^{\policy_{\tau_i}}(-; e_\pair, \kernel)$ and $\bias^{\policy_{\tau_i}}(-; e_\pair, \kernel)$. 
        Using the Poisson equation, we obtain:
        \begin{align*}
            \visits_\pair(\tau_{i+1}) - \visits_\pair(\tau_i) 
            & =
            \sum_{t=\tau_i}^{\tau_{i+1}-1}
            \gain^{\policy_{\tau_i}}(\State_t; e_\pair, \kernel)
            +
            \sum_{t=\tau_i}^{\tau_{i+1}-1}
            \parens*{
                e_{\State_{t+1}} - \kernel(\Pair_t)
            } \bias^{\policy_{\tau_i}}(e_\pair, \kernel)
            \\
            & \phantom{=}
            +
            \parens*{
                \bias^{\policy_{\tau_i}}(\State_{\tau_{i}}; e_\pair, \kernel)
                -
                \bias^{\policy_{\tau_i}}(\State_{\tau_{i+1}}; e_\pair, \kernel)
            }
            \\
            & \ge
            \sum_{t=\tau_i}^{\tau_{i+1}-1}
            \gain^{\policy_{\tau_i}}(\State_t; e_\pair, \kernel)
            +
            \sum_{t=\tau_i}^{\tau_{i+1}-1}
            \parens*{
                e_{\State_{t+1}} - \kernel(\Pair_t)
            } \bias^{\policy_{\tau_i}}(e_\pair, \kernel)
            - \frac 1\epsilon
        \end{align*}
        where $\epsilon$ is any positive quantity smaller than $(\max_\policy \max_\pair \vecspan{\bias^\policy(e_\pair, \kernel)})^{-1} > 0$. 

        Let $\mathcal{I}_\mathrm{sub}^\pair := \braces{i \in \mathcal{I}_\mathrm{sub}: \pair \in \Reach(\policy_{\tau_i}, \State_{\tau_{i+1}-1})}$.

        Because the segment $[\tau_i, \tau_{i+1})$ is a piece of episode, $\policy_{\tau_i}$ is used all throughout the segment hence a pair that is reachable at time $\tau_{i+1}-1$ is necessarily reachable during the entire segment. 
        Therefore, if $i \in \mathcal{I}_\mathrm{sub}^\pair$, then $\gain^{\policy_{\tau_i}}(\State_t; e_\pair, \kernel) > 0$ for all $t \in [\tau_i, \tau_{i+1}-1)$. 
        Further assume that $\epsilon$ is smaller than $\min \braces{\gain^\policy(\state; e_\pair, \kernel) : \pair \in \Reach(\policy, \state), \state \in \states, \policy \in \policies} > 0$.
        We obtain:
        \begin{align*}
            \visits_\pair(\tau_{i+1}) - \visits_\pair(\tau_i) 
            & \ge
            \epsilon (\tau_{i+1} - \tau_i)
            +
            \sum_{t=\tau_i}^{\tau_{i+1}-1}
            \parens*{
                e_{\State_{t+1}} - \kernel(\Pair_t)
            } \bias^{\policy_{\tau_i}}(e_\pair, \kernel)
            - \frac 1\epsilon
            .
        \end{align*}
        Summing for $i$ provides
        \begin{equation}
            \nonumber
            \max_{i \in \mathcal{I}_\mathrm{sub}} 
            \visits_\pair(\tau_{i+1}) - \visits_\pair(\tau)
            \ge
            \epsilon 
            \sum_{i \in \mathcal{I}_\mathrm{sub}^\pair} 
            (\tau_{i+1} - \tau_i)
            +
            \sum_{i \in \mathcal{I}_\mathrm{sub}^\pair} 
            \sum_{t=\tau_i}^{\tau_{i+1}-1}
            \parens*{
                e_{\State_{t+1}} - \kernel(\Pair_t)
            } \bias^{\policy_{\tau_i}}(e_\pair, \kernel)
            - \frac {\abs{\mathcal{I}_\mathrm{sub}^\pair}}\epsilon
            .
        \end{equation}
        Recall that for $i \in \mathcal{I}_\mathrm{sub}$, the segment last until the next episode and $\gain^{\policy_{\tau_i}}(\State_t, \model) < \optgain(\model)$ holds for all $t \in [\tau_i, \tau_{i+1})$.  
        Meanwhile, coherence guarantees that, on $\bigcap_{t=\tau}^{\tau+T-1} F_t$, we have $\visits_\pair(\tau_{i+1}) \le \visits_\pair(\tau) + \varphi(\tau) + 1$ for all $i \in \mathcal{I}_\mathrm{sub}$ and $\pair \notin \pairs^*(\model)$.  
        So, for all $\pair \notin \pairs^*(\model)$ and on $\bigcap_{t=\tau}^{\tau+T-1} F_t$, we have
        \begin{equation}
        \nonumber
            \varphi(\tau) + 1 
            \ge 
            \epsilon 
            \sum_{i \in \mathcal{I}_\mathrm{sub}^\pair} 
            (\tau_{i+1} - \tau_i)
            +
            \sum_{i \in \mathcal{I}_\mathrm{sub}^\pair} 
            \sum_{t=\tau_i}^{\tau_{i+1}-1}
            \parens*{
                e_{\State_{t+1}} - \kernel(\Pair_t)
            } \bias^{\policy_{\tau_i}}(e_\pair, \kernel)
            - \frac {\abs{\mathcal{I}_\mathrm{sub}^\pair}}\epsilon
            .
        \end{equation}
        By coherence and on $\bigcap_{t=\tau}^{\tau+T-1} F_t$ again, we see that $i \in \mathcal{I}_\mathrm{sub}$ belongs to one $\mathcal{I}_\mathrm{sub}^\pair$ for some $\pair \notin \pairs^*(\model)$ at least. 
        Summing for $\pair \notin \pairs^*(\model)$, we obtain the claim.
    \end{proof}

    \noindent
    \STEP{2}
    \textit{
        There exists a constant $\eta > 0$ such that
        \begin{equation}
        \label{equation_coherence_suboptimal_2}
            \forall x \ge 0,
            \quad
            \Pr \parens*{
                \abs*{\mathcal{I}_\mathrm{sub}} \ge x + \frac 1\eta \varphi(\tau)
                \text{~and~} \bigcap_{t=\tau}^{\tau+T-1} F_t
            }
            \le 
            \exp \parens*{- \eta x}.
        \end{equation}
        Moreover, $\eta$ can be chosen independently of $F, \tau, T$ and $\varphi$.
    }
    \medskip
    \begin{proof}
        Denote $\mathcal{T}_\mathrm{sub}(\tau, \tau+T) := \bigcup_{i \in \mathcal{I}_\mathrm{sub}} [\tau_i, \tau_{i+1})$ the time instants when $\gain^{\policy_t}(\State_t, \model) < \optgain(\model)$. 

        Introduce the quantity $\phi(t) := \sum_z \brackets*{\varphi(\tau) + N_z(\tau) - N_z(t)}_+$ for $t \in [\tau, \tau+T)$, which is non-increasing by construction. 
        By coherence and on $F_t$, if $t \in \mathcal{T}_\mathrm{sub}(\tau, \tau+T)$ then there exists a reachable $z$ such that $\varphi(\tau) + N_z(\tau) - N_z(t) > 0$.
        The crucial remark is that for $i \in \mathcal{I}_\mathrm{sub}$ with $[\tau_i, \tau_{i+1}) \subseteq [t_k, t_{k+1})$, two things may hold at time $\tau_{i+1}$:
        (1) Either $i+1 \in \mathcal{I}_\mathrm{opt}$, meaning that a state from which $\policy_{\tau_i}$ is optimal has been reached; 
        (2) Or $i+1 \notin \mathcal{I}_\mathrm{sub}$ and $\tau_{i+1} = t_{k+1}$, in which case $\State_{\tau_{i+1}}$ has been already visited since $\tau_i$.
        For (2), remark indeed that $\State_{\tau_{i+1}}$ has been visited already since $t_k$ by regenerativity of episodes (\Cref{definition_regenerative_episodes}), but if $t_k \ne \tau_i$ then $\gain^{\policy_{\tau_i}}(\State_t, \model) = \optgain(\model)$ for all $t \in [t_k, \tau_i)$ hence $\State_{t_{k+1}}$ cannot appear within the collection of states visited in the time-range $[t_k, \tau_i)$. 
        Combining (1) and (2), we conclude that conditionally on the history $\History_{\tau_i}$, every reachable pair $\pair \in \Reach(\policy_{\tau_i}, \State_{\tau_i})$ from which $\policy_{\tau_i}$ is sub-optimal has positive probability $\epsilon(S_{\tau_i}, \pi_{\tau_i}, z, M)$ to be visited until $\tau_{i+1}$. 
        Letting $\epsilon := \min_{s, \pi, z} \epsilon(s, \pi, z, M) > 0$, we get:
        \begin{align*}
            &
            \Pr \parens*{
                \phi(\tau_{i+1}) < \phi(\tau_i)
                \middle|
                \History_{\tau_i}, i \in \mathcal{I}_\mathrm{sub}, F_{\tau_i}
            }
            \\
            & \ge 
            \min_{
                \scriptsize
                \begin{array}{c}
                    z \equiv (\state,\action) \in \Reach(S_{\tau_i}, \pi_{\tau_i})
                    \\
                    \gain^{\policy_{\tau_i}}(\state, \model) < \optgain(\model)
                \end{array}
            }
            \Pr \parens*{
                N_z(\tau_{i+1}) > N_z(\tau_i)
                \middle |
                \History_{\tau_i}, i \in \mathcal{I}_\mathrm{sub}, F_{\tau_i}
            }
            \\
            & \ge 
            \epsilon
            .
        \end{align*}
        Let $\phi_0(\tau) := S A \varphi(\tau)$ and denote $F_{\tau:\tau+T} := \bigcap_{t=\tau}^{\tau+T-1} F_t$.
        On $F_{\tau:\tau+T}$, $\phi$ can only decrease up to $\phi_0(\tau)$ times before reaching zero, and once it has reached zero, we cannot have $t \in \mathcal{T}_\mathrm{sub}(\tau, \tau+T)$ anymore. 
        Accordingly, for all $m \ge 1$, $\abs*{\mathcal{I}_\mathrm{sub}} \ge m + \phi_0(\tau)$ implies on $F_{\tau:\tau+T}$ that the first in the first $m + \phi_0(\tau)$ elements of $\mathcal{I}_\mathrm{sub}$, at least $m$ of them are such that $\phi(\tau_{i+1}) = \phi(\tau_i)$.
        Introduce the short-hand $U_{\tau_i} := \indicator{\phi(\tau_{i+1}) = \phi(\tau_i)}$.
        For $\lambda > 0$ and $m \ge 1$, we have:
        \begin{align*}
            \psi(m) 
            & :=
            \Pr\parens*{
                \abs*{\mathcal{I}_\mathrm{sub}} \ge m + \phi_0(\tau)
                \text{~and~} F_{\tau:\tau+T}
            }
            \\
            & = 
            \Pr \parens*{
                \sum_{j=1}^{m + \phi_0(\tau)} U_{\tau_j} \ge m
                \text{~and~} F_{\tau:\tau+T}
            }
            \\
            & = \EE \brackets*{
                \eqindicator{
                    \exp\parens*{\lambda\sum_{j=1}^{m + \phi_0(\tau)} U_{\tau_j}}
                    \ge
                    \exp(\lambda m)
                }
                \indicator{F_{\tau:\tau+T}}
            }
            \\
            & \le 
            \exp(- \lambda m)
            \EE \brackets*{
                \exp\parens*{\lambda\sum_{j=1}^{m + \phi_0(\tau)} U_{\tau_j}}
                \indicator{F_{\tau:\tau+T}}
            }
            \\
            & \overset{(\dagger)}\le 
            \exp(- \lambda m)
            \EE \brackets*{
                \exp\parens*{\lambda\sum_{j=1}^{m + \phi_0(\tau)-1} U_{\tau_j}}
                \cdot
                {
                    \indicator{F_{\tau:\tau_{m+\phi_0(\tau)}}}
                    \cdot
                    \indicator{F_{\tau_{m+\phi_0(\tau)}}}
                    \atop
                    \EE\brackets*{\exp(\lambda U_{\tau_m+\phi_0(\tau)}) \middle| F_{\tau_{m+\phi_0(\tau)}}}
                }
            }
            \\
            & \overset{(\ddagger)}\le
            \exp(- \lambda m)
            \EE \brackets*{
                \exp\parens*{\lambda\sum_{j=1}^{m + \phi_0(\tau)-1} U_{\tau_j}}
                \indicator{F_{\tau:\tau_{m+\phi_0(\tau)}}}
                \cdot
                \exp\parens*{\lambda(1-\epsilon) + \frac{\lambda^2}8}
            }
            \\
            & \hspace{0.4em}\vdots{}
            \\
            & \le
            \exp \parens*{
                - \lambda m + \lambda (1-\epsilon) (m+\phi_0(\tau)) + (m+\phi_0(\tau)) \frac{\lambda^2}8
            }
            .
        \end{align*}
        In the above, $(\dagger)$ use that $\indicator{F_{\tau:\tau_{m+\phi_0}}} \cdot \indicator{F_{\tau_{m+\phi_0}}} \le \indicator{F_{\tau:\tau+T}}$ and $(\ddagger)$ is an application of Hoeffding's Lemma together with the fact that $\EE[U_{\tau_i}|F_{\tau_i}] \indicator{F_{\tau_i}} \le 1-\epsilon$.
        Assume that $m$ is large enough so that $\epsilon m > (1-\epsilon) \phi_0(\tau)$.
        Then we continue by factorizing the polynomial within the exponential and minimizing in $\lambda$, straight forward algebra shows that for $m \ge \tfrac{2\phi_0(\tau)}\epsilon$, we have:
        \begin{equation}
            \Pr \parens*{
                \abs*{\mathcal{I}_\mathrm{sub}} \ge m + \phi_0(\tau)
                \text{~and~} F_{\tau:\tau+T}
            }
            \le 
            \exp \parens*{- \frac{3\epsilon^2m}{4}}.
        \end{equation}
        We conclude accordingly by choosing $\eta = \Theta(1 + \frac 2\epsilon)$. 
    \end{proof}

    \noindent
    \STEP{3}
    \textit{
        There exists constants $C_0, C_1, C_2, C_3 > 0$ such that
        \begin{equation}
            \forall x \ge 0,
            \quad
            \Pr \parens*{ 
                \sum_{i \in \mathcal{I}_\mathrm{sub}} (\tau_{i+1}-\tau_i)
                >
                x + C_3 \varphi(\tau)
                \text{~and~}
                \bigcap_{t=\tau}^{\tau+T-1} F_t
            }
            \le
            C_1 T^{C_2} \exp\parens*{- C_0 x}
            .
        \end{equation}
        Moreover, $C_0, C_1, C_2, C_3$ can be chosen independently of $F, \tau, T$ and $\varphi$. 
    }
    \medskip
    \begin{proof}
        Using a time-uniform Azuma-Hoeffding's inequality (see \cite[Lemma~5]{bourel_tightening_2020}), we have for all $\delta > 0$,
        \begin{equation}
            \nonumber
            \Pr \parens*{ 
                \sum_{i \in \mathcal{I}_\mathrm{sub}} 
                \sum_{t=\tau_i}^{\tau_{i+1}-1}
                \parens*{
                    e_{\State_{t+1}} - \kernel(\Pair_t)
                } \sum_{\pair \in \pairs} \bias^{\policy_{\tau_i}}(e_\pair, \kernel)
                <
                - \frac 1\epsilon \sqrt{
                    \sum_{i \in \mathcal{I}_\mathrm{sub}} 
                    (\tau_{i+1} - \tau_i)
                    \log\parens*{\tfrac T\delta}
                }
            }
            \le
            \delta
            .
        \end{equation}
        Combined with \eqref{equation_coherence_suboptimal_1} from \STEP{1}, we obtain an equation of the form $x \le \alpha + \beta \sqrt{x}$ with $x = \sum_{i \in \mathcal{I}_\mathrm{sub}} (\tau_{i+1} - \tau_i)$, $\alpha = \frac {1} \epsilon (\abs{\pairs} (\varphi(\tau) + 1) + \frac 1\epsilon \abs{\mathcal{I}_\mathrm{sub}})$ and $\beta = \frac 1\epsilon \sqrt{\log(T/\delta)}$. 
        Simple algebra shows that $x \le 2 \alpha + 2 \beta^2$. 
        In other words, we have shown that:
        \begin{equation}
            \nonumber
            \forall \delta > 0,
            \quad
            \Pr \parens*{ 
                \sum_{i \in \mathcal{I}_\mathrm{sub}} (\tau_{i+1}-\tau_i)
                >
                C_0 \log\parens*{\tfrac T\delta}
                + C_1 \varphi(\tau)
                + C_2 \abs{\mathcal{I}_\mathrm{sub}}
                \text{~and~}
                \bigcap_{t=\tau}^{\tau+T-1} F_t
            }
            \le
            \delta
        \end{equation}
        for some model dependent constants $C_0, C_1, C_2 > 0$. 
        Use the sub-exponential tail property of $\abs{\mathcal{I}_\mathrm{sub}}$ \eqref{equation_coherence_suboptimal_2} from \STEP{2} to obtain a sub-exponential tail for $\sum_{i \in \mathcal{I}_\mathrm{sub}} (\tau_{i+1} - \tau_i)$. 
    \end{proof}

    \noindent
    \STEP{4}
    \textit{
        There exist constants $C_0, C_1, C_2, C_3 > 0$ such that, for all $\eta > 0$, 
        \begin{equation}
            \Pr \parens*{
                \sum_{j \in \mathcal{J}^+_\mathrm{sub}}
                \sum_{t=\tau_j^+}^{\tau_{j+1}^+-1}
                \ogaps(\Pair_t)
                >
                x + C_3 \varphi(\tau)
                \text{~and~}
                \bigcap_{t = \tau}^{\tau+T} F_t
            }
            \le
            C_1 T^{C_2} \exp(-C_0x)
            .
        \end{equation}
        Moreover, $C_0, C_1, C_2, C_3$ can be chosen independently of $F, \tau, T$ and $\varphi$. 
    }
    \medskip
    \begin{proof}
        Combine \eqref{equation_coherence_suboptimal_segments} with the result of \STEP{3}.
    \end{proof}

    \subsection{Upper bounding the regret on optimal segments}
    \label{appendix_coherence_optimal_segments}
    
    We start by merging consecutive optimal segments. 
    This is done by setting:
    \begin{equation}
    \begin{split}
        \tau_1^+ & := \inf \braces*{\tau_i: i \in \mathcal{I}_{\mathrm{opt}}}
        \\
        \tau_{2j}^+ & := \inf \braces*{\tau_i > \tau_{2j-1}^+: i \in \mathcal{I}_\mathrm{sub}}\\
        \tau_{2j+1}^+ & := \inf \braces*{\tau_i > \tau_{2j}^+ : i \in \mathcal{I}_\mathrm{opt}}
    \end{split}
    \end{equation}
    that design a macroscopic decomposition of $[\tau, \tau+T-1)$ into time-segments, of which $(\tau_i)$ is a refinement. 
    Remark that if $j$ is even, then $[\tau_j^+, \tau_{j+1}^+) \subseteq \bigcup_{i \in \mathcal{I}_\mathrm{sub}} [\tau_i, \tau_{i+1})$ and conversely, if $j$ is odd, then $[\tau_j^+, \tau_{j+1}^+) \setminus \bigcup_{i \in \mathcal{I}_\mathrm{sub}} [\tau_i, \tau_{i+1}) = \emptyset$.
    We write $j \in \mathcal{J}^+_\mathrm{sub}$ and $j \in \mathcal{J}^+_\mathrm{opt}$ respectively. 

    By non-degeneracy of the model $M$, all asymptotically optimal policies of $M$ have the same (unique) invariant probability measure $\mu^* \in \kernels(\pairs)$.
    On segments $[\tau_j^+, \tau_{j+1}^+)$ with $j \in \mathcal{J}^+_\mathrm{opt}$, $\ogaps(\Pair_t)$ can only be positive if the optimal recurrent states $\states(\mathrm{supp}(\mu^*))$ have not been reached yet. 
    The proof consists in showing that when $j \in \mathcal{J}^+_\mathrm{opt}$, the optimal recurrent class is quickly reached on $[\tau_j^+, \tau_{j+1}^+)$. 
    Indeed, setting $\tau_{j+1}^* := \tau_{j+1}^+ \wedge \inf \braces{t > \tau_j^+ : \mu^*(S_t) > 0}$ the reaching time to $\mathrm{supp}(\mu^*)$ after $\tau_j^+$, we have:\footnote{$\imeasure$ is a measure on $\pairs$. For $\state \in \states$, we write $\imeasure(\state) := \sum_{\action \in \actions(\state)} \imeasure(\state,\action)$.}
    \begin{equation}
    \label{equation_coherence_optimal_segments_1}
    \begin{aligned}
        \sum_{j \in \mathcal{J}^+_\mathrm{opt}}
        \sum_{t=\tau_j^+}^{\tau_{j+1}^+-1}
        \ogaps(\Pair_t)
        & \le 
        \parens*{\max_{z \in \pairs} \ogaps(\pair)}
        \sum_{j \in \mathcal{J}^+_\mathrm{opt}} \parens*{N_{\pairs^-(M)}(\tau_{j+1}^+) - N_{\pairs^-(M)}(\tau_j^+)}
        \\
        & =
        \parens*{\max_{z \in \pairs} \ogaps(\pair)}
        \sum_{j \in \mathcal{J}^+_\mathrm{opt}} \parens*{\tau_{j+1}^* - \tau_j^+}
        .
    \end{aligned}
    \end{equation}
    We now upper bound the RHS.

    \bigskip
    \noindent
    \STEP{1}
    \textit{
        There exists a constant $D_* > 0$ as well as an adapted sequence $(h_t)$ with $\vecspan{h_t} \le D^*$ s.t.:
        \begin{equation*}
        \begin{aligned}
            \sum_{j \in \mathcal{J}^+_\mathrm{opt}} \parens*{\tau_{j+1}^* - \tau_j^+}
            & \le
            2 D^* \parens*{
                \abs*{\mathcal{J}^+_\mathrm{opt}}
                + \sum_{j \in \mathcal{J}^+_\mathrm{opt}} \abs*{\braces*{
                    t_\ell \in (\tau_j^+, \tau_{j+1}^+) : \mu^*(S_{t_\ell} ) = 0
                }}
            }
            \\
            & \phantom{\le} 
            + \sum_{j \in \mathcal{J}^+_\mathrm{opt}} \sum_{t=\tau_j^+}^{\tau_j^*-1} \parens*{e_{S_{t+1}} - p_{Z_t}} h_t
            .
        \end{aligned}
        \end{equation*}
        Moreover, $D_*$ is independent of $F, \tau, T$ and $\varphi$. 
    }
    \medskip
    \begin{proof}
        Notice that $[\tau_j^+, \tau_{j+1}^+)$ is of the form $[t'_k, t_{k+1}) \uplus \biguplus_\ell [t_\ell, t_{\ell+1})$ where $[t'_k \in [t_k, t_{k+1}]$ is a time such that $g^{\pi_{t-1}}(\State_{t-1}; M) < g^{\pi_t}(\State_t; M) = g^*(S_t;M)$.
        Consider the reward function $f(z) := \indicator{z \notin \pairs^*(\model)}$.
        Over an episode $[t_\ell, t_{\ell+1}) \subseteq [\tau_j^+, \tau_{j+1}^+)$, the gain and the bias of $\pi^\ell$ associated to this reward function are respectively denoted $g^{(\ell)}$ and $h^{(\ell)}$. 
        Because the recurrent pairs under $\pi^\ell$ from $S_{t_\ell}$ are $\mathrm{supp}(\mu^*)$, we have $g^{(\ell)}(\state) = 0$ for all $(\state,\action) \in \Reach(S_{t_\ell}, \pi^\ell, M)$ and $h^{(\ell)}(\state) = 0$ for all $(\state,\action) \in \mathrm{supp}(\mu^*)$.
        Let $D^* < \infty$ be the maximum $\mathrm{sp}(h^{(\ell)})$ possible over all $\pi^\ell \in \Pi$.
        Using Poisson's equation $g^{(\ell)}(s) + h^{(\ell)}(s) = f(s, \pi^\ell(s)) + p(s, \pi^\ell(s)) h^{(\ell)}$, we obtain:
        \begin{align*}
            (-) 
            & := \tau_{j+1}^* - \tau_j^+
            \\
            & = N_{\pairs^-(M)}(\tau_{j+1}^+) - N_{\pairs^-}(\tau_j^+) 
            \\
            & = 
            \sum_\ell 
            \parens*{h^{(\ell)}(S_{t_\ell}) - h^{(\ell)}(S_{t_{\ell+1}})}
            + \sum_\ell \sum_{t=t_\ell}^{t_{\ell+1}-1}
            \parens*{e_{S_{t+1}} - p_{S_t, A_t}} h^{(\ell)}
            \\
            & \le 
            2 D^*
            + \sum_{\ell: t_\ell \in (\tau_j^+, \tau_{j+1}^+)}
            \parens*{h^{(\ell)}(S_{t_\ell}) - h^{(\ell)}(S_{t_{\ell+1}})}
            + \sum_\ell \sum_{t=t_\ell}^{t_{\ell+1}-1}
            \parens*{e_{S_{t+1}} - p_{S_t, A_t}} h^{(\ell)}
            \\
            & \overset{(\dagger)}=
            2 D^*
            + \sum_{\ell > k}
            \indicator{t_\ell < \tau_j^*}
            \parens*{h^{(\ell)}(S_{t_\ell}) - h^{(\ell)}(S_{t_{\ell+1}})}
            \\
            & \phantom{
                {} \overset{(\dagger)}=
                2 D^*
            }
            {} + \sum_\ell \sum_{t=t_\ell}^{t_{\ell+1}-1}
            \indicator{t < \tau_j^*}
            \parens*{e_{S_{t+1}} - p_{S_t, A_t}} h^{(\ell)}
            \\
            & \overset{(\ddagger)}\le
            2 D^* \parens*{
                1 + \abs*{\braces*{
                    t_\ell \in (\tau_j^+, \tau_{j+1}^+) : \mu^*(S_{t_\ell}) = 0
                }}
            }
            + \sum_{t=\tau_j}^{\tau_j^*-1}
            \parens*{e_{S_{t+1}} - p_{S_t, A_t}} h_t
        \end{align*}
        where $(\dagger)$ follows from $h^{(\ell)} = 0$ on the support of $\mu^*$, and $(\dagger)$ introduces $h_t$ as the unique $h^{(\ell)}$ such that $t \in [t_\ell, t_{\ell+1})$.
        Conclude by summing over $i \in \mathcal{J}^+_\mathrm{opt}$.
    \end{proof}

    \noindent
    \STEP{2}
    \textit{
        There exist constants $C_1, C_2, C_3, C_4 > 0$ such that, for all $\eta > 0$, 
        \begin{equation}
            \forall x \ge 0,
            \quad
            \Pr \parens*{
                \sum_{j \in \mathcal{J}^+_\mathrm{opt}} \parens*{\tau_{j+1}^* - \tau_j^+} 
                >
                x + C_4 \varphi(\tau)
                \text{~and~}
                \bigcap_{t = \tau}^{\tau+T} F_t
            }
            \le 
            C_1 T^{C_2} \exp(-C_3x).
        \end{equation}
        Moreover, $C_1, C_2, C_3, C_4$ can be chosen independently of $F, \tau, T$ and $\varphi$. 
    }
    \medskip
    \begin{proof}
        We bound every term appearing in \STEP{1}.

        The \strong{first term} involves $\abs{\mathcal{J}^+_\mathrm{opt}}$.
        Because elements of $\mathcal{J}^+_\mathrm{opt}$ and $\mathcal{J}^+_\mathrm{sub}$ are intertwined, we $\abs{\mathcal{J}^+_\mathrm{opt}} \le 1 + \abs{\mathcal{J}^+_\mathrm{sub}}$. 
        Moreover, since macroscopic segments $[\tau_j^+, \tau_{j+1}^+)$ are unions of segments $[\tau_i, \tau_{i+1})$, we have $\abs{\mathcal{J}^+_\mathrm{sub}} \le \abs{\mathcal{I}_\mathrm{sub}}$ that has been bounded in \eqref{equation_coherence_suboptimal_2} already.
        Accordingly, $\abs{\mathcal{J}^+_\mathrm{sub}}$ has sub-exponential tails on the good event $\bigcap_{t=\tau}^{\tau+T-1} F_t$:
        \begin{equation}
            \forall x \ge 0,
            \quad
            \Pr \parens*{
                \abs*{\mathcal{J}^+_\mathrm{sub}} \ge x + \frac 1c \varphi(\tau)
                \text{~and~} \bigcap_{t=\tau}^{\tau+T-1} F_t
            }
            \le 
            \exp \parens*{- c x}
        \end{equation}
        where $c > 0$ is a model dependent constant.

        For the \strong{second term}, remark that for each $t_\ell \in [\tau_j, \tau_{j+1})$ with $j \in \mathcal{J}^+_\mathrm{opt}$, the probability that the episode ends with $S_{t_{\ell+1}} \in \mathrm{supp}(\mu^*)$ is positive because of the regenerativity property (\Cref{definition_regenerative_episodes}) of \eqref{equation_vanishing_multiplicative}.
        This is also true for the first (possibly) truncated episode $[t'_k, t_{k+1})$ that starts the macroscopic segment $[\tau_j^+, \tau_{j+1}^+)$ because as the gain $g^{\pi_{t}}(S_t; M)$ increases from $t'_k-1$ to $t'_k$ to the optimal $g^{\policy_t}(\State_t, M) = \optgain(\State_t; M)$, all states that are reachable from $S_t$ under $\pi_t$ cannot have been visited yet during the episode.
        In the end, the probability of reaching $\mathrm{supp}(\mu^*)$ by the end of the episode is lower bounded by some $\epsilon'(\pi_{t_\ell}, S_{t_\ell}, M) > 0$ and we denote $\epsilon' > 0$ the minimum for all possible values of $\pi^\ell$ and $S_{t_\ell}$.
        We conclude that $\Pr(\mu^*(S_{t_{\ell+1}}) > 0 \mid \History_{t_\ell}) > \epsilon'$. 
        Accordingly,
        \begin{equation*}
            U_{\tau_j^+}
            :=
            \abs*{\braces*{
                t_\ell \in (\tau_j^+, \tau_{j+1}^+)
                :
                \mu^*(X_{t_\ell})
                =
                0
            }}
        \end{equation*}
        is stochastically dominated by a geometric distribution $\mathrm{G}(\epsilon')$. 
        Using bounds on tails of geometric random variables (\Cref{lemma_tail_bound_geometric}), we obtain:
        \begin{equation}
            \label{equation_coherence_optimal_segments_2}
            \Pr\parens*{
                \sum_{j \in \mathcal{J}^+_\mathrm{opt}}
                \abs*{\braces*{
                    t_\ell \in (\tau_j^+, \tau_{j+1}^+) : \mu^*(S_{t_\ell}) = 0
                }}
                >
                \abs*{\mathcal{J}^+_\mathrm{opt}} \parens*{1 + \tfrac 2{\epsilon'}}
                + \tfrac{2\eta \log(T)}{\log\parens*{\frac 1{1-\epsilon'}}}
            }
            \le
            T^{-\eta}
            .
        \end{equation}
        The \strong{third term} $\sum_{j \in \mathcal{J}^+_\mathrm{opt}} \sum_{t=\tau_j}^{\tau_{j+1}^*-1} (e_{S_{t+1}} - p_{Z_t}) h_t$ is the sum of a martingale difference sequence, each term having span at most $D^*$ by \STEP{1}.
        By applying a time-uniform Azuma-Hoeffding's inequality (see \cite[Lemma~5]{bourel_tightening_2020}), we obtain:
        \begin{equation}
            \label{equation_coherence_optimal_segments_3}
            \Pr \parens*{
                \sum_{j \in \mathcal{J}^+_\mathrm{opt}} 
                \sum_{t=\tau_j}^{\tau_{j+1}^*-1}
                (e_{S_{t+1}} - p_{Z_t}) h_t
                > 
                D^* \sqrt{
                    \sum_{j \in \mathcal{J}^+_\mathrm{opt}}
                    \parens*{\tau_{j+1}^* - \tau_j^+}
                    \parens*{\tfrac 12 + \eta} \log(1+T)
                }
            }
            \le
            T^{-\eta}
            .
        \end{equation}
        Combining the bound of the first term, \eqref{equation_coherence_optimal_segments_2} and \eqref{equation_coherence_optimal_segments_3}, we see that there exists $C_1, C_2, C_3, C_4$ such that for all $\eta > 0$, with probability $1 - 3T^{-\eta}$, 
        \begin{align*}
            & \sum_{j \in \mathcal{J}^+_\mathrm{opt}}
            \parens*{\tau_{j+1}^* - \tau_j^+}
            \\
            & \le
            C_1 + (C_2 + \eta C_3) (\log(T) + \varphi(\tau))
            + C_4 \sqrt{\sum_{j\in\mathcal{J}^+_\mathrm{opt}} \parens*{\tau_{j+1}^* - \tau_j^+} \parens*{\tfrac 12 + \eta} \log(T)}.
        \end{align*}
        This is an equation of the form $x \le \alpha + \beta \sqrt{x}$ that implies in particular $x \le 2(\alpha + \beta^2)$. 
        We conclude by rearranging terms of the equation. 
    \end{proof}

    \noindent
    \STEP{3}
    \textit{
        There exist constants $C_1, C_2, C_3, C_4 > 0$ such that, for all $\eta > 0$, 
        \begin{equation}
            \Pr \parens*{
                \sum_{j \in \mathcal{J}^+_\mathrm{opt}}
                \sum_{t=\tau_j^+}^{\tau_{j+1}^+-1}
                \ogaps(\Pair_t)
                >
                x + C_4 \varphi(\tau)
                \text{~and~}
                \bigcap_{t = \tau}^{\tau+T} F_t
            }
            \le
            C_1 T^{C_2} \exp(-C_3x)
            .
        \end{equation}
        Moreover, $C_1, C_2, C_3, C_4$ can be chosen independently of $F, \tau, T$ and $\varphi$. 
    }
    \medskip
    \begin{proof}
        Invoke \eqref{equation_coherence_optimal_segments_1} and apply the result of \STEP{2}.
    \end{proof}

    \subsection{Combining everything}
    \label{appendix_coherence_combining}

    Conclude by combining \eqref{equation_coherence_partioning} with \Cref{appendix_coherence_suboptimal_segments} \STEP{4} and \Cref{appendix_coherence_optimal_segments} \STEP{3}.

    \subsection{Technical result: Tails of geometric random variables}

    In this section, we provide a result on the tails of sums of geometric random variables. 

    \begin{lemma}[Tails of Geometric Random Variables]
        \label{lemma_tail_bound_geometric}
        Let $(X_i)$ be a sequence of i.i.d.~random variable of distribution $\mathrm{G}(p)$, and let $S_n := X_1 + \ldots + X_n$ be their sum.
        Then, for all $c \ge 2$ and $t \ge 0$,
        \begin{equation*}
            \Pr\parens*{
                S_n \ge c \parens*{\EE[S_n] + t}
            }
            \le 
            (1 - p)^t
            \exp\parens*{
                -\tfrac{(2c-3)n}{4}
            }.
        \end{equation*}
    \end{lemma}
    \begin{proof}
        This proof is standard and was found on \texttt{math.stackexchange.com}\footnote{
            \resizebox{0.95\linewidth}{!}{
                \url{https://math.stackexchange.com/questions/110691/tail-bound-on-the-sum-of-independent-non-identical-geometric-random-variables}
            }
        }.
        We rely on Chernoff's method as usual by using the Laplace transform $\EE[e^{sX_i}]$. 
        We have:
        \begin{equation*}
            \Pr(S_n \ge c(\EE[S_n]+t))
            \le
            e^{-sct}
            e^{-scn/p}
            \prod_{i=1}^n
            \EE[e^{sX_i}].
        \end{equation*}
        We compute the Laplace transform of $X_i$: $\EE[e^{sX_i}] = (1 - \frac{1-e^s}{p})^{-1}$. 
        Setting $s = - \frac 1c \log(1-p)$, we have $\exp(-sc) = 1-p$.
        In the above formula, we readily obtain:
        \begin{equation*}
            \Pr(S_n \ge c(\EE[S_n]+t))
            \le
            (1-p)^t 
            \exp\parens*{
                n\parens*{
                    \tfrac ap
                    - \log\parens*{1 - \tfrac bp}
                }
            }
        \end{equation*}
        where $a := \log(1-p)$ and $b = 1 - (1-p)^{1/c}$. 
        We want that exponential to decrease quickly to $0$ with $n$, i.e., we want $a/p - \log(1-b/p) < 0$. 
        By Bernoulli's inequality, we have $b = 1 - (1-p)^{1/c} \le p/c \le p/2$, hence $b/p \le \frac 12$. 
        Moreover, for $z \in (0, \tfrac 12]$, $\log(1-z) \ge -z-z^2$, hence
        \begin{equation*}
            \tfrac ap - \log\parens*{1 - \tfrac bp}
            \le
            \tfrac ap \tfrac bp - \log(\tfrac 12)
            \le 
            \tfrac 12 \parens*{\tfrac ab + \tfrac 32}.
        \end{equation*}
        Finally, since $a = \log(1-p)$ and $b \le p/c$, it follows that $\frac ab \le \frac{c\log(1-p)}p \le -c$, so we get $\frac ab - \log(1 - \frac bp) \le \frac 12(-c + \frac 32)$. 
        This concludes the proof. 
    \end{proof}

    \clearpage
    \section{Regret of exploration guarantees of \texorpdfstring{\eqref{equation_vanishing_multiplicative}}{(VM)}}
    \label{appendix_regret_of_exploration}

    In this appendix, we provide the details of \Cref{section_vm_regret_guarantees} and behind the proof of \Cref{theorem_main}, assertion 3. 
    Below is a map of the proof.
    It will be reported throughout the proof to keep track of where the current lemma of interest in used in the proof's architecture. 

    \begin{figure}[ht]
        \centering
        \small
        \resizebox{\linewidth}{!}{\begin{tikzpicture}
            \tikzstyle{lemma}=[text width=3.3cm, align=center, draw, rounded corners, fill=white]

            \node[lemma, text width=2cm] (0) at (0, 0) {\texttt{EVI}-based algorithm};

            \node[lemma,dashed] (21) at (8, -1.5) {Linear visits $N_z(T) = \Omega(T)$ on $\pairs^{**}(M)$};
            \node[lemma] (31) at (12, -1.5) {Shaking effect (\Cref{appendix_shaking})};

            \node[lemma] (10) at (4, 0) {Asymptotic regime (\Cref{appendix_asymptotical})};
            \node[lemma,dashed] (20) at (8, 1.5) {Logarithmic visits $N_z(T) = \OH(\log(T))$ outside $\pairs^{**}(M)$};
            \node[lemma] (30) at (12, 1.5) {Shrinking effect (\Cref{appendix_shrinking})};

            \draw[color=black, fill=white] (14.5, 0) circle(0.25cm);
            \node (star) at (14.5, 0) {\Large $*$};
            \node[lemma] (40) at (17, 0) {\strong{Local} coherence (\Cref{section_establishing_coherence})};
            \node[lemma] (50) at (17, -1.75) {$\RegExp(T) = \OH(\log(T))$};

            \draw[->, >=stealth] (0) to (0, 1) to node[midway, above] {\scriptsize \Cref{lemma_coherence}} node[midway, below] {\scriptsize \strong{Global} coherence} (3.5, 1) to (3.5, 0.5);

            \draw[->, >=stealth] (4.5, 0.5) to (4.5, 2.5) to node[midway, above] {\scriptsize \Cref{lemma_asymptotic_regime}} (7.5, 2.5) to (7.5, 2.27);
            \draw[->, >=stealth] (8.5, 2.27) to (8.5, 2.5) to node[midway, above] {\scriptsize \Cref{lemma_shrinking}} (11.5, 2.5) to (11.5, 2.03);
            \draw[->, >=stealth] (12.5, 2.03) to (12.5, 2.5) to (14.5, 2.5) to (14.5, 0.25);

            \draw[->, >=stealth] (4.5, -0.5) to (4.5, -2.5) to node[midway, below] {\scriptsize \Cref{lemma_asymptotic_regime}} (7.5, -2.5) to (7.5, -2.27);
            \draw[->, >=stealth] (8.5, -2.27) to (8.5, -2.5) to node[midway, below] {\scriptsize \Cref{lemma_shaking}} (11.5, -2.5) to (11.5, -2.03);
            \draw[->, >=stealth] (12.5, -2.03) to (12.5, -2.5) to (14.5, -2.5) to (14.5, -0.25);

            \draw[->, >=stealth] (5.8, 0) to node[pos=0.88, above] {\scriptsize (\Cref{section_establishing_coherence})} (14.25, 0);
            \draw[->, >=stealth] (14.75, 0) to (40);
            \draw[->, >=stealth] (16, -0.535) to node[midway, right] {\scriptsize \Cref{lemma_coherence}} (16, -1.2);
        \end{tikzpicture}}
        \caption{
            \label{figure_regret_of_exploration_analysis}
            Proof map of regret of exploration guarantees.
        }
    \end{figure}
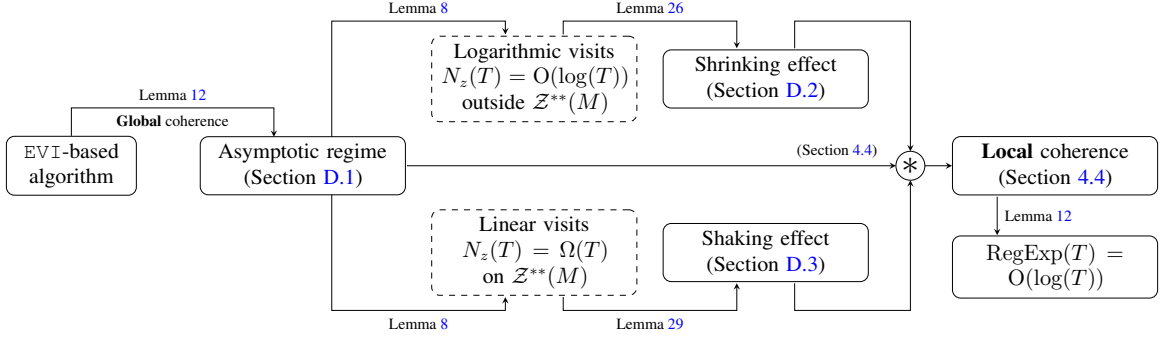
    
    \paragraph{Outline}
    The appendix is organized as such.
    In \Cref{appendix_asymptotical} we prove \Cref{lemma_asymptotic_regime}, describing the asymptotic visit rates. 
    We continue by establishing formal version of the shrinking-shaking effect, discussed informally in \Cref{section_shrinking_shaking}, beginning with the shrinking effect in \Cref{appendix_shrinking} and continuing with the shaking effect in \Cref{appendix_shaking}.
    We conclude by linking the shrinking-shaking effect and the asymptotic visit rates to a local coherence property of \Cref{lemma_local_coherence} in \Cref{appendix_local_coherence}, which is the last step of the proof of the assertion 3 of \Cref{theorem_main}, see \Cref{section_establishing_coherence}.
    
    \subsection{The asymptotic regime of \texorpdfstring{\eqref{equation_vanishing_multiplicative}}{(VM)}: Proof of \texorpdfstring{\Cref{lemma_asymptotic_regime}}{Lemma 8}}
    \label{section_asymptotic_regime}
    \label{appendix_asymptotical}

    In this section, we provide a proof of \Cref{lemma_asymptotic_regime}:

    \par
    \bigskip
    \noindent
    \textbf{\Cref{lemma_asymptotic_regime}}
    \textit{
        Let $M \in \models$ be non-degenerate satisfying \Cref{assumption_interior}.
        Assume that \texttt{KLUCRL} is run while managing episodes with $f$-\eqref{equation_vanishing_multiplicative} with arbitrary $f > 0$.
        There exists $\lambda > 0$ such that:
        \begin{align*}
            \forall z \notin \optpairs(M), \quad
            & \Pr^M \parens*{
                \exists T, \forall t \ge T :
                N_z(t) < \lambda \log(t)
            } = 1, 
            \text{~and}
            \\
            \forall z \in \optpairs(M), \quad
            & \Pr^M \parens*{
                \exists T, \forall t \ge T :
                N_z(t) > \tfrac 1\lambda t
            } = 1.
        \end{align*}
    }

    \begin{figure}[ht]
        \centering
        \small
        \resizebox{\linewidth}{!}{\begin{tikzpicture}
            \tikzstyle{lemma}=[text width=3.3cm, align=center, draw, rounded corners, fill=white]

            \node[lemma, text width=2cm] (0) at (0, 0) {\texttt{EVI}-based algorithm};

            \node[lemma,dashed, fill=yellow!50] (21) at (8, -1.5) {Linear visits $N_z(T) = \Omega(T)$ on $\pairs^{**}(M)$};
            \node[lemma] (31) at (12, -1.5) {Shaking effect (\Cref{appendix_shaking})};

            \node[lemma, fill=yellow!50] (10) at (4, 0) {Asymptotic regime (\Cref{appendix_asymptotical})};
            \node[lemma,dashed, fill=yellow!50] (20) at (8, 1.5) {Logarithmic visits $N_z(T) = \OH(\log(T))$ outside $\pairs^{**}(M)$};
            \node[lemma] (30) at (12, 1.5) {Shrinking effect (\Cref{appendix_shrinking})};

            \draw[color=black, fill=white] (14.5, 0) circle(0.25cm);
            \node (star) at (14.5, 0) {\Large $*$};
            \node[lemma] (40) at (17, 0) {\strong{Local} coherence (\Cref{section_establishing_coherence})};
            \node[lemma] (50) at (17, -1.75) {$\RegExp(T) = \OH(\log(T))$};

            \draw[->, >=stealth] (0) to (0, 1) to node[midway, above] {\scriptsize \Cref{lemma_coherence}} node[midway, below] {\scriptsize \strong{Global} coherence} (3.5, 1) to (3.5, 0.5);

            \draw[->, >=stealth] (4.5, 0.5) to (4.5, 2.5) to node[midway, above] {\scriptsize \Cref{lemma_asymptotic_regime}} (7.5, 2.5) to (7.5, 2.27);
            \draw[->, >=stealth] (8.5, 2.27) to (8.5, 2.5) to node[midway, above] {\scriptsize \Cref{lemma_shrinking}} (11.5, 2.5) to (11.5, 2.03);
            \draw[->, >=stealth] (12.5, 2.03) to (12.5, 2.5) to (14.5, 2.5) to (14.5, 0.25);

            \draw[->, >=stealth] (4.5, -0.5) to (4.5, -2.5) to node[midway, below] {\scriptsize \Cref{lemma_asymptotic_regime}} (7.5, -2.5) to (7.5, -2.27);
            \draw[->, >=stealth] (8.5, -2.27) to (8.5, -2.5) to node[midway, below] {\scriptsize \Cref{lemma_shaking}} (11.5, -2.5) to (11.5, -2.03);
            \draw[->, >=stealth] (12.5, -2.03) to (12.5, -2.5) to (14.5, -2.5) to (14.5, -0.25);

            \draw[->, >=stealth] (5.8, 0) to node[pos=0.88, above] {\scriptsize (\Cref{section_establishing_coherence})} (14.25, 0);
            \draw[->, >=stealth] (14.75, 0) to (40);
            \draw[->, >=stealth] (16, -0.535) to node[midway, right] {\scriptsize \Cref{lemma_coherence}} (16, -1.2);
        \end{tikzpicture}}
    \end{figure}

    \paragraph{Outline of the proof}
    The proof relies on the coherence lemma (\Cref{lemma_coherence}).
    We show that the confidence regions are such that, if a sub-optimal policy is played, one of the pairs responsible for its optimistic gain must be sub-sampled.
    This provides a ``global'' coherence property, see \STEP{1}, this is used show that sub-optimal pairs are visited at most logarithmically often in the asymptotic regime, see \STEP{2} and \eqref{equation_asymptotic_4_5}.
    However, the coherence lemma (\Cref{lemma_confidence_region}) cannot be invoked unless \Cref{assumption_interior} holds and $\model$ is non-degenerate.
    We deduce in parallel that non optimal pairs, i.e., $\pairs \setminus \optpairs(\model)$, are visited at most logarithmically often in \STEP{3} with \eqref{equation_asymptotic_5}, and that optimal pairs, i.e., $\optpairs(\model)$, are visited at least linearly often in \STEP{4} with \eqref{equation_asymptotic_6}.

    \bigskip
    \noindent
    \STEP{1}
    \textit{
        There exists a sequence of adapted events $(F_t)$ satisfying $\Pr(\exists T, \forall t \ge T: F_t) = 1$ and a function $\varphi : \N \to \R_+$ with $\varphi(t) = \OH(t)$ s.t.~the algorithm is $((F_t), \lfloor \log(T) \rfloor, T, \varphi)$-coherent.
    }
    \medskip
    \begin{proof}
        Introduce the good event $E_t := \braces*{M \in \models(t)}$.
        By design of the confidence region (see \Cref{lemma_confidence_region}), we know that $\Pr(\exists t \ge T: \model \notin \models(t)) = \OH(T^{-1})$, so $\Pr(\exists T, \forall t \ge T: E_t) = 1$.
        Let $T \ge 1$ and set $T_0 := \lfloor \log(T) \rfloor$.
        Pick $t \in \braces{T_0, \ldots, T}$ and let $\braces{t_k, \ldots, t_{k+1}-1}$ be the unique episode it falls in.
        We denote $\pi \equiv \pi_{t_k}$ for short and assume that $\pi$ is sub-optimal from $S_t$, i.e.,
        \begin{equation}
            \optgain(\State_t; \model) > \gain^\policy(\State_t; \model)
            .
        \end{equation}
        So there exists a class of pairs $\pairs'$ which is recurrent under $\policy$, with $\pairs' \subseteq \Reach(\policy, \State_t)$ and such that $\gain^\policy(\state; \model) < \optgain(\state; \model)$ for every $\state \in \states(\pairs')$. 
        Let $\state_0 \in \states(\pairs')$. 

        Denote $\Delta_g := \min \braces*{g^*(s; M) - g^\policy(s; M) : \pi \in \Pi, s \in \states, \optgain(\state;, \model) > g^\policy(\state; M)} > 0$ the minimal gain gap in $M$.
        Because $\pi$ is output by \texttt{EVI} (\Cref{appendix_evi}) at time $t_k$, it is optimistically optimal at time $t_k$ and $g^*(\State_t; \models(t_k)) = g(\State_t; \tilde{r}_\pi, \tilde{p}_\pi)$ for some $\tilde{r}_\pi \in \prod_s \rewards_{s, \pi(s)}(t_k)$ and $\tilde{p}_\pi \in \prod_s \kernels_{s, \pi(s)}(t_k)$.
        Furthermore, on $E_t$, we have $D(\models(t)) \le D(M)$ hence every policy returned by \texttt{EVI} (\Cref{appendix_evi}) has optimistic bias span at most $D(M)$ and its optimistic gain has span equal to $0$.
        We have, on $E_{t_k}$,
        \begin{align*}
            \Delta_g
            & \le g^*(\state_0; M) - g^\policy(\state_0; M)
            \\
            & \overset{(\dagger)}\le 
            g^{\pi_{t_k}}(\state_0; \models(t_k)) - g^\policy(\state_0; M)
            \\
            & \overset{(\ddagger)}\le 
            \norm*{\tilde{r} - r}_{\infty, \Reach(\pi, \state_0)}
            + \tfrac 12 D(M) \norm*{\tilde{p} - p}_{1, \Reach(\pi, \state_0)}
            .
        \end{align*}
        In $(\dagger)$, we have used that, $g^*(\state_0; M) \le g^*(\state_0; \models(t_k)) = g^*(\State_{t_k}; \models(t_k)) = g^\pi(\State_{t_k}; \models(t_k))$ on $E_{t_k}$.
        In $(\ddagger)$, we first invoke a gain deviation inequality (\Cref{lemma_gain_deviations}), then rely on the fact that by \Cref{assumption_interior}, the optimistic gain of $\pi$ computed by \texttt{EVI} only depends on pairs that are reachable from $\state_0$ under $\pi$ on $M$.
        One of the two terms of the RHS of the above equation must be at least $\tfrac 12 \Delta_g$.
        For instance, $D(M) \norm*{\tilde{p} - p}_{1, \Reach(\pi, \state_0)} \ge \Delta_g$.
        We have:
        \begin{equation}
        \label{equation_asymptotic_2}
        \begin{split}
            \Delta_g
            & \le
            D(M) \parens*{
                \norm*{\tilde{p} - \hat{p}_{t_k}}_{1, \Reach(\pi, \state_0)}
                +
                \norm*{\hat{p}_{t_k} - p}_{1, \Reach(\pi, \state_0)}
            }
            \\
            & =
            D(M) \parens*{
                \min_{z \in \Reach(\pi, \state_0)}
                \norm*{\tilde{p}(\pair) - \hat{p}_{t_k}(\pair)}_{1}
                +
                \min_{z \in \Reach(\pi, \state_0)}
                \norm*{\hat{p}_{t_k}(\pair) - p(\pair)}_{1}
            }
            .
        \end{split}
        \end{equation}
        Now, given $\tilde{\kernel} \in \kernels_\pair (t)$, we have $\visits_\pair (t) \KL(\hat{\kernel}_{t}(\pair)||\tilde{\kernel}(\pair)) \le \abs{\states} \log(2et)$, see \eqref{equation_confidence_region}.
        By Pinsker's inequality, we deduce that there are constants $\alpha, \beta > 0$ (independent of $t \ge 1$, $\pair \in \pairs$ and $\tilde{\kernel} \in \kernels_\pair (t)$) such that $\visits_\pair (t) \norm{\tilde{\kernel}(\pair) - \hat{\kernel}_{t}(\pair)}_1^2 \le \alpha \log(\beta t)$. 
        Accordingly,
        \begin{equation*}
            \kernels_z(t) 
            \subseteq 
            \braces*{
                \tilde{p}_z \in \kernels(\states):
                N_z(t) \norm*{\tilde{p}_z - \hat{p}_z(t)}_1^2 \le \alpha \log(\beta t)
            }
            =: \kernels'_z(t)
            .
        \end{equation*}
        Since $\Pr(\exists T \ge 1: \kernel(\pair) \in \kernels_\pair(t)) = 1$, we deduce that $\Pr(\exists T \ge 1: \kernel(\pair) \in \kernels'_\pair(t)) = 1$.
        Injecting this in \eqref{equation_asymptotic_2}, we see that on the asymptotically almost sure event $F_{t_k}^p := \braces*{\forall z, p_z \in \kernels'_z(t)}$, we have:
        \begin{equation}
        \label{equation_asymptotic_3}
            \Delta_g 
            \le 
            2 D(M) 
            \min_{z \in \Reach(\pi, \state_0)}
            \sqrt{\frac{\alpha\log(\beta t_k)}{N_z(t_k)}}
            \overset{(\dagger)}\le
            2 D(M) 
            \min_{z \in \Reach(\pi, \state_0)}
            \sqrt{\frac{2\alpha\log(2\beta t)}{N_z(t)}}
        \end{equation}
        where $(\dagger)$ uses that the \eqref{equation_vanishing_multiplicative} guarantees $N_z(t_{k+1}) \le 2 N_z(t_k)$ and $t_{k+1} \le 2 t_k$.
        Solving \eqref{equation_asymptotic_3} in $N_z(t)$, we find a condition of the form $N_z(t) \le \alpha' \log(\beta' t)$.

        The same rationale can be used to handle the case where $\norm*{\tilde{r} - r}_{\infty, \Reach(\pi, \state_0)} \ge \tfrac 12 \Delta_g$, dealing with the design of another asymptotically almost sure event $F_t^r := \braces*{\forall z, r_z \in \rewards_z'(t)}$ and ending with the same kind of upper-bound on $N_z(t)$.
        In the end, setting $F_t := \bigcap_{t'=\lfloor t/2\rfloor}^t F_{t'}^r \cap F_{t'}^p \cap E_{t'}$ and $\varphi(T_0) = \alpha' \log(\beta' t)$, we see that the algorithm is $((F_t), T_0, T, \varphi)$-coherent.
    \end{proof}

    \noindent
    \STEP{2}
    \textit{
        There exists $C > 0$ such that:
        \begin{equation}
        \label{equation_asymptotic_4_5}
            \Pr \parens*{
                \exists T,
                \forall t \ge T,
                \forall z \in \pairs^-(M):
                N_z(t) \le C \log(t)
            }
            =
            1
            .
        \end{equation}
    }
    \begin{proof}
        Since $M$ is non-degenerate, coherence can be converted to regret guarantees (\Cref{lemma_coherence}): Applying \Cref{lemma_coherence} following \STEP{1}, there exist constants $C_1, C_2 > 0$ such that:
        \begin{equation}
        \label{equation_asymptotic_4}
            \forall T \ge 1,
            \quad
            \Pr\parens*{
                \Reg(\log(T), T) \ge C_1 + C_2 \log(T)
                \text{~and~}
                \bigcap_{t=\lfloor\log(T)\rfloor}^T F_t
            }
            \le
            T^{-2}.
        \end{equation}
        Since $N_z(T) \le N_z(T_0) + \ogaps(\pair)^{-1} \Reg(T_0, T)$, the condition $\Reg(\log(T), T) \le C_1 + C_2 \log(T)$ is converted to $N_z(T) \le C'_1 + C'_2 \log(T)$ for all $z \in \pairs^-(M)$. 
        We have:
        \begin{align*}
            & \Pr\parens*{
                \forall T,
                \exists t \ge T,
                \exists z \in \pairs^-(M):
                N_z(t) > C'_1 + C'_2 \log(t)
            }
            \\
            & \overset{(\dagger)}=
            \Pr\parens*{
                \forall T,
                \exists t \ge T,
                \exists z \in \pairs^-(M):
                N_z(t) > C'_1 + C'_2 \log(t) 
                \text{~and~}
                \bigcap_{t=\lfloor \log(T)\rfloor}^T
                F_t
            }
            \\
            & \le
            \Pr\parens*{
                \forall T,
                \exists t \ge T,
                \exists z \in \pairs^-(M):
                \Reg(\log(T), T) > C_1 + C_2 \log(T)
                \text{~and~}
                \bigcap_{t=\lfloor \log(T)\rfloor}^T
                F_t
            }
            \\
            & =
            \lim_{T\to\infty}
            \sum_{t \ge T}
            \sum_{z \in \pairs^-(M)}
            \Pr\parens*{
                \Reg(\log(T), T) > C_1 + C_2 \log(T)
                \text{~and~}
                \bigcap_{t=\lfloor \log(T)\rfloor}^T
                F_t
            }
            \\
            & \overset{(\ddagger)}\le 
            SA \lim_{T \to \infty} \tfrac 1T 
            = 0.
        \end{align*}
        In the above, $(\dagger)$ follows by $\Pr(\limsup F_t) = 1$ and $(\ddagger)$ by \eqref{equation_asymptotic_4}.
        Up to assuming $t$ large enough, we eventually have $C_2' \log(T) \ge C_1'$ hence the constant term can be ignored.
    \end{proof}

    \noindent
    \STEP{3}
    \textit{
        There exists $C > 0$ such that:
        \begin{equation}
        \label{equation_asymptotic_5}
            \Pr \parens*{
                \exists T, \forall t \ge T,
                \forall z \notin \pairs^{**}(M):
                N_z(t) \le C \log(t)
            }
            .
        \end{equation}
    }
    \begin{proof}
        Because $M$ is non-degenerate, $\pairs^*(M)$ defines a unique policy that we denote $\policy^*$, given by $\policy(\state) = \action$ where $\action \in \actions(\state)$ is the unique action such that $(\state, \action) \in \weakoptimalpairs(\model)$. 

        Introduce the reward function $f(z) := \indicator{z \notin \pairs^{**}(M)}$. 
        Let $g^f, h^f$ and $\Delta^f$ be the gain, bias and gap functions of $\pi^*$ in $M$ endowed with the reward function $f$.
        Remark that $g^f(s) = 0$, that $h^f(s) = 0$ for $(s, \pi^*(s)) \in \pairs^{**}(M)$ and that $\Delta^f(z) = 0$ for $z \in \pairs^*(M)$.
        Denote $H_f := \max\braces{\vecspan{h^f}, \max_z \abs{\Delta^f(z)}}$.
        Therefore:
        \begin{align*}
            \sum_{z \notin \pairs^{**}(M)} N_z(T)
            & = \sum_{t=1}^T f(Z_t)
            \\
            & = \sum_{t=1}^T \parens*{
                \parens*{e_{S_t} - p(Z_t)} h^f - \Delta_f(Z_t)
            }
            \\
            & \le 
            H^f 
            + \sum_{t=1}^T \indicator{Z_t \notin \pairs^{**}(M)} \parens*{e_{S_{t+1}} - p(Z_t)} h^f 
            + H^f \sum_{z \in \pairs^-(M)} N_z(T)
            \\
            & \overset{(\dagger)}\le
            H^f \parens*{
                1 
                + 2 \sqrt{\sum_{z \notin \pairs^{**}(M)} N_z(T) \log(T)}
                + \sum_{z \in \pairs^-(M)} N_z(T)
            }
            \\
            & \overset{(\ddagger)} \le
            H^f \parens*{
                1 
                + 2 \sqrt{\sum_{z \notin \pairs^{**}(M)} N_z(T) \log(T)}
                + S A C \log(T)
            }
        \end{align*}
        where $(\dagger)$ holds with probability $1 - T^{-2}$ by Azuma-Hoeffding's inequality (see \cite[Lemma~5]{bourel_tightening_2020}), and $(\dagger)$ holds on the asymptotically almost sure event $(\forall z \in \pairs^-(M), N_z(T) \le C \log(T))$ (see \eqref{equation_asymptotic_4_5}).
        This is an equation of the form $n \le \alpha + \beta \sqrt{n}$ that implies in particular $n \le 2(\alpha + \beta^2)$.
        In the end, we get:
        \begin{equation*}
            \Pr \parens*{
                \forall T,
                \exists t \ge T:
                \sum_{z \notin \pairs^{**}(M)} N_z(t)
                \le 
                2H^f(1 + S A C \log(T) + 4 \log(T))
            }
            = 1.
        \end{equation*}
        This concludes the proof.
    \end{proof}

    \noindent
    \STEP{4}
    \textit{
        There exists $c > 0$ such that:
        \begin{equation}
        \label{equation_asymptotic_6}
            \Pr \parens*{
                \exists T, \forall t \ge T, \forall z \in \pairs^{**}(M):
                N_z(t) \ge c t
            }
            = 1.
        \end{equation}
    }
    \begin{proof}
        This is established with a similar technique than \eqref{equation_asymptotic_5} in \STEP{3}.
        By non-degeneracy of $M$, $\pairs^*(M)$ defines a unique policy that we denote $\pi^*$.
        Fix $z_0 \in \pairs^{**}(M)$ and introduce the reward function $f(z) = \indicator{z = z_0}$.
        Remark that $g^f(s) = c > 0$ for all $s \in \states$ and that $\Delta^f(z) = 0$ for all $z \in \pairs^*(M)$.
        Let $H^f := \vecspan{h^f} \vee \max_z \abs{\Delta^f(z)}$.
        We have:
        \begin{align*}
            N_{z_0}(T)
            & := 
            \sum_{t=1}^T f(Z_t)
            \\
            & =
            cT + \sum_{t=1}^T \parens*{
                \parens*{e_{S_t} - p(Z_t)} h^f
                - \Delta_f (Z_t)
            }
            \\
            & \ge 
            cT 
            - \sum_{t=1}^T \indicator{Z_t \in \pairs^-(M)} \parens*{e_{S_{t+1}} - p(Z_t)} h^f
            - H^f \sum_{z \in \pairs^-(M)} N_z(T)
            \\
            & \ge 
            cT - 2 \sqrt{H^fSAC} \cdot \log(T)
            - H^f S A C \log(T)
            \sim cT
        \end{align*}
        where the last inequality holds with probability $1 - T^{-2}$ on the asymptotically almost sure event $(\forall z \in \pairs^-(M): N_z(T) \le C \log(T))$ given by \eqref{equation_asymptotic_4_5}.
        We conclude accordingly.
    \end{proof}

    \paragraph{About \Cref{assumption_interior}}
    In the coherence property, the first statement, which is about the reachability of sub-sampled pairs, is not guaranteed to hold if we run \texttt{KLUCRL} on an arbitrary model.
    The issue lies in the fact that the high optimistic gain of a policy may be due states that are unreachable under the optimistically optimal policy.
    This is because in the confidence region $\models(t)$, there may be models with a richer transition structure than the true hidden model $M$. 
    This is where \Cref{assumption_interior} seems necessary.
    \Cref{assumption_interior} is roughly equivalent to stating that the support of the transitions of $M$ are known in advance. 
    We conjecture that this assumption cannot be removed without a significant rework of \texttt{EVI}.
    Under \Cref{assumption_interior}, the optimistic gain of a policy $\pi$ from a fixed state $s$ only depends on $\rewards_z(t), \kernels_z(t)$ for pairs $z$ that are reachable from $s$ under $\pi$ on $M$.
    This echoes the reachability requirement of sub-sampled pairs.

    \subsection{The shrinking effect: Formal version of \texorpdfstring{\Cref{lemma_informal_shrinking_shaking}}{Informal Property 9}}
    \label{appendix_shrinking}

    In this section, we provide a proof of a formalized version of the \strong{shrinking effect} part of \Cref{lemma_informal_shrinking_shaking}.

    \begin{figure}[ht]
        \centering
        \small
        \resizebox{\linewidth}{!}{\begin{tikzpicture}
            \tikzstyle{lemma}=[text width=3.3cm, align=center, draw, rounded corners, fill=white]

            \node[lemma, text width=2cm] (0) at (0, 0) {\texttt{EVI}-based algorithm};

            \node[lemma,dashed] (21) at (8, -1.5) {Linear visits $N_z(T) = \Omega(T)$ on $\pairs^{**}(M)$};
            \node[lemma] (31) at (12, -1.5) {Shaking effect (\Cref{appendix_shaking})};

            \node[lemma] (10) at (4, 0) {Asymptotic regime (\Cref{appendix_asymptotical})};
            \node[lemma,dashed] (20) at (8, 1.5) {Logarithmic visits $N_z(T) = \OH(\log(T))$ outside $\pairs^{**}(M)$};
            \node[lemma, fill=yellow!50] (30) at (12, 1.5) {Shrinking effect (\Cref{appendix_shrinking})};

            \draw[color=black, fill=white] (14.5, 0) circle(0.25cm);
            \node (star) at (14.5, 0) {\Large $*$};
            \node[lemma] (40) at (17, 0) {\strong{Local} coherence (\Cref{section_establishing_coherence})};
            \node[lemma] (50) at (17, -1.75) {$\RegExp(T) = \OH(\log(T))$};

            \draw[->, >=stealth] (0) to (0, 1) to node[midway, above] {\scriptsize \Cref{lemma_coherence}} node[midway, below] {\scriptsize \strong{Global} coherence} (3.5, 1) to (3.5, 0.5);

            \draw[->, >=stealth] (4.5, 0.5) to (4.5, 2.5) to node[midway, above] {\scriptsize \Cref{lemma_asymptotic_regime}} (7.5, 2.5) to (7.5, 2.27);
            \draw[->, >=stealth] (8.5, 2.27) to (8.5, 2.5) to node[midway, above] {\scriptsize \Cref{lemma_shrinking}} (11.5, 2.5) to (11.5, 2.03);
            \draw[->, >=stealth] (12.5, 2.03) to (12.5, 2.5) to (14.5, 2.5) to (14.5, 0.25);

            \draw[->, >=stealth] (4.5, -0.5) to (4.5, -2.5) to node[midway, below] {\scriptsize \Cref{lemma_asymptotic_regime}} (7.5, -2.5) to (7.5, -2.27);
            \draw[->, >=stealth] (8.5, -2.27) to (8.5, -2.5) to node[midway, below] {\scriptsize \Cref{lemma_shaking}} (11.5, -2.5) to (11.5, -2.03);
            \draw[->, >=stealth] (12.5, -2.03) to (12.5, -2.5) to (14.5, -2.5) to (14.5, -0.25);

            \draw[->, >=stealth] (5.8, 0) to node[pos=0.88, above] {\scriptsize (\Cref{section_establishing_coherence})} (14.25, 0);
            \draw[->, >=stealth] (14.75, 0) to (40);
            \draw[->, >=stealth] (16, -0.535) to node[midway, right] {\scriptsize \Cref{lemma_coherence}} (16, -1.2);
        \end{tikzpicture}}
    \end{figure}

    In \Cref{lemma_shrinking} below, we show that if $\visits_\pair (t) = \OH(\log(t))$ and under a good event, the kernel confidence region $\kernels_\pair (t)$ remains confined in the confidence region $\kernels_\pair (t_{k(i)-1})$ at time $t_{k(i)-1}$, the beginning of the previous {exploitation} episode (when the current policy is gain optimal). 
    For rewards, the shrinking effect is shown strict by quantifying its speed.
    The shrinking speed is shown to be faster than any $(\frac 1t)^\eta$ for $\eta > 0$.
    This will be essential later, so that the shrinking effect on non-optimal pairs completely dominates the shaking effect on optimal pairs.

    \begin{lemma}
    \label{lemma_shrinking}
        Let $(t_{k(i)})$ be the enumeration of exploration episodes, and let $T \ge 1$.
        Fix $\lambda > 0$ and $\pair \in \pairs$.
        For all $\delta, \eta > 0$, we can find $\epsilon, m, C > 0$ such that:
        \begin{equation*}
        \begin{gathered}
            \Pr \parens*{
                \exists t \in \braces*{t_{k(i)}, \ldots t_{k(i)}+T}
                :
                {
                    \kernels_\pair(t) \not \subseteq \kernels_\pair(t_{k(i)-1}) 
                    \atop
                    \text{~and~}
                    F_t \text{~and~}
                    \visits_\pair(t) > \visits_\pair(t_{k(i)}) + C \log\parens*{\tfrac T\delta}
                }
            }
            \le \delta,
            \\
            \Pr \parens*{
                \exists t \in \braces{t_{k(i)}, \ldots, t_{k(i)}+T}
                :
                {
                    \sup \rewards_\pair(t) > \sup \rewards_\pair(t_{k(i)-1}) - \tfrac{\visits_\pair(t) - \visits_\pair(t_{k(i)})}{C \cdot (t_{k(i)})^\eta}
                    \atop
                    \text{~and~} F_t \text{~and~}
                    N_z(t) > N_z(t_{k(i)}) + C \log\parens*{\tfrac T\delta}
                }
            }
            \le \delta
        \end{gathered}
        \end{equation*}
        with $F_{t} := (\visits_\pair(t) < \frac 1\lambda \log(t), \KL(\hat{\kerrew}_t(\pair)||\kerrew(\pair)) < \epsilon, t > m)$ where $\kerrew(\pair) \equiv (\reward(\pair), \kernel(\pair))$. 
    \end{lemma}

    \subsubsection{A ``large'' shrinking effect for kernels}

    We beginning with a proof of the shrinking effect for the confidence regions of kernels. 
    The shrinking is shown large, in the sense that we show a property of the form ``$\kernels_\pair(t) \subseteq \kernels_\pair (t_{k(i)-1})$'' but do not quantify how smaller than $\kernels_\pair (t_{k(i)-1})$ the region $\kernel_\pair(t)$ is subjected to be. 

    \begin{lemma}[Shrinking effect, kernels]
        \label{lemma_shrinking_kernels}
        Let $(t_{k(i)})$ be the enumeration of exploration episodes, and let $T \ge 1$.
        Fix $\lambda > 0$ and $\pair \in \pairs$.
        For all $\delta > 0$, we can find $\epsilon, m, C > 0$ such that:
        \begin{equation*}
            \Pr \parens*{
                \exists t \in \braces*{t_{k(i)}, \ldots t_{k(i)}+T}
                :
                {
                    \kernels_\pair(t) \not \subseteq \kernels_\pair(t_{k(i)-1}) 
                    \atop
                    \text{~and~}
                    F_{t_{k(i)}-1} \cap F_t \text{~and~}
                    \visits_\pair(t) > \visits_\pair(t_{k(i)}) + C \log\parens*{\tfrac T\delta}
                }
            }
            \le \delta
        \end{equation*}
        with $F_{t} := (\visits_\pair(t) < \frac 1\lambda \log(t), \KL(\hat{\kerrew}_t(\pair)||\kerrew(\pair)) < \epsilon, t > m)$ where $\kerrew(\pair) \equiv (\reward(\pair), \kernel(\pair))$. 
        
    \end{lemma}

    \begin{proof}
        We write $\visits_\pair(t,t') := \visits_\pair(t') - \visits_\pair(t)$ the number of times $\pair \in \pairs$ is visited between the times $t$ and $t'$.
        We write $w_{t_{k(i)-1}, t}(\pair) := \hat{\kernel}_t(\pair) - \hat{\kernel}_{t_{k(i)-1}}(\pair)$ the change of kernel from time $t_{k(i)-1}$ to $t$ for the pair $\pair \in \pairs$. 
        Fix $\epsilon, \lambda, m > 0$, $\pair \in \pairs$ and introduce the event:
        \begin{equation}
            F_{t} 
            \equiv 
            F_t^{(\epsilon, \lambda, m)}
            := 
            \parens*{
                \visits_\pair(t) < \frac 1\lambda \log(t), 
                \KL(\hat{\kerrew}_t(\pair)||\kerrew(\pair)) < \epsilon, 
                t > m
            }
        \end{equation}

        \par
        \bigskip
        \noindent
        \STEP{1}
        \textit{
            There exists a function $\lambda \mapsto m_\lambda \in \N$ such that, for $m \ge m_\lambda$, we have:
            \begin{equation*}
                \Pr \parens*{
                    \exists t \in [t_{k(i)}, t_{k(i)}+T]
                    :
                    F_{t_{k(i)-1}}, 
                    \norm{w_{t_{k(i)-1}, t}(\pair)}_1
                    >
                    \tfrac { 
                        2 
                        + \epsilon^2 \visits_\pair(t_{k(i)}, t)
                        + \sqrt{\abs{\states} \visits_\pair(t_{k(i)},t) \log\parens*{\frac T\delta}}
                    }{
                        \visits_\pair(t)
                    }
                }
                \le
                \delta
            \end{equation*}
        }
        \begin{proof}
            With straight-forward algebra, we check that $w_{t_{k(i)-1}, t}(\pair)$ is equal to
            \begin{equation}
            \label{equation_likelihood_noise}
                \frac 1{\visits_\pair(t)}
                \parens*{
                    \visits_\pair(t_{k(i)-1}, t) \parens*{\kernel(\pair) - \hat{p}_{t_{k(i)-1}}(\pair)}
                    +
                    \sum_{i=t_{k(i)-1}}^{t-1} \indicator{\Pair_i = \pair} \parens*{e_{\State_{i+1}}-\kernel(\pair)}
                }
                .
            \end{equation}
            On the $F_{t_{k(i)-1}}$, we know that $\KL(\hat{\kernel}_{t_{k(i)-1}}(\pair)||\kernel(\pair)) < \epsilon$, so by Pinsker's inequality, follows $\norm{\kernel(\pair) - \hat{\kernel}_{t_{k(i)-1}}(\pair)}_1 \le \epsilon^2$. 
            So $\norm{w_{t_{k(i)-1}, t}(\pair)}_1 \le \tfrac 1{\visits_\pair(t)} (\visits_\pair(t_{k(i)-1}, t) \epsilon^2 + \norm{\sum_i \indicator{\Pair_i = \pair} (e_{S_{i+1}} - \kernel(\pair))}_1)$, consisting in two terms. 
            The first term is an error a priori, while the second is the norm of a martingale which is the sum of $\visits_\pair(t_{k(i)-1}, t)$ terms.
            On $F_{t_{k(i)-1}}$, we have:
            \begin{align*}
                \visits_\pair(t_{k(i)}) 
                & \le 
                \floor{\parens*{1 + f(t_{k(i)-1})}\visits_\pair(t_{k(i)-1}))} + 1
                \\
                & \overset{(\dagger)}\le
                \visits_\pair(t_{k(i)-1}) + 1 
                + \floor*{\tfrac 1\lambda f(t_{k(i)-1}) \log(t_{k(i)-1})}
                \\
                & \overset{(\ddagger)}=
                \visits_{\pair}(t_{k(i)-1}) + 1
            \end{align*}
            where $(\dagger)$ is by definition on $F_{t_{k(i)-1}}$ and $(\ddagger)$ holds for $t \to \infty$ since $f(t) = \oh(\log(t)^{-1})$, hence provided that $t_{k(i)-1} \ge \frac 12 t_{k(i)} \ge \frac 12 m$ is large enough with respect to $\lambda$, e.g., $m \ge m_\lambda \in \N$. 
            Accordingly, we have $\visits_\pair(t_{k(i)-1}, t_{k(i)}) \le 1$ on $F_{t_{k(i)-1}}$.
            So, on $F_{t_{k(i)-1}}$, we have:
            \begin{equation*}
                \norm*{w_\pair(t_{k(i)-1}, t)}_1
                \le 
                \frac 1{\visits_\pair(t)}
                \parens*{
                    2
                    + \visits_\pair(t_{k(i)}, t) \epsilon^2
                    +
                    \norm*{
                        \sum_{i=t_{k(i)}}^{t-1} \indicator{\Pair_i = \pair} \parens*{e_{\State_{i+1}}-\kernel(\pair)}
                    }_1
                }
                .
            \end{equation*}
            Applying Weissman's inequality (see \cite{weissman_inequalities_2003} or \cite[Equation~(44)]{auer_near_optimal_2009} or \Cref{lemma_weissman_maximal}), the martingale can then be bounded as follows:
            \begin{equation*}
            \begin{aligned}
                & 
                \Pr \parens*{
                    \exists t \in [t_{k(i)}, t_{k(i)}+T],
                    ~
                    \norm*{
                        \sum_{i=t_{k(i)}}^{t-1}
                        \indicator{\Pair_i = \pair} \parens*{e_{\State_{i+1}} - \kernel(\pair)}
                    }_1
                    \ge
                    \tsqrt{
                        \abs{\states} \visits_\pair(t_{k(i)}, t) \log\parens*{\tfrac T\delta}
                    }
                }
                \\
                & \le 
                \delta
                .
            \end{aligned}
            \end{equation*}
            We conclude accordingly.
        \end{proof}

        \noindent
        \STEP{2}
        \textit{
            Assume that $\epsilon < (\frac{\abs{\states}\log(T)}T)^{1/4}$ and $m \ge m_\lambda$.
            Then, for all $\delta > 0$, we have:
            \begin{equation*}
                \Pr \parens*{
                    \exists t \in [t_{k(i)}, t_{k(i)}+T]
                    :
                    F_{t_{k(i)-1}}, 
                    \norm{w_{t_{k(i)-1}, t}(\pair)}_1
                    >
                    \tfrac { 
                        2 \parens*{
                            1 
                            + 
                            \sqrt{\abs{\states} \visits_\pair(t_{k(i)},t) \log\parens*{\frac T\delta}}
                        }
                    }{
                        \visits_\pair(t)
                    }
                }
                \le
                \delta
            \end{equation*}
        }
        \begin{proof}
            We know that for $t \in \braces{t_{k(i}, \ldots, t_{k(i)}+T}$, we have $\visits_\pair(t_{k(i)}, t) \le T$. 
            Solve $\epsilon^2 T < \sqrt{\abs{\states} T \log(T)}$ in $\epsilon$ and invoke \STEP{1}.
        \end{proof}

        \noindent
        \STEP{3}
        \textit{
            There exists $\epsilon_\pair > 0$ such that, for all $\kernel'(\pair)$ satisfying $\KL(\kernel'(\pair)||\kernel(\pair)) < \epsilon_\pair$, we have $\mathrm{supp}(\kernel') \supseteq \mathrm{supp}(\kernel)$ and $\kernel'(\state|\pair) \ge \frac 12 \kernel(\state|\pair)$ for all $\state \in \states$. 
        }
        \medskip
        \begin{proof}
            Denote $x := \KL(\kernel'(\pair)||\kernel(\pair))$.
            By Pinkser's inequality, we have $\norm{\kernel'(\pair) - \kernel(\pair)}_1 \le \sqrt{2x}$, so
            \begin{equation*}
                \forall \state \in \states,
                \quad
                \abs{\kernel'(\state|\pair) - \kernel(\state|\pair)} \le \sqrt{2x}
                .
            \end{equation*}
            Assume that $\sqrt{2x} \le \frac 12 \min_{\state \in \mathrm{supp}(\kernel(\pair))} \kernel(\state|\pair)$. 
            Then $\kernel'(\state|\pair) \ge \frac 12 \kernel(\state|\pair)$ for all $\state \in \states$ and in particular, $\kernel'(\pair) \gg \kernel(\pair)$. 
            Hence the result. 
        \end{proof}

        \noindent
        \STEP{4}
        \textit{
            For $\epsilon < (\frac{\abs{\states} \log(T)}T)^{1/4}$ and for $m \ge m_\lambda$, for all $\delta > 0$ and $m \ge t_\delta \in \N$, we have:
            \begin{equation*}
                \Pr \parens*{
                    \exists t \in \braces{t_{k(i)}, \ldots, t_{k(i)}+T}
                    :
                    {
                        \displaystyle
                        \visits_\pair(t_{k(i)}, t) 
                        \ge
                        \frac{T \lambda^2 \abs{\states} \log\parens*{\tfrac T\delta}}{c^2}
                        +
                        4 \log \parens*{\frac ec}^2
                        \atop
                        \text{and~}
                        F_t
                        \text{~and~}
                        F_{t_{k(i)-1}}
                        \text{~and~}
                        \kernels_\pair(t) \not \subseteq \kernels_\pair(t_{k(i)-1})
                    }
                }
                \le \delta
                .
            \end{equation*}
        }
        \begin{proof}
            Let $\tilde{\kernel}(\pair) \in \kernels_\pair(t)$. 
            We derive conditions on $\visits_\pair (t_{k(i)}, t)$ such that $\tilde{\kernel}(\pair) \in \kernels_\pair(t_{k(i)-1})$ with high probability, by looking at when $\visits_{\pair}(t_{k(i)-1}) \KL(\hat{\kernel}_{t_{k(i)-1}}(\pair)||\tilde{\kernel}(\pair)) \le \alpha \log (\beta t_{k(i)-1})$ where $\alpha = \abs{\states}$ and $\beta = 2e$.
            Let $\states(\pair) := \kernel(\pair)$, which is the same as the support of $\hat{\kernel}_t(\pair)$ on $F_t$ by \STEP{3}.
            We have:
            \begin{align}
                & \notag
                \visits_\pair (t_{k(i)-1}) 
                \KL(\hat{\kernel}_{t_{k(i)-1}}(\pair)||\tilde{\kernel}(\pair))
                \\
                & = \notag
                \visits_\pair (t_{k(i)-1})
                \KL(
                    \hat{\kernel}_{t}(\pair) - w_{t_{k(i)-1}, t}(\pair)
                    ||
                    \tilde{\kernel}(\pair)
                )
                \\
                & = \notag
                \visits_\pair (t_{k(i)-1})
                \sum_{\state \in \states(\pair)}
                \parens*{
                    \hat{\kernel}_t(\state|\pair) - w_{t_{k(i)-1}, t}(\state|\pair)
                } \log \parens*{
                    \frac{
                        \hat{\kernel}_t(\state|\pair) - w_{t_{k(i)-1}, t}(\state|\pair)
                    }{
                        \tilde{\kernel}(\state|\pair)
                    }
                }
                \\
                & = \notag
                \visits_\pair (t_{k(i)-1})
                \sum_{\state \in \states(\pair)}
                \parens*{
                    \hat{\kernel}_t(\state|\pair) - w_{t_{k(i)-1}, t}(\state|\pair)
                } \parens*{
                    \log \parens*{
                        \frac{
                            \hat{\kernel}_t(\state|\pair)
                        }{
                            \tilde{\kernel}(\state|\pair)
                        }
                    }
                    + \log \parens*{
                        1-
                        \frac{
                            w_{t_{k(i)-1}, t}(\state|\pair)
                        }{
                            \hat{\kernel}_t(\state|\pair)
                        }
                    }
                }
                \\
                & = \notag
                \visits_\pair (t_{k(i)-1}) \parens*{
                    \KL(\hat{\kernel}_{t}(\pair)||\tilde{\kernel}(\pair))
                    - 
                    \sum_{\state \in \states(\pair)}
                    w_{t_{k(i)-1},t}(\state|\pair) \parens*{
                        \log \parens*{\hat{\kernel}_t(\state|\pair)}
                        +
                        \log \parens*{
                            \frac{
                                1
                            }{
                                \tilde{\kernel}(\state|\pair)
                            }
                        }
                    }
                }
                \\
                & ~~ + 
                \visits_{\pair} (t_{k(i)-1}) 
                \sum_{\state \in \states(\pair)}
                \parens*{
                    \hat{\kernel}_t(\state|\pair) - w_{t_{k(i)-1}, t}(\state|\pair)
                } \log \parens*{
                    1-
                    \frac{
                        w_{t_{k(i)-1}, t}(\state|\pair)
                    }{
                        \hat{\kernel}_t(\state|\pair)
                    }
                }
                .
            \label{equation_likelihood_noise_1}
            \end{align}
            Let $c := 2 \min_{\state \in \states(\pair)} \kernel(\state|\pair)$.
            By \STEP{3}, $\min_{\state \in \states(\pair)} \hat{\kernel}_t (\state|\pair) \ge c$ on $F_{t}$.
            So $\abs{\log(\hat{\kernel}_t(\state|\pair))} \le \log (\frac 1c)$ for all $\state \in \states(\pair)$. 

            Furthermore, as $\tilde{\kernel}(\pair) \in \kernels_\pair (t)$, we have $\visits_\pair (t) \KL(\hat{\kernel}_t(\pair)||\tilde{\kernel}(\pair)) \le \alpha \log(\beta t)$ where $\alpha = \abs{\states}$ and $\beta = 2e$ by construction of $\kernels_\pair(t)$, see \eqref{equation_confidence_region}.
            Writing $\entropy(\hat{\kernel}_t(\pair)) := - \sum_{\state} \hat{\kernel}_t(\state|\pair) \log(\hat{\kernel}_t(\state|\pair))$ the Shannon entropy of $\hat{\kernel}_t(\pair)$, we have
            \begin{align*}
                \alpha \log(\beta t)
                & \ge
                \visits_\pair (t) 
                \sum_{\state \in \states(\pair)}
                \hat{\kernel}_t(\state|\pair) 
                \log \parens*{
                    \frac{\hat{\kernel}_t(\state|\pair)}{\tilde{\kernel}(\state|\pair)}
                }
                \\
                & \ge
                \visits_\pair (t) 
                \parens*{
                    \sum_{\state \in \states(\pair)}
                    \hat{\kernel}_t(\state|\pair)
                    \log \parens*{\frac 1{\tilde{\kernel}(\state|\pair)}}
                    -
                    \entropy(\hat{\kernel}_t(\pair))
                }
                \\
                & \ge
                \visits_\pair (t) 
                \parens*{
                    \sum_{\state \in \states(\pair)}
                    \hat{\kernel}_t(\state|\pair)
                    \log \parens*{\frac 1{\tilde{\kernel}(\state|\pair)}}
                    -
                    \log \abs{\states}
                }
            \end{align*}
            so we find that $\log\parens{\frac 1{\tilde{\kernel}(\state|\pair)}} \le \frac 1{c \visits_\pair(t)} \parens{\alpha \log(\beta t) + \log \abs{\states}} \le \frac {\alpha \log(\beta' t)}{c\visits_\pair(t)}$ for some $\beta' > 0$.
            Using this to continue the computations from \eqref{equation_likelihood_noise_1} and further using $\log(1+x) \le x$, we have:
            \begin{align*}
                &
                \visits_\pair (t_{k(i)-1}) 
                \KL(\hat{\kernel}_{t_{k(i)-1}}(\pair)||\tilde{\kernel}(\pair))
                \\
                & \le 
                \visits_\pair (t_{k(i)-1}) \parens*{
                    \KL(\hat{\kernel}_{t}(\pair) || \tilde{\kernel}(\pair))
                    +
                    \norm{w_{t_{k(i)-1}, t}(\pair)}_1 \parens*{
                        \frac{
                            \alpha \log(\beta't)
                        }{c \visits_\pair (t)}
                        +
                        \log \parens*{\frac 1c}
                    }
                }
                \\
                & \quad -
                \visits_{\pair} (t_{k(i)-1}) 
                \sum_{\state \in \states(\pair)}
                \parens*{
                    \hat{\kernel}_t(\state|\pair) - w_{t_{k(i)-1}, t}(\state|\pair)
                }
                \frac
                { w_{t_{k(i)-1}, t}(\state|\pair) }
                { \hat{\kernel}_t(\state|\pair) }
                .
                \\
                & \le 
                \visits_\pair (t_{k(i)-1}) \parens*{
                    \KL(\hat{\kernel}_{t}(\pair) || \tilde{\kernel}(\pair))
                    +
                    \norm{w_{t_{k(i)-1}, t}(\pair)}_1 \parens*{
                        \frac{
                            \alpha \log(\beta't)
                        }{c \visits_\pair (t)}
                        +
                        \log \parens*{\frac 1c}
                        +
                        1
                    }
                }
                \\
                & \quad
                +
                \visits_\pair (t_{k(i)-1})
                \frac{\norm{w_{t_{k(i)-1},t}(\pair)}_2^2}{c}
                \\
                & \overset{(\dagger)}\le
                \alpha \log(\beta t)
                - \tfrac{\visits_\pair(t_{k(i)-1}, t) \alpha \log(\beta t)}{\visits_\pair (t)}
                + \tfrac{\sqrt{\abs{\states} \visits_\pair(t_{k(i)}, t) \log\parens*{T/\delta}}}{\visits_\pair(t)}
                \parens*{
                    \tfrac{
                        \alpha \log(\beta't)
                    }{c \visits_\pair (t)}
                    +
                    \log \parens*{\tfrac 1c}
                    +
                    1
                }
                \\
                & \quad + \OH \parens*{
                    \tfrac 1{\visits_\pair (t)}
                    \parens*{
                        \tfrac{
                            \alpha \log(\beta't)
                        }{c \visits_\pair (t)}
                        +
                        \log \parens*{\tfrac 1c}
                        +
                        1
                    }
                    +
                    \tfrac{1 + \abs{\states} \visits_\pair(t_{k(i)}, t) \log(T/\delta)}{c \visits_\pair(t)}
                }
                \\
                & \le 
                \alpha \log(\beta t)
                +
                \tfrac{\alpha \log(\beta t)}{\visits_\pair(t)} \parens*{
                    - \visits_\pair(t_{k(i)}, t)
                    + \tfrac{\log(\beta t)}{\log(\beta' t)}
                    \parens*{
                        \tfrac{\log(\beta' t)}{c \visits_\pair(t)}
                        \tsqrt{\abs{\states} \log\parens*{\tfrac T\delta}}
                        + \log\parens*{\tfrac ec}
                    }
                    \tsqrt{\visits_\pair(t_{k(i)}, t)}
                }
                \\
                & \quad
                + \OH \parens*{
                    \tfrac{T \log(T/\delta)}{\visits_\pair (t)}
                }
                \\
                & \overset{(\ddagger)}\le
                \alpha \log(\beta t_{k(i)-1})
                +
                \tfrac{\alpha \log(\beta t)}{\visits_\pair(t)} \parens*{
                    - \visits_\pair(t_{k(i)}, t)
                    + 2 \parens*{
                        \tfrac \lambda c
                        \tsqrt{\abs{\states} \log\parens*{\tfrac T\delta}}
                        + \log\parens*{\tfrac ec}
                    }
                    \tsqrt{\visits_\pair(t_{k(i)}, t)}
                }
                \\
                & \quad
                + \OH \parens*{
                    \tfrac{\visits_\pair(t_{k(i)}, t) \log(T/\delta)}{\visits_\pair (t)}
                }
            \end{align*}
            where $(\dagger)$ follows from \STEP{2} and holds with probability $1 - \delta$ on $F_{t_{k(i)-1}}$, 
            and $(\ddagger)$ follows by using that (1) $t_{k(i)-1} \le 3 t$ if $t$ is large enough, (2) that $\visits_\pair (t) < \frac 1\lambda \log(t)$ on $F_t$ and (3) that $\log(\beta't) / \log(\beta t) \le 2$ for $t$ large enough. 
            We want the RHS to be smaller than $\alpha \log(\beta t_{k(i)-1})$.
            For large $t$, we can neglect the second order term in $\visits_\pair(t_{k(i)}, t) \log(T/\delta)/\visits_\pair(t)$ when $t \gg T/\delta$, because $\log(\beta t) \gg \log(T/\delta)$.
            This leads to a condition of the form:
            \begin{equation*}
                \visits_\pair(t_{k(i)}, t)
                \ge
                2 \parens*{
                    \tfrac \lambda c
                    \tsqrt{\abs{\states} \log\parens*{\tfrac T\delta}}
                    + \log\parens*{\tfrac ec}
                }
                \tsqrt{\visits_\pair(t_{k(i)}, t)}
            \end{equation*}
            that leads immediately to the claimed result by using $(a + b)^2 \le 2 a^2 + 2 b^2$. 
        \end{proof}
        \STEP{4} concludes the proof. 
    \end{proof}

    \subsubsection{A ``strict'' shrinking effect for rewards}

    We continue with the shrinking effect for rewards.
    The proof is essentially similar to the shrinking effect for kernels (\Cref{lemma_shrinking_kernels}) but the result is more precise, because we quantify the speed of the shrinking phenomenon.
    Therefore, the proof requires an extra step. 

    \begin{lemma}[Shrinking effect, rewards]
    \label{lemma_shrinking_rewards}
        Let $(t_{k(i)})$ be the enumeration of exploration episodes, and let $T \ge 1$.
        Fix $\lambda > 0$ and $\pair \in \pairs$.
        For all $\delta, \eta > 0$, we can find $\epsilon, m, C > 0$ such that:
        \begin{equation*}
        \begin{gathered}
            \Pr \parens*{
                \exists t \in \braces{t_{k(i)}, \ldots, t_{k(i)}+T}
                :
                {
                    \max \rewards_\pair(t) > \max \rewards_\pair(t_{k(i)-1}) - \tfrac{\visits_\pair(t) - \visits_\pair(t_{k(i)})}{C \cdot (t_{k(i)})^\eta}
                    \atop
                    \text{~and~} F_{t_{k(i)-1}} \cap F_t \text{~and~}
                    N_z(t) > N_z(t_{k(i)}) + C \log\parens*{\tfrac T\delta}
                }
            }
            \le \delta
        \end{gathered}
        \end{equation*}
        with $F_{t} := (\visits_\pair(t) < \frac 1\lambda \log(t), \KL(\hat{\kerrew}_t(\pair)||\kerrew(\pair)) < \epsilon, \visits_\pair(t) > m)$ where $\kerrew(\pair) \equiv (\reward(\pair), \kernel(\pair))$. 
    \end{lemma}

    \begin{proof}
        The proof is essentially similar to \Cref{lemma_shrinking_kernels}.
        For rewards however, \Cref{lemma_shrinking_rewards} quantifies the shrinking speed, hence we need to refine what is being said at the end of the proof of \Cref{lemma_shrinking_kernels}. 
        Following \STEP{4} of the previous proof, for
        \begin{equation*}
            \frac 12 \visits_\pair(t_{k(i)}, t)
            \ge
            2 \parens*{
                \tfrac \lambda c
                \tsqrt{\abs{\states} \log\parens*{\tfrac T\delta}}
                + \log\parens*{\tfrac ec}
            }
            \tsqrt{\visits_\pair(t_{k(i)}, t)}
            ,
        \end{equation*}
        we essentially have, on $F_t \cap F_{t_{k(i)-1}}$, that
        \begin{equation}
        \label{equation_proof_shrinking_rewards_1}
            \visits_{\pair}(t_{k(i)-1})
            \KL(\hat{\reward}_{t_{k(i)-1}}(\pair)||\tilde{\reward}(\pair))
            \le
            \parens*{
                1 - \frac{\visits_\pair(t_{k(i)}, t)}{2 \visits_\pair (t)}
            } \alpha \log(\beta t)
        \end{equation}
        for all $\tilde{\reward}(\pair) \in \rewards_t (\pair)$. 
        Introduce the optimistic rewards $\reward^+_t(\pair) := \max \rewards_\pair (t)$ and $\reward^+_{t_{k(i)-1}}(\pair) := \max \rewards_\pair (t_{k(i)-1})$, and let $\omega^+_{t_{k(i)-1}, t} (\pair) := \reward^+_t(\pair) - \reward^+_{t_{k(i)-1}}(\pair)$ be their difference. 
        Following \eqref{equation_proof_shrinking_rewards_1}, we have
        \begin{equation*}
            \KL(\hat{\reward}_{t_{k(i)-1}}(\pair)||\reward^+_{t}(\pair))
            \le
            \parens*{
                1 - \frac{\visits_\pair(t_{k(i)}, t)}{2 \visits_\pair (t)}
            } 
            \cdot \KL(\hat{\reward}_{t_{k(i)-1}}(\pair)||\reward^+_{t_{k(i)-1}}(\pair))
            .
        \end{equation*}
        Approximating $\KL(\hat{\reward}_{t_{k(i)-1}}(\pair)||\reward^+_{t_{k(i)-1}}(\pair) + w^+_{t_{k(i)-1}, t}(\pair))$ by its Taylor expansion at first order, we find:
        \begin{equation*}
            \KL(\hat{\reward}_{t_{k(i)-1}}(\pair)||\reward^+_t(\pair))
            \approx
            \KL(\hat{\reward}_{t_{k(i)-1}}(\pair)||\reward^+_{t_{k(i)-1}}(\pair))
            +
            \frac{\reward^+_{t_{k(i)-1}}(\pair) - \hat{\reward}_{t_{k(i)-1}}(\pair)}{\reward^+_{t_{k(i)-1}}(\pair)(1 - \reward^+_{t_{k(i)-1}}(\pair))}
            w^+_{t_{k(i)-1}, t}(\pair)
        \end{equation*}
        so that, at first order, we obtain the equation:
        \begin{equation*}
            \frac{\reward^+_{t_{k(i)-1}}(\pair) - \hat{\reward}_{t_{k(i)-1}}(\pair)}{\reward^+_{t_{k(i)-1}}(\pair)(1 - \reward^+_{t_{k(i)-1}}(\pair))}
            w^+_{t_{k(i)-1}, t}(\pair)
            \approx
            - \frac{\visits_{t_{k(i)-1}, t}(\pair) \alpha \log(\beta t_{k(i)-1})}{\visits_\pair (t_{k(i)-1})^2}
        \end{equation*}
        and solving in $w^+_{t_{k(i)-1}, t}(\pair)$ provides:
        \begin{equation}
        \label{equation_proof_shrinking_rewards_2}
            w^+_{t_{k(i)-1}, t}(\pair)
            \approx
            - \frac{\visits_{t_{k(i)-1}, t}(\pair) \alpha \log(\beta t_{k(i)-1})}{\visits_\pair (t_{k(i)-1})^2}
            \cdot 
            \frac{\reward^+_{t_{k(i)-1}}(\pair)(1 - \reward^+_{t_{k(i)-1}}(\pair))}{\reward^+_{t_{k(i)-1}}(\pair) - \hat{\reward}_{t_{k(i)-1}}(\pair)}
            .
        \end{equation}
        The question is how close to the boundary ${\reward^+_{t_{k(i)-1}}(\pair)(1 - \reward^+_{t_{k(i)-1}}(\pair))}$ can be. 
        Thanks to \STEP{3} of the proof of \Cref{lemma_shrinking_kernels}, on $F_{t_{k(i)-1}}$, $\hat{\reward}_{t_{k(i)-1}}(\pair)$ and $\reward(\pair)$ have the same support with $\hat{\reward}_{t_{k(i)-1}}(\pair) \le 2 \reward(\pair) - 1 < 1$. 
        By writing $\KL(x||y) = - \entropy(x) + x \log (\frac 1y) + (1 - x) \log(\frac 1{1-y})$, the inequality $\visits_\pair(t_{k(i)-1}) \KL(\hat{\reward}_{t_{k(i)-1}}(\pair)||\reward^+_{t_{k(i)-1}}(\pair)) = \alpha \log (\beta t_{k(i)-1})$ leads to:
        \begin{align}
            1 - \reward^+_{t_{k(i)-1}}(\pair)
            & \ge
            \notag
            \exp \parens*{
                - \frac{\frac{\alpha \log(\beta t_{k(i)-1})}{\visits_\pair (t_{k(i)-1}) } - \entropy(\hat{\reward}_{t_{k(i)-1}}(\pair))}{1 - \hat{\reward}_{t_{k(i)-1}}(\pair)}
            }
            \\
            & \ge
            \label{equation_proof_shrinking_rewards_3}
            2^{\frac 2{1 - \reward(\pair)}}
            \parens*{
                \beta t_{k(i)-1}
            }^{
                - \frac{2 \alpha}{(1 - \reward(\pair)) \visits_\pair(t_{k(i)-1})}
            }
            = \Omega \parens*{
                (t_{k(i)-1})^{-\eta}
            }
        \end{align}
        provided that $\visits_\pair (t_{k(i)-1}) \ge \frac{2\alpha}{\eta(1 - \reward(\pair))}$. 
        To conclude, we inject \eqref{equation_proof_shrinking_rewards_3} into \eqref{equation_proof_shrinking_rewards_2} together with the fact that, on $F_{t_{k(i)-1}}$, we have $\visits_\pair (t_{k(i)-1}) < \frac 1\lambda \log(t_{k(i)-1})$, to get:
        \begin{equation*}
            w^+_{t_{k(i)-1}, t}(\pair)
            \lesssim
            - \frac{
                \visits_\pair(t_{k(i)}, t)
            }{\visits_\pair(t_{k(i)-1})}
            \cdot \Omega \parens*{
                (t_{k(i)-1})^{-\eta}
            }
            =
            - \Omega \parens*{
                \frac{
                    \visits_\pair(t_{k(i)}, t)
                }{
                    (t_{k(i)-1})^\eta \log(t_{k(i)-1})
                }
            }
            .
        \end{equation*}
        This concludes the proof. 
    \end{proof}

    \subsection{The shaking effect: Proof of \texorpdfstring{\Cref{lemma_shaking}}{Lemma 29}}
    \label{appendix_shaking}

    In this section, we provide a proof of a formalized version of the \strong{shaking effect} part of \Cref{lemma_informal_shrinking_shaking}.

    \begin{figure}[ht]
        \centering
        \small
        \resizebox{\linewidth}{!}{\begin{tikzpicture}
            \tikzstyle{lemma}=[text width=3.3cm, align=center, draw, rounded corners, fill=white]

            \node[lemma, text width=2cm] (0) at (0, 0) {\texttt{EVI}-based algorithm};

            \node[lemma,dashed] (21) at (8, -1.5) {Linear visits $N_z(T) = \Omega(T)$ on $\pairs^{**}(M)$};
            \node[lemma, fill=yellow!50] (31) at (12, -1.5) {Shaking effect (\Cref{appendix_shaking})};

            \node[lemma] (10) at (4, 0) {Asymptotic regime (\Cref{appendix_asymptotical})};
            \node[lemma,dashed] (20) at (8, 1.5) {Logarithmic visits $N_z(T) = \OH(\log(T))$ outside $\pairs^{**}(M)$};
            \node[lemma] (30) at (12, 1.5) {Shrinking effect (\Cref{appendix_shrinking})};

            \draw[color=black, fill=white] (14.5, 0) circle(0.25cm);
            \node (star) at (14.5, 0) {\Large $*$};
            \node[lemma] (40) at (17, 0) {\strong{Local} coherence (\Cref{section_establishing_coherence})};
            \node[lemma] (50) at (17, -1.75) {$\RegExp(T) = \OH(\log(T))$};

            \draw[->, >=stealth] (0) to (0, 1) to node[midway, above] {\scriptsize \Cref{lemma_coherence}} node[midway, below] {\scriptsize \strong{Global} coherence} (3.5, 1) to (3.5, 0.5);

            \draw[->, >=stealth] (4.5, 0.5) to (4.5, 2.5) to node[midway, above] {\scriptsize \Cref{lemma_asymptotic_regime}} (7.5, 2.5) to (7.5, 2.27);
            \draw[->, >=stealth] (8.5, 2.27) to (8.5, 2.5) to node[midway, above] {\scriptsize \Cref{lemma_shrinking}} (11.5, 2.5) to (11.5, 2.03);
            \draw[->, >=stealth] (12.5, 2.03) to (12.5, 2.5) to (14.5, 2.5) to (14.5, 0.25);

            \draw[->, >=stealth] (4.5, -0.5) to (4.5, -2.5) to node[midway, below] {\scriptsize \Cref{lemma_asymptotic_regime}} (7.5, -2.5) to (7.5, -2.27);
            \draw[->, >=stealth] (8.5, -2.27) to (8.5, -2.5) to node[midway, below] {\scriptsize \Cref{lemma_shaking}} (11.5, -2.5) to (11.5, -2.03);
            \draw[->, >=stealth] (12.5, -2.03) to (12.5, -2.5) to (14.5, -2.5) to (14.5, -0.25);

            \draw[->, >=stealth] (5.8, 0) to node[pos=0.88, above] {\scriptsize (\Cref{section_establishing_coherence})} (14.25, 0);
            \draw[->, >=stealth] (14.75, 0) to (40);
            \draw[->, >=stealth] (16, -0.535) to node[midway, right] {\scriptsize \Cref{lemma_coherence}} (16, -1.2);
        \end{tikzpicture}}
    \end{figure}

    In \Cref{lemma_shaking} below, we show that if $\visits_\pair (t) = \Omega(t)$ and under a good event, the reward-kernel confidence region $\kerrews_\pair (t) := \rewards_\pair (t) \times \kernels_\pair (t)$ barely changes compared to its state $\kerrews_\pair (t_{k(i)-1})$ at time $t_{k(i)-1}$, the beginning of the previous exploitation episode. 
    The amount of displacement is quantified in Hausdorff distance and is shown of order $\sqrt{\log(t)/t}$.
    This will be negligible with respect to the displacements of the confidence region due to the shrinking effect, of which the order of magnitude is $\Omega((\frac 1t)^\eta)$ for all $\eta > 0$.

    \begin{lemma}
    \label{lemma_shaking}
        Let $(t_{k(i)})$ be the enumeration of exploration episodes, and let $T \ge 1$.
        Fix $\lambda, z \in \pairs$ and for two sets $\mathcal{U}, \mathcal{V} \subseteq \R^n$, denote $d_\mathrm{H}(\mathcal{U}, \mathcal{V})$ the Hausdorff distance induced by the one-norm. 
        We can find $c, m > 0$ such that:
        \begin{align*}
            \parens*{\textit{kernels}} \quad
            & F_{t_{k(i)}} \supseteq 
            \parens*{
                \forall t \in [{t_{k(i)}}, {t_{k(i)}}+T]: d_\mathrm{H}(\kernels_z(t), \kernels_z({t_{k(i)-1}})) 
                \le 
                \sqrt{\tfrac {c \log(t)}t}
            }, 
            \\
            \parens*{\textit{rewards}} \quad
            & F_{t_{k(i)}} \supseteq 
            \parens*{
                \forall t \in [{t_{k(i)}}, {t_{k(i)}}+T]: d_\mathrm{H}(\rewards_z(t), \rewards_z({t_{k(i)-1}})) 
                \le 
                \sqrt{\tfrac {c \log(t)}t}
            }
        \end{align*}
        where $F_{t_{k(i)}} := (N_z({t_{k(i)-1}}) > \lambda {t_{k(i)-1}}, {t_{k(i)}} > m) \cap (\forall t \in [t_{k(i)-1}, t_{k(i)}], \model \in \models(t))$.
    \end{lemma}

    \begin{proof}
        We provide the argument for kernels, as the argument for rewards is the same in a smaller dimension.
        By Pinsker's inequality, $\norm{\hat{\kernel}_t(\pair) - \kernel'(\pair)}_1 \le 2 \KL(\hat{\kernel}_t(\pair)||\kernel'(\pair))$, so on $F_t$ and for all $\kernel'(\pair) \in \kernels_\pair (t)$, we have $\norm{\hat{\kernel}_t(\pair) - \kernel'(\pair)}_1 \le (\lambda t)^{-1} {2 \alpha \log(\beta t)}$. 
        On $F_t$, we further have $\kernel(\pair) \in \kernels_\pair (t)$ as well, so $\norm{\hat{\kernel}_t(\pair) - \kernel(\pair)}_1 \le (\lambda t)^{-1} \cdot 2 \alpha \log(\beta t)$.
        We deduce that, on $F_t$:
        \begin{equation}
        \notag
            \kernels_\pair (t) 
            \subseteq 
            \braces*{
                \kernel'(\pair) \in \probabilities(\states)
                :
                \norm{\kernel'(\pair) - \kernel(\pair)}_1 \le 2 \sqrt{\frac{2 \alpha \log(\beta t)}{\lambda t}}
            }
            .
        \end{equation}
        The result is therefore obtained by estimating the Hausdorff distance between $\ell_1$-ball of radius $\Theta(\sqrt{\log(t)/t})$ centered at $\kernel(\pair)$.  
    \end{proof}

    \subsection{Combining everything together: Proof of \texorpdfstring{\Cref{lemma_local_coherence}}{Lemma 13}}
    \label{appendix_local_coherence}

    Combining the shrinking-shaking effect and the asymptotic visit rates of optimal and non-optimal pairs, we establish the local coherence property of \Cref{lemma_local_coherence}.

    \begin{figure}[ht]
        \centering
        \small
        \resizebox{\linewidth}{!}{\begin{tikzpicture}
            \tikzstyle{lemma}=[text width=3.3cm, align=center, draw, rounded corners, fill=white]

            \node[lemma, text width=2cm] (0) at (0, 0) {\texttt{EVI}-based algorithm};

            \node[lemma,dashed] (21) at (8, -1.5) {Linear visits $N_z(T) = \Omega(T)$ on $\pairs^{**}(M)$};
            \node[lemma] (31) at (12, -1.5) {Shaking effect (\Cref{appendix_shaking})};

            \node[lemma] (10) at (4, 0) {Asymptotic regime (\Cref{appendix_asymptotical})};
            \node[lemma,dashed] (20) at (8, 1.5) {Logarithmic visits $N_z(T) = \OH(\log(T))$ outside $\pairs^{**}(M)$};
            \node[lemma] (30) at (12, 1.5) {Shrinking effect (\Cref{appendix_shrinking})};

            \draw[color=black, fill=white] (14.5, 0) circle(0.25cm);
            \node (star) at (14.5, 0) {\Large $*$};
            \node[lemma, fill=yellow!80] (40) at (17, 0) {\strong{Local} coherence (\Cref{section_establishing_coherence})};
            \node[lemma] (50) at (17, -1.75) {$\RegExp(T) = \OH(\log(T))$};

            \draw[->, >=stealth] (0) to (0, 1) to node[midway, above] {\scriptsize \Cref{lemma_coherence}} node[midway, below] {\scriptsize \strong{Global} coherence} (3.5, 1) to (3.5, 0.5);

            \draw[->, >=stealth] (4.5, 0.5) to (4.5, 2.5) to node[midway, above] {\scriptsize \Cref{lemma_asymptotic_regime}} (7.5, 2.5) to (7.5, 2.27);
            \draw[->, >=stealth] (8.5, 2.27) to (8.5, 2.5) to node[midway, above] {\scriptsize \Cref{lemma_shrinking}} (11.5, 2.5) to (11.5, 2.03);
            \draw[->, >=stealth] (12.5, 2.03) to (12.5, 2.5) to (14.5, 2.5) to (14.5, 0.25);

            \draw[->, >=stealth] (4.5, -0.5) to (4.5, -2.5) to node[midway, below] {\scriptsize \Cref{lemma_asymptotic_regime}} (7.5, -2.5) to (7.5, -2.27);
            \draw[->, >=stealth] (8.5, -2.27) to (8.5, -2.5) to node[midway, below] {\scriptsize \Cref{lemma_shaking}} (11.5, -2.5) to (11.5, -2.03);
            \draw[->, >=stealth] (12.5, -2.03) to (12.5, -2.5) to (14.5, -2.5) to (14.5, -0.25);

            \draw[->, >=stealth] (5.8, 0) to node[pos=0.88, above] {\scriptsize (\Cref{section_establishing_coherence})} (14.25, 0);
            \draw[->, >=stealth] (14.75, 0) to (40);
            \draw[->, >=stealth] (16, -0.535) to node[midway, right] {\scriptsize \Cref{lemma_coherence}} (16, -1.2);
        \end{tikzpicture}}
    \end{figure}

    This is the last step in the proof of \Cref{theorem_main}, assertion 3.

    \par
    \medskip
    \noindent
    \textbf{\Cref{lemma_local_coherence} (Local coherence)}
    \textit{
        Let $\model \in \models^+$ be a non-degenerate explorative model.
        Consider running \texttt{KLUCRL} with model satisfying \Cref{assumption_interior} and assume that episodes are managed with the $f$-\eqref{equation_vanishing_multiplicative} with $f(t) = \oh\parens{\frac 1{\log(t)}}$.
        Let $(t_{k(i)})$ be the enumeration of exploration episodes. 
        Then, there exists a constant $C(\model) > 0$ such that, for all $T \ge 1$ and $\delta > 0$, there is an adapted sequence of events $(E_t)$ and a function $\varphi : \N \to \R$ such that:
        \begin{enumerate}
            \item 
                For all $i \ge 1$, the algorithm is $(E_t, t_{k(i)}, T, \varphi)$-coherent;
                \vspace{-.66em}
            \item 
                $\Pr\parens[\big]{\bigcup_{t=t_{k(i)}}^{t_{k(i)}+T-1} E_t^c} \le \delta + \oh(1)$ when $i \to \infty$;
                \vspace{-.66em}
            \item 
                $\varphi(t) \le 1 + C \log\parens{\frac T\delta} + \oh(1)$ when $t \to \infty$.
        \end{enumerate}
    }

    \begin{proof}
        By correctness of the confidence region, $\Pr(\exists T, \forall t \ge T: \forall \pi, g^\pi(\models(t)) \ge g^\policy(M)) = 1$, hence a policy with optimistic gain less than $g^*(M)$ won't be optimistically optimal on this event, so won't be the result of \texttt{EVI}.
        Considering an exploration time $t_{k(i)}$, we know that the policy of the previous episode was optimal in $M$, hence $\optgain(\models(t_{k(i)-1})) = \gain^{\policy^*}(\models(t_{k(i)-1}))$ where $\pi^* \in \Pi^*(M)$.
        By \Cref{assumption_interior}, we know that $\optgain(\models(t_{k(i)-1}))$ only depends on $\rewards_z(t_{k(i)-1})$ and $\kernels_z(t_{k(i)-1})$ for $z \in \pairs^{**}(M)$ where $N_z(t_{k(i)-1}) \ge \lambda t_{k(i)-1}$ by \Cref{lemma_asymptotic_regime}.
        Using \Cref{lemma_gain_deviations} to quantify the sensibility of the gain to kernel and reward perturbations, we get that
        \begin{equation}
        \label{equation_at_exploitation_episodes}
            g^*(M) 
            \le
            g^*(\models(t_{k(i)-1}))
            \le 
            g^*(M) 
            + \OH \parens*{ \!\! \sqrt{ \frac{\log(t_{k(i)-1})}{t_{k(i)-1}} } }
        \end{equation}
        holds with probability one when $i \to \infty$.

        Fix $t \in \braces{t_{k(i)}, \ldots, t_{k(i)}+T-1}$.
        Recall that a policy that \texttt{EVI} outputs must have optimistic gain with span zero.
        Let $\pi$ be the output of \texttt{EVI} at time $t' \in [t_{k(i)}, t]$, and assume that 
        (1)~$\pi$ is sub-optimal in $M$ from $S_t$, so that there exists $s \in \states$ such that $g^\policy(s; M) < g^*(s;M)$ and $s$ is reachable from $\State_t$ under $\pi$; and
        (2)~that $N_z(t) > N_z(t_{k(i)}) + C \log(T/\delta)$ for all $z \in \pairs$, where $C$ is given by the shrinking-shaking \Cref{lemma_shrinking,lemma_shaking}.
        Without loss of generality, we can assume that $s$ is recurrent under $\pi$ on $M$ and let $\pairs' \subseteq \pairs$ be the associated recurrent component of pairs.
        By \Cref{assumption_interior}, we see that $\gain^\pi(s; \models(t))$ only depends on data on $\pairs'$.
        Since $\pi$ was output by \texttt{EVI}, $\gain^\pi(\models(t))$ only depends on data on $\pairs'$.
        Let $\pairs'_- := \pairs' \setminus \pairs^{**}(M)$.
        This set is non-empty because $g^{\pi}(\state; M) < \optgain(\state; M)$.
        Let $\pairs'_+ := \pairs' \cap \pairs^{**}(M)$.
        We have:
        \begin{align*}
            g^\pi(\models(t)) 
            & =
            \sup_{\tilde{r} \in \rewards_\pi(t)} \sup_{\tilde{p} \in \kernels_\pi(t)} g(r, p)
            =
            \sup_{\tilde{r} \in \rewards_{\pairs'}(t)}
            \sup_{\tilde{p} \in \kernels_{\pairs'}(t)}
            g\parens*{
                \tilde{r},
                \tilde{p}
            }
            \\
            & \overset{(\dagger)}\le 
            \sup_{\tilde{r} \in \rewards_{\pairs'}(t_{k(i)-1})}
            \sup_{\tilde{p} \in \kernels_{\pairs'}(t_{k(i)-1})}
            g\parens*{
                \tilde{r} - \tfrac{\log(T/\delta)}{\log(t_{k(i)})} \cdot e_{\pairs'_-} + \sqrt{\tfrac{c\log(t_{k(i)})}{t_{k(i)}}} \cdot e_{\pairs'_+},
                \tilde{p}
            }
            \\
            & \overset{(\ddagger)}\le 
            g^\pi(\models(t_{k(i)-1})) 
            + \sqrt{\tfrac {c \log(t)}t}
            - \eta(M, \pi) \tfrac{\log(T/\delta)}{\log (t_{k(i)})}
            \\
            & \sim
            g^\pi(\models(t_{k(i)-1})) 
            - \tfrac{\eta(M, \pi) \log(T/\delta)}{\log (t_{k(i)})}
            \\
            & \overset{\eqref{equation_at_exploitation_episodes}}\le 
            g^*(M) 
            + \OH \parens*{ \!\! \sqrt{ \tfrac{\log(t_{k(i)-1})}{t_{k(i)-1}} } }
            - \tfrac{\eta(M, \pi) \log(T/\delta)}{\log (t_{k(i)})}
            < g^*(M)
        \end{align*}
        where the last inequality hold for $t_{k(i)}$ large enough.
        In the above, $(\dagger)$ holds on the events specified by the shrinking-shaking behavior of confidence regions, see \Cref{lemma_shrinking,lemma_shaking}; 
        and $(\ddagger)$ is a technical result on exit probabilities, stating that even though we take a supremum on $\tilde{p} \in \kernels_{\pairs'}(t_{k(i)}-1)$, the choice of $\tilde{p}$ will put positive probability mass $\eta(M, \pi) > 0$ on $\pairs'_-$ in its associated invariant probability measures.

        This is justified as follows. 
        On $\pairs_+' \equiv \pairs' \cap \pairs^{**}(\model)$, the number of visits is $\omega(t_{k(i)-1})$ hence $\kernels_\pair(t_{k(i)-1})$ is nearly equal to $\braces{\kernel_\pair}$ for all $\pair \in \pairs_+'$; In fact, for all fixed $\epsilon > 0$, we can assume that $\kernels_\pair(t_{k(i)-1}) \subseteq \braces{\tilde{\kernel}_\pair : \norm{\tilde{\kernel}_\pair - \kernel_\pair}_1 < \epsilon}$ with overwhelming probability provided that $t_{k(i)-1}$ is large enough. 
        Let $(\tilde{\reward}^\policy, \tilde{\kernel}^\policy) \in \models^\policy(t_{k(i)-1})$ be an optimistic model of $\policy$ (see \Cref{appendix_evi}) and let $\tilde{\imeasure}^\policy$ be the empirical invariant measure of $\policy$ starting from $\state$ under the optimistic model. 
        Using that $\vecspan{\gain(\tilde{\reward}^\policy, \tilde{\kernel}^\policy)} = 0$, we assume that $\tilde{\kernel}^\policy$ has a single recurrent class $\pairs''$ up to restricting to that class. 
        By correctness of the confidence region, a policy output by \texttt{EVI} has optimistic gain higher than $\optgain(\model)$ and since the optimistic model is nearly equal to the true model on $\pairs'_+$, we deduce that $\pairs''$ must contain elements of $\pairs'_-$ (otherwise $\policy$ is optimal in $\model$). 
        We see that under $\tilde{\kernel}^\policy$, for every element of $\pairs'' \cap \pairs'_+$ there must be a path to an element of $\pairs'' \cap \pairs'_-$ of length at most $\abs{\states}-1$ and probability at least $c_\epsilon(\model) := (\min_{\pair \in \pairs'_+} \min \braces{\kernel(\state|\pair) > 0 : \state \in \states} - \epsilon)^{\abs{\states}-1}$, which is well-defined and positive for $\epsilon > 0$ small enough. 
        So there must be $\pair \in \pairs'' \cap \pairs'_-$ such that $\tilde{\imeasure}(\pair) \ge \abs{\states}^{-1} c_\epsilon(\model)$. 
        Set $\eta(\model, \policy) := \tfrac 12 c_0(\model)$. 
        For $\epsilon$ small enough and on mild concentration events, we have:
        {
            \small
            \begin{equation*}
                \gain\parens*{
                    \tilde{r}^\policy - \tfrac{\log(T/\delta)}{\log(t_{k(i)})} \cdot e_{\pairs'_-} + \sqrt{\tfrac{c\log(t_{k(i)})}{t_{k(i)}}} \cdot e_{\pairs'_+},
                    ,
                    \tilde{\kernel}^\policy
                }
                \le
                g^\policy(\models(t_{k(i)-1})) 
                + \sqrt{\tfrac {c \log(t)}t}
                - \eta(M, \pi) \tfrac{\log(T/\delta)}{\log (t_{k(i)})}
                .
            \end{equation*}
        }

        This justifies $(\ddagger)$.

        Overall, we have $g^\policy(\models(t)) < g^*(M) \le g^*(\models(t))$ on the event $E_t := \bigcap_{\pair \in \pairs} E_t^\pair$ with $E_t^\pair$ given by, for $\pair \notin \optpairs(\model)$:
        \begin{equation*}
            \parens*{
                F^\pair_{t_{k(i)}}
                ,
                \brackets*{
                    \hspace{-0.5em}
                    \begin{array}{c}
                        \kernels_z(t) \subseteq \kernels_z(t_{k(i)-1}) 
                        \\
                        \text{or~}
                        N_z(t) \le N_z(t_{k(i)}) + C \log\parens*{\tfrac T\delta}
                    \end{array}
                    \hspace{-0.5em}
                }
                ,
                \brackets*{
                    \hspace{-0.5em}
                    \begin{array}{c}
                        \sup \rewards_z(t) \le \sup \rewards_z(t_{k(i)-1}) - \tfrac{N_z(t) - N_z(t_{k(i)})}{C \log(t_{k(i)})}
                        \\
                        \text{~or~}
                        N_z(t) \le N_z(t_{k(i)}) + C \log\parens*{\tfrac T\delta}
                    \end{array}
                    \hspace{-0.5em}
                }
            }
        \end{equation*}
        and for $\pair \in \optpairs(\model)$:
        \begin{equation*}
            \parens*{
                F^\pair_{t_{k(i)}}
                ,
                d_\mathrm{H}(\kernels_z(t), \kernels_z({t_{k(i)-1}})) 
                \le 
                \sqrt{\tfrac {c \log(t)}t}
                ,
                d_\mathrm{H}(\rewards_z(t), \rewards_z({t_{k(i)-1}})) 
            }
        \end{equation*}
        where, for $\pair \notin \optpairs(\model)$, $F^\pair_{t_{k(i)}}$ is the event appearing in the shrinking effect lemma (\Cref{lemma_shrinking}), and for $\pair \in \optpairs(\model)$, $F^\pair_{t_{k(i)}}$ is the event appearing in the shaking effect lemma (\Cref{lemma_shaking}); In both cases, we have $\Pr(\exists i, \forall j \ge i: F^\pair_{t_{k(j)}}) = 1$ provided that the rate $\lambda > 0$ in the definition of $F_{t_{k(i)}}^\pair$ is chosen accordingly to the asymptotic regime of the algorithm (\Cref{lemma_asymptotic_regime}). 
        We deduce that on $E_t$, $\pi$ will be rejected as soon as \eqref{equation_vanishing_multiplicative} triggers, because its optimistic gain is no more optimistically optimal.
        By \eqref{equation_vanishing_ends}, as soon as a pair $z \notin \pairs^{**}(M)$ is about to be visited for the second time in the episode, the episode will stop. 
        We therefore have shown that while $\gain^{\policy}(\State_t;\model) < \optgain(\State_t;\model)$ and on $E_t$, there exists $z \equiv (\state, \action)$ that is reachable from $S_t$ under $\pi$ such that $N_z(t) < N_z(t_k) + 1 + C \log(T/\delta)$ and $\gain^\policy(\state; \model) < \optgain(\state; \model)$.

        Accordingly, we have shown that the algorithm is $(E, t_{k(i)}, T, \varphi)$-coherent, with $\Pr(\exists t \in [t_{k(i)}, t_{k(i)}+T] : E_t^c) \le \delta + \oh(1)$ when $i \to \infty$ and $\varphi(t) = 1 + C \log(T/\delta)$. 
    \end{proof}

    \clearpage
    \section{Model dependent regret guarantees via coherence}
    \label{appendix_model_dependent}

    In the proof of the regret of exploration guarantees, \Cref{lemma_coherence} is used twice and two different coherence properties are invoked.
    Coherence is first used in a \emph{global} form to derive the almost sure asymptotic regime.
    Indeed, the first step of the proof (see \Cref{section_asymptotic_regime}) consists in showing that the algorithm is $((F_t), \ceil{\log(T)}, T, \varphi)$-coherent for $\varphi(\ceil{\log(T)}) = \OH(\log(T))$ where the sequence of events $(F_t)$ is asymptotically almost-sure, i.e., $\Pr(\exists T, \forall t \ge T: F_t) = 1$. 
    Then, coherence is used in a \emph{local} form to derive the regret of exploration guarantees. 
    Indeed, the whole point of \Cref{section_establishing_coherence} is to show that the algorithm is $(E, t_{k(i)}, T, \varphi)$-coherent where $(t_{k(i)})$ is the sequence of exploration episodes, $\Pr(\exists T, \forall t \ge T: E_t) = 1$ and $\varphi(t_{k(i)}) = \OH(\log(T))$. 

    In this appendix, we show a third application of coherence properties: model dependent regret guarantees.

    \subsection{A general model dependent regret bound via coherence}

    We provide first a general result. 

    \begin{theorem}
    \label{theorem_general_model_dependent_regret}
        Consider an episodic algorithm with {\upshape (1)} weakly regenerative episodes and {\upshape (2)} such that there exists an adapted sequence of events $(F_t)$ with $\Pr\parens{\bigcup_{t=T}^\infty F_t^c} = \OH(\frac 1T)$ such that the algorithm is $((F_t), T, T, \varphi)$-coherent for all $T \ge 1$.
        Then, for all non-degenerate model $\model$, 
        \begin{equation}
            \Reg(T; \model)
            =
            \OH \parens*{\sum_{m=0}^{\lceil \log_2(T) \rceil -1} \varphi(2^m)} + \OH(\log(T))
        \end{equation}
        when $T \to \infty$. 
    \end{theorem}

    \begin{proof}
        Let $n := \lceil \log_2(T) \rceil$. 
        For all $m \le n$, the algorithm is $(F, 2^m, 2^m, \varphi)$-coherent, has weakly regenerative episodes, and $M$ is non-degenerate, so we invoke \Cref{lemma_coherence} and obtain, for $x \ge 0$,
        \begin{align*}
            & 
            \Pr \parens*{
                \Reg(2^m, 2^{m+1})
                \ge
                x + C_4 \varphi(2^n)
            }
            \\
            & \le
            \Pr \parens*{
                \Reg(2^m, 2^{m+1})
                \ge
                x + C_4 \varphi(2^n),
                \bigcap_{t=2^m}^{2^{m+1}-1} F_t
            }
            +
            \Pr \parens*{
                \bigcup_{t=2^m}^{\infty} F_t^c
            }
            \\
            & \le
            \exp\parens*{
                - \frac{x}{C_2} + C_3 m \log(2) + \log(C_1)
            } + \OH \parens*{2^{-m}}
        \end{align*}
        where $C_1, C_2, C_3, C_4$ are model dependent constants. 
        For $x \ge C_2(C_1 + (1 + C_3) \log(2) m)$, the RHS is $\OH(2^{-m})$. 
        In other words, $\Reg(2^m, 2^{m+1}) = \OH(\varphi(2^m))$.
        Summing for $m \ge 1$, we get:
        \begin{align*}
            \Reg(T)
            & :=
            \sum_{m=0}^{n-1}
            \Reg(2^m, 2^{m+1})
            \\
            & =
            \OH \parens*{\sum_{m=0}^{n-1} \varphi(2^m) + 1}
            \\
            & =
            \OH \parens*{\sum_{m=0}^{\lceil \log_2(T) \rceil -1} \varphi(2^m)} + \OH(\log(T)).
        \end{align*}
        This is the intended result.
    \end{proof}

    A few comments are in order. 
    First, the requirement $\Pr(\bigcup_{t=T}^\infty F_t^c) = \OH(\frac 1T)$ is slightly overshoot and can be weakened depending on the asymptotic properties of $\varphi$ and the desired bound. 
    Second, the proof technique can be directly adapted to obtain bounds in probability rather than in expectation. 
    Last, but perhaps the most important, is that this bound only holds for non-degenerate models (\Cref{definition_non_degeneracy}). 
    While every model can be made non-degenerate up to smooth reward perturbations, non-degenerate models are a bit special, because the weakly optimal pair is unique from every state (unique Bellman optimal policy), and $\optpairs(\model)$ has a unique communicating component (unique gain optimal component), see \Cref{appendix_non_degeneracy}.
    The proof of \Cref{lemma_coherence}, which is key here, inevitably relies on non-degeneracy.
    Yet, degenerate models are easy to find.
    When \cite{ortner_online_2010} discusses the necessity for episodes (see his Figure~2), he exhibits a degenerate model for that purpose.
    This simple example is a good starting point to understand why coherence and weakly regenerative episodes are insufficient to provide regret bounds on degenerate models. 

    \subsection{A model dependent regret bound for \texorpdfstring{\eqref{equation_vanishing_multiplicative}}{(VM)}}

    \Cref{theorem_general_model_dependent_regret} is applied to \texttt{KLUCRL} managing episodes with a $f$-\eqref{equation_vanishing_multiplicative} rule, by showing that such algorithms satisfy a $((F_t), T, T, \varphi)$-coherence property with a budget function $\varphi(T) = \OH(\log(T))$, leading to $\OH(\log(T) \log\log(T))$ regret bounds.

    \begin{theorem}
    \label{theorem_model_dependent_regret}
        Let $\model$ be a non-degenerate model.
        Consider running \texttt{KLUCRL} with $\model$ satisfying \Cref{assumption_interior} and assume that episodes are managed with a $f$-\eqref{equation_vanishing_multiplicative} with $f > 0$.
        Then:
        \begin{equation}
            \Reg(T; M)
            = 
            \OH \parens*{
                \log(T) \log\log(T)
            }
            .
        \end{equation}
    \end{theorem}

    \begin{proof}
        Consider the good events $E_t := (\model \in \models(t))$ and $F_t := \bigcap_{t'=(t-\abs{\pairs})/2}^t E_t$.

        By design $\Pr(\bigcup_{t=T}^\infty E_t^c) = \OH(\frac 1T)$, see \Cref{lemma_confidence_region}, so $\Pr(\bigcup_{t=T}^\infty F_t^c) = \OH(\frac 1T)$ as well.
        We show that the algorithm is $(F_t, T, T, \varphi)$-coherent for $\varphi(T) = \OH(\log(T))$. 
        The result will then follow by \Cref{theorem_general_model_dependent_regret} using that $\integral \log(x) dx = x \log x - x$. 

        By Pinsker's inequality, for all $\epsilon > 0$, there exists $C \equiv C_\epsilon > 0$ such that, if $\visits_\pair (t) \ge C \log(t)$, then:
        \begin{equation}
            \kernels_\pair(t) \subseteq \braces*{\tilde{\kernel}_\pair : \norm{\tilde{\kernel}_\pair - \hat{\kernel}_\pair(t)}_1 < \tfrac 12 \epsilon}
            \quad\text{and}\quad
            \rewards_\pair(t) \subseteq \braces*{\tilde{\reward}_\pair : \norm{\tilde{\reward}_\pair - \hat{\reward}_\pair(t)}_\infty < \tfrac 12 \epsilon}
        \end{equation}
        Introduce the gain gap $\Delta_g := \min \braces{\norm{\gain^\policy(\model) - \optgain(\model)}_\infty : \policy \notin \optpolicies(\model)} > 0$. 
        Whenever $\model \in \models(t)$, we have $\optgain(\model) \le \optgain(\models(t))$. 
        Let $\policy$ be a policy output by \texttt{EVI} at time $t$ and assume that $\visits_\pair(t) \ge C \log(t)$ for all $\pair \in \pairs$. 
        It has optimistic bias with span at most $\diameter(\model)$, hence by \Cref{lemma_gain_deviations}, we have:
        \begin{equation}
            \norm{\gain^\policy(\models_t) - \gain^\policy(\model)}_\infty
            \le
            \epsilon \parens*{
                1 + \tfrac 12 \diameter(\model)
            }
        \end{equation}
        yet $\gain^\policy(\models_t) \ge \optgain(\models_t) \ge \optgain(\model)$.
        So, provided that $\epsilon (1 + \tfrac 12 \diameter(\model)) < \Delta_g$, $\policy$ necessarily achieves optimal gain.
        We assume from now on that $\epsilon(1 + \frac 12 \diameter(\model)) < \Delta_g$ is true.

        Now, assume that $\policy_t$ is such that $\gain^{\policy_t}(\State_t, \model) < \optgain(\State_t, \model)$. 
        By construction of \texttt{EVI}-based algorithms, $\policy_t$ is the output of \texttt{EVI} for $t_k$ with $t \in [t_k, t_{k+1})$, hence is the optimistically optimal policy at time $t_k$.
        By assumption $\gain^{\policy_{t_k}}(\State_t;\model) < \optgain(\State_t; \model)$, so assuming that 
        \begin{equation}
            E_{t_k} \equiv (\model \in \models(t_k))
        \end{equation}
        holds, we deduce from the previous argument that there must be $\pair \in \pairs$ such that $\visits_\pair(t_k) < C \log(t_k)$. 
        Since $\gain^{\policy_{t_k}}(\State_t, \model) < \optgain(\State_t, \model)$, $\Reach(\policy_{t_k}, \State_t)$ must contain a recurrent component of $\policy_{t_k}$ on which the achieved gain is sub-optimal.
        Pick one, denoted $\pairs'$. 
        Thanks to \Cref{assumption_interior}, the optimistic gain of $\gain^{\policy_{t_k}}(\state, \models(t_k))$ for $\state \in \states(\pairs')$ only depends on pairs among $\Reach(\policy_{t_k}, \state)$ and yet $\gain^{\policy_{t_k}}(\state; \models(t_k)) \ge \optgain(\state, \model)$. 
        So there must be a sub-sampled pair in $\pairs'$, i.e., there exists $(\state, \action) \in \pairs'$ such that $\visits_{\state,\action}(t_k) < C \log(t_k)$; This pair is reachable from $\State_t$ under $\policy_t$ and $\gain^{\policy_t}(\state; \model) < \optgain(\state; \model)$ by construction of $\pairs'$.
        Last, but not least, is that by construction of \eqref{equation_vanishing_multiplicative}, we have $t \le 2t_k + \abs{\pairs}$ and $\visits_{\state,\action}(t) \le 2 \visits_{\state,\action}(t_k) + 1$. 
        So, on the event $F_t := \bigcap_{t'=(t-\abs{\pairs})/2}^t E_t$, 
        \begin{equation}
             \exists \pair \equiv (\state, \action) \in \Reach(\policy_t, \State_t):
            \quad
            \visits_\pair(t) \le 2 C \log(t) + 1
            \text{~and~}
            \gain^{\policy_t}(\state;\model) < \optgain(\state;\model).
        \end{equation}
        Setting $\varphi(t) := 2 C \log(2t) + 1$, we have shown that the algorithm is $((F_t), T, T, \varphi)$-coherent.
        We have $\varphi(T) = \OH(\log(T))$ and $\Pr(\bigcup_{t=T}^\infty F_t^c) = \OH(\frac 1T)$.
        Conclude by applying \Cref{theorem_general_model_dependent_regret}.
    \end{proof}

    The result is remarkable in that $f$ is basically arbitrary.
    It allows for $f(t)$ decreasing arbitrarily fast, hence for linearly many episodes, meaning that \texttt{KLUCRL} can nearly be episode-less on non-degenerate models, at the expense of minimax guarantees (see \Cref{theorem_minimax_regret}). 
    This remark is to be combined with the observation that optimistic algorithms (\Cref{appendix_minimax}) cannot be episode-less on degenerate models in general, see \cite{ortner_online_2010}.
    In tandem, this indicates that coherence alone cannot provide regret guarantees beyond non-degenerate models. 
    If the model dependent regret guarantees are obtained ``for free'' from coherence, extending such guarantees to degenerate models would require a different approach and most likely assumptions on the function $f$. 

    Whether the $\OH(\log\log(T))$ factor can be removed remains an open question. 

    \clearpage
    \section{A few technical results on Markov decision processes}
    \label{appendix_technical}

    In this appendix, we provide a few useful technical results on Markov decision processes.
    In \Cref{appendix_gain_function}, we provide a general result on the sensibility of the gain function to parameters, that is used at many places in this work. 
    In \Cref{appendix_non_degeneracy}, we give a few insights regarding the non-degeneracy assumption of \Cref{definition_non_degeneracy}.
    Lastly, we dedicate the last \Cref{appendix_proof_of_theorem_linear_regexp} to the proof of \Cref{theorem_linear_regexp}, showing that the regret of exploration of existing algorithms is linear on explorative Markov decision processes. 

    \subsection{Sensibility of the gain function to parameters}
    \label{appendix_gain_function}

    In \Cref{lemma_gain_deviations}, we explain how the gain function of a policy is subjected to vary under perturbation of the reward vector $\reward$ and the transition kernel $\kernel$. 
    The gain function is shown to be $1$-Lipschitz with respect to rewards, and $\frac 12 \vecspan{\bias^\policy}$-Lipschitz with respect to kernels, where $\bias^\policy (\state) := \lim \EE_{\state}^\policy [\sum_{t=1}^T (\Reward_t - \gain^\policy(\State_t))]$ is the bias function of the policy. 

    \begin{lemma}
        \label{lemma_gain_deviations}
        Let $\model \equiv (\pairs, \kernel, \reward)$ and $\hat{\model} \equiv (\pairs, \hat{\kernel}, \hat{\reward})$ be two Markov decision processes and fix $\policy \in \policies$ a policy.
        If $\vecspan{\gain^\policy(\model)} = 0$, then
        \begin{align*}
            & \norm*{
                \gain^{\policy}(\hat{\model})
                - 
                \gain^{\policy}(\model)
            }_\infty
            \\
            & \le
            \max_{\state \in \states} \braces*{
                \abs[\big]{\hat{\reward}(\state, \policy(\state)) - \reward(\state, \policy(\state))}
                +
                \frac 12 \vecspan{\bias^\policy(\model)}
                \norm[\big]{\hat{\kernel}(\state, \policy(\state)) - \kernel(\state, \policy(\state))}_1
            }
            .
        \end{align*}
    \end{lemma}
    \begin{proof}
        Let $T \ge 1$ and let $\state \in \states$ be an initial state.
        Set $\epsilon_\reward^\policy := \norm{\hat{\reward}^\policy - \reward^\policy}_\infty$ and $\epsilon_\kernel^\policy := \norm{\hat{\kernel}^\policy - \kernel^\policy}_1$.
        \begin{align*}
            & \EE_\state^{\policy, \hat{\model}} \brackets*{
                \sum_{t=0}^{T-1} \Reward_t
            }
            \\
            & = 
            \EE_\state^{\policy, \hat{\model}} \brackets*{
                \sum_{t=0}^{T-1}
                \hat{\reward}^\policy(\State_t)
            }
            \\
            & \le
            \EE_\state^{\policy, \hat{\model}} \brackets*{
                \sum_{t=0}^{T-1}
                \reward^\policy(\State_t)
            }
            + T \epsilon_\reward^\policy
            \\
            & \overset{(\dagger)}=
            \EE_\state^{\policy, \hat{\model}} \brackets*{
                \sum_{t=0}^{T-1}
                \parens*{
                    \gain^\policy(\State_t)
                    + 
                    \parens*{
                        e_{\State_t} - \kernel(\State_t, \Action_t)
                    } \bias^\policy
                }   
            }
            + T \epsilon_\reward^\policy
            \\
            & \overset{(\ddagger)}\le
            T \gain^\policy(\state)
            +
            \EE_\state^{\policy, \hat{\model}} \brackets*{
                \sum_{t=0}^{T-1}
                \parens*{
                    \parens*{
                        e_{\State_{t+1}} - \hat{\kernel}(\State_t, \Action_t)
                    } \bias^\policy
                    +
                    \parens*{
                        \hat{\kernel}(\State_t, \Action_t) - \kernel(\State_t, \Action_t)
                    } \bias^\policy
                }
            }
            \\
            & \phantom{
                {} \overset{(\ddagger)}\le
                T \gain^\policy(\state)
            }
            {} + \vecspan{\bias^\policy}
            + T \epsilon_\reward^\policy
            \\
            & \overset{(\S)}=
            T \parens*{
                \gain^\policy(\state)
                + \epsilon_\reward^\policy
                + \tfrac 12 \vecspan{\bias^\policy} \epsilon_\kernel^\policy
            }
            + \vecspan{\bias^\policy}
        \end{align*}
        where 
        $(\dagger)$ invokes the Poisson equation $\gain^\policy(\State_t) + \bias^\policy(\State_t) = \reward^\policy(\State_t) + \kernel^\policy(\State_t) \bias^\policy$,
        $(\ddagger)$ uses that $\gain^\policy(\State_t) = \gain^\policy(\state)$ for all $t \ge 0$ and 
        $(\S)$ that, if $p, p' \in \probabilities(\states)$ and $u \in \R^\states$ then $\abs{(p' - p) u} \le \frac 12 \vecspan{u} \norm{p' - p}_1$.
        Dividing by $T$ and letting it go to infinity, we obtain the desired upper-bound.
        The lower bound is obtained similarly.
    \end{proof}

    \subsection{The space of non-degenerate Markov decision processes}
    \label{appendix_non_degeneracy}

    In this section, we discuss of non-degeneracy assumption, found in \Cref{definition_non_degeneracy}, and made in \Cref{theorem_main} for both model dependent regret guarantees and regret of exploration guarantees.  
    We argue that while not all Markov decision processes are non-degenerate, most of them are. 

    \begin{blackblock}
    \begin{theorem}[Characterizations of non-degenerate MDPs]
    \label{theorem_characterization_non_degenerate}
        Let $\model \equiv (\pairs, \reward, \kernel)$ be a communicating Markov decision process.
        The following statements are equivalents.
        \begin{enumerate}
            \item $\model$ is non-degenerate in the sense of \Cref{definition_non_degeneracy}: There is a unique policy satisfying the Bellman equations (i) $\gain^\policy(\state) = \max_{\action \in \actions(\state)} \braces{\kernel(\state, \action) \gain^\policy}$ and (ii) $\gain^\policy (\state) + \bias^\policy (\state) = \max_{\action \in \actions (\state)} \braces{\reward(\state, \action) + \kernel(\state, \action) \bias^\policy}$ for all $\state \in \states$, and this policy is unichain. 

            \item $\wkoptpairs(\model)$ is robust to reward noise, i.e., there exists $\epsilon > 0$ such that if $\norm{\reward' - \reward}_\infty < \epsilon$, then $\wkoptpairs(\reward', \kernel) = \wkoptpairs(\reward, \kernel)$; 
        \end{enumerate}
    \end{theorem}
    \end{blackblock}

    The first characterization is the definition, stating that the Bellman optimal policy is unique and unichain.
    The second characterization states that the set of weakly optimal pairs is robust to reward perturbations, in other words, that the gap function $\ogaps(-)$ has locally constant support. 

    We start with a lemma, showing that the uniqueness of weakly optimal actions is almost sure up to smooth perturbation of the reward function. 

    \begin{lemma}
    \label{lemma_non_degeneracy_noise}
        Let $\model \equiv (\pairs, \reward, \kernel)$ be a communicating Markov decision process.
        Let $U(\pair)$ be i.i.d.~random variables of distribution $\mathrm{N}(0, 1)$. 
        Then $\model_U := (\pairs, \reward + U, \kernel)$ has unique weakly optimal actions almost surely, i.e., $\abs{\wkoptpairs(\model_U)} = \abs{\states}$ almost surely.
    \end{lemma}
    \begin{proof}
        If a model $\model' \equiv (\pairs, \reward', \kernel)$ does not have unique optimal actions, then there exist $\state \in \states$ as well as $\action \ne \action' \in \actions(\state)$ such that $(\state, \action), (\state, \action') \in \wkoptpairs(\model')$.
        In particular, we have:
        \begin{equation}
        \label{equation_proof_non_degeneracy_noise_1}
            \reward'(\state, \action) + \kernel(\state, \action) \optbias(\reward', \kernel)
            =
            \reward'(\state, \action') + \kernel(\state, \action') \optbias(\reward', \kernel)
            .
        \end{equation}
        Because $\optbias$ is obtained as the bias vector of some policy, we have in particular:
        \begin{equation}
        \label{equation_proof_non_degeneracy_noise_2}
            \exists \policy \in \policies,
            \quad
            \reward'(\state, \action) + \kernel(\state, \action) \bias^\policy(\reward', \kernel)
            =
            \reward'(\state, \action') + \kernel(\state, \action') \bias^\policy(\reward', \kernel)
        \end{equation}
        which is of the form ``$\exists \policy \in \policies, f^\policy(\reward') = 0$'' where $f^\policy$ are a linear forms. 
        It happens that all are non-degenerate. 
        Indeed, denoting $(e_\pair)$ the canonical basis of $\R^\pairs$, we see that for all $\policy \in \policies$, either $f^\policy(e_{(\state, \action)}) \ne 0$ or $f^\policy(e_{(\state, \action')}) \ne 0$ depending on whether $\policy(\state) = \action$ or $\policy(\state) \ne \action$. 
        It follows that the set of $\reward' \in \R^\pairs$ satisfying \eqref{equation_proof_non_degeneracy_noise_2} is a union of hyperplanes, hence is negligible with respect to the Lebesgue measure. 
        It follows that $\Pr \parens*{r + U \text{~satisfies~\eqref{equation_proof_non_degeneracy_noise_1}}} = 0$. 
    \end{proof}
    
    \begin{proof}{\bf of \Cref{theorem_characterization_non_degenerate}}
        To begin with, note that if $\optpolicy$ is the unique Bellman optimal policy of a Markov decision process $\model$, $\optpolicy$ is bias optimal by standard theory \cite[§9.2]{puterman_markov_1994}, i.e., $\gain^{\optpolicy}(\model) = \optgain(\model)$ and $\bias^{\optpolicy}(\model) = \optbias(\model)$.
        In particular, $\wkoptpairs(\model) = \braces{(\state, \optpolicy(\state)): \state \in \states}$. 

        Now, assume (\textit{1.}) and let $\optpolicy$ be the unique Bellman optimal policy of $\model$.  
        We show (\textit{2.}) by contradiction.
        Assume that $\wkoptpairs(\model)$ is not robust to reward noise.
        So, because $2^\pairs$ is finite, there exists $\pairs_0 \subseteq \pairs$ with $\pairs_0 \ne \wkoptpairs(\model)$, together with a sequence $\reward_n \to \reward$ such that $\wkoptpairs(\reward_n, \kernel) = \pairs_0$ for all $n \ge 1$.
        Up to infinitesimal perturbation of $\reward_n$, we can further assume that $\abs{\pairs_0} = \abs{\states}$ by \Cref{lemma_non_degeneracy_noise}.
        So, $\pairs_0$ defines a policy $\policy_0$ where $\policy_0(\state)$ picks the unique action $\action$ such that $(\state, \action) \in \pairs_0$. 
        By definition of $\wkoptpairs(\model)$, this policy is the unique bias optimal policy of every $\model_n \equiv (\pairs, \reward_n, \kernel)$.
        In particular, $\policy_0$ is Bellman optimal in $\model_n$.
        Now, the gain $\gain^{\policy_0}$ and bias $\bias^{\policy_0}$ are $1$-Lipschitz in $\reward$.
        By taking $n \to \infty$, we conclude that $\policy_0$ is Bellman optimal in $\model$ as well.
        By (\textit{1.}) uniqueness of the Bellman optimal policy of $\model$, we have $\policy_0 = \policy^*$ hence $\pairs_0 = \wkoptpairs(\model)$; A contradiction. 

        Conversely, assume (\textit{2.}). 
        By \Cref{lemma_non_degeneracy_noise}, it implies that $\abs{\wkoptpairs(\model)} = \abs{\states}$, hence that the bias optimal policy is unique. 
        We prove (\textit{1.}) by contradiction, so either $\model$ has multiple Bellman optimal policies, or its Bellman optimal policy is not unichain.

        Assume that the Bellman optimal policy $\optpolicy$ is not unichain and let $\pairs_1$ and $\pairs_2$ be two disjoint recurrent components of $\optpolicy$. 
        Consider the model $\model'_\epsilon := (\pairs, \reward + \epsilon \indicator{\pairs_1}, \kernel)$ for $\epsilon > 0$. 
        We see that $\optgain(\model'_\epsilon) = \optgain(\model) + \epsilon$, while the gain of every unichain policy $\policy$ with recurrent component $\pairs_2$ is $\gain^\policy(\model'_\epsilon) = \optgain(\model)$. 
        Using the formula
        \begin{equation*}
            \gain^{\policy}(\state; \model'_\epsilon) 
            =
            \optgain(\state; \model'_\epsilon)
            - \lim_{T \to \infty} 
            \EE_\state^{\policy, \model'_\epsilon} \brackets*{
                \frac 1T
                \sum_{t=1}^T
                \ogaps(\Pair_t; \model'_\epsilon)
            },
        \end{equation*}
        we conclude that there exists $\pair \in \pairs_2$ such that $\ogaps(\pair; \model'_\epsilon) > 0$, hence $\pair \notin \wkoptpairs(\model'_\epsilon)$.
        So $\wkoptpairs(\model'_\epsilon) \ne \wkoptpairs(\model)$; A contradiction.

        Now, assume that $\model$ has multiple Bellman optimal policies, say $\optpolicy_1$ and $\optpolicy_2$.
        Without loss of generality, we assume that $\optpolicy_2$ is bias optimal in $\model$, i.e., that $\optpolicy_2 (\state)$ picks the unique action $\action$ such that $(\state, \action) \in \wkoptpairs(\model)$, which exists by (\textit{2.}).
        With the same argument as before, we can show that both $\optpolicy_1$ and $\optpolicy_2$ are unichain. 
        We can also show that $\optpolicy_1$ and $\optpolicy_2$ must have the same recurrent component; Otherwise, we introduce $\pairs_1^*$ and $\pairs_2^*$ their respective components, we consider $\model'_\epsilon := (\pairs, \reward + \epsilon \indicator{\pairs_1^* \setminus \pairs_2^*}, \kernel)$ and invoke the same rationale. 
        So $\bias^{\optpolicy_1} = \bias^{\optpolicy_2} = \optbias$ on the recurrent states of $\optpolicy_1$ and $\optpolicy_2$. 
        Let $\ogaps_1 (\state, \action) := \gain^{\optpolicy_1}(\state) +  \bias^{\optpolicy_1}(\state) - \reward(\state, \action) - \kernel(\state, \action) \bias^{\optpolicy_1}$ be the gap function of $\optpolicy_1$ in $\model$.
        Note that $\ogaps_1 \ge 0$ by definition of $\optpolicy_1$. 
        Now, let $\state_0$ be a recurrent state of $\optpolicy_1$ and $\optpolicy_2$ and let $\tau_0 := \inf \braces{t \ge 1 : \State_t = \state_0}$ be the reaching time to $\state_0$.
        For all $\state \in \states$, we have
        \begin{align*}
            \bias^{\optpolicy_2}(\state)
            & =
            \EE^{\policy_2, \model}_{\state} \brackets*{
                \sum_{t=1}^{\tau_0-1}
                \parens*{
                    \bias^{\optpolicy_2} (\State_t)
                    - 
                    \bias^{\optpolicy_2} (\State_{t+1})
                }
            }
            + \bias^{\optpolicy_2} (\state_0)
            \\
            & \overset{(\dagger)}=
            \EE^{\policy_2, \model}_{\state} \brackets*{
                \sum_{t=1}^{\tau_0-1}
                \parens*{
                    \reward(\Pair_t) - \optgain(\State_t)
                }
            }
            + \bias^{\optpolicy_2} (\state_0)
            \\
            & \overset{(\ddagger)}=
            \EE^{\policy_2, \model}_{\state} \brackets*{
                \sum_{t=1}^{\tau_0-1}
                \parens*{
                    \reward(\Pair_t) - \gain^{\optpolicy_1}(\State_t)
                }
            }
            + \bias^{\optpolicy_1} (\state_0)
            \\
            & \overset{(\S)}=
            \EE^{\policy_2, \model}_{\state} \brackets*{
                \sum_{t=1}^{\tau_0-1}
                \parens*{
                    \bias^{\optpolicy_1} (\State_t)
                    - 
                    \bias^{\optpolicy_1} (\State_{t+1})
                    - 
                    \ogaps_1 (\Pair_t)
                }
            }
            + \bias^{\optpolicy_1} (\state_0)
            \\
            & \overset{(\$)}\le
            \bias^{\optpolicy_1}(\state)
        \end{align*}
        where 
        $(\dagger)$ invokes the Poisson equation of $\optpolicy_2$;
        $(\ddagger)$ uses that $\optpolicy_1$ is gain optimal and that $\bias^{\optpolicy_1}(\state_0) = \bias^{\optpolicy_2}(\state_0)$;
        $(\S)$ follows by definition of the gap function $\ogaps_1$; and
        $(\$)$ uses that $\ogaps_1 \ge 0$. 
        So $\optbias = \bias^{\optpolicy_2} \le \bias^{\optpolicy_1}$.
        So $\optpolicy_1$ is bias optimal.
        By (\textit{2.}), the bias optimal policy of $\model$ is unique, so $\optpolicy_1 = \optpolicy_2$; A contradiction.
    \end{proof}

    Combining \Cref{theorem_characterization_non_degenerate} and \Cref{lemma_non_degeneracy_noise}, we obtain the following result.

    \begin{corollary}[Non-degeneracy is almost-sure]
    \label{corollary_non_degeneracy_noise}
        Let $\model \equiv (\pairs, \reward, \kernel)$ be a communicating Markov decision process.
        Let $U(\pair)$ be i.i.d.~random variables of distribution $\mathrm{N}(0, 1)$. 
        Then $\model_U := (\pairs, \reward + U, \kernel)$ is almost-surely non-degenerate.
    \end{corollary}
    
    This result states that if a Markov decision process is degenerate, almost all its neighbors are non-degenerate.
    For instance, fixing the kernel then picking the reward function uniformly at random in $[0, 1]^\pairs$, the resulting Markov decision process is non-degenerate with probability one. 
    This supports the idea that, although many Markov decision processes are degenerate, most are non-degenerate. 

    \subsection{Proof of \texorpdfstring{\Cref{theorem_linear_regexp}}{Theorem 6}: Algorithms based on \texorpdfstring{\eqref{equation_doubling_trick}}{(DT)} have linear regret of exploration}
    \label{appendix_proof_of_theorem_linear_regexp}

    In this paragraph, we prove \Cref{theorem_linear_regexp}.

    \bigskip
    \par\noindent
    \strong{\Cref{theorem_linear_regexp}.}
    \textit{
        Fix a pair space $\pairs$ and let $\models$ be the space of all recurrent models with pairs $\pairs$. 
        Let $f : \N \to (0, \infty)$ be such that $\lim f(n) = + \infty$. 
        Any no-regret episodic learner $\planner$ satisfying:
        \begin{equation}
        \begin{gathered}
            \forall k \ge 1,
            \exists \pair \in \pairs,
            \quad
            \visits_{t_{k+1}}(\pair) \ge \visits_{t_k}(\pair) + f\parens*{\visits_{t_k}(\pair)}
            \\
            \exists c > 0,
            \forall t \ge 0, \forall (\state, \action) \in \pairs,
            \quad
            \policy_t(\action|\pair) \ge c \text{~or~} \policy_t(\action|\pair) = 0
        \end{gathered}
        \end{equation}
        has linear regret of exploration on the explorative sub-space of $\models$, i.e., for all $\model \in \models^+$, we have $\RegExp(T) = \Omega(T)$ a.s.~when $T \to \infty$.
    }
    \medskip

    \begin{proof}
        Let $\model \in \models^+$.
        By \Cref{theorem_characterization_explorative}, $\abs*{\explorationepisodes} = \infty$ almost surely. 
        Denote $(t_{k(i)})$ the enumeration of exploration times. 
        Because $\model$ is recurrent, every policy is recurrent on $\model$ thus $\Reach(\policy, \model, \state) \cap \suboptimalpairs(\model) \ne \emptyset$ if, and only if $\gain^\policy(\model) < \optgain(\model)$, where $\state$ is an arbitrary state.
        From \eqref{equation_linear_condition}, we see that:
        \begin{equation}
            \Pr \parens*{
                \lim_{t \to \infty} 
                \min \braces*{\visits_t(\state, \action): \policy_t(\action|\state) > 0}
                =
                \infty
            }
            =
            1.
        \end{equation}
        It follows that $\liminf (t_{k(i)+1} - t_{k(i)}) = \infty$. 
        Below, we write $\imeasure^\policy$ the asymptotic empirical measure of play of $\policy \in \policies$, given by $\imeasure^\policy(\pair|\state; \model) := \lim \frac 1T \EE_{\state}^{\policy, \model}[\sum_{t=1}^T \indicator{\Pair_t = \pair}]$, i.e., $\imeasure^\policy(\pair|\state; \model)$ is the average amount of time that $\policy$ spends playing $\pair$ under $\model$ starting from $\state \in \states$.
        In particular, for all $T \ge 0$, we have:
        \begin{align*}
            & \RegExp(T) 
            \\
            & \overset{(*)}\ge
            \limsup_{i \to \infty}
            \parens*{
                \EE^{\planner, \model} \brackets*{
                    \sum_{t=t_{k(i)}}^{t_{k(i)}+T-1}
                    \ogaps(\Pair_t; \model)
                }
            }
            \\
            & \overset{(\dagger)}\ge
            \limsup_{i \to \infty}
            \parens*{
                \EE^{\planner, \model} \brackets*{
                    T \min(\imeasure^{\policy_{t_{k(i)}}}(\model))
                    \ogaps_\mathrm{min}(\model)
                    - \diameter(\policy_{t_{k(i)}}; \model)
                }
            }
            \\ 
            & \overset{(\ddagger)} \ge
            T \alpha - \beta
        \end{align*}
        where
        $(*)$ follows from by definition;
        $(\dagger)$ is obtained by writing the Poisson equation of $\policy_{t_{k(i)}}$ for the reward function $f_i(\pair) = \indicator{\pair = \pair_i}$ where $\pair_i$ is any sub-optimal pair played by $\policy_{t_{k(i)}}$, and $\diameter(\policy_{t_{k(i)}}; \model)$ is the span of the bias function of $\policy_{t_{k(i)}}$ under $f$; and
        $(\ddagger)$ introduces $\alpha := \min_{\policy} \min(\imeasure^\policy(\model)) \ogaps_\mathrm{\min}(\model) > 0$ and $\beta := \max_{\policy} \diameter(\policy; \model) < \infty$. 
    \end{proof}

    \clearpage
    \section{The class of explorative MDPs}
    \label{appendix_exploration}

    In this appendix, we study the spaces of Markov decision processes for which the regret of exploration (\Cref{definition_regret_of_exploration}) is well-defined. 
    By construction, the regret of exploration is well-defined if, and only if the number of exploration times (\Cref{definition_exploration_time}) is infinite and we naturally investigate when this is exactly the case.
    As motivated in \Cref{section_well_definition}, we need a technical accommodation: We focus on learning algorithms with sub-linearly many episodes.
    These algorithms are these for which the performance in practice is actually comparable to the policies from which they pick actions. 
    Under this technical assumption, explorative Markov decision processes correspond to those that, intuitively, cannot be learned by playing the optimal policy only. 
    In \Cref{theorem_characterization_explorative}, we provide four characterizations of explorative environments. 
    Every one of them is of a different nature, that we explain below.
    
    \begin{blackblock}
        \begin{theorem}[Characterizations of explorative MDPs]
        \label{theorem_characterization_explorative}
            Let $\models \equiv \product_{\pair \in \pairs} (\rewards_\pair \times \kernels_\pair)$ be a \strong{convex} ambient space in product form.
            Let $\model \in \models$ be a non-degenerate Markov decision process.
            The following assertions are equivalent:
            \begin{enumerate}
                \item $\model \notin \models^+$, i.e., $\model$ is not explorative;
                    \vspace{-.66em}
                \item $\model$ has empty \strong{confusing set}, i.e., $\confusing(\model) = \emptyset$, see \eqref{equation_confusing_set};
                    \vspace{-.66em}
                \item There exists a consistent learner $\learner$, i.e., such that $\Reg(T; \model', \learner) = \oh(T^\epsilon)$ for all $\model' \in \models$ and $\epsilon > 0$, such that $\Reg(T; \model, \learner) = \oh(\log(T))$;
                    \vspace{-.66em}
                \item There exists a robust learner $\learner$, i.e., such that $\sup_{\model' \in \models} \Reg(T; \model', \learner) = \oh(T)$, such that $\Reg(T; \model, \learner) = \OH(1)$;
            \end{enumerate}
        \end{theorem}
    \end{blackblock}

    Each characterization in \Cref{theorem_characterization_explorative} is to be understood as follows.

    The first characterization (\textit{1.}) is simply the definition from the main text (\Cref{definition_explorative}): A Markov decision process is explorative if the regret of exploration of no-regret algorithms with sub-linearly many episodes is well-defined.
    The other characterizations relate the concept of explorative MDPs to more common settings. 
    The second characterization (\textit{2.}) is computational. 
    A Markov decision process is explorative if its confusing set, given by\footnote{The notation ``$\model \ll \model^\dagger$'' is about the absolute continuity of $\model$ with respect to $\model^\dagger$. It means that $\reward(\pair) \ll \reward^\dagger(\pair)$ and $\kernel(\pair) \ll \kernel^\dagger(\pair)$ for all $\pair \in \pairs$.}
    \begin{equation}
    \label{equation_confusing_set}
        \confusing(\model) 
        :=
        \braces[\Big]{
            \model^\dagger \in \models:
            \model \ll \model^\dagger
            , 
            \model = \model^\dagger \mathrm{~on~} \optpairs(\model)
            ,
            \optpolicies(\model^\dagger) \cap \optpolicies(\model) = \emptyset
        },
    \end{equation}
    is empty.
    The confusing set is a natural object that arises in instance dependent approaches to regret minimization, see \cite{lai_asymptotically_1985,burnetas_optimal_1997,tranos_regret_2021,boone_regret_2025}---although as shown by the fourth characterization in \Cref{theorem_characterization_explorative}, it is also linked to instance independent frameworks.
    In practice, the second characterization provides a simple way to test if a Markov decision process is explorative. 
    The third characterization~(\textit{3.}) is relevant to regret minimization in the instance dependent setting.
    It is known that most MDPs are such that consistent learners satisfy $\Reg(T; \model, \learner) = \Omega(\log(T))$. 
    However, non-explorative MDPs are those for which it is somehow possible to have regret $\oh(\log(T))$.
    The fourth characterization (\textit{4.}) is relevant to regret minimization in the problem independent (or minimax) setting, stating that we can find robust learners with bounded regret on $\model$. 

    \paragraph{Outline}
    The main goal of this appendix is to establish \Cref{theorem_characterization_explorative}.
    \Cref{theorem_characterization_explorative} provides many characterizations of $\models^+$, but each of them come with long-winded and exhausting proofs, especially if written in full details. 
    So, we begin by providing some intuition on why explorative MDPs are necessary in the first place. 
    In \Cref{appendix_exploration_counter_example}, we describe a simple Markov decision process that is \emph{not} explorative, and for which we show that the famous \texttt{UCRL2} of \cite{auer_near_optimal_2009} has bounded regret. 
    Because the regret of exploration is the object of focus in this work, we believe that the most important part of \Cref{theorem_characterization_explorative} is to show that if $\confusing(\model) \ne \emptyset$, then $\model$ is explorative and the regret of exploration is well-defined. 
    Therefore, \Cref{appendix_confusing_implies_explorative} is dedicated to a fully detailed proof of this result, hence providing a simple condition under which the analysis of the regret of exploration is meaningful in the first place.
    The remaining equivalences of \Cref{theorem_characterization_explorative} are bonus.
    In \Cref{appendix_confusing_implies_regret}, we show that Markov decision processes with non-empty confusing set cannot be learned trivially: the regret of consistent learning algorithms must grow logarithmically with $T$ (\Cref{proposition_confusing_consistent}) and the regret of robust learning algorithms must be unbounded (\Cref{proposition_confusing_robust}).
    We consider the converse results in \Cref{appendix_empty_confusing_implies_non_explorative}, showing that when the confusing set of $\model$ is empty, then $\model$ is non-explorative, that there are robust algorithms with bounded regret on $\model$, and consistent learning algorithms with sub-logarithmic regret on $\model$.
    This completes the proof of \Cref{theorem_characterization_explorative}.

    In \Cref{appendix_interior_are_explorative}, we discuss how common the property ``$\model \in \models^+$'' is.
    We explain that it depends on the amount of structure of $\models^+$: $\models \setminus \models^+$ can be large if $\models$ is heavily structured, and small otherwise.
    In particular, we show that all non-degenerate interior models (\Cref{assumption_interior}) are explorative in the ambient set of all Markov decision processes (see \Cref{proposition_interior_models_are_explorative}). 

    \subsection{An example of non-explorative Markov decision process}
    \label{appendix_exploration_counter_example}

    Not every Markov decision process is explorative, and as a matter of fact, they are easily found. 
    Such MDPs can be learned within a finite exploration phase because efficient learning algorithm can eliminated sub-optimal policies just by having information on optimal ones. 
    In \Cref{figure_non_explorative}, we provide an example of a non-explorative environment. 

    \begin{figure}[ht]
        \centering
        \begin{tikzpicture}
            \node[circle, draw] (0) at (0*72-36:2) {$0$};
            \node[circle, draw] (1) at (1*72-36:2) {$1$};
            \node[circle, draw] (2) at (2*72-36:2) {$2$};
            \node[circle, draw] (3) at (3*72-36:2) {$3$};
            \node[circle, draw] (4) at (4*72-36:2) {$4$};
            \node[circle, draw] (5) at (4, 0) {$5$};
    
            \draw[->,>=stealth] (0) to node[midway,left] {$0.1$} (1);
            \draw[->,>=stealth] (1) to node[midway,above] {$0.9$} (2);
            \draw[->,>=stealth] (2) to node[midway,left] {$0.9$} (3);
            \draw[->,>=stealth] (3) to node[midway,left] {$0.9$} (4);
            \draw[->,>=stealth] (4) to node[midway,below] {$0.9$} (0);
    
            \draw[->,>=stealth,dashed] (1) to node[midway,above] {$0.95$} (5);
            \draw[->,>=stealth,dashed] (5) to node[midway,below] {$0.8$} (0);
        \end{tikzpicture}
        \caption{
            \label{figure_non_explorative}
            An example of a non-explorative Markov decision process.
            From all states, there is a single choice of action excepted at the marked state ($*$) where there are two actions (dashed and solid lines).
            Choices of action deterministically lead to the state indicated by the arrow.
            Rewards are Bernoulli, with means indicated by the labels.
            \vspace{-1.5em}
        }
    \end{figure}
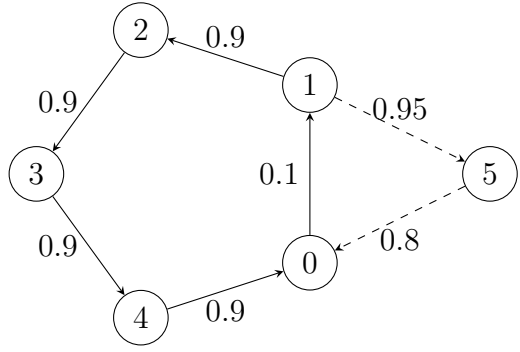

    \paragraph{Notations and intuition}
    Let $\kernel$ be the transition kernel of $\model$ as described by \Cref{figure_non_explorative}.
    Let 
    \begin{equation*}
        \models 
        := 
        \braces[\big]{
            \model' 
            : 
            \forall \pair \in \pairs,
            ~
            \kernel'(\pair) = \kernel(\pair) \mathrm{~and~} \reward(\pair) \in [0, 1]
        }
    \end{equation*}
    be the set of Bernoulli-reward Markov decision processes with the same transition structure than $\model$. 
    On $\model$, there are two policies $\policy^*$ and $\policy^-$, respectively looping on the 5-cycle or the 3-cycle.
    By looping on the 5-cycle, the algorithm learns its rewards very well, hence can claim that the 3-cycle's average reward is upper bounded by $\frac{1+1+0.1+\varepsilon_t}{3}$ because unknown rewards are bounded by 1. This is smaller than a lower bound for the 5-cycle $\frac{0.9+ 0.9+0.9+0.9+0.1 - \varepsilon_t}{5}$ (where $\varepsilon_t$ is vanishing with $t$).
    Therefore, the algorithm has no need to visit the dashed arrows infinitely often.
    What we have just justified is that: (1) there is no $\model' \in \models$ that coincide with $\model$ on the 5-cycle, which is such that $\optpolicies(\model') \ne \braces{\policy^*}$, meaning that $\confusing(\model) = \emptyset$ and echoing the characterization (\textit{2.}) of \Cref{theorem_characterization_explorative}; (2) the property $\confusing(\model) = \emptyset$ can be exploited by some optimistic algorithm to have uncommonly small regret on $\model$ specifically, echoing the characterizations (\textit{3.}) and (\textit{4.}) of \Cref{theorem_characterization_explorative}. 
    We show this second point more formally with \texttt{UCYCLE} \cite{ortner_online_2010}, a variant of \texttt{UCRL2} that is specialized to learning deterministic transition Markov decision processes such as in \Cref{figure_non_explorative}.

    \begin{algorithm}[H]
        \begin{equation*}
            \rewards(t)
            := 
            \prod_{\pair \in \pairs} 
            \braces*{
                \tilde{\reward}(\pair) \in [0, 1]
                :
                \tilde{\reward}(\pair) 
                \le 
                \hat{\reward}_t(\pair) + \sqrt{\tfrac{2 \log(SAt)}{\visits_\pair(t)}}
            }
            \quad \textrm{and} \quad
            \kernels(t) := \braces*{\kernel}
            .
        \end{equation*}
        \begin{algorithmic}[1]
            \STATE $k \gets 0$, initialize $\policy^0$;
            \FOR{$t=0, 1, \ldots$}
                \IF{\eqref{equation_doubling_trick} triggers}
                    \STATE
                        $k \gets k+1$; $t_k \gets t$;
                    \STATE 
                        $\policy_{t_k} \gets \texttt{EVI}(\models(t_k), 0, 0^\states)$;
                \ENDIF
                \STATE
                    Set $\policy_t \gets \policy_{t_k}$ and play $\Action_t \gets \policy_t(\State_t)$.
            \ENDFOR
        \end{algorithmic}
        \caption{
            \label{algorithm_ucycle}
            \texttt{UCYCLE}: \texttt{UCRL2} for deterministic transition models
        }
    \end{algorithm}

    To be absolutely accurate, \Cref{algorithm_ucycle} is not exactly the same algorithm as \cite{ortner_online_2010}, that we have simplified to ease the exposition. 
    It is essentially the same algorithm as \texttt{UCRL2} of \cite{auer_near_optimal_2009} with prior information on the transition kernel of $\model$. 
    The proof of its model independent regret guarantees on $\models$ can be directly adapted from \cite{ortner_online_2010}, or from our own \Cref{appendix_minimax} by removing the error terms relative to the learning of transition kernels.

    \begin{proposition}
    \label{proposition_bounded_regret}
        \texttt{UCYCLE} (see \Cref{algorithm_ucycle}) is robust on $\models$, with
        \begin{equation*}
            \sup_{\model' \in \models} \Reg(T; \model', \texttt{UCYCLE}) 
            = 
            \OH\parens*{
                \tsqrt{\abs{\pairs} T \log(T)}
            }
            .
        \end{equation*}
        Moreover, for $\model'$ as given by \Cref{figure_non_explorative}, we have $\Reg(T; \model, \texttt{UCYCLE}) = \OH(1)$.
    \end{proposition}

    \paragraph{Remark}
    The example of \Cref{figure_non_explorative} is robust to reward perturbation.
    It means that it is {non-degenerate} in the sense of \Cref{definition_non_degeneracy}.
    It follows that by identifying $\model' \in \models$ as a reward vector $\reward' \in [0, 1]^\pairs$, it means the set $\models'$ of $\model' \in \models$ where the 5-cycle dominates the 3-cycle in the fashion described above has positive Lebesgue measure. 
    As $\models^+ \supseteq \models'$, it follows that a large portion of $\models$ is made of non-explorative Markov decision processes: By picking $\reward'$ uniformly at random, the obtained MDP is non-explorative with positive probability. 

    \bigskip
    \def\proofname{Proof of \Cref{proposition_bounded_regret}}
    \begin{proof}
        The assertion on the model independent regret guarantees is well-known, see \cite{ortner_online_2010} and \Cref{appendix_minimax}.
        We focus on proving that it has bounded regret on the model $\model$ given in \Cref{figure_non_explorative}.

        The model $\model$ is identified with its reward vector $\reward$. 
        Remark that the only pair with positive Bellman gap is $(1, \dagger)$ with Bellman gap $\ogaps(1, \dagger) \le 1$. 
        So, the regret is upper-bounded by $\abs*{\braces{t \le T : \policy_t = \policy^-}}$.
        We are left to bound how many times the sub-optimal policy $\policy^-$ is played. 
        A simple property induced by the doubling trick \eqref{equation_doubling_trick} is that $t_{k+1} \le 3t_k$. 
        So, if $\policy_t = \policy^-$, then there exists $t' \in [\frac 13 t, t]$ such that $\policy^-$ is the result of \texttt{EVI}, i.e., $g^{\policy^-}(\models({t'})) > g^{\policy^*}(\rewards({t'})) + \frac 1{t'}$. 

        Let $c := 3 \cdot \frac{0.9+0.9+0.9+0.9+0.1}{5} - 2 = 0.22$, which is the threshold on the reward that one should have on $(0, *)$ in order to make $\policy^-$ better than $\policy^*$. 
        Since
        \begin{align*}
            g^{\policy^-}(\models(t))
            & \le 
            \frac 13\parens*{
                2 + \hat r_t(0, *) + \sqrt{\frac{2\log(\abs{\pairs}t)}{\visits_{0, *}(t)}}
            }
            \text{~and~}
            g^{\policy^*}(\models(t)) \ge \indicator{r \in \rewards(t)} g^{\policy^*}(\model),
        \end{align*}
        we have:
        \begin{align}
            \nonumber
            (*) & :=
            \EE \abs*{\braces*{
                t \ge 1: \policy_t \ne \policy^{-}
            }}
            \\
            \nonumber
            & \le 300 + 
            \sum_{t \ge 300}
            \sum_{t'=t/3}^t
            \Pr \parens*{
                g^{\policy^-}(\models({t'})) > g^{\policy^*}(\models(t')) + \tfrac 1{100}
            }
            \\
            & 
            \label{appendix:equation:toto24}
            \le 300 + 
            \sum_{t\ge 300}
            \sum_{t' = t/3}^t
            \parens*{
                \Pr \parens*{
                    \hat r_{t'}(0, *) + \sqrt{\frac{2\log(\abs{\pairs}t')}{\visits_{0, *}(t')}} > 0.21
                }
                + 
                \Pr \parens*{
                    \model \notin \models(t')
                }
            }.
        \end{align}
        For the first term, remark that $\visits_{0, *}(t') \ge \frac 15t'$ almost surely when $t' \ge 5$. 
        For $t'$ large enough so that $\sqrt{10\log(\abs{\pairs}t')/t'} < 0.01$, we have
        \begin{align*}
            (**)
            & :=
            \Pr \parens*{
                \hat{\reward}_{t'}(0, *) + \sqrt{\frac{2\log(\abs{\pairs}t')}{\visits_{0, *}(t')}} > 0.21
            } 
            \\
            & \le 
            \Pr \parens*{
                \exists n \in [\tfrac 15t', t']:
                \visits_{0, *}(t') = n, \hat{\reward}_{t'}(0, *) + \sqrt{\tfrac{2\log(\abs{\pairs}t')}{n}} > 0.21
            }
            \\
            & \le 
            \sum_{n=\frac15t'}^\infty
            \Pr\parens*{
                \visits_{0, *}(t') = n, \hat{\reward}_{t'}(0, *) - r(0, *) > 0.2
            } \\
            & \overset{(\dagger)}\le 
            \sum_{n=\frac15t'}^\infty 
            \exp\parens*{
                -\frac 8{10000}n 
            } = \frac{\exp\parens*{-\frac 1{6250}t'}}{1 - \exp\parens*{-\frac 1{1250}}} = \OH\parens*{\exp(-\tfrac 1{6250} t')}
        \end{align*}
        where $(\dagger)$ follows from Azuma-Hoeffding's inequality.
        For the second term, we have 
        \begin{align*}
            \Pr\parens*{\model \notin \models(t')}
            & = \Pr \parens*{
                \exists \pair \in \pairs, 
                \abs*{\hat r_{t'}(\pair) - r_z} > \sqrt{\frac{2\log(\abs{\pairs}t')}{\visits_\pair(t')}}
            }
            \\
            & \le 
            \sum_z 
            \sum_{n=1}^\infty 
            \Pr \parens*{
                N_{\pair}(t') = n, 
                \abs*{\hat r_{t'}(\pair) - r(z)} > \sqrt{\frac{2\log(\abs{\pairs}t')}{n}}
            } \\
            & \overset{(\dagger)}\le 2\abs{\pairs} \sum_{n=1}^\infty \exp\parens*{-4\log(\abs{\pairs}t') \cdot n}
            \\
            & \le \frac{2\abs{\pairs}}{(t'\abs{\pairs})^4} \cdot \frac 1{1 - (t'\abs{\pairs})^{-4}} \le \frac 4{\abs{\pairs}^3t'^4} = \OH\parens*{t'^{-4}}. 
        \end{align*}
        where $(\dagger)$ follows from Azuma-Hoeffding's inequality. 
        Overall, injecting it all in \eqref{appendix:equation:toto24}, we obtain $\EE\abs*{\braces{t \ge 1 : \policy_t \ne \policy^-}} < \infty$. 
        We conclude accordingly that $\Reg(T; \model, \texttt{UCYCLE}) = \OH(1)$. 
    \end{proof}
    \def\proofname{Proof}

    \subsection{MDPs with non-empty confusing sets are explorative}
    \label{appendix_confusing_implies_explorative}

    In this section, we show that Markov decision processes with non-empty confusing set are explorative, see \Cref{proposition_well_defined}.
    This is (\textit{1.}) $\Rightarrow$ (\textit{2.}) in \Cref{theorem_characterization_explorative} for which we show the transposition $\neg$(\textit{2.}) $\Rightarrow$ $\neg$(\textit{1.}).
    This result is absolutely necessary to justify that the analysis of the regret of exploration is formally based. 

    \begin{proposition}
    \label{proposition_well_defined}
        Let $\model \in \models$.
        If $\confusing(\model) \ne \emptyset$, then every no-regret algorithm $\learner$ with sub-linearly many episodes has infinitely many exploration episodes on $\model$, almost surely. 
    \end{proposition}

    \paragraph{Proof sketch}
    Recall that, by definition, $k \ge 1$ is an \strong{exploration episode} if (1) $g^*(M) = g(\pi^k, S_{t_k}, M)$ and (2) $\Reach(\pi^k, S_{t_k}, M) \cap \pairs^-(M) \ne \emptyset$, see (\Cref{definition_exploration_episode}).
    In order to show that there are infinitely many exploration episodes, we have to show that the learning process alternates infinitely often between periods of times when the played policy is gain optimal, and others when there is a reachable sub-optimal pair.
    \STEP{1} is a preliminary technical fact.
    In \STEP{2}, we show with \eqref{equation_well_definition_2A} that the process is infinitely many times on the recurrent part of a gain optimal policy.
    In \STEP{2}, we show with \eqref{equation_well_definition_3_0} that the process must play sub-optimal pairs infinitely often.
    Combining both in \STEP{4}, we show that the number of exploration times is infinite, and each are finite with probability one.

    \paragraph{Notations}
    For $\policy \in \policies$, we write $\mathrm{Rec}(\policy)$ the set of states that are recurrent under $\policy$ on $M$, i.e., $\state \in \states$ such that $\Pr_{\state}^{\policy, \model}(\forall m, \exists n \ge m: \State_n = \state) = 1$.

    \bigskip
    \noindent
    \STEP{1}
    \textit{
        For every model $M \in \models$, there exists a constant $C(M) > 0$ such that whatever the learning algorithm, we have:
        \begin{equation}
        \label{equation_well_definition_1A}
            \EE^M\brackets*{
                \sum_{t=1}^T \parens*{
                    g^*(S_t, M) - g^{\pi_t}(\State_t, M)
                    + \indicator{S_t \notin \mathrm{Rec}(\pi_t)}
                }
            }
            \le
            \Reg(T; \model)
            + C(M) \EE^M \abs*{\episodes(T)}
            .
        \end{equation}
    }
    \begin{proof}
        In the proof below, we drop the dependency in $M$ in the notations. 
        If $\pi \in \Pi$, we denote $\mathrm{Rec}(\pi)$ the recurrent states of $\pi$ in $M$.
        We have:
        \begin{align*}
            (*) 
            & = \Reg(T; \model)
            \\
            & = 
            \EE\brackets*{
                \sum_{t=1}^T \Delta^*(Z_t)
            }
            \\
            & \overset{(\dagger)}= 
            \EE \brackets*{
                \sum_{t=1}^T 
                \parens*{g^*(S_t) - r(Z_t) + \parens*{e_{S_t} - p(Z_t)} h^*}
            }
            \\
            & \ge
            \EE \brackets*{
                \sum_{t=1}^T (g^* - r(Z_t))
            }
            - \vecspan{h^*}
            \\
            & \overset{(\ddagger)}\ge 
            \underbrace{
                \EE \brackets*{
                    \sum_{k=1}^{\abs*{\episodes(T)}} \sum_{t=t_k}^{t_{k+1}-1}
                    \eqindicator{S_t \in \mathrm{Rec}(\pi_t)}
                    (g^*(S_t) - r(Z_t))
                }
            }_{\mathrm{A}}
            - \underbrace{
                \EE \brackets*{
                    \sum_{k=1}^{\abs*{\episodes(T)}} \sum_{t=t_k}^{t_{k+1}-1}
                    \eqindicator{S_t \notin \mathrm{Rec}(\pi_t)}
                }
            }_{\mathrm{B}}
            \\
            & \phantom{=} - \vecspan{h^*}
        \end{align*}
        where $(\dagger)$ uses the Bellman equation $h^*(s) + g^*(s) = r(s,a) + p(s,a)h^* + \Delta^*(s,a)$, and $(\ddagger)$ uses that $g^*(S_t) - r(Z_t) \ge -1$. 
        We bound A and B separately. 
        Let $D_* := \max_{\pi} \max_s \EE_s^\pi[\inf\braces{t \ge 1 : S_t \in \mathrm{Rec}(\pi)}] < \infty$ be the worst hitting time to a recurrent class in $M$. 
        We have:
        \begin{equation}
        \label{equation_well_definition_1B}
            \mathrm{B}
            =
            \EE \brackets*{
                \sum_{k=1}^{\abs*{\episodes(T)}} \sum_{t=t_k}^{t_{k+1}-1}
                \inf\braces*{t > t_k : S_t \in \mathrm{Rec}(\pi^k)}
            }
            \le
            D_* \EE[\abs*{\episodes(T)}]
            .
        \end{equation}
        Meanwhile, introduce $t'_k := t_{k+1} \wedge \inf \braces{t > t_k: S_t \in \mathrm{Rec}(\pi^k)}$ and $H := \max_\pi \vecspan{h^\pi} < \infty$ the worst bias span.
        We have:
        \begin{align}
            \mathrm{A}
            & = \nonumber
            \EE \brackets*{
                \sum_{k=1}^{\abs*{\episodes(T)}} \sum_{t=t'_k}^{t_{k+1}-1}
                (g^*(S_t) - r(Z_t))
            }
            \\
            & \overset{(\dagger)}= \nonumber
            \EE \brackets*{
                \sum_{k=1}^{\abs*{\episodes(T)}} \sum_{t=t'_k}^{t_{k+1}-1}
                \parens*{
                    g^*(S_t) - g^{\pi^k}(S_t) + \parens*{p(Z_t) - e_{S_t}}h^{\pi^k}
                }
            }
            \\
            & \ge \nonumber
            \EE \brackets*{
                \sum_{k=1}^{\abs*{\episodes(T)}} \sum_{t=t'_k}^{t_{k+1}-1}
                \parens*{
                    g^*(S_t) - g^{\policy_t}(S_t) 
                }
            }
            - H \EE[\abs*{\episodes(T)}]
            \\
            & \overset{(\ddagger)}\ge
        \label{equation_well_definition_1C}
            \EE \brackets*{
                \sum_{t=1}^T
                \parens*{
                    g^*(S_t) - g^{\policy_t}(S_t) 
                }
            }
            - H \EE[\abs*{\episodes(T)}]
        \end{align}
        where $(\dagger)$ uses the Poisson equation $h^{\pi^k}(s) + g^{\pi^k}(s) = r(s, \pi^k(s)) + p(s, \pi^k(s)) h^{\pi^k}$ and $(\ddagger)$ that $g^*(S_t) \ge g^{\pi_t}(S_t)$ for all $t \ge 1$.
        Combining \eqref{equation_well_definition_1B} and \eqref{equation_well_definition_1C}, we get:
        \begin{equation*}
            \EE \brackets*{
                \sum_{t=1}^T
                \parens*{
                    g^*(S_t) - g^{\policy_t}(S_t) 
                }
            }
            +
            \EE\brackets*{
                \sum_{t=1}^T \eqindicator{S_t \notin \mathrm{Rec}(\pi_t)}
            }
            \le
            \Reg(T)
            + (2D_* + H) \EE[\abs*{\episodes(T)}]
            .
        \end{equation*}
        Conclude the proof by setting $C := 2 D_* + H < \infty$.
    \end{proof}

    \noindent
    \STEP{2}
    \textit{
        Assume that the algorithm is no-regret and has sub-linearly many episodes in expectation.
        Then:
        \begin{equation}
        \label{equation_well_definition_2A}
            \Pr \parens*{
                \forall T, \exists t \ge T:
                g^*(S_t, M) = g^{\policy_t}(S_t, M)
                \text{~and~}
                S_t \in \mathrm{Rec}(\pi_t)
            }
            =
            1.
        \end{equation}
    }
    \begin{proof}
        Assume on the contrary that $\Pr(\forall T, \exists t \ge T: g^*(S_t, M) = g^{\policy_t}(S_t, M) \wedge S_t \in \mathrm{Rec}(\pi_t)) = 1 - \delta$ with $\delta > 0$.
        Accordingly, there exists $T_0 \ge 1$ such that:
        \begin{equation*}
            \Pr \parens*{
                \forall t \ge T_0,
                g_{S_t}^*(M) > g^\policy_t(S_t, M)
                \text{~or~}
                S_t \notin \mathrm{Rec}(\pi_t)
            }
            \ge \tfrac 12 \delta.
        \end{equation*}
        Let $\Delta_g := \min \braces{g^*(s, M) - g^\pi(s, M): \pi \in \Pi, s \in \states, g^*(s, M) > g^\pi(s, M)}$ be the gain-gap of $M$.
        We have $\Delta_g \in (0, 1]$ and thus:
        \begin{align*}
            (*) 
            & := \EE\brackets*{
                \sum_{t=1}^T 
                \parens*{g^*(S_t, M) - g^{\pi_t}(S_t, M)}
            }
            + \EE \brackets*{
                \sum_{t=1}^T \eqindicator{S_t \notin \mathrm{Rec}(\pi_t)}
            }
            \\
            & \ge
            \Delta_g
            \EE \brackets*{
                \sum_{t=1}^T
                \eqindicator{
                    g^*(S_t, M) > g^{\pi_t}(S_t, M)
                    \text{~or~}
                    S_t \notin \mathrm{Rec}(\pi_t)
                }
            }
            \\
            & \ge 
            \Delta_g (T - T_0) \Pr \parens*{
                \forall t \ge T_0,
                g^*(S_t, M) > g^{\pi_t}(S_t , M)
                \text{~or~}
                S_t \notin \mathrm{Rec}(\pi_t)
            }
            \\
            & \ge 
            \tfrac 12 \Delta_g \delta (T - T_0)
            = \Omega(T)
            .
        \end{align*}
        Meanwhile, we know that $\Reg(T; \model) = \oh(T)$ and $\EE[\abs*{\episodes(T)}] = \oh(T)$, so that by \STEP{1} \eqref{equation_well_definition_1A}, we also have $(*) = \oh(T)$, a contradiction. 
    \end{proof}

    \noindent
    \STEP{3}
    \textit{
        If $\confusing(M) \ne \emptyset$, then every no-regret algorithm satisfies
        \begin{equation}
        \label{equation_well_definition_3_0}
            \Pr^M \parens*{
                \forall T, \exists t > T: \Delta^*(Z_t) > 0
            }
            = 
            1.
        \end{equation}
    }
    \begin{proof}
        On the contrary, assume that $\Pr^M( \forall T, \exists t > T: \Delta^*(Z_t) > 0) = 1 - \delta$ with $\delta > 0$.
        Accordingly, there exists $m \ge 1$ such that:
        \begin{equation}
        \label{equation_well_definition_3A}
            \tfrac 12 \delta 
            \le 
            \Pr^M \parens*{
                \forall t > m: \Delta^*(Z_t) = 0
            } 
            \le
            \Pr^M \parens*{
                \forall t \ge 1:
                \sum_{z \in \pairs^-(M)} N_t(\pair) \le m
            }
            .
        \end{equation}
        We show that $z \in \pairs^-(M)$ can be changed to $z \notin \pairs^{**}(M)$ in \eqref{equation_well_definition_3A}, see \eqref{equation_well_definition_3B}. 
        To see this, introduce the reward function $f(z) := \indicator{z \in \pairs^{**}(M)}$ and let $g^f, h^f$ and $\Delta^f$ the respective gain, bias and gap functions of the optimal policy $\pi^*$ of $M$ (defined by $\pi^*(s) = a$ the unique $a \in \actions(s)$ such that $(s,a) \in \pairs^{*}(M)$) under reward function $f$ and kernel $p(M)$.
        Remark that $g^f(s) = 1$ for all $s \in \states$ and that, by construction of $\pi^*$, $\Delta^f(z) = 0$ for all $z \in \pairs^{*}(M)$.
        Denote $H^f := \mathrm{sp}\parens*{h^f} \vee \max_z \abs*{\Delta^f(z)}$.
        We have:
        \begin{align*}
            & \sum_{z \in \pairs^{**}(M)}
            N_z(T)
            \\
            & = \sum_{t=1}^T f(Z_t)
            \\
            & \overset{(\dagger)}= 
            \sum_{t=1}^T \parens*{1 + \parens*{e_{S_t} - p(Z_t)}h^f - \Delta^f(Z_t)}
            \\
            & \ge 
            T - H^f
            - \sum_{t=1}^T \Delta^f(Z_t)
            + \sum_{t=1}^T \parens*{e_{S_{t+1}} - p(Z_t)} h^f
            \\
            & \overset{(\ddagger)}\ge 
            T - H^f
            - H^f \sum_{t=1}^T \indicator{Z_t \notin \pairs^*(M)}
            + \sum_{t=1}^T \indicator{Z_t \notin \pairs^{**}(M)} \parens*{e_{S_{t+1}} - p(Z_t)} h^f
        \end{align*}
        where $(\dagger)$ uses the Bellman equation $1 + h^f(s) = f(s,a) + p^f(s,a) h^f + \Delta^f(s,a)$, and $(\ddagger)$ that $h^f(s) = 0$ for all $(s, \pi^*(s)) \in \pairs^{**}(M)$.
        For $\pairs' \subseteq \pairs$, denote $N_T({\pairs'}) := \sum_{z \in \pairs'} N_T(z)$.
        The first sum is equal to $\sum_{t=1}^T \indicator{Z_t \notin \pairs^{*}(M)} = N_T({\pairs^-(M)})$.
        The RHS of the above equation is bounded using a time-uniform Azuma-Hoeffding inequality (see \cite[Lemma~5]{bourel_tightening_2020}), showing that:
        \begin{equation*}
            \Pr \parens*{
                \exists T \ge 1:
                {
                    \sum_{t=1}^T \indicator{Z_t \notin \pairs^{**}(M)} \parens*{e_{S_{t+1}} - p(Z_t)} h^f
                    \atop
                    <
                    - H^f \sqrt{
                        N_T({\pairs^{**}(M)^c})
                        \log\parens*{
                            \tfrac{4N_{\pairs^{**}(M)^c}(T)}{\delta}
                        }
                    }
                }
            }
            \le
            \tfrac 14 \delta
        \end{equation*}
        Using that $N_T({\pairs^{**}(M)^c}) = T - N_T({\pairs^{**}(M)})$, we obtain that, with probability at least $\tfrac 14 \delta$, for all $T \ge 1$, we have:
        \begin{align*}
            & T - N_T({\pairs^{**}(M)^c})
            \\
            & \ge
            T 
            - H^f \parens*{1 + N_T({\pairs^-(M)})}
            - H^f \sqrt{
                N_T({\pairs^{**}(M)^c})
                \log\parens*{
                    \tfrac{4N_T({\pairs^{**}(M)^c})}{\delta}
                }
            }
            \\
            & \ge 
            T 
            - H^f \parens*{1 + m}
            - H^f \sqrt{
                N_T({\pairs^{**}(M)^c})
                \log\parens*{
                    \tfrac{4N_T({\pairs^{**}(M)^c})}{\delta}
                }
            }
            .
        \end{align*}
        Rearranging terms, we get that with probability at least $\tfrac 14\delta$, for all $T \ge 1$, we have:
        \begin{align*}
            & N_T({\pairs^{**}(M)^c})
            \\
            & \le
            H^f \parens*{
                1 + m
                + \sqrt{
                    N_T({\pairs^{**}(M)^c})
                    \log\parens*{ N_T({\pairs^{**}(M)^c}) }
                }
                + \sqrt{
                    N_T({\pairs^{**}(M)^c})
                    \log\parens*{\tfrac 4\delta}
                }
            }
            .
        \end{align*}
        Denoting $n := N_T({\pairs^{**}(M)})$, we have an equation of the form $n \le \alpha + \beta \sqrt{n \log(n)} + \gamma \sqrt{n}$. 
        For $n \ge 3$, $n \log(n) \ge n$ hence we can simplify the upper-bound to $n \le \alpha + (\beta + \gamma) \sqrt{n \log(n)}$.
        Dividing by $\log(n) \ge 1$ and setting $m := n/\log(n)$, we get $m \le \alpha + (\beta + \gamma) \sqrt{m}$, and simple algebra leads to:
        \begin{equation*}
            \frac{n}{\log(n)} = m \le 2\parens*{\alpha + (\beta + \gamma)^2}.
        \end{equation*}
        Further using $\log(n) \le \sqrt{n}$, we get $n \le 4(\alpha + (\beta+\gamma)^2)^2$.
        We conclude that there exists a constant $m'$ such that:
        \begin{equation}
        \label{equation_well_definition_3B}
            \Pr^{M} \parens*{
                \forall t \ge 1,
                \sum_{z \notin \pairs^{**}(M)}
                N_t(z)
                \le m'
            } \ge \tfrac 14 \delta.
        \end{equation}
        Now that \eqref{equation_well_definition_3B} is established, we finally derive a contradiction by relying on a change of measure argument.
        Let $M^\dagger \in \confusing(M)$, which is non-empty by assumption.
        For short, the transition kernels and reward distributions of $M$ (respectively $M^\dagger$) are denoted $p$ and $r$ (respectively $p^\dagger$ and $r^\dagger$).
        We introduce the log-likelihood-ratio of observations $H_t := (S_t, A_t, R_1, \ldots, A_{t-1}, R_{t-1}, S_t)$ as:
        \begin{equation*}
            L(t)
            \equiv L(H_t)
            :=
            \sum_{s,a} \sum_{i < t-1} 
            \indicator{S_i = s, A_t = a}
            \log\parens*{
                \frac{r_{s,a}(R_i)}{r_{s,a}^\dagger(R_i)}
                \frac{p_{s,a}(S_{i+1})}{p^\dagger_{s,a}(S_{i+1})}
            }
            .
        \end{equation*}
        It is known since \cite{marjani_navigating_2021} that if $\event$ is a $\sigma(H_t)$-measurable event, then $\Pr^{M^\dagger}(\event) = \EE^M[\indicator{\event} \exp(-L(t))]$.
        Since $M \ll M^\dagger$, there exists a constant $c > 0$ such that, for all $z \in \pairs$, we have $\log[(r_z(\alpha)/r_z^\dagger(\alpha)) \cdot (p_z(s')/p_z^\dagger(s'))] \le \log(c)$ with the convention $0/0 = 0$.
        For $z \in \pairs^{**}(M)$, the LHS logarithm is null.
        Therefore, we have:
        \begin{align*}
            & \Pr^{M^\dagger} \parens*{
                \sum_{z \notin \pairs^{**}(M)} N_t(z) \le m'
            }
            \\
            & =
            \EE^M \brackets*{
                \eqindicator{ \sum_{z \notin \pairs^{**}(M)} N_t(z) \le m' }
                \exp \parens*{
                    - L(t)
                }
            }
            \\
            & \ge
            \EE^M \brackets*{
                \eqindicator{ \sum_{z \notin \pairs^{**}(M)} N_t(z) \le m' }
                \exp \parens*{
                    - \sum_{z \notin \pairs^{**}(M)}
                    N_t(z) \log(c)
                }
            }
            \\
            & \ge c^{-m'}
            \Pr^{M} \parens*{
                \sum_{z \notin \pairs^{**}(M)} N_t(z) \le m'
            }
            \ge
            c^{-m'} \delta := \delta' > 0.
        \end{align*}
        Accordingly, the algorithm has probability at least $\delta'$ to spend at most $m'$ visits outside $\pairs^{**}(M)$ when running on $M^\dagger$.
        This will be in contradiction $M^\dagger \in \confusing(M)$ and the consistency of the algorithm.
        Indeed, since $M^\dagger \gg M$ coincides with $M$ on $\pairs^{**}(M)$, we see that the optimal policy $\pi^*$ of $M$ has unique recurrent class $\pairs^{**}(M)$ in $M^\dagger$.
        Yet, $\pi^* \notin \Pi^*(M^\dagger)$, hence $\pairs^{**}(M) \cap \pairs^-(M^\dagger) \ne \emptyset$, i.e., there exists $z \in \pairs^{**}(M)$ such that $\Delta^*(z; M^\dagger) > 0$.
        We further link the number of visits of this $z$ to the total number of visits of $\pairs^{**}(M)$ with the same technique that the one used to convert \eqref{equation_well_definition_3A} to \eqref{equation_well_definition_3B}.
        
        Introduce the reward function $f(z') := \indicator{z' = z}$, and let $g^f, h^f, \Delta^f$ be the gain, bias and gaps functions of the policy $\pi^*$ in $M^\dagger$.
        There exists $\epsilon > 0$ such that $g^f(s) = \epsilon$ for all $s \in \states$.
        Letting $C := \mathrm{sp}\parens*{h^f} \vee \max_{z'} \abs*{\Delta^f(z')} < \infty$.
        For all $T \ge 1$, we have
        \begin{align*}
            N_T(z)
            & = 
            \sum_{t=1}^T f(Z_t)
            = \sum_{t=1}^T 
            \parens*{
                \epsilon + \parens*{e_{S_t} - p(Z_t)}h^f - \Delta^f(Z_t)
            }
            \\
            & \ge
            T \epsilon 
            - C - C N_{\pairs^{**}(M)^c}(T)
            + \sum_{t=1}^T \parens*{e_{S_{t+1}} - p(Z_t)} h^f
            \\
            & \overset{(\dagger)}\ge
            T \epsilon - C (1 + m') - C \sqrt{T \log\parens*{\tfrac {2T}{\delta'}}}
            \sim T \epsilon
            .
        \end{align*}
        where $(\dagger)$ holds with probability $\tfrac 12 \delta' > 0$ uniformly for $T \ge 1$, by invoking a time-uniform Azuma-Hoeffding (see \cite[Lemma~5]{bourel_tightening_2020}) to lower-bound the right-hand martingale.
        We accordingly obtain, when $T \to \infty$,
        \begin{equation}
        \label{equation_well_definition_3C}
            \Reg(T; \model^\dagger)
            \gtrsim \tfrac 12 \epsilon \delta' \Delta^*(z; M^\dagger) T 
            = \Omega(T)
            .
        \end{equation}
        So \eqref{equation_well_definition_3C} is in contradiction with the consistency of the algorithm.
    \end{proof}

    \noindent
    \STEP{4}
    \textit{
        If the algorithm is no-regret, has sub-linearly many episodes, then for all $M \in \models$ such that $\confusing(M) \ne \emptyset$, we have:
        \begin{equation}
            \Pr^M \parens*{
                \forall T, \exists t \ge T:
                g^*(M) = g^{\pi_{t-1}}(S_{t-1}, M)
                \text{~and~}
                \Reach(\pi_t, S_t, M) \cap \pairs^{-}(M) \ne \emptyset
            } = 1.
        \end{equation}
        Moreover, the stopping times $t$ enumerating times such that $g^*(S_{t-1}, M) = g^{\pi_{t-1}}(S_{t-1}, M)$ and $\Reach(\pi_t, S_t, M) \cap \pairs^-(M) \ne \emptyset$ are exploration times; Hence there are infinitely many of them with probability one.
    }
    \medskip
    \begin{proof}
        This is obtained by combining \eqref{equation_well_definition_2A} of \STEP{2} and \eqref{equation_well_definition_3_0} of \STEP{3}.
        We have:
        \begin{align}
        \label{equation_well_definition_4B}
            \Pr^M \parens*{
                \forall T, \exists t \ge T:
                g^*(S_{t-1},M) = g^{\pi_t}(S_t, M)
                \text{~and~}
                S_t \in \mathrm{Rec}(\pi_t)
            } & = 1, \text{~and}
            \\
        \label{equation_well_definition_4C}
            \Pr^M \parens*{
                \forall T, \exists t \ge T:
                \Reach(\pi_t, S_t, M) \cap \pairs^-(M) \ne \emptyset
            }
            & = 1
            .
        \end{align}
        By non-degeneracy of $M$, if $g^*(S_t, M) = g^{\pi_t}(S_t, M)$ and $S_t \in \mathrm{Rec}(\pi_t)$, then $\Reach(\pi_t, S_t, M) = \pairs^{**}(M)$ which is disjoint from $\pairs^-(M)$. 
        Define:
        \begin{align*}
            \tau_1 
            & := \inf \braces*{
                t \ge 1 : 
                g^*(S_t, M) = g^{\pi_t}(S_t, M) \text{~and~} S_t \in \mathrm{Rec}(\pi_t)
            }
            ,
            \\
            \tau_{2i}
            & := \inf \braces*{
                t > \tau_{2i-1}: 
                \Reach(\pi_t, S_t, M) \cap \pairs^-(M) \ne \emptyset
            }
            , 
            \\
            \tau_{2i+1}
            & := \inf \braces*{
                t > \tau_{2i} : 
                g^*(S_t, M) = g^{\pi_t}(S_t, M) \text{~and~} S_t \in \mathrm{Rec}(\pi_t)
            }
        \end{align*}
        Then $(\tau_i)$ is an increasing sequence of stopping times, and by \eqref{equation_well_definition_4B} \eqref{equation_well_definition_4C} applied in tandem, we show by induction that $\Pr^M(\tau_i < \infty) = 1$ for all $i \ge 1$.
        By non-degeneracy of $M$, at $t = \tau_{2i+1}$, the current policy is gain optimal and the process is currently on the optimal class $\pairs^{**}(M)$.
        Because $\pairs^{**}(M)$ is the disjoint union of sink components of $\pairs^*(M)$, hence the only way to exit $\pairs^{**}(M)$ is by playing a $z \in \pairs^-(M)$.
        Therefore, we see that for $t = \tau_{2i}$, we must have $g^*(S_{t-1}, M) = g^{\pi_{t-1}}(S_{t-1}, M)$ with $\pi_{t-1} \ne \pi_t$. 
        Accordingly, every $\tau_{2i}$ are change of episodes that are exploration episodes.
    \end{proof}

    This proves \Cref{proposition_well_defined}. 
    \hfill
    \QED

    \subsection{Instance (in)dependent regrets for MDPs with non-empty confusing sets}
    \label{appendix_confusing_implies_regret}

    In this section, we show (\textit{3.}) $\Rightarrow$ (\textit{2.}) and (\textit{4.}) $\Rightarrow$ (\textit{2.}) in \Cref{theorem_characterization_explorative} by showing the transpositions $\neg$(\textit{2.}) $\Rightarrow$ $\neg$(\textit{3.}) in \Cref{proposition_confusing_consistent} and $\neg$(\textit{2.}) $\Rightarrow$ $\neg$(\textit{4.}) in \Cref{proposition_confusing_robust}.
    In the statements below, we borrow the terminology introduced by \Cref{theorem_characterization_explorative}.
    A learning algorithm $\learner$ is said \strong{consistent} (on $\models$) if $\Reg(T; \model, \learner) = \oh(T^\epsilon)$ for all $\epsilon > 0$ and $\model \in \models$.
    A learning algorithm $\learner$ is said \strong{robust} (relatively to $\models$) if $\sup_{\model' \in \models} \Reg(T; \model', \learner) = \oh(T)$. 

    \begin{proposition}[Consistent algorithms]
    \label{proposition_confusing_consistent}
        Let $\model \in \models$ be such that $\confusing(\model) \ne \emptyset$. 
        Then every consistent learning algorithm $\learner$ satisfies $\Reg(T; \model, \learner) = \Omega(\log(T))$. 
    \end{proposition}
    \begin{proof}
        This is a consequence of \cite[Corollary~7]{boone_regret_2025}, that shows that every consistent learning algorithm $\learner$ is such that, for all $\model^\dagger \in \confusing(\model)$, we have:
        \begin{equation*}
            \EE^{\model, \learner} \brackets*{
                \sum_{\pair \in \pairs} 
                \visits_\pair(T)
                \KL(\kerrew(\pair)||\kerrew^\dagger(\pair))
            }
            \ge
            \log(T) + \oh(\log(T))
        \end{equation*}
        where $\kerrew(\pair) = (\reward(\pair), \kernel(\pair))$ is the reward-kernel tuple. 

        Fix $\model^\dagger \in \confusing(\model)$.
        Let $c := \max \braces{\KL(\kerrew(\pair)||\kerrew^\dagger(\pair)) : \kerrew(\pair) \ne \kerrew^\dagger(\pair)}$, that satisfies $c < \infty$ since $\model \ll \model^\dagger$.
        By definition of $\confusing(\model)$, $\model = \model^\dagger$ coincide on $\optpairs(\model)$, so $\KL(\kerrew(\pair)||\kerrew^\dagger(\pair)) = 0$ for all $\pair \in \optpairs(\model)$.
        Together with $\EE[\sum_{\pair \in \pairs} \visits_\pair(T) \KL(\kerrew(\pair)||\kerrew^\dagger(\pair))] \le \abs{\pairs} \max_{\pair \in \pairs} \EE[\visits_\pair(T)] \KL(\kerrew(\pair)||\kerrew^\dagger(\pair))$, we deduce that:
        \begin{equation*}
            \forall T \ge 1,
            \quad
            \max_{\pair \notin \optpairs(\model)}
            \EE^{\model, \learner} \brackets*{
                \visits_\pair (T)
            }
            \ge \frac{\log(T) + \oh(\log(T))}{\abs{\pairs}c}.
        \end{equation*}
        Now, by \cite[Lemma~8]{boone_regret_2025}, there exist constants $\alpha, \beta > 0$ such that $\EE^{\model, \learner}[\sum_{t=1}^T \indicator{\Pair_t \notin \optpairs(\model)}] \le \alpha \Reg(T; \model, \learner) + \beta$ for all $T \ge 1$, so
        \begin{equation*}
            \Reg(T; \model, \learner)
            \ge 
            \frac{
                \max_{\pair \notin \optpairs(\model)}
                \EE^{\model, \learner} \brackets*{
                    \visits_\pair (T)
                }
                - \beta
            }{
                \alpha
            }
            \ge
            \frac{
                \log(T) + \oh(\log(T))
            }{
                \abs{\pairs} c \alpha
            },
        \end{equation*}
        hence $\Reg(T; \model, \learner) = \Omega(\log(T))$. 
    \end{proof}

    \begin{proposition}[Robust algorithms]
    \label{proposition_confusing_robust}
        Let $\model \in \models$ be such that $\confusing(\model) \ne \emptyset$. 
        Then every robust learning algorithm $\learner$ satisfies $\Reg(T; \model, \learner) = \omega(1)$, i.e., $\Reg(T; \model, \learner) \to \infty$. 
    \end{proposition}
    \begin{proof}
        This is a direct consequence of (\textbf{STEP 3}) of the proof of \Cref{proposition_well_defined}, see \eqref{equation_well_definition_3_0}.
        Indeed, robust algorithms are by no-regret.
        So, by \eqref{equation_well_definition_3_0}, the function given by
        \begin{equation*}
            f(T)
            := 
            \inf_{T' \ge T}
            \braces*{
                \Pr^{\model, \learner} \parens[\big]{
                    \exists t \in \braces{T, \ldots, T'-1}
                    :
                    \ogaps(\Pair_t) > 0
                }
                \ge 
                \frac 12
            }
        \end{equation*}
        satisfies $T + 1 \le f(T) < \infty$ for all $T \ge 0$. 
        Introduce the (deterministic) sequence $T_1 := 1$ and $T_{k+1} := f(T_k)$, and introduce its pseudo-inverse $g(T) := \sup \braces{k \ge 1: T_{k+1} \le T}$. 
        Since $f(T) < \infty$, we have $T_k \to \infty$ and $g(T) \to \infty$.
        For $T \ge T_2$, we have:
        \begin{align*}
            \Reg(T; \model, \learner)
            & =
            \EE^{\model, \learner} \brackets*{
                \sum_{t=1}^T
                \ogaps(\Pair_t)
            }
            \\
            & \ge 
            \sum_{k=1}^{g(T)}
            \EE^{\model, \learner} \brackets*{
                \sum_{t=T_k}^{T_{k+1}-1}
                \ogaps(\Pair_t)
            }
            \\
            & \overset{(\dagger)}\ge 
            \sum_{k=1}^{g(T)} 
            c \Pr^{\model, \learner} \parens[\big]{
                \exists t \in \braces{T_k, \ldots, T_{k+1}-1}
                :
                \ogaps(\Pair_t) > 0
            }
            \overset{(\ddagger)}\ge
            \frac {c \cdot g(T)}2
        \end{align*}
        where $c := \min \braces{\ogaps(\pair): \ogaps(\pair) > 0} > 0$ is the minimum positive Bellman gap.
        We have $\frac {cg(T)}2 \to \infty$ when $T \to \infty$, hence the conclusion.
    \end{proof}

    \subsection{MDPs with empty confusing sets are non-explorative}
    \label{appendix_empty_confusing_implies_non_explorative}

    In this section, we show that if $\confusing(\model) = \emptyset$, then $\model$ is non-explorative, there exists a consistent learning algorithm $\learner$ such that $\Reg(T; \model, \learner) = \oh(\log(T))$ and there exists a robust learning algorithm $\learner'$ such that $\Reg(T; \model, \learner') = \OH(1)$.
    This corresponds to (\textit{2.}) $\Rightarrow$ (\textit{1.}), (\textit{3.}), and (\textit{4.}) in \Cref{theorem_characterization_explorative} hence completing the proof of all the equivalences. 
    The implication (\textit{2.}) $\Rightarrow$ (\textit{4.}), stating the existence of a robust learning algorithm $\learner$ with $\Reg(T; \model, \learner) = \OH(1)$ is done first, with \Cref{proposition_confusing_empty_robust}. 
    The proof is constructive, as we introduce a biased variant of \texttt{KLUCRL} \cite{filippi_optimism_2010} that is specialized to have bounded regret on $\model$, see \Cref{algorithm_biased_robust_klucrl}. 
    We prove (\textit{2.}) $\Rightarrow$ (\textit{1.}), i.e., that $\model$ is non-explorative, in \Cref{proposition_confusing_empty_non_explorative} and with the same algorithm. 
    For (\textit{2.}) $\Rightarrow$ (\textit{1.}) and the proof of the existence of a consistent learning algorithm $\learner$ such that $\Reg(T; \model, \learner) = \oh(\log(T))$, we provide the construction of the algorithm and simply sketch the proof.

    \subsubsection{A robust algorithm specialized to a non-explorative model}
    \label{appendix_confusing_empty_robust}

    We begin by providing a robust algorithm $\learner$ such that $\Reg(T; \model, \learner) = \OH(1)$ for $\model$ specifically. 

    \begin{proposition}
    \label{proposition_confusing_empty_robust}
        Consider a \strong{convex} ambient space $\models^* \equiv \product_{\pair \in \pairs} (\rewards^*_\pair \times \kernels^*_\pair)$ in product form and let $\model \in \models^*$ be non-degenerate.
        If $\confusing(\model) = \emptyset$, there exists a learning algorithm $\learner$ that 
        (1)~is robust,
        (2)~makes sub-linearly many episodes and
        (3) satisfies $\Reg(T; \model, \learner) = \OH(1)$.
    \end{proposition}

    The algorithm that we consider is a variant of \texttt{KLUCRL} managing episodes with \eqref{equation_doubling_trick}, that is specialized for $\model$. 
    Also, we have to take into account that $\models^*$ may not be the whole set of Markov decision processes with pair space $\pairs$, i.e., we may have $\models^* \ne \product_{\pair \in \pairs} ([0, 1] \times \probabilities(\states))$. 
    It must be taken into account by the learning algorithm, as the property ``$\confusing(\model) = \emptyset$'' depends on $\models^*$---by definition \eqref{equation_confusing_set}, $\confusing(\model) \subseteq \models^*$ so if one increases $\models^*$ to $\models'$, the confusing set of $\model$ relatively to $\models'$ may become non-empty. 
    So, the confidence region of \texttt{KLUCRL}, $\models(t)$, is constrained to $\models^*$ to eventually exploit that $\confusing(\model) = \emptyset$ relatively to $\models^*$. 

    \paragraph{Notations}
    We introduce the natural optimal policy of $\model$, given by $\optpolicy(\state) = \action$ where $\action \in \actions(\state)$ is the unique element such that $(\state, \action) \in \weakoptimalpairs(\model)$.
    Further introduce:
    \begin{equation}
    \nonumber
    \begin{gathered}
        \rewards_\pair(t; \models^*)
        :=
        \braces[\big]{
            \tilde{\reward}_\pair \in \rewards^*_\pair
            :
            \visits_\pair(t)
            \KL(\hat{\reward}_\pair(t)||\tilde{\reward}_\pair)
            \le
            \log(2t) + \log\parens*{e(1 + \visits_\pair(t))}
        }
        \\
        \kernels_\pair(t; \models^*)
        :=
        \braces*{
            \tilde{\kernel}_\pair \in \kernels^*_\pair
            :
            \visits_\pair(t)
            \KL(\hat{\kernel}_\pair(t)||\tilde{\kernel}_\pair)
            \le
            \log(2t) + \abs{\states} \log\parens*{e\parens*{1 + \frac{\visits_\pair(t)}{\abs{\states}-1}}}
        }
    \end{gathered}
    \end{equation}
    Note that unlike \eqref{equation_confidence_region} of the vanilla \texttt{KLUCRL}, $\log(t)$ is changed to $\log(2t)$. 
    This is done so that $\Pr(\exists t \ge T: \model \notin \models(t; \models^*)) = \OH(T^{-2})$ instead of $\OH(T^{-1})$ as in the vanilla version. 
    The confidence region for $\optpolicy$ is $\models_{\optpolicy}(t; \models^*) := \product_{\state \in \states} (\rewards_{\state, \optpolicy(\state)}(t; \models^*) \times \kernels_{\state, \optpolicy(\state)}(t; \models^*))$. 
    Similarly to the whole confidence region $\models(t; \models^*)$, it can be seen as a Markov decision process with compact action space by extending its action space (see \Cref{appendix_evi}) and \texttt{EVI} can be run on $\models_{\optpolicy}(t; \models^*)$ to compute the optimistic gain of $\optpolicy$, written $\gain^{\optpolicy}(\models(t;\models^* ))$. 

    \paragraph{Idea of the algorithm}
    The designed algorithm is working by epochs of doubling sizes.
    Given an epoch $\braces{2^m, \ldots, 2^{m+1}-1}$, it starts by iterating $\optpolicy$ $(2^m)^{2/3}$ times in a row. 
    After that initial phase, the algorithms runs an altered version of \texttt{KLUCRL} that uses \texttt{EVI} specifically biased for $\optpolicy$, that, when several policies are nearly optimistically optimal, prioritizes $\optpolicy$. 

    \noindent
    \begin{minipage}[t]{.49\linewidth}
        \begin{algorithm}[H]
            \begin{algorithmic}[1]
                \FOR{epochs $m = 0, 1, 2, \ldots$}
                    \STATE Iterate $\optpolicy$ for $t = 2^m, \ldots, 2^m + 2^{2m/3}$;
                    \FOR{$t = 2^m + 2^{2m/3}, \ldots, 2^{m+1}$}
                        \IF{\eqref{equation_doubling_trick} triggers \strong{or} $t = 2^m + 2^{2m/3}$}
                            \STATE $k \gets k+1$, $t_k \gets t$;
                            \STATE $\policy_{t_k} \gets \texttt{EVI-b}_{\optpolicy}(\models(t; \models^*), t)$;
                        \ENDIF
                        \STATE Set $\policy_t \gets \policy_{t_k}$ and play $\Action_t \gets \policy_t(\State_t)$.
                    \ENDFOR
                \ENDFOR
            \end{algorithmic}
            \caption{
                \label{algorithm_biased_robust_klucrl}
                \texttt{KLUCRL}$(\optpolicy, \models^*)$
            }
        \end{algorithm}
    \end{minipage}%
    \hfill%
    \begin{minipage}[t]{.49\linewidth}
        \begin{algorithm}[H]
            \begin{algorithmic}[1]
                \STATE Compute $\tilde{\policy} \gets \texttt{EVI}(\widetilde{\models})$;
                \STATE Compute $\tilde{\gain}^* \gets \gain^*(\widetilde{\models})$;
                \STATE Compute $\tilde{\gain}^{\policy} \gets \gain^{\policy}(\widetilde{\models})$;
                \IF{$\tilde{\gain}^* > \tilde{\gain}^{\policy} + \frac{\log(t)}{t^{1/3}}$}
                    \RETURN $\tilde{\policy}$;
                \ELSE
                    \RETURN $\policy$.
                \ENDIF
            \end{algorithmic}
            \caption{
                \label{algorithm_biased_robust_evi}
                $\texttt{EVI-b}_\policy(\widetilde{\models}, t)$
            }
        \end{algorithm}
    \end{minipage}

    \par
    \bigskip
    \def\proofname{Proof of \Cref{proposition_confusing_empty_robust}}
    \begin{proof}
        \def\proofname{Proof}
        Proving that the algorithm is robust on $\models^*$ follows a similar line that \Cref{appendix_minimax}, that we won't detail here.
        The idea is that the forced exploration with $\optpolicy$ last for at most $T^{2/3}$ time steps of the learning process, accounting for a regret of order $T^{2/3}$ if $\optpolicy$ is not optimal. 
        Later, the algorithm deploys policy that are $\frac{\log(T)}{T^{1/3}}$-optimistically optimal, inducing an extra cost of $T^{2/3} \log^2(T)$ compared to the vanilla analysis of \texttt{KLUCRL}.
        Therefore, the model independent regret is \texttt{KLUCRL}$(\optpolicy, \models^*)$ is of order $T^{2/3} \log^2(T) = \oh(T)$. 

        Meanwhile, \texttt{KLUCRL}$(\optpolicy, \models^*)$ makes $\OH(\log(T))$ episodes: one for each epoch when playing $\optpolicy$, and the others are triggered by \eqref{equation_doubling_trick} that is known to produce logarithmically many episodes, see \cite{auer_near_optimal_2009} or \Cref{appendix_number_episodes}.

        Last but not least, we argue that $\Reg(T; \model) = \OH(1)$.
        Because this is an instance dependent result, the argument is different from \Cref{appendix_minimax}.
        The idea is to show that 
        \begin{equation}
        \label{equation_proof_confusing_non_explorative_1}
            \Pr^\model\parens[\big]{
                \exists t \in \braces*{2^m, \ldots, 2^{m+1}-1}
                :
                \policy_t \ne \optpolicy
            } = 
            \OH\parens*{4^{-m}}
            .
        \end{equation}
        Following \eqref{equation_proof_confusing_non_explorative_1}, we have:
        \begin{align*}
            \Reg(T; \model)
            & \le
            \sum_{m=0}^\infty
            \EE^{\model} \brackets*{
                \sum_{t=2^m}^{2^{m+1}-1}
                \ogaps(\Pair_t)
            }
            \\
            & \le 
            \sum_{m=0}^\infty
            2^m \max(\ogaps) \Pr^\model\parens[\big]{
                \exists t \in \braces*{2^m, \ldots, 2^{m+1}-1}
                :
                \policy_t \ne \optpolicy
            }
            \\
            & \overset{(\dagger)}\le
            \max(\ogaps) \sum_{m=0}^\infty \OH\parens*{2^{-m}}
            =
            \OH(1)
        \end{align*}
        where $(\dagger)$ follows from \eqref{equation_proof_confusing_non_explorative_1}.
        We now explain how \eqref{equation_proof_confusing_non_explorative_1} is established. 

        \bigskip
        \par
        \noindent
        \STEP{1}
        \textit{
            There exists $c > 0$ such that, for $m$ large enough and $\pair \in \optpairs(\model)$, we have:
            \begin{equation*}
                \Pr^\model \parens*{
                    \exists t \in \braces{2^m + 2^{2m/3}, \ldots, 2^{m+1}-1}
                    :
                    \visits_\pair (t) < c 2^{2m/3}
                }
                = 
                \OH(4^{-m})
                .
            \end{equation*}
        }
        \begin{proof}
            The recurrent pairs of $\optpolicy$ are precisely $\optpairs(\model)$, i.e., $\Pr^{\model, \optpolicy}\braces{\forall m, \exists n \ge m: \Pair_n = \pair} = 1$. 
            Because the state space is finite, it follows that $\min_{\pair \in \optpairs(\model) }\EE^{\model, \optpolicy}[\visits_\pair(t)] \ge c t$ for some $c > 0$ when $t \to \infty$. 
            It means that under $\optpolicy$, every optimal pair is visited linearly many times in expectation. 
            Fixing $\pair_0 \in \optpairs(\model)$ and setting $f(\pair) = \indicator{\pair = \pair_0}$, we show that $\visits_{\pair_0}(t) \ge c t$ holds in probability as well.
            This is done as follows.
            Seeing $f$ as a reward function, $\optpolicy$ has an associated gain and bias functions that we denote $\gain^f$ and $\bias^f$. 
            These satisfy a Poisson equation $f(\state, \policy(\state)) + \kernel(\state, \policy(\state)) \bias^f = \gain^f(\state) + \bias^f(\state)$.
            By non-degeneracy, $\optpolicy$ is unichain so $\vecspan{\gain^f} = 0$, and we see that $\gain^f(\state) \ge c$ for all $\state \in \states$. 
            We continue as follows: If we only iterate $\optpolicy$, then
            \begin{align*}
                \visits_{\pair_0} (t)
                & = \sum_{i=1}^t f(\Pair_t)
                \\
                & =
                \sum_{i=1}^t 
                \parens*{
                    \gain^f(\State_t) + \bias^f(\State_t) - \kernel(\Pair_t) \bias^f
                }
                \ge
                c t - \vecspan{\bias^f}
                + \sum_{i=1}^t \parens*{e_{\State_{t+1}} - \kernel(\Pair_t)} \bias^f
                .
            \end{align*}
            The RHS is a martingale and each term is almost surely bounded by $\vecspan{\bias^f}$. 
            By Azuma-Hoeffding's inequality, it is therefore bounded by $\vecspan{\bias^f}\sqrt{t \log(\alpha t)/2} = \oh(t)$ with probability $1 - \frac 1{t^\alpha}$.
            So, provided that $t$ is large enough, we have $\visits_\pair (t) \ge \frac 12 ct$ with probability $1 - \frac 1{t^\alpha}$.

            Now, we know that on the time range $\braces{2^m, \ldots, 2^m + 2^{2m/3}}$, the algorithm exclusively iterates $\optpolicy$, so $\optpolicy$ is iterate $t = 2^{2m/3}$ times. 
            Pick $\alpha = 3$.
        \end{proof}

        \par
        \noindent
        \STEP{2}
        \textit{
            There exists $C > 0$ such that for all $\pair \in \optpairs(\model)$, we have
            \begin{align*}
                & \Pr^\model \parens*{
                    {
                        \exists t \in \braces{2^m + 2^{2m/3}, \ldots, 2^{m+1}-1}
                        \atop
                        \exists \widetilde{\models}_t \equiv (\pairs, \tilde{\reward}_t, \tilde{\kernel}_t) \in \models(t; \models^*)
                    }
                    :
                    \norm*{\tilde{\kernel}_t(\pair) - \kernel(\pair)}_1
                    +
                    \abs*{\tilde{\reward}_t(\pair) - \reward(\pair)}
                    >
                    \frac{C \sqrt{\log(t)}}{t^{1/3}}
                }
                \\
                & =
                \OH \parens*{4^{-m}}
            \end{align*}
        }
        \begin{proof}
            Fix $\pair \in \optpairs(\model)$.
            By \STEP{1}, we know that $\visits_\pair (t) > c2^{2m/3}$ with probability $1 - \OH(4^{-m})$ uniformly for $t \in \braces{2^m + 2^{2m/3}, \ldots, 2^{m+1}-1}$.
            Using \cite[Proposition~1]{jonsson2020planning} we find that for $m$ large enough, 
            \begin{equation*}
                \Pr^\model \parens*{
                    \exists t \ge 1
                    :
                    \visits_\pair(t) \ge c 2^{2m/3}
                    \text{~and~}
                    \KL(\hat{\reward}_t(\pair)||\reward(\pair))
                    >
                    \frac{\log(2e \cdot 4^m)}{c 2^{2m/3}}
                }
                \le
                4^{-m}
                .
            \end{equation*}
            So, by Pinsker's inequality, it follows that for $m$ large enough, we have $\abs{\hat{\reward}_t(\pair) - \reward(\pair)} \le C_r \sqrt{\log(t)} t^{-1/3}$ with probability $1 - \OH(4^{-m})$ uniformly for $t = 2^m + 2^{2m/3}, \ldots, 2^{m+1}-1$, where $C_r > 0$ is some constant. 
            Now, by design of the confidence region $\rewards_\pair(t; \models^*)$, every $\tilde{\reward}_t(\pair) \in \rewards_\pair(t; \models^*)$ satisfies $\abs{\tilde{\reward}_t(\pair) - \hat{\reward}_t(\pair)} \le \frac 1{\visits_\pair(t)} \log(4et)$. 
            By triangular inequality, we conclude that, for $\pair \in \optpairs(\model)$, the inequality
            \begin{equation*}
                \abs*{\tilde{\reward}_t(\pair) - \reward(\pair)} 
                \le 
                \frac{C \sqrt{\log(t)}}{t^{1/3}}
            \end{equation*}
            holds uniformly for $t =2^m + 2^{2m/3}, \ldots, 2^{m+1}-1$ with probability $1 - \OH(4^{-m})$. 
            Transition kernels are treated similarly. 
        \end{proof}

        \noindent
        \STEP{3}
        \textit{
            There exists a constant $C > 0$ such that
            \begin{align*}
                & 
                \Pr \parens*{
                    \exists t =2^m + 2^{2m/3}, \ldots, 2^{m+1}-1
                    :
                    \gain^{\optpolicy}(\models(t; \models^*))
                    +
                    \frac{C \sqrt{\log(t)}}{t^{1/3}}
                    \le
                    \optgain(\models(t; \models^*))
                }
                \\
                & = \OH(4^{-m})
                .
            \end{align*}
        }
        \begin{proof}
            This is where we use that $\confusing(\model) = \emptyset$.

            Let $\event_m$ be the event under which $\model \in \models(t; \models^*)$ and $\norm*{\tilde{\kernel}_t(\pair) - \kernel_\pair}_1 + \abs{\tilde{\reward}_t(\pair) - \reward(\pair)} \le C \sqrt{\log(t)}{t^{-1/3}}$ for all optimal pair $\pair \in \optpairs(\model)$, uniformly for $t = 2^{m} + 2^{2m/3}, \ldots, 2^{m+1}-1$ and $\tilde{\model}_t \equiv (\pairs, \tilde{\reward}_t, \tilde{\kernel}_t) \in \models(t; \models^*)$, where $C > 0$ is given by \STEP{2}.
            Then, following \STEP{2} and the design of $\models(t; \models^*)$, we have $\Pr(\event_m) = 1 - \OH(4^{-m})$. 

            Fix $t \in \braces{2^m + 2^{2m/3}, \ldots, 2^{m+1}-1}$ and let $\tilde{\model} \in \models(t; \models^*)$ be a plausible model. 
            Since $\models^*$ is convex, we can assume that $\tilde{\model} \gg \model$ up to changing $\tilde{\model}$ to $(1 - \lambda) \tilde{\model} + \lambda \model$ for some arbitrarily small $\lambda > 0$. 
            We show that, on $\event_m$, 
            \begin{equation}
            \label{equation_proof_confusing_non_explorative_1_5}
                \gain^{\optpolicy}(\model) 
                + 
                \frac{C_g\sqrt{\log(t)}}{t^{1/3}}
                \ge 
                \optgain(\tilde{\model}) 
            \end{equation}
            for $C_g > 0$ some constant, independent from $\tilde{\model}$ and $m$. 
            Since, on $\event_m$ again, we further have $\gain^{\optpolicy}(\models(t; \models^*)) \ge \optgain(\model)$, the result will follow from $\Pr(\event_m) = 1 - \OH(4^{-m})$. 

            We now show \eqref{equation_proof_confusing_non_explorative_1_5}.
            Let $\model \cup \tilde{\model}$ be the Markov decision process with states $\states$ where the choice of an action from $\state$ consists in (1) choosing $\action \in \actions(\state)$ in the vanilla sense and (2) choosing whether the transition is made using $(\reward(\state, \action), \kernel(\state, \action))$ or $(\tilde{\reward}(\state, \action), \tilde{\kernel}(\state, \action))$. 
            Note that $\model \cup \tilde{\model}$ is still a MDP with finite action space, that $\optgain(\model \cup \tilde{\model}) \ge \max\braces{\optgain(\model), \optgain(\tilde{\model})}$ and $\diameter(\model \cup \tilde{\model}) \le \diameter(\model)$. 
            In particular, $\model \cup \tilde{\model}$ is communicating and $\vecspan{\optbias(\model \cup \tilde{\model})} \le \diameter(\model)$. 
            Using \texttt{EVI} on $\model \cup \tilde{\model}$ (see \Cref{appendix_evi}) to compute its optimal gain, we extract a MDP $\model^\dagger$ such that $\optgain(\model^\dagger) = \optgain(\model \cup \tilde{\model})$ and $\vecspan{\optbias(\model^\dagger)} \le \diameter(\model)$. 
            By construction, $\model^\dagger$ is a blend of $\model$ and $\tilde{\model}$, in the sense that $\reward^\dagger(\pair) \in \braces{\reward(\pair), \tilde{\reward}(\pair)}$ and $\kernel^\dagger(\pair) \in \braces{\kernel(\pair), \tilde{\kernel}(\pair)}$. 
            Let $\model^\dagger_0$ be obtained from $\model^\dagger$ by setting 
            \begin{equation*}
                \model^\dagger_0 (\pair)
                :=
                \begin{cases}
                    \model(\pair) & \text{if~} \pair \in \optpairs(\model);
                    \\
                    \model^\dagger (\pair) & \text{if~} \pair \notin \optpairs(\model).
                \end{cases}
            \end{equation*}
            Note that since $\models^*$ is in product form and $\models(t; \models^*) \subseteq \models^*$, we have $\model_0^\dagger \in \models^*$. 
            It follows that $\model^\dagger_0 \in \models^*$ and $\model^\dagger_0 = \model$ on $\optpairs(\model)$.
            Moreover, since $\tilde{\model} \gg \model$, we have $\model^\dagger \gg \model$ and $\model^\dagger_0 \gg \model$.
            So, because $\confusing(\model) = \emptyset$, we have $\optpolicy \in \optpolicies(\model_0^\dagger)$. 
            Accordingly,
            \begin{equation}
                \gain^{\optpolicy}(\model^\dagger_0) 
                = \optgain(\model^\dagger_0)
                = \optgain(\model)
            \end{equation}
            where the second equality follows from the observation that $\model^\dagger_0$ is a copy of $\model$ on $\optpairs(\model)$. 
            Now, by construction $\event_m$, we know that the width of the confidence region is $\OH(\sqrt{\log(t)} t^{-1/3})$ on pairs of $\optpairs(\model)$. 
            Using the gain deviation inequality of \Cref{lemma_gain_deviations}, we conclude that on $\event_m$, 
            \begin{equation}
            \begin{gathered}
                \norm*{
                    \gain^{\optpolicy}(\model)
                    -
                    \gain^{\optpolicy}(\model^\dagger)
                }_\infty
                \le
                \parens*{1 + \vecspan{\optbias(\model^\dagger)}}
                \OH \parens*{
                    \frac{\sqrt{\log(t)}}{t^{1/3}}
                }
                =
                \OH \parens*{
                    \frac{\sqrt{\log(t)}}{t^{1/3}}
                }
                \\
                \norm*{
                    \optgain(\model^\dagger_0)
                    -
                    \optgain(\model^\dagger)
                }_\infty
                \le
                \parens*{1 + \vecspan{\optbias(\model^\dagger)}}
                \OH \parens*{
                    \frac{\sqrt{\log(t)}}{t^{1/3}}
                }
                =
                \OH \parens*{
                    \frac{\sqrt{\log(t)}}{t^{1/3}}
                }
            \end{gathered}
            \end{equation}
            where every $\OH(-)$ hides constants that are independent of $\tilde{\model}$ and $m$. 
            Since $\optgain(\model^\dagger) \ge \optgain(\tilde{\model})$, we conclude accordingly. 
        \end{proof}

        Finally, \eqref{equation_proof_confusing_non_explorative_1} is an immediate consequence of (\STEP{3}). 
        (\STEP{3}) states that, uniformly for $t = 2^m + 2^{2m/3}, \ldots, 2^{m+1}-1$, we have:
        \begin{equation}
        \label{equation_proof_confusing_non_explorative_2}
            \gain^{\optpolicy}(\models(t; \models^*))
            +
            \frac{C \sqrt{\log(t)}}{t^{1/3}}
            >
            \optgain(\models(t; \models^*))
        \end{equation}
        with probability $1 - \OH(4^{-m})$, where $C > 0$ is some constant independent of $m$. 
        Now, by design of $\texttt{EVI-b}_{\optpolicy}$ (\Cref{algorithm_biased_robust_evi}), if at $t = t_k \in \braces{2^m + 2^{2m/3}, \ldots, 2^{m+1}}$, we have $\gain^{\optpolicy}(\models(t; \models^*)) + \log(t) t^{-1/3} \ge \optgain(\models(t; \models^*))$, then \texttt{EVI-b}$_{\optpolicy}$ outputs $\optpolicy$. 
        So, on the event prescribed by \eqref{equation_proof_confusing_non_explorative_2} and for $t \ge \exp(C^2)$, $\texttt{EVI-b}_{\optpolicy}$ outputs $\optpolicy$. 
        Hence \eqref{equation_proof_confusing_non_explorative_1}.
    \end{proof}
    \def\proofname{Proof}

    \subsubsection{Models with non-empty confusing set are non-explorative}

    With the same algorithm, we show that non-degenerate Markov decision processes with empty confusing set are non-explorative.

    \begin{proposition}
    \label{proposition_confusing_empty_non_explorative}
        Consider a \strong{convex} ambient space $\models^* \equiv \product_{\pair \in \pairs} (\rewards^*_\pair \times \kernels^*_\pair)$ in product form and let $\model \in \models^*$ be non-degenerate.
        If $\confusing(\model) = \emptyset$, then $\model$ is non-explorative, i.e., there exists a learning algorithm $\learner$ 
        (1)~with sub-linearly many episodes,
        (2)~which is no-regret on $\models^*$ and
        (3)~that has finitely many exploration episodes. 
    \end{proposition}
    \begin{proof}
        We consider the algorithm \texttt{KLUCRL}$(\optpolicy, \models^*)$ (\Cref{algorithm_biased_robust_klucrl}), introduced for the proof of \Cref{proposition_confusing_empty_robust}. 
        Following \eqref{equation_proof_confusing_non_explorative_1}, we have:
        \begin{align*}
            \Pr^\model \parens*{
                \exists t \ge 2^m
                :
                \policy_t \ne \optpolicy
            }
            & \le 
            \sum_{n = m}^\infty 
            \Pr^\model \parens*{
                \exists t \in \braces{2^n, \ldots, 2^{n+1}-1}
                :
                \policy_t \ne \optpolicy
            }
            \\
            & =
            \OH \parens*{
                \sum_{n = m}^\infty 4^{-n}
            }
            = \OH \parens*{4^{-m}}
            = \oh_{m \to \infty}(1)
            .
        \end{align*}
        So $\Pr^\model\parens{\exists T, \forall t \ge T: \policy_t = \policy^*} = 1$. 
        Because every pair $\pair$ that $\optpolicy$ can reach satisfies $\ogaps(\pair; \model) = 0$ by construction of $\optpolicy$, it follows that $\Pr^\model\parens{\exists T, \forall t \ge T: \ogaps(\Pair_t) = 0} = 1$. 
        So necessarily, the number of exploration times (\Cref{definition_exploration_time}) is almost-surely finite.
    \end{proof}

    \subsubsection{A consistent algorithm specialized to a non-explorative model}

    We conclude by arguing that the robust algorithm of for \Cref{proposition_confusing_empty_robust}, $\texttt{KLUCRL}(\policy^*, \models^*)$, can be adapted into a consistent learning algorithm $\learner'$ such that $\Reg(T; \model, \learner') = \oh(\log(T))$. 

    \begin{proposition}
    \label{proposition_confusing_empty_consistent}
        Consider a \strong{convex} ambient space $\models^* \equiv \product_{\pair \in \pairs} (\rewards^*_\pair \times \kernels^*_\pair)$ in product form and let $\model \in \models^*$ be non-degenerate.
        If $\confusing(\model) = \emptyset$, there exists a learning algorithm $\learner$ that 
        (1)~is consistent; and
        (2) satisfies $\Reg(T; \model, \learner) = \oh(\log(T))$.
    \end{proposition}

    The considered algorithm is a reworked version of \texttt{KLUCRL}$(\policy^*, \models^*)$ (\Cref{algorithm_biased_robust_klucrl}). 
    When tuning \Cref{algorithm_biased_robust_klucrl} and \Cref{algorithm_biased_robust_evi} to provide a robust algorithm, \Cref{algorithm_biased_robust_klucrl} forces $\Omega(T^{2/3})$ iterations of the policy $\policy^*$.
    This is incompatible with consistency, as the latter implies that the model dependent regret is $\Omega(T^{2/3})$ when $\policy^*$ is not gain optimal. 
    Instead, the algorithm will force $\OH(\log^{3}(T))$ iterations of $\policy^*$, losing robustness but achieving consistency. 
    Then, the biased \texttt{EVI} adds a bonus to favor the selection of $\policy^*$, compensating the noise on the estimation of $\gain^{\policy^*}$.
    \vspace{-1em}

    \noindent
    \begin{minipage}[t]{.49\linewidth}
        \begin{algorithm}[H]
            \begin{algorithmic}[1]
                \FOR{epochs $m = 0, 1, 2, \ldots$}
                    \STATE Set $\psi(2^m) \gets \log(2^m)$;
                    \STATE Play $\optpolicy$ for $t = 2^m, \ldots, 2^m + \psi(2^m)$
                    \FOR{$t = 2^m + \psi(2^m), \ldots, 2^{m+1}$}
                        \IF{\eqref{equation_doubling_trick} triggers \strong{or} $t = 2^m + \psi(2^m)$}
                            \STATE $k \gets k+1$, $t_k \gets t$;
                            \STATE $\policy_{t_k} \gets \texttt{EVI-b}_{\optpolicy}(\models(t; \models^*), t)$;
                        \ENDIF
                        \STATE Set $\policy_t \gets \policy_{t_k}$ and play $\Action_t \gets \policy_t(\State_t)$.
                    \ENDFOR
                \ENDFOR
            \end{algorithmic}
            \caption{
                \label{algorithm_consistent_robust_klucrl}
                \texttt{KLUCRL}$'(\optpolicy, \models^*)$
            }
        \end{algorithm}
    \end{minipage}%
    \hfill%
    \begin{minipage}[t]{.49\linewidth}
        \begin{algorithm}[H]
            \begin{algorithmic}[1]
                \STATE Compute $\tilde{\policy} \gets \texttt{EVI}(\widetilde{\models})$;
                \STATE Compute $\tilde{\gain}^* \gets \gain^*(\widetilde{\models})$;
                \STATE Compute $\tilde{\gain}^{\policy} \gets \gain^{\policy}(\widetilde{\models})$;
                \IF{$\tilde{\gain}^* > \tilde{\gain}^{\policy} + \sqrt{\frac{1}{\log(t)}}$}
                    \RETURN $\tilde{\policy}$;
                \ELSE
                    \RETURN $\policy$.
                \ENDIF
            \end{algorithmic}
            \caption{
                \label{algorithm_biased_consistent_evi}
                $\texttt{EVI-b}'_\policy(\widetilde{\models}, t)$
            }
        \end{algorithm}
    \end{minipage}

    \bigskip
    \def\proofname{Sketch of proof}
    \begin{proof}
        Proving \Cref{proposition_confusing_empty_consistent} in full detail is, again, tedious. 
        As the proof follows from techniques that are quite similar to \Cref{proposition_confusing_empty_robust}, we only leave the main ideas. 
        We fix $\model' \in \models^*$ and look at whether $\model' = \model$ or $\model' \ne \model$.

        If $\model' \ne \model$, we prove that $\Reg(T; \model') = \OH(\log^3(T))$ when $T \gg \exp\braces{(\optgain(\model') - \gain^{\policy^*}(\model'))^{-1}}^2$.
        $T$ needs to be large in order to compensate for the bonus that \Cref{algorithm_biased_consistent_evi} puts on the optimistic gain of $\policy^*$, that intentionally overshoots the empirical noise on $\gain^{\policy^*}$, i.e., the value of $\abs{\gain^{\policy^*}(\hat{\model}_t) - \gain^{\policy^*}(\model)}$. 
        Beyond that subtlety, we enumerate $\pair \notin \wkoptpairs(\model)$ and we distinguish the cases where $\Pair_t = \pair$ depending on whether (1) $\pair$ is transient under the currently deployed policy or (2) $\pair$ is recurrent under the currently deployed policy and (3) $\pair$ is a recurrent pair of $\optpolicy$ and $t \in \bigcup_m \braces{2^m, \ldots, 2^m + \psi(m)}$ is within a forced exploration phase.
        The first is proportionally to the number of episodes, which is $\OH(\log(T))$, while the second implies that confidence regions are wrong.
        As discussed above, when taking account of the bonus in \Cref{algorithm_biased_consistent_evi}, confidence regions start to be correct when $T \gg \exp\braces{(\optgain(\model') - \gain^{\policy^*}(\model'))^{-1}}^2$, leading to an error that sums as $\OH(1)$ overall. 
        The third accounts for $\OH(\log^3(T))$ which is eventually the dominant term in the regret bound. 

        If $\model' = \model$, we prove that $\Reg(T; \model) = \OH(1)$ with the exact same proof technique as \Cref{proposition_confusing_empty_robust}, by establishing an equation in the style of \eqref{equation_proof_confusing_non_explorative_1} with the same proof strategy \STEP{1, 2, 3}, adapted to a forced exploration of $\Theta(\log^3(T))$ rather than $\Theta(T^{1/3})$. 
    \end{proof}
    \def\proofname{Proof}

    \subsection{Interior Markov decision processes are universally explorative}
    \label{appendix_interior_are_explorative}
    
    We conclude the discussion on explorative spaces by discussing examples of explorative spaces, and how common the property ``$\model \in \models^+$'' may be. 
    As shown in \Cref{appendix_exploration_counter_example}, when the ambient space $\models$ is a fixed kernel space (i.e., is of the form $\models \equiv \product_{\pair \in \pairs} (\rewards_\pair \times \braces{\kernel(\pair)}$ for some fixed $\kernel \in \probabilities(\states)^\pairs$) and $\rewards_\pair \subseteq [0,1]$, large portions of $\models$ may be non-explorative. 
    For instance, taking 
    \begin{equation*}
        \models := \braces*{
            \model \equiv (\pairs, \reward, \kernel)
            :
            \reward \in [0, 1]^\pairs
            \text{~and~}
            \forall \pair \in \pairs,
            \forall \state \in \states,
            \kernel(\state|\pair) \in \braces{0, 1}
        }
    \end{equation*}
    the space of deterministic transition Markov decision processes with pair space $\pairs$, one can generalize the example of \Cref{appendix_exploration_counter_example} to show that as soon as $\abs{\pairs} > \abs{\states}$, a model picked in $\models$ uniformly at random is non-explorative with positive probability.  

    The take-away of this observation is that if $\models$ is structured, then $\models^+$ can be large. 
    In \Cref{proposition_interior_models_are_explorative} below, we prove a complementary result: If $\models$ has no structure, then every (non-degenerate) model in the interior (see \Cref{assumption_interior}) of $\models$ is explorative. 
    It follows (see \Cref{corollary_random_explorative}) that if the ambient space $\models$ has no structure and if $\model \in \models$ is picked uniformly at random, then $\model$ is explorative almost surely. 
    This property is very convenient for experiments: Any Markov decision processes that is generated randomly is a good environment to investigate regret of exploration guarantees. 

    \begin{proposition}[Interior implies explorative]
    \label{proposition_interior_models_are_explorative}
        Let $\models \equiv \product_{\pair \in \pairs} ([0, 1] \times \probabilities(\states))$ be the space of all Markov decision processes with pair space $\pairs$, with $\abs{\pairs} > \abs{\states}$.
        Then every non-degenerate (n.d.) interior model is explorative, i.e., 
        \begin{equation*}
            \braces*{
                \model \equiv (\pairs, \reward, \kernel)
                \text{~n.d.}
                :
                \forall \pair \in \pairs,
                \mathrm{supp}(\reward(\pair)) = \braces{0, 1}
                \text{~and~}
                \mathrm{supp}(\kernel(\pair)) = \states
            }
            \subseteq 
            \models^+
            .
        \end{equation*}
    \end{proposition}
    \begin{proof}
        Let $\model \in \models$ be a non-degenerate interior model. 
        We show that $\confusing(\model) \ne \emptyset$.
        Because $\abs{\pairs} > \abs{\states}$ and $\model$ is non-degenerate, there exists $\pair_0 \in \pairs \setminus \wkoptpairs(\model)$. 
        By definition, $\ogaps(\pair_0; \model) > 0$ so that this pair cannot be recurrent under any gain optimal policy of $\model$. 
        For $\epsilon > 0$, define $\model_\epsilon \equiv (\pairs, \reward_\epsilon, \kernel_\epsilon)$ the Markov decision process given by:
        \begin{equation*}
            \kernel_\epsilon (\pair)
            := 
            \begin{cases}
                \kernel(\pair) & \text{if~} \pair \ne \pair_0
                \\
                (1 - \epsilon) e_{\state_0} + \epsilon \cdot \kernel(\pair) & \text{if~} \pair = \pair_0
            \end{cases}
            \text{~~~and~~~}
            \reward_\epsilon (\pair)
            := 
            \begin{cases}
                \reward(\pair) & \text{if~} \pair \ne \pair_0
                \\
                1 - \epsilon + \epsilon \cdot \reward(\pair) & \text{if~} \pair = \pair_0
            \end{cases}
        \end{equation*}
        where $\pair_0 \equiv (\state_0, \action_0)$ and $e_{\state_0}$ is the Dirac at $\state_0$.
        By construction, $\model_\epsilon \gg \model$ for all $\epsilon > 0$. 
        Let $\policy$ be the (randomized) policy that picks actions uniformly at random from $\state \ne \state_0$ and with $\policy(\state_0) = \action_0$. 
        It is clear that as $\epsilon \to 0$, we have $\gain^\policy(\model_\epsilon) \to 1$ because $\policy$ converges to a unichain policy that converges to a loop on $\state_0$ where it scores $1 - \epsilon$. 
        Now, if $\optpolicy \in \optpolicies(\model)$, we have $\gain^{\optpolicy}(\model_\epsilon) = \gain^{\optpolicy}(\model)$ since $\optpolicy$ does not pick $\action_0$ from $\state_0$. 
        So, for $\epsilon > 0$ small enough, we have $\gain^{\policy}(\model_\epsilon) > \gain^{\optpolicy}(\model_\epsilon)$ for all $\optpolicy \in \optpolicies(\model)$.
        For such $\epsilon > 0$, we have $\optpolicies(\model_\epsilon) \cap \optpolicies(\model) = \emptyset$ and it follows that $\model_\epsilon \in \confusing(\model)$. 
        So $\confusing(\model) \ne \emptyset$ and $\model$ is explorative from \Cref{theorem_characterization_explorative}.
    \end{proof}

    \begin{corollary}
    \label{corollary_random_explorative}
        Let $\models \equiv \product_{\pair \in \pairs} ([0, 1] \times \probabilities(\states))$ be the space of all Markov decision processes with pair space $\pairs$, with $\abs{\pairs} > \abs{\states}$.
        Let $\model \equiv (\pairs, \reward, \pairs)$ where $\reward(\pair) \sim \mathrm{U}[0, 1]$ and $\kernel(\pair) \sim \mathrm{U}(\probabilities(\states))$ are sampled independently. 
        Then $\model \in \models^+$ almost surely. 
    \end{corollary}
    
    \begin{proof}
        If $\model$ is picked at random as described above, then $\model$ is interior with probability one. 
        Meanwhile, the set of degenerate models of a fixed arbitrary kernel $\kernel \in \probabilities(\states)$ is of measure zero, see \Cref{corollary_non_degeneracy_noise}.
        Integrating, the set of degenerate models is negligible in $\models$ for the Lebesgue measure.
        Hence, if $\model$ is picked at random described, then $\model$ is non-degenerate with probability one. 
    \end{proof}
\end{document}